\documentclass{article}

  \usepackage[nonatbib,final]{neurips_2023}

\usepackage[utf8]{inputenc} % allow utf-8 input
\usepackage[T1]{fontenc}    % use 8-bit T1 fonts
\usepackage{hyperref}       % hyperlinks
\usepackage{url}            % simple URL typesetting
\usepackage{booktabs}       % professional-quality tables
\usepackage{amsfonts}       % blackboard math symbols
\usepackage{nicefrac}       % compact symbols for 1/2, etc.
\usepackage{microtype}      % microtypography
\usepackage{xcolor}         % colors
\usepackage{placeins}         % colors
\usepackage[round]{natbib}

\usepackage{amsthm}
\usepackage{amsmath}
\usepackage{graphicx}
\usepackage[capitalize,noabbrev]{cleveref}
\usepackage{algorithm}
\usepackage{algorithmic}
\usepackage{subcaption}
\usepackage[export]{adjustbox}% http://ctan.org/pkg/adjustbox

\theoremstyle{plain}
\newtheorem{theorem}{Theorem}[section]

\newtheorem{corollary}[theorem]{Corollary}
\theoremstyle{definition}

\newtheorem{assumption}[theorem]{Assumption}
\theoremstyle{remark}
\newtheorem{remark}[theorem]{Remark}
\newcommand{\E}{\mathbb{E}}
\newcommand{\p}{\mathbb{P}}
\newcommand{\D}{\mathcal{D}}
\newcommand{\Hnull}{\D_{\mathrm{test}}^{\mathrm{null}}}

\usepackage{xr}
\title{Derandomized novelty detection with FDR control via conformal e-values}

\author{%
  Meshi Bashari \\
  Department of Electrical and Computer Engineering\\
  Technion IIT\\
  Haifa, Israel \\
  \texttt{meshi.b@campus.technion.ac.il} \\
  \And
  Amir Epstein \\
  Citi Innovation Lab \\
  Tel Aviv, Israel \\
  \texttt{amir.epstein@citi.com} \\
  \AND
  Yaniv Romano \\
  Department of Electrical and Computer Engineering \\
  Department of Computer Science \\
  Technion IIT\\
  Haifa, Israel \\
  \texttt{yromano@technion.ac.il} \\
  \And
  Matteo Sesia \\
  Department of Data Sciences and Operations \\
  University of Southern California \\
  Los Angeles, California, USA \\
  \texttt{sesia@marshall.usc.edu} \\
}

\begin{document}

\maketitle

\begin{abstract}
Conformal inference provides a general distribution-free method to rigorously calibrate the output of any machine learning algorithm for novelty detection. 
While this approach has many strengths, it has the limitation of being randomized, in the sense that it may lead to different results when analyzing twice the same data, and this can hinder the interpretation of any findings.
We propose to make conformal inferences more stable by leveraging suitable conformal {\em e-values} instead of {\em p-values} to quantify statistical significance. This solution allows the evidence gathered from multiple analyses of the same data to be aggregated effectively while provably controlling the false discovery rate.
Further, we show that the proposed method can reduce randomness without much loss of power compared to standard conformal inference, partly thanks to an innovative way of weighting conformal e-values based on additional side information carefully extracted from the same data. Simulations with synthetic and real data confirm this solution can be effective at eliminating random noise in the inferences obtained with state-of-the-art alternative techniques, sometimes also leading to higher power.
\end{abstract}

\section{Introduction}

\subsection{Background and motivation}

A common problem in statistics and machine learning is to determine which samples, among a collection of new observations, were drawn from the same distribution as a reference data set \citep{wilks1963multivariate,riani2009finding,chandola2009anomaly}.
This task is known as {\em novelty detection, out-of-distribution testing, or testing for outliers}, and it arises in numerous applications
within science, engineering, and business, including for example in the context of medical diagnostics~\citep{tarassenko1995novelty}, security monitoring~\citep{zhang2013medmon}, and fraud detection~\citep{ahmed2016survey}.
This paper looks at the problem from a model-free perspective, in the sense that it does not rely on parametric assumptions about the data-generating distributions, which are generally unknown and complex.
Instead, we apply powerful machine learning models for one-class \citep{moya1993one} or binary classification to score the new samples based on how they {\em conform} to patterns observed in the reference data, and then we translate such scores into rigorous tests using conformal inference.

Conformal inference \citep{vovk2005algorithmic,lei2013distribution} provides flexible tools for extracting provably valid novelty detection tests from any {\em black-box} model.
The simplest implementation is based on random sample splitting. This consists of training a classifier on a subset of the reference data,
%(possibly including some labeled outlier samples), 
and then ranking the output score for each test point against the corresponding scores evaluated out-of-sample for the hold-out reference data.
As the latter do not contain outliers, the aforementioned rank is uniformly distributed under the null hypothesis that the test point is not an outlier \citep{laxhammar2015inductive,smith2015conformal,guan2019prediction}, as long as some relatively mild {\em exchangeability} assumptions hold.
In other words, this calibration procedure yields a conformal {\em p-value} that can be utilized to test for outliers while rigorously controlling the probability of making a {\em false discovery}---incorrectly labeling an {\em inlier} data point as an ``outlier".
Further, split-conformal inference produces only weakly dependent p-values for different test points \citep{conformal-p-values}, allowing exact control of the expected proportion of false discoveries---the {\em false discovery rate} (FDR)---with the powerful Benjamini-Hochberg (BH) filter \citep{BH}.

As visualized in Figure~\ref{fig:illustration-oc-conformal}, a limitation of split-conformal inference is that it is randomized---its results for a given data set are unpredictable because they depend on how the reference samples are divided between the training and calibration subsets.
However, higher stability is desirable in practice, as randomized methods generally tend to be less reliable and more difficult to interpret~\citep{murdoch2019definitions,yu2020}.
This paper addresses the problem of making conformal inferences more stable by developing a principled method to powerfully aggregate tests for outliers obtained with repeated splits of the same data set, while retaining provable control of the FDR.
This problem is challenging because dependent p-values for the same hypothesis are difficult to aggregate without incurring into a significant loss of power \citep{vovk2020combining,vovk2022admissible}. 
\begin{figure}[!htb]
    \centering
    \begin{subfigure}[t]{0.35\textwidth}
\includegraphics[width=\textwidth]{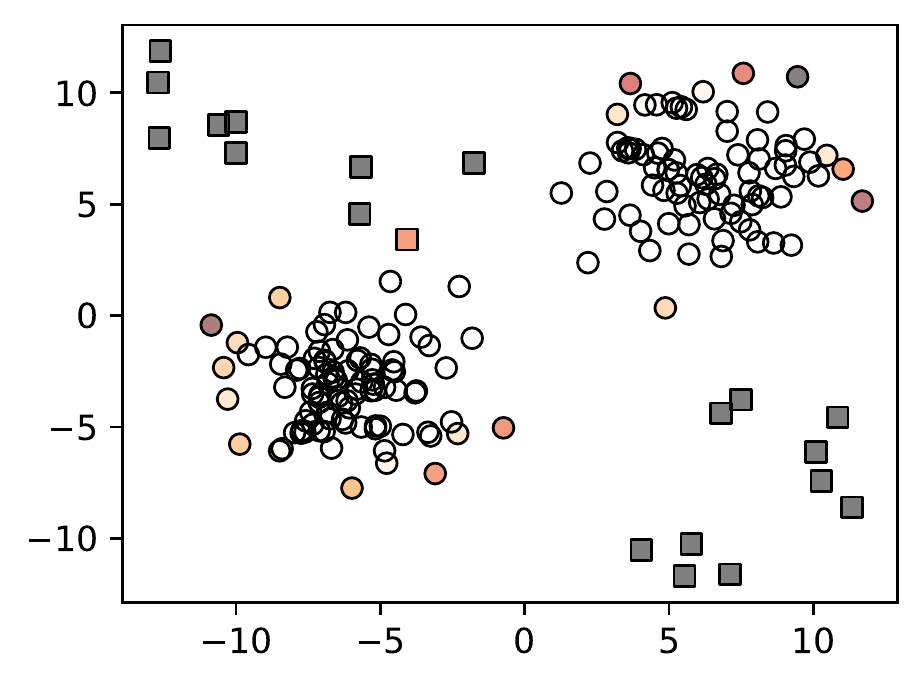}
    \caption{\texttt{Standard conformal.}}
    \label{fig:illustration-oc-conformal}
    \end{subfigure}
    \begin{subfigure}[t]{0.35\textwidth}
\includegraphics[width=\textwidth]{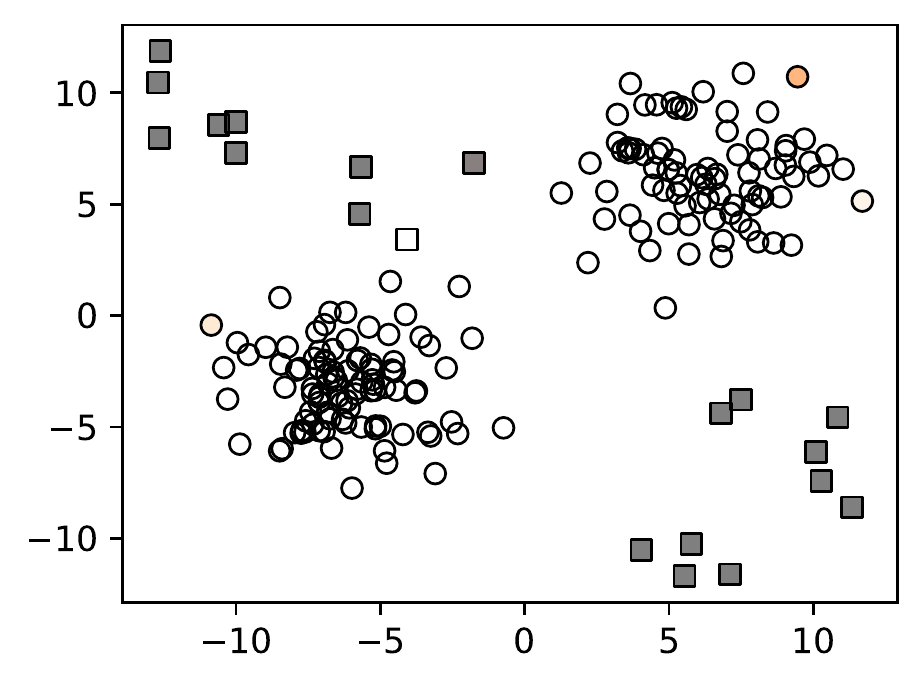}
    \caption{\texttt{derandomized conformal.}}
    \label{fig:illustration-e-oc-conformal}
    \end{subfigure}
    \centering
    \includegraphics[width=0.16\textwidth]{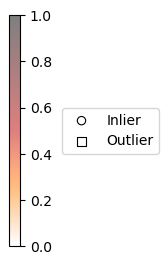}
    \caption{Demonstration on two-dimensional synthetic data of standard conformal (a) and derandomized conformal (b) inferences for novelty detection. Circles denote true inliers and squares denote outliers. The colors indicate how often each test point is reported as a possible outlier over 100 independent analyses of the same data. By carefully aggregating evidence from 10 distinct analyses based on independent splits of the same data, the proposed derandomized approach discovers the same outliers consistently and is less likely to make random false discoveries.
%the proposed derandomized outlier detection method \ref{fig:illustration-e-oc-conformal}, \texttt{E-OC-Conformal}, applied with $K=10$, compared to that of its randomized benchmark \ref{fig:illustration-oc-conformal}, \texttt{OC-Conformal}. The color indicates the rejection rate of each test sample.
}
    \label{fig:illustration}
\end{figure}

\subsection{Main contributions}

This paper utilizes carefully constructed conformal {\em e-values} \citep{e-value} instead of {\em p-values} to quantify statistical significance when testing for outliers under FDR control.
The advantage of e-values is that they make it possible to aggregate the results of mutually dependent tests in a relatively simple way, enabling an effective approach to derandomize conformal inferences.
Our contribution is to develop a martingale-based method inspired by \citet{drand-kn} that leverages e-value ideas efficiently, as different types of e-values can be constructed but not all would be powerful in our context due to the discrete nature of the statistical evidence in conformal inference. 
We further refine this method and boost power by adaptively weighting our conformal e-values based on an estimate of the out-of-sample accuracy of each underlying machine learning model.
A preview of the performance of our solution is given by Figure~\ref{fig:illustration-e-oc-conformal}, which shows that our method can achieve power comparable to that of standard conformal inferences while mitigating the algorithmic randomness.

%\clearpage

\subsection{Related work}

This paper builds upon {\em e-values} \citep{e-value}: quantitative measures of statistical evidence, alternative to p-values, that lend themselves well to the derandomization of data-splitting procedures and to FDR control under dependence \citep{eBH}.
There exist several generic methods for converting any p-value into an e-value \citep{e-value}. While those {\em p-to-e calibrators} could be applied for our novelty detection problem, their power turns out to be often quite low due to the fact that conformal p-values are discrete and cannot take very small values unless the sample size is extremely large; see the Supplementary Section \ref{app:baselines} for more details.

Therefore, we propose a novel construction of (slightly generalized) e-values inspired by the work of \citet{drand-kn} on the derandomization of the knockoff filter \citep{barber2015controlling}, which focused on a completely different high-dimensional variable selection problem.
A different approach for producing e-values in the context of conformal inference can also be found in  \citet{ignatiadis2023evalues}, although the latter did not focus on derandomization. Our approach differs from that of \citet{ignatiadis2023evalues} because we construct e-values simultaneously for the whole test set, aiming to control the FDR, instead of operating one test point at a time. Simulations show that our approach tends to yield higher power, especially if the test data contain many outliers.

Our second novelty consists of developing a principled method for assigning data-driven weights to conformal e-values obtained from different machine learning models, in such a way as to further boost power.
This solution re-purposes {\em transductive} \citep{vovk2013transductive} conformal inference ideas to leverage information contained in the test data themselves while calibrating the conformal inferences, increasing the power to detect outliers similarly to \citet{ml-fdr} and \citet{liang2022integrative}.

While this paper focuses on derandomizing split-conformal inferences, there exist other distribution-free methods that can provide finite-sample tests for novelty detection, such as full-conformal inference \citep{vovk2005algorithmic} and cross-validation+ \citep{barber2019predictive}.
Those techniques are more computationally expensive but have the advantage of yielding relatively more stable conformal p-values because they do not rely on a single random data split. However, full-conformal inference and cross-validation+ also produce conformal p-values with more complicated dependencies, which make it difficult to control the FDR without large losses in power \citep{benjamini2001control} or very expensive computations \citep{fithian2022conditional,liang2022integrative}.

Finally, prior works studied how to stabilize conformal predictors by calibrating the output of an {\em ensemble} of simpler models \citep{lofstrom2013effective,beganovic2018ensemble,linusson2020efficient,kim2020predictive,gupta2022nested}. However, we consider a distinct problem as we focus on derandomizing conformal novelty detection methods while controlling the FDR.

\section{Relevant technical background}

\subsection{Notation and problem setup}
\label{sec:notations}
Consider $n$ observations, $X_i\in \mathbb{R}^d$, sampled exchangeably (or, for simplicity, independent and identically distributed) from some unknown distribution $P_0$, for all $i \in \mathcal{D} = [n] = \{1,\ldots,n\}$.
Then, imagine observing a test set of $n_{\text{test}}$ ``unlabeled'' samples $X_j \in \mathbb{R}^d$.
The problem is to test, for each $j \in \D_{\mathrm{test}} = [n+n_{\text{test}}] \setminus [n]$, the {\em null hypothesis} that $X_j$ is also an {\em inlier}, in the sense that it was randomly sampled from $P_0$ exchangeably with the data in $\mathcal{D}$. We refer to a rejection of this null hypothesis as the {\em discovery} that $X_j$ is an {\em outlier}, and we indicate the set of true inlier test points as $\D^{\mathrm{null}}_{\mathrm{test}}$, with $n^{\mathrm{null}}_{\mathrm{test}}=|\D^{\mathrm{null}}_{\mathrm{test}}|$.
For each $j \in \D_{\mathrm{test}}$, define $R_j$ as the binary indicator of whether $X_j$ is labeled by our method as an outlier. Then, the goal is to discover as many true outliers as possible while controlling the FDR, defined as $\mathrm{FDR}=\mathbb{E}[(\sum_{j \in \D^{\mathrm{null}}_{\mathrm{test}}} R_j) / \max\{1, \sum_{j \in \D_{\mathrm{test}}} R_j\}]$.

\subsection{Review of FDR control with conformal p-values} \label{sec:fdr-conformal}

After randomly partitioning $\D$ into two disjoint subsets $\D_{\mathrm{train}}$ and $\D_{\mathrm{cal}}$, of cardinality $n_{\mathrm{train}}$ and $n_{\mathrm{cal}}=n-n_{\mathrm{train}}$ respectively, the standard approach for computing split-conformal p-values begins by training a one-class classification model on the data indexed by $\D_{\mathrm{train}}$. This model is applied out-of-sample to compute conformity scores $\hat{S}_i$ and $\hat{S}_j$ for all calibration and test points $i \in \D_{\mathrm{cal}}$ and $j \in \D_{\mathrm{test}}$, with the convention that larger scores suggest evidence of an outlier.
Assuming without loss of generality that all scores take distinct values (otherwise, ties can be broken at random by adding a little noise), a conformal p-value $\hat{u}(X_j)$ for each $j \in \D_{\mathrm{test}}$ is then calculated by taking the relative rank of $\hat{S}_j$ among the $\hat{S}_i$ for all $i \in \D_{\mathrm{cal}}$: $\hat{u}(X_j) = ( 1 + \sum_{i \in \D_{\mathrm{cal}}} \mathbb{I}\{ \hat{S}_j \leq \hat{S}_i \})  / (1+n_{\mathrm{cal}})$.
If the null hypothesis for $X_j$ is true, $\hat{S}_j$ is exchangeable with $\hat{S}_i$ for all $i \in \D_{\mathrm{cal}}$, and $\hat{u}(X_j)$ is uniformly distributed on $\{1/(1+n_{\mathrm{cal}}), 2/(1+n_{\mathrm{cal}}), \ldots, 1\}$. Since this distribution is stochastically larger than the continuous uniform distribution on $[0,1]$,  one can say that $\hat{u}(X_j)$ is a valid conformal p-value.
Note however that the p-values $\hat{u}(X_j)$ and $\hat{u}(X_{j'})$ for two different test points $j,j' \in \D_{\mathrm{test}}$ are not independent of one another, even conditional on $\D_{\mathrm{train}}$, because they share the same calibration data.

Despite their mutual dependence, conformal p-values can be utilized within the BH filter to simultaneously probe the $n_{\text{test}}$ hypotheses for all test points while controlling the FDR.
A convenient way to explain the BH filter is as follows \citep{storey2002direct}. Imagine rejecting the null hypothesis for all test points $j$ with $\hat{u}(X_j) \leq s$, for some threshold $s \in [0,1]$.
By monotonicity of $\hat{u}(X_j)$, this amounts to rejecting the null hypothesis for all test points $j$ with $\hat{S}_j \geq t$, for some appropriate threshold $t \in \mathbb{R}$.
An intuitive estimate of the proportion of false discoveries incurred by this rule is:
\begin{align}
\label{eq:fdp_hat-classic}
    \widehat{\text{FDP}}(t)
    = \frac{ n_{\mathrm{test}} }{1+n_{\mathrm{cal}} } \cdot
        \frac{ 1 + \sum_{i\in \D_{\mathrm{cal}}} \mathbb{I} \{ \hat{S}_i \geq t\} }{
    \sum_{j\in \D_{\mathrm{test}}} \mathbb{I} \{ \hat{S}_j \geq t\}}.
\end{align}
This can be understood by noting that $\sum_{j\in \D_{\mathrm{test}}} \mathbb{I} \{ S_j^{(k)}\geq t \}$ is the total number of discoveries, while the numerator should behave similarly to the (latent) number of false discoveries in $\D_{\mathrm{test}}$ due to the exchangeability of $\hat{S}_i$ and $\hat{S}_j$ under the null hypothesis.
With this notation, it can be shown that the BH filter applied at level $\alpha \in (0,1)$ computes an adaptive threshold
\begin{align}
\label{eq:threshold-classic}
\hat{t}^{\mathrm{BH}} = \min \left\{ t\in \{\hat{S}_i\}_{i \in \mathcal{D}_{\mathrm{cal}} \cup \mathcal{D}_{\mathrm{test}}} : \widehat{\text{FDP}}(t) \leq \alpha \right\},
\end{align}
and rejects all null hypotheses $j$ with $\hat{S}_j \geq \hat{t}^{\mathrm{BH}}$; see \citet{fairness} for a derivation of this connection.
This procedure was proved by \citet{conformal-p-values} to control the FDR below $\alpha$.

\subsection{Review of FDR control with \texttt{AdaDetect}}  \label{sec:adadetect}

Recently, \citet{ml-fdr} proposed \texttt{AdaDetect}, a more sophisticated version of the method reviewed in Section~\ref{sec:fdr-conformal}.
The main innovation of \texttt{AdaDetect} is that it leverages a binary classification model instead of a one-class classifier. In particular, \texttt{AdaDetect} trains a binary classifier to distinguish the inlier data in $\D_{\mathrm{train}}$ from the mixture of inliers and outliers contained in the union of $\D_{\mathrm{cal}}$ and $\D_{\mathrm{test}}$.
The key idea to achieve FDR control is that the training process should remain invariant to permutations of the calibration and test samples.
While the true inlier or outlier nature of the observations in $\D_{\mathrm{test}}$ is obviously unknown at training time, \texttt{AdaDetect} can still extract some useful information from the test data which would otherwise be ignored by the more traditional split-conformal approach reviewed in Section~\ref{sec:fdr-conformal}. In particular, \texttt{AdaDetect} can leverage the test data to automatically tune any desired model hyper-parameters in order to approximately maximize the number of discoveries.
A similar idea also motivates the alternative method of {\em integrative} conformal p-values proposed by \citet{liang2022integrative}, although the latter requires the additional assumption that some labeled outlier data are available, and is therefore not discussed in equal detail within this paper.

Despite a more sophisticated use of the available data compared to the split-conformal method reviewed in Section~\ref{sec:fdr-conformal}, \texttt{AdaDetect} still suffers from the same limitation that it must calibrate its inferences based on a single random data subset $\D_{\mathrm{cal}}$, and thus its results remain aleatory.
For simplicity, Section~\ref{sec:weighted-agg-conformal-e-values} begins by explaining how to derandomize standard split-conformal inferences; then, the proposed method will be easily extended in Section~\ref{sec:ada-training} to derandomize \texttt{AdaDetect}.

\section{Method}

\subsection{Derandomizing split-conformal inferences}
\label{sec:weighted-agg-conformal-e-values}

Consider $K \geq 1$ repetitions of the split-conformal analysis reviewed in Section~\ref{sec:fdr-conformal}, each starting with an independent split of the same reference data into $\D_{\mathrm{train}}^{(k)}$ and $\D_{\mathrm{cal}}^{(k)}$.
For each repetition $k \in [K]$, after training the machine learning model on $\D_{\mathrm{train}}^{(k)}$ and computing conformity scores on $\D_{\mathrm{cal}}^{(k)}$ and $\D_{\mathrm{test}}$, one can estimate the false discovery proportion corresponding to the rejection of all test points with scores above a fixed rejection threshold $t \in \mathbb{R}$, similarly to~\eqref{eq:fdp_hat-classic}, with:
\begin{equation}
\label{eq:fdp_hat}
    \widehat{\text{FDP}}^{(k)}(t)
    = \frac{ n_{\mathrm{test}} }{n_{\mathrm{cal}} } \cdot
        \frac{ \sum_{i\in \D_{\mathrm{cal}}^{(k)}} \mathbb{I} \{ \hat{S}_i^{(k)} \geq t \} }{
    \sum_{j\in \D_{\mathrm{test}}} \mathbb{I} \{ \hat{S}_j^{(k)}\geq t \}}.
\end{equation}
Note that the estimate in~\eqref{eq:fdp_hat} differs slightly from that in~\eqref{eq:fdp_hat-classic} as it lacks the ``+1'' constant term in the numerator and denominator.
While it is possible to include the ``+1'' terms in~\eqref{eq:fdp_hat}, this is not needed by our theory and we have observed that it often makes our method unnecessarily conservative.
For any fixed $\alpha_{\mathrm{bh}} \in (0,1)$, let $\hat{t}^{(k)}$ be the corresponding BH threshold~\eqref{eq:threshold-classic} at the nominal FDR level $\alpha_{\mathrm{bh}}$:
\begin{equation}
\label{eq:threshold}
\hat{t}^{(k)} = \min \{ t\in \tilde{\mathcal{D}}^{(k)}_{\mathrm{cal-test}} : \widehat{\text{FDP}}^{(k)}(t) \leq \alpha_{\mathrm{bh}} \},
\end{equation}
where $\tilde{\mathcal{D}}^{(k)}_{\mathrm{cal-test}} = \{ \hat{S}_{i}^{(k)} \}_{i \in \D_{\mathrm{test}} \cup \D_{\mathrm{cal}}^{(k)} }$.
For each test point $j \in \D_{\mathrm{test}}$, define the following rescaled indicator of whether $\hat{S}_j^{(k)}$ exceeds $\hat{t}^{(k)}$:
\begin{align} \label{e-value}
    e_j^{(k)}
        = (1 + n_{\mathrm{cal}}) \cdot
          \frac{\mathbb{I}\{ \hat{S}_j^{(k)}\geq \hat{t}^{(k)}\} }
               {1+\sum_{i\in \D_{\mathrm{cal}}^{(k)}} \mathbb{I} \{ \hat{S}_i^{(k)}\geq \hat{t}^{(k)} \} }.
\end{align}
Intuitively, this quantifies not only whether the $j$-th null hypothesis would be rejected by the BH filter at the nominal FDR level $\alpha_{\mathrm{bh}}$, but also how extreme $\hat{S}_j^{(k)}$ is relative to the calibration scores. In other words, a large $e_j^{(k)}$ suggests that the test point may be an outlier, where this variable can take any of the following values: 0, 1, $(1+n_{\mathrm{cal}})/n_{\mathrm{cal}}$, $(1+n_{\mathrm{cal}})/(n_{\mathrm{cal}}-1)$, \dots, $(1+n_{\mathrm{cal}})$. This approach, inspired by \citet{drand-kn}, is not the only possible way of constructing e-values to derandomize conformal inferences, as discussed in Supplementary Section \ref{app:baselines}. However, we will show that it works well in practice and it typically achieves higher power compared to standard p-to-e calibrators \citep{e-value} applied to conformal p-values.
This advantage partly derives from the fact that~\eqref{e-value} can gather strength from many different test points, and partly from the fact that it is not a proper e-value according to the original definition of \citet{e-value}, in the sense that its expected value may be larger than one even if $X_j$ is an inlier.
Instead, we will show that our e-values satisfy a relaxed {\em average validity} property \citep{drand-kn} that is sufficient to guarantee FDR control while allowing more numerous discoveries.

After evaluating~\eqref{e-value} for all $j \in \D_{\mathrm{test}}$ and all $k \in [K]$, we aggregate the evidence against the $j$-th null hypothesis into a single statistic $\bar{e}_j$ by taking a weighted average:
$$
\bar{e}_j = \sum_{k=1}^K w^{(k)} e_j^{(k)}, \qquad \sum_{k=1}^{K} w^{(k)} = 1,
$$
based on some appropriate normalized weights $w^{(k)}$.
Intuitively, the role of $w^{(k)}$ is to allow for the possibility that the machine learning models based on different realizations of the training subset may not all be equally powerful at separating inliers from outliers.
In the remainder of this section, we will take these weights to be known a-priori for all $k \in [K]$, thus representing relevant {\em side information}; e.g., in the sense of \citet{genovese2006false} and \citet{ren2020knockoffs}.
For simplicity, one may think for the time being of trivial uninformative weights $w^{(k)} = 1/K$.
Of course, it would be preferable to allow these weights to be data-driven, but such an extension is deferred to Section~\ref{sec:weights} for conciseness.

Having calculated aggregate {\em e-values} $\bar{e}_j$ with the procedure described above, which is outlined by Algorithm~\ref{drand-e-value} in the Supplementary Material, our method rejects the null hypothesis for all $j\in \D_{\mathrm{test}}$ whose $\bar{e}_j$ is greater than an adaptive threshold calculated by applying the eBH filter of \citet{eBH}, which is outlined for completeness by Algorithm~\ref{alg:ebh} in the Supplementary Material.
We refer to \citet{eBH} for a more detailed discussion of the eBH filter. Here, it suffices to recall that the eBH filter computes an adaptive rejection threshold based on the $n_{\mathrm{test}}$ input e-values and on the desired FDR level $\alpha \in (0,1)$.
Then, our following result states that the overall procedure is guaranteed to control the FDR below $\alpha$, under a relatively mild exchangeability assumption.

\begin{assumption} \label{exchangeable}
The inliers in $\mathcal{D}$ and the null test points are exchangeable conditional on the non-null test points.
\end{assumption}

\begin{theorem}\label{thm:e-fdr}
Suppose Assumption~\ref{exchangeable} holds.
Then, the e-values computed by Algorithm~\ref{drand-e-value} satisfy:
\begin{equation} \label{eq:e-values-valid}
    \sum_{j\in \Hnull} \E \left[ \bar{e}_j\right] \leq n_{\mathrm{test}}.
\end{equation}
\end{theorem}

The proof of Theorem~\ref{thm:e-fdr} is in the Supplementary Section \ref{app:proofs}. Combined with Theorem~2 from~\citet{drand-kn}, this result implies our method controls the FDR below the desired target level $\alpha$.

\begin{corollary}[\citet{drand-kn}] \label{thm:e-fdr-cor}
The eBH filter of \citet{eBH} applied at level $\alpha \in (0,1)$ to e-values $\{\bar{e}_j\}_{j \in \D_{\mathrm{test}}}$, satisfying~\eqref{eq:e-values-valid} guarantees FDR $\leq \alpha$.
\end{corollary}

\begin{remark} Assumption~\ref{exchangeable} does not require that the inliers are independent of the outliers.
\end{remark}

\begin{remark} Theorem~\ref{thm:e-fdr} holds regardless of the value of the hyper-parameter $\alpha_{\mathrm{bh}}$ of Algorithm~\ref{drand-e-value}, which appears in~\eqref{eq:threshold}. See Section~\ref{sec:alpha_t} for further details about the choice of $\alpha_{\mathrm{bh}}$.
\end{remark}

\subsection{Leveraging data-driven weights} \label{sec:weights}

Our method can be extended to leverage adaptive weights based on the data in $\mathcal{D}$ and $\mathcal{D}_{\mathrm{test}}$, as long as each weight $w^{(k)}$ is invariant to permutations of the test point with the corresponding calibration samples in $\mathcal{D}_{\mathrm{cal}}^{(k)}$.
In other words, we only require that these weights be written in the form of
\begin{align} \label{eq:invariant-weights}
w^{(k)} = \omega ( \tilde{\mathcal{D}}^{(k)}_{\mathrm{cal-test}} ).
\end{align}
The function $\omega$ may depend on $\D_{\mathrm{train}}^{(k)}$ but not on $\D_{\mathrm{cal}}^{(k)}$ or $\D_{\mathrm{test}}$.
An example of a useful weighting scheme satisfying this property is at the end of this section.
The general method is summarized by Algorithm~\ref{drand-e-value-adaptive} in the Supplementary Material, which extends Algorithm~\ref{drand-e-value}.
This produces e-values that control the FDR in conjunction with the eBH filter of \citet{eBH}.

\begin{theorem}\label{thm:e-fdr-adaptive}
Suppose Assumption~\ref{exchangeable} holds.
Then, the e-values computed by Algorithm~\ref{drand-e-value-adaptive} satisfy~\eqref{eq:e-values-valid}, as long as the adaptive weights obey~\eqref{eq:invariant-weights}.
\end{theorem}

An example of a valid weighting function applied in this paper is the following.
Imagine having some prior side information suggesting that the proportion of outliers in $\D_{\mathrm{test}}$ is approximately $\gamma \in (0,1)$.
Then, a natural choice to measure the quality of the $k$-th model is to let $\tilde{w}^{(k)} = |\tilde{v}^{(k)}|$, where $\tilde{v}^{(k)}$ is the standard t-statistic for testing the difference in means between the top $\lceil n_{\mathrm{test}} \cdot \gamma \rceil $ largest values in $\tilde{\mathcal{D}}^{(k)}_{\mathrm{cal-test}}$ and the remaining ones.
See Algorithm~\ref{alg:t-weights} in the Supplementary Material for further details.
Intuitively, Algorithm~\ref{alg:t-weights} tends to assign larger weights to models achieving stronger out-of-sample separation between inliers and outliers.
Of course, this approach may not always be optimal but different weighting schemes could be easily accommodated within our framework.

\subsection{Derandomizing \texttt{AdaDetect} with \texttt{E-AdaDetect}} \label{sec:ada-training}

The requirement discussed in Section~\ref{sec:weights} that the data-adaptive weights should be invariant to permutations of the calibration and test samples is analogous to the idea utilized by \texttt{AdaDetect} \citep{ml-fdr} to train more powerful machine learning models leveraging also the information contained in the test set; see Section~\ref{sec:adadetect}.
This implies that Theorem~\ref{thm:e-fdr-adaptive} remains valid even if our method is implemented based on $K$ machine learning models each trained by looking also at the unordered union of all data points in $\D_{\mathrm{cal}}^{(k)} \cup \D_{\mathrm{test}}$, for each $k \in [K]$.
See Algorithm~\ref{drand-e-value-adaptive-ada} in the Supplementary Material for a detailed implementation of this extension of our method, which we call \texttt{E-AdaDetect}.

\subsection{Tuning the FDR hyper-parameter}
\label{sec:alpha_t}

As explained in Section~\ref{sec:weighted-agg-conformal-e-values}, our method involves a hyper-parameter $\alpha_{\mathrm{bh}}$ controlling the BH thresholds $\hat{t}^{(k)}$ in~\eqref{eq:threshold}.
Intuitively, higher values of $\alpha_{\mathrm{bh}}$ tend to increase the number of both test and calibration scores exceeding the rejection threshold at each of the $K$ iterations.
Such competing effects make it generally unclear whether increasing $\alpha_{\mathrm{bh}}$ leads to larger e-values in~\eqref{e-value} and hence higher power. This trade-off was studied by \citet{drand-kn} while derandomizing the knockoff filter, and they suggested setting $\alpha_{\mathrm{bh}} < \alpha$.
In this paper, we adopt $\alpha_{\mathrm{bh}} = \alpha/10$, which we have observed to work generally well in our context, although even higher power can sometimes be obtained with different values of $\alpha_{\mathrm{bh}}$, especially if the number of outliers in the test set is large.
While we leave it to future research to determine whether further improvements are possible, it is worth noting that a straightforward extension of our method, not explicitly implemented in this paper, can be obtained by further averaging e-values obtained with different choices of $\alpha_{\mathrm{bh}}$.
Such extension does not affect the validity of~\eqref{eq:e-values-valid} due to the linearity of expected values.

\section{Numerical experiments}

\subsection{Setup and performance metrics}

This section compares empirically the performance of \texttt{AdaDetect} and our proposed derandomized method described in Section~\ref{sec:ada-training}, namely \texttt{E-AdaDetect}.
Both procedures are deployed using a binary logistic regression classifier \citep{ml-fdr} as the base predictive model. 
The reason why we focus on derandomizing \texttt{AdaDetect} instead of traditional split-conformal inferences based on a one-class classifier \citep{conformal-p-values} is that we have observed that \texttt{AdaDetect} often achieves higher power on the data considered in this paper, which makes it a more competitive benchmark.
However, additional experiments reporting on the performance of our derandomization method applied in combination with one-class classifiers can be found in the Supplementary Sections \ref{app:synthetic-experiments-OC} and \ref{app:real-experiments-OC}.

As the objective of this paper is to powerfully detect outliers while mitigating algorithmic randomness, we assess the performance of each method over $M=100$ independent analyses based on the same fixed data and the same test set.
For each repetition $m$ of the novelty detection analysis based on the fixed data, we identify a subset $\mathcal{R}^{(m)} \subseteq \D_{\mathrm{test}}$ of likely outliers (the rejected null hypotheses) and evaluate the average power and false discovery proportion, namely
\begin{align} \label{eq:def-power-fdr}
  & \widehat{\text{Power}} = \frac{1}{M} \sum_{m=1}^M \frac{|\mathcal{R}^{(m)} \cap \mathcal{D}_{\text{test}}^{\text{non-null}} |}{ |\mathcal{D}_{\text{test}}^{\text{non-null} }|},
  & \widehat{\text{FDR}} = \frac{1}{M} \sum_{m=1}^M \frac{|\mathcal{R}^{(m)} \cap \mathcal{D}_{\text{test}}^{\text{null}} |}{\max\{|\mathcal{R}^{(m)}|,1\}},
\end{align}
 where $\mathcal{D}_{\text{test}}^{\text{non-null}} = \mathcal{D}_{\text{test}} \setminus \mathcal{D}_{\text{test}}^{\text{null}}$ indicates the true outliers in the test set.
The average false discovery proportion defined in~\eqref{eq:def-power-fdr} is not the FDR, which is the quantity we can theoretically guarantee to control. In fact, $\text{FDR} = \mathbb{E}[\widehat{\text{FDR}}]$, with expectation taken with respect all randomness in the data. Nonetheless, we will see that this average false discovery proportion is also controlled in practice within all data sets considered in this paper. 
The advantage of this setup is that it makes it natural to estimate algorithmic variability by observing the consistency of each rejection across independent analyses. In particular, after defining $R_{j,m}$ as the indicator of whether the $j$-th null hypothesis was rejected in the $m$-th analysis, we can evaluate the average variance in the rejection events:
\begin{align} \label{eq:def-var}
        \widehat{\text{Variance}} & = \frac{1}{n_{\text{test}}}\sum_{j=1}^{n_{\text{test}}} \frac{1}{M-1} \sum_{m=1}^{M} \left( R_{j,m} - \bar{R}_j\right)^2,
\end{align}
where $\bar{R}_j = (1/M) \sum_{m=1}^{M} R_{j,m}$.
Intuitively, it would be desirable to maximize power while simultaneously minimizing both the average false discovery proportion and the variability. 
In practice, however, these metrics often compete with one another; hence why we focus on comparing power and variability for methods designed to control the FDR below the target level $\alpha=0.1$.

\subsection{Experiments with synthetic data} \label{sec:synthetic}

Synthetic reference and test data consisting of 100-dimensional vectors $X$ are generated as follows.
The reference set contains only inliers, drawn i.i.d.~from the standard normal distribution with independent components, $\mathcal{N}(0, I_{100})$.
Unless specified otherwise, the test set contains 90\% inliers and 10\% outliers, independently sampled from $\mathcal{N}(\mu, I_{100})$.
The first 5 entries of $\mu$ are equal to a constant parameter, to which we refer as the {\em signal amplitude}, while the remaining 95 entries are zeros.
The size of the reference set is $n=2000$, with 1000 samples in the training subset and $1000$ in the calibration subset.
The size of the test set is $n_{\mathrm{test}}=1000$.
Both \texttt{E-AdaDetect} and \texttt{AdaDetect} are applied based on the same logistic regression classifier with default hyper-parameters.

\subsubsection{The effect of the signal strength}

Figure \ref{fig:data_difficulty} compares the performance of \texttt{E-AdaDetect} (applied with $K=10$) to that of \texttt{AdaDetect}, as a function of the signal amplitude. The results confirm both methods control the FDR but ours is less variable, as expected. 
The comparison becomes more interesting when looking at power: \texttt{AdaDetect} tends to detect more outliers on average if the signal strength is low, but \texttt{E-AdaDetect} can also outperform by that metric if the signals are strong.
This may be explained as follows.
If the signal strength is high, most true discoveries produced by \texttt{AdaDetect} are relatively stable across different analyses, while false discoveries may be more aleatory, consistently with the illustration of Figure~\ref{fig:illustration}. 
Such situation is ideal for derandomization, which explains why \texttt{E-AdaDetect} is able to simultaneously achieve high power and low false discovery proportion. 
By contrast, if the signals are weak, the true outlier discoveries produced by \texttt{AdaDetect} are relatively scarce and unpredictable, thus behaving not so differently from the false findings.
In this case, one could argue that stability becomes even more important to facilitate the interpretation of any findings, and that may justify some loss in average power.

\begin{figure*}[!htb]
  \centering
  \includegraphics[width=0.31\textwidth]{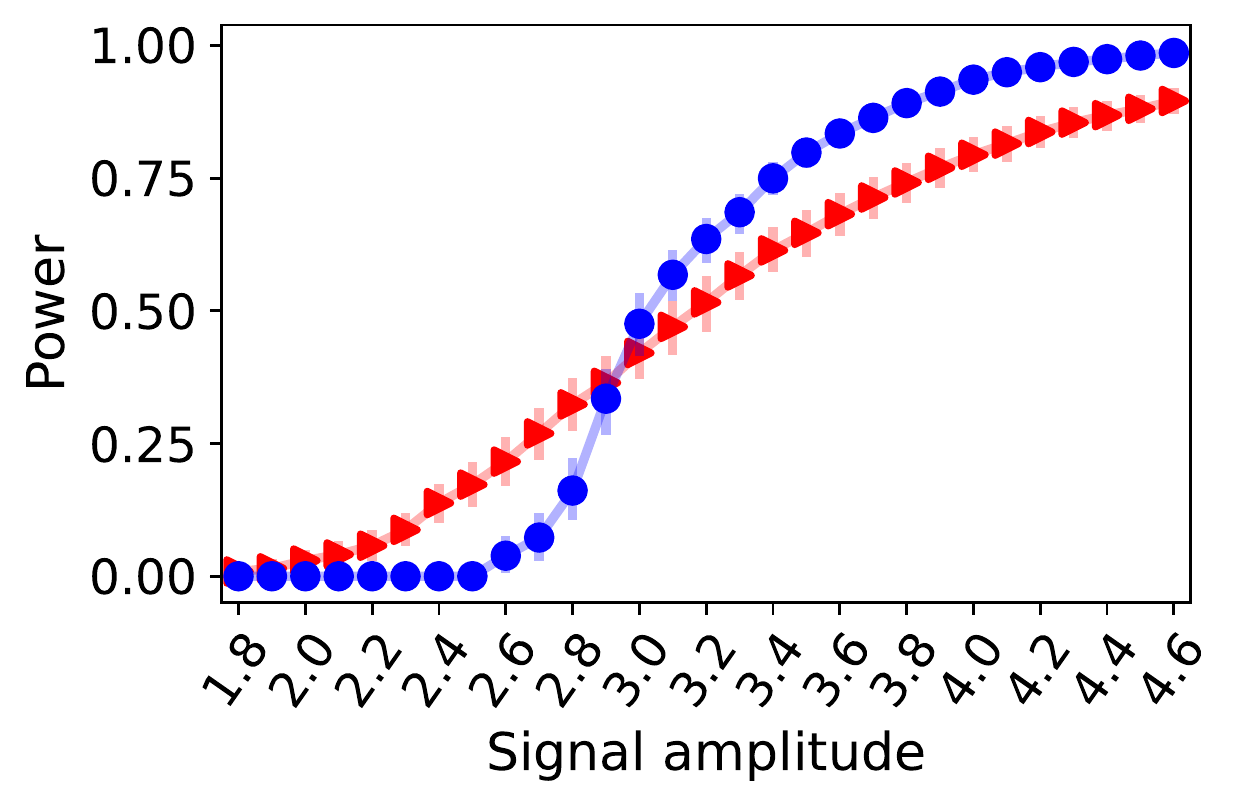}
  \includegraphics[width=0.31\textwidth]{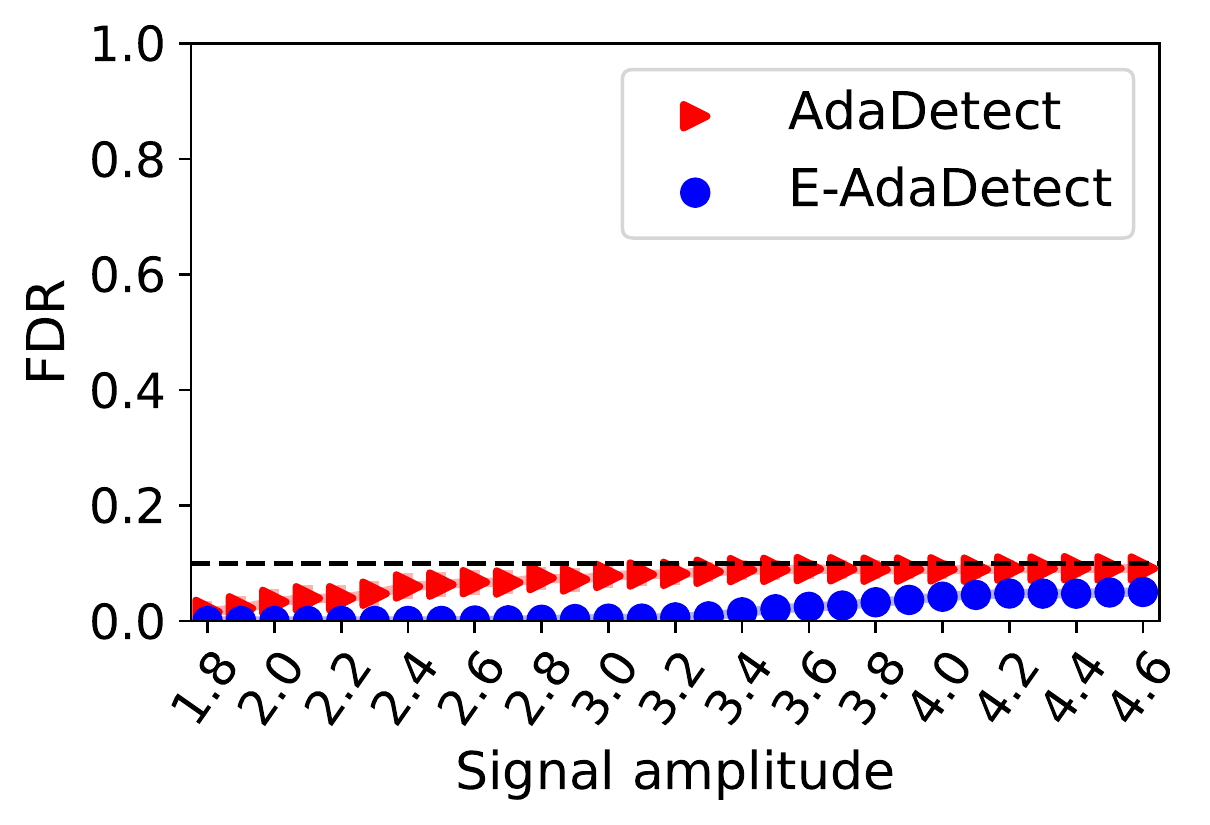}
  \includegraphics[width=0.31\textwidth]{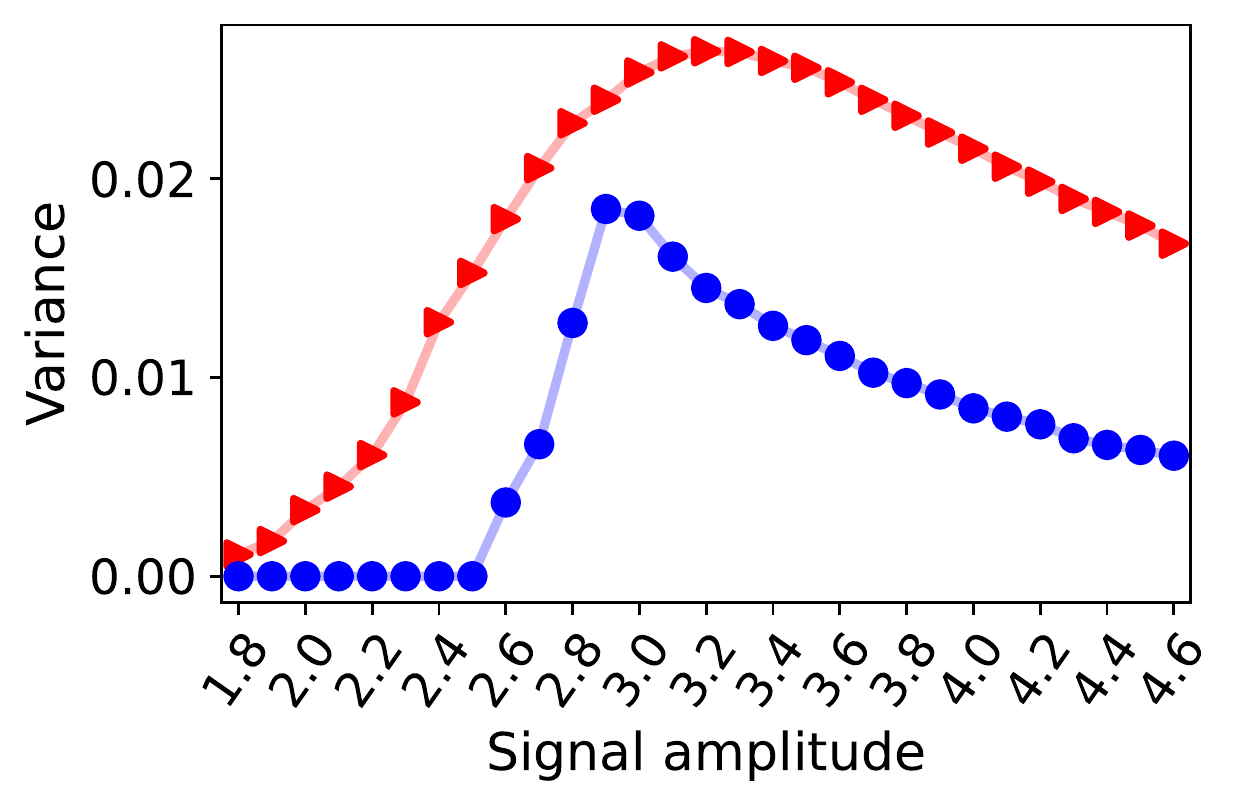}
  \caption{Performance on synthetic data of the proposed derandomized outlier detection method, \texttt{E-AdaDetect}, applied with $K=10$, compared to that of its randomized benchmark, \texttt{AdaDetect}, as a function of the signal strength. Both methods leverage a logistic regression binary classifier. Left: average proportion of true outliers that are discovered (higher is better). Center: average proportion of false discoveries (lower is better). Right: variability of the findings (lower is better).}
\label{fig:data_difficulty}
  \end{figure*}

\subsubsection{The effect of the number of analyses $K$}

Figure~\ref{fig:iterations} investigates the effect of varying the number of analyses $K$ aggregated by \texttt{E-AdaDetect}. 
Here, the signal amplitude is fixed to 3.4 (strong signals), while $K$ is varied between 1 and 30.
As expected, the results show that the variability of the findings obtained with \texttt{E-AdaDetect} decreases as $K$ increases.
The average proportion of false discoveries obtained with \texttt{E-AdaDetect} also tends to decrease when $K$ is large, which can be understood by noting that spurious findings are less likely to be reproduced consistently across multiple independent analyses of the same data.
Regarding power, the average number of true outliers detected by \texttt{E-AdaDetect} appears to monotonically increase with $K$, although this is not always true in other situations, as shown in the Supplementary Section \ref{app:synthetic-experiments-adadetect}.
In fact, if the signals are weak, \texttt{E-AdaDetect} may lose some power with larger values of $K$ (although some $K>1$ may be optimal), consistently with the results shown in Figure~\ref{fig:data_difficulty}. 
Thus, we recommend practitioners to utilize larger values of $K$ in applications where higher power is expected.
Finally, note that the power of \texttt{E-AdaDetect} is generally lower compared to that of \texttt{AdaDetect} in the special case of $K=1$, although this is not a practically relevant value of $K$ because it does not allow any derandomization.
The reason why the power of \texttt{E-AdaDetect} is lower when $K=1$ is that this method relies on the eBH filter. The latter is relatively conservative as an FDR-controlling strategy because it requires no assumptions about the dependencies of the input statistics.

\begin{figure*}[!htb]
  \centering
\includegraphics[width=0.32\textwidth]{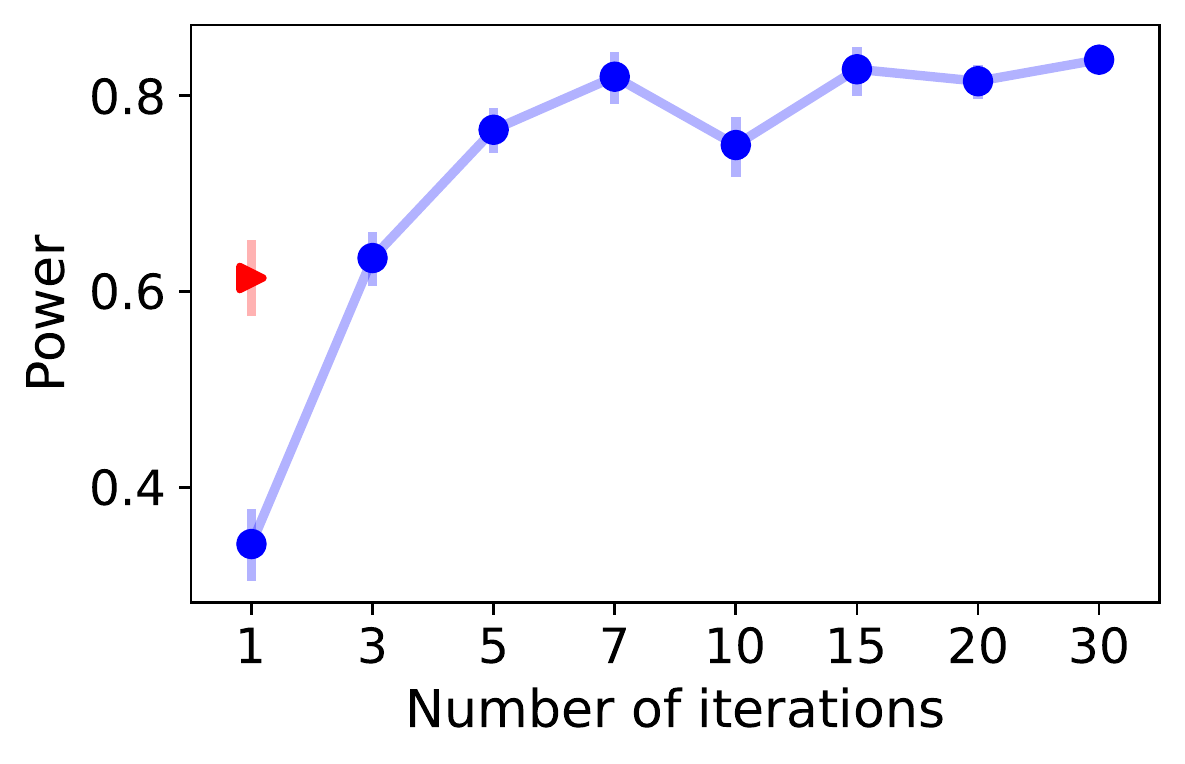}
\includegraphics[width=0.32\textwidth]{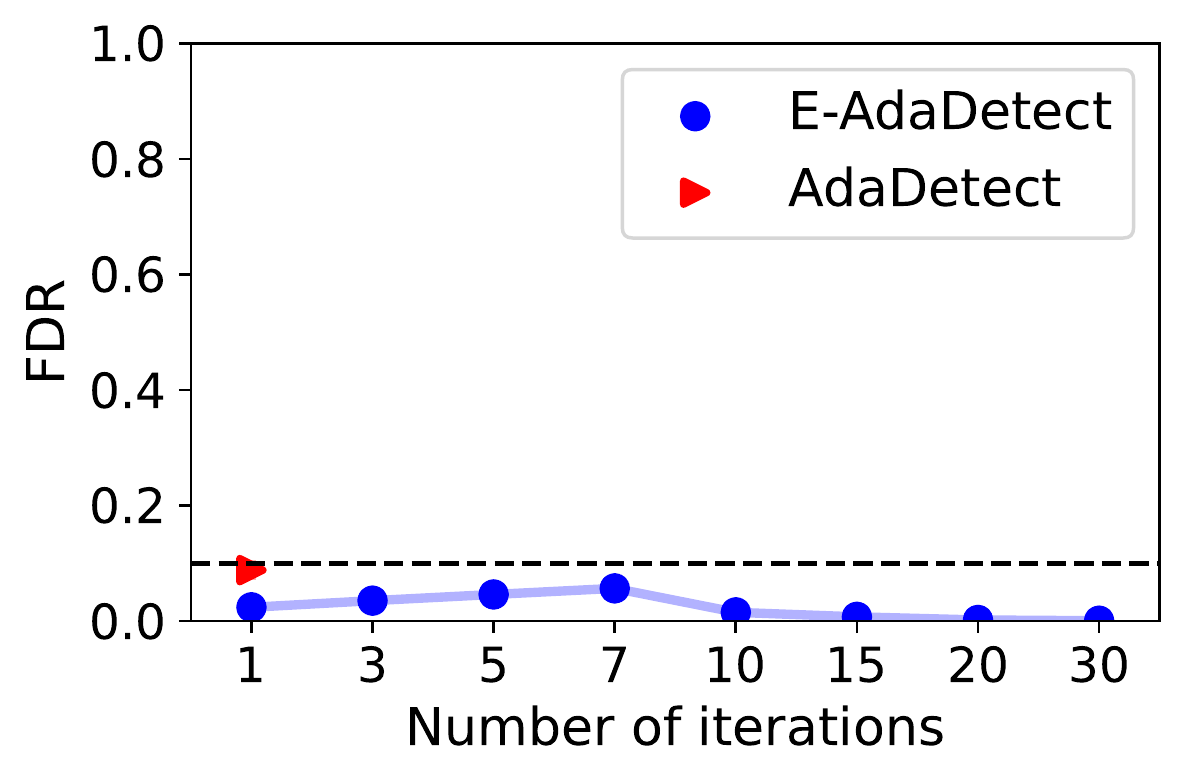}
\includegraphics[width=0.32\textwidth]{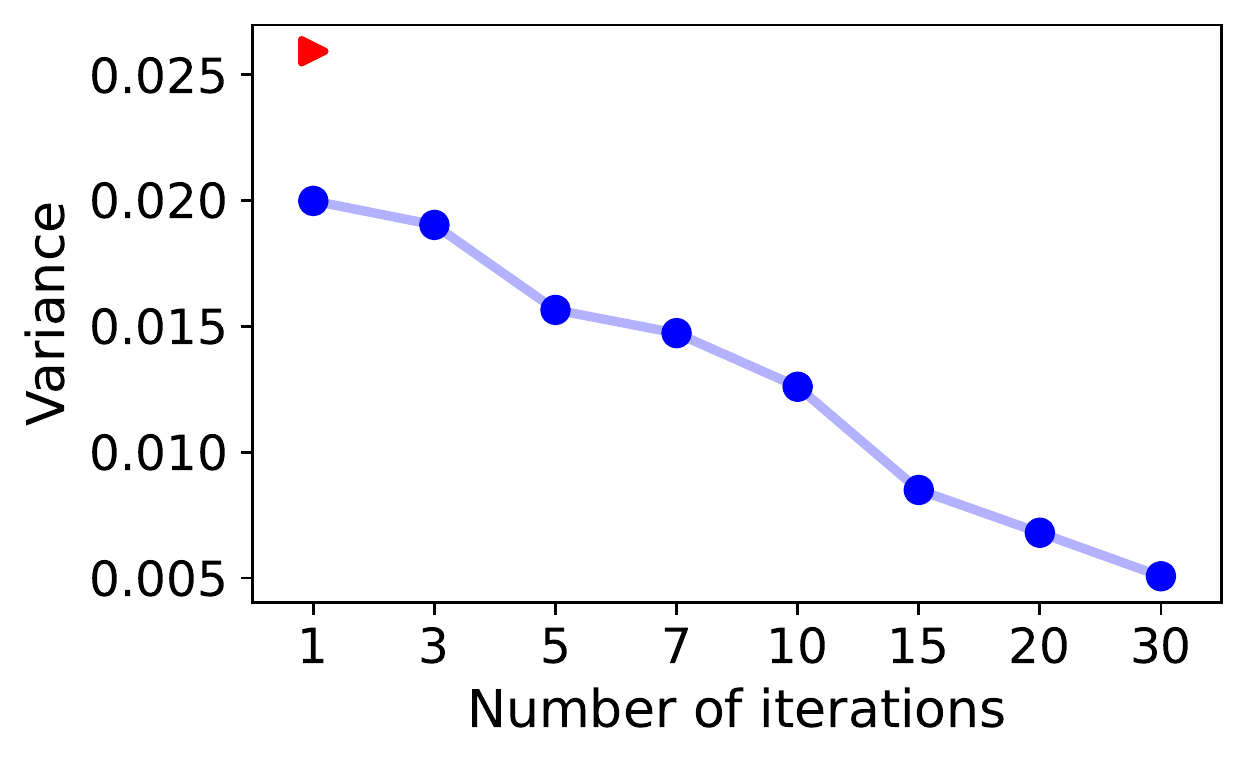}
  \caption{Performance on synthetic data of \texttt{E-AdaDetect}, as a function of the number $K$ of derandomized analyses, compared to \texttt{AdaDetect}. Note that the latter can only be applied with a single data split (or iteration).
The signal amplitude is $3.4$.
Other details are as in Figure~\ref{fig:data_difficulty}. 
}
\label{fig:iterations}
  \end{figure*}

\subsubsection{The effect of the weighting strategy}

This section highlights the practical advantage of being able to use data-adaptive model weights within \texttt{E-AdaDetect}. For this purpose, we carry out experiments similar to those of Figure~\ref{fig:data_difficulty},
but leveraging a logistic regression model trained with different choices of hyper-parameters in each of the $K$ analyses.
Specifically, we fit a sparse logistic regression model using $K=10$ different values of the regularization parameter. To induce higher variability in the predictive rules, one model was trained with a regularization parameter equal to $0.0001$, while the others were trained with regularization parameters equal to $1$, $10$, $50$, and $100$, respectively.
Then, we apply \texttt{E-AdaDetect} using different weighting schemes: constant equal weights (`uniform'), data-driven weights calculated with the t-statistic approach (``t-test'') summarized by Algorithm~\ref{alg:t-weights}, and a simple alternative {\em trimmed average} data-driven approach (``avg. score'') outlined by Algorithm~\ref{alg:avg-weights} in the Supplementary Material.
The results in Figure \ref{fig:weight} show that the data-driven aggregation scheme based on t-statistics is the most effective one, often leading to much higher power.
We have chosen not to compare \texttt{E-AdaDetect} to the automatic \texttt{AdaDetect} hyper-parameter tuning strategy proposed in Section 4.5 of \citet{ml-fdr} because we found that it does not perform very well in our experiments, possibly due to the relatively low sample size.

\begin{figure*}[!htb]
  \centering
\includegraphics[width=0.32\textwidth]{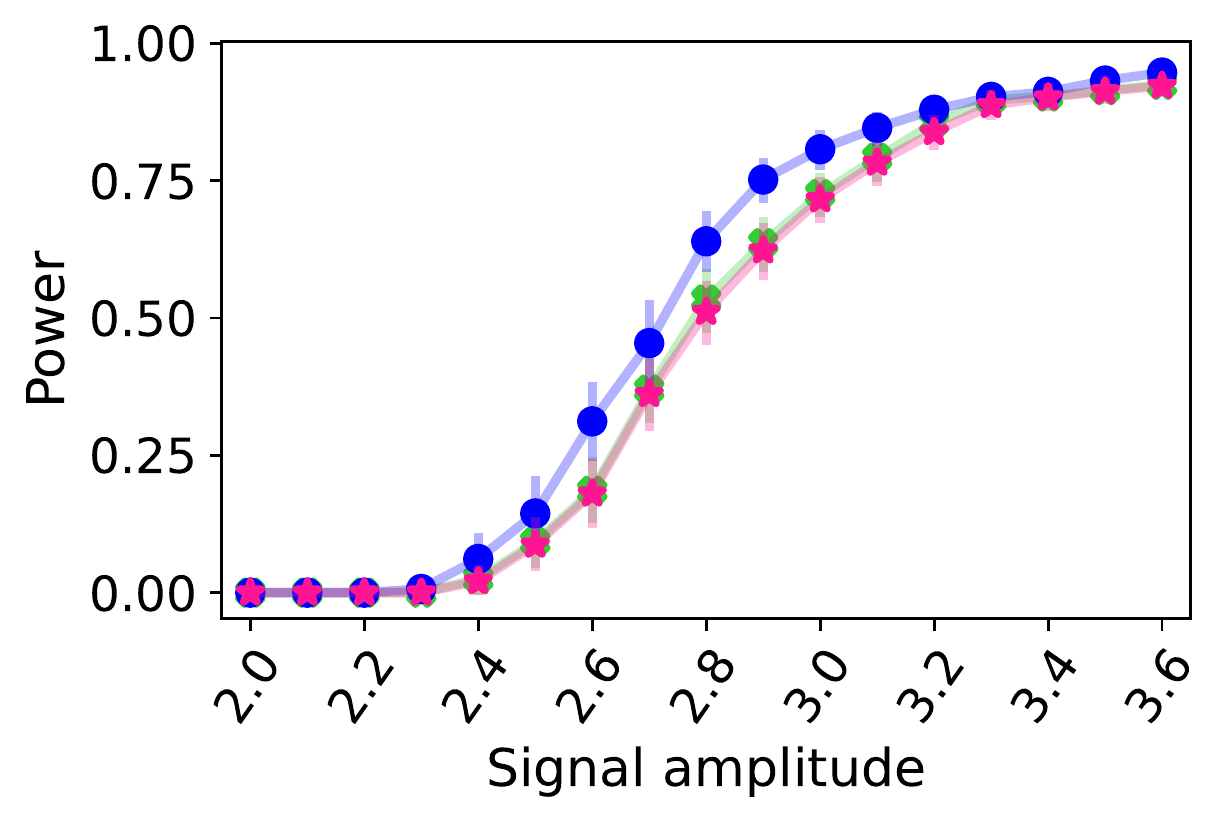}
\includegraphics[width=0.32\textwidth]{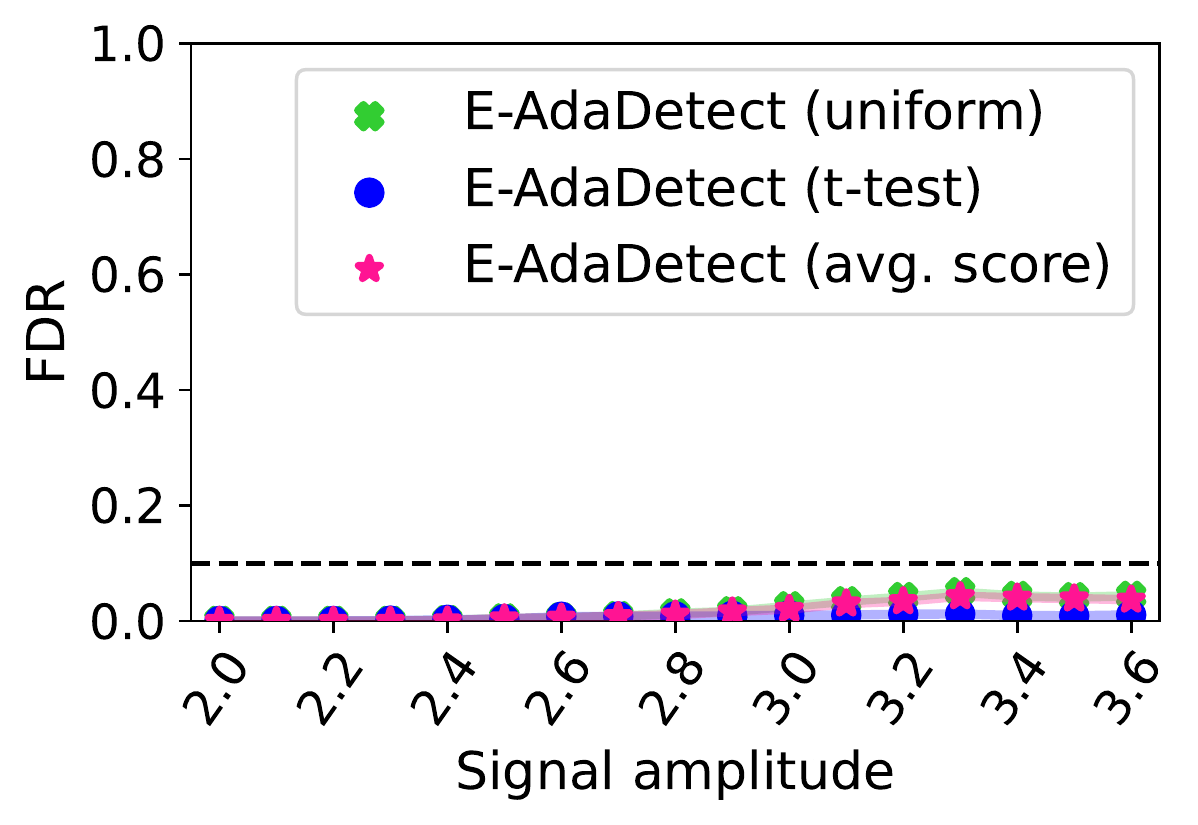}
\includegraphics[width=0.32\textwidth]{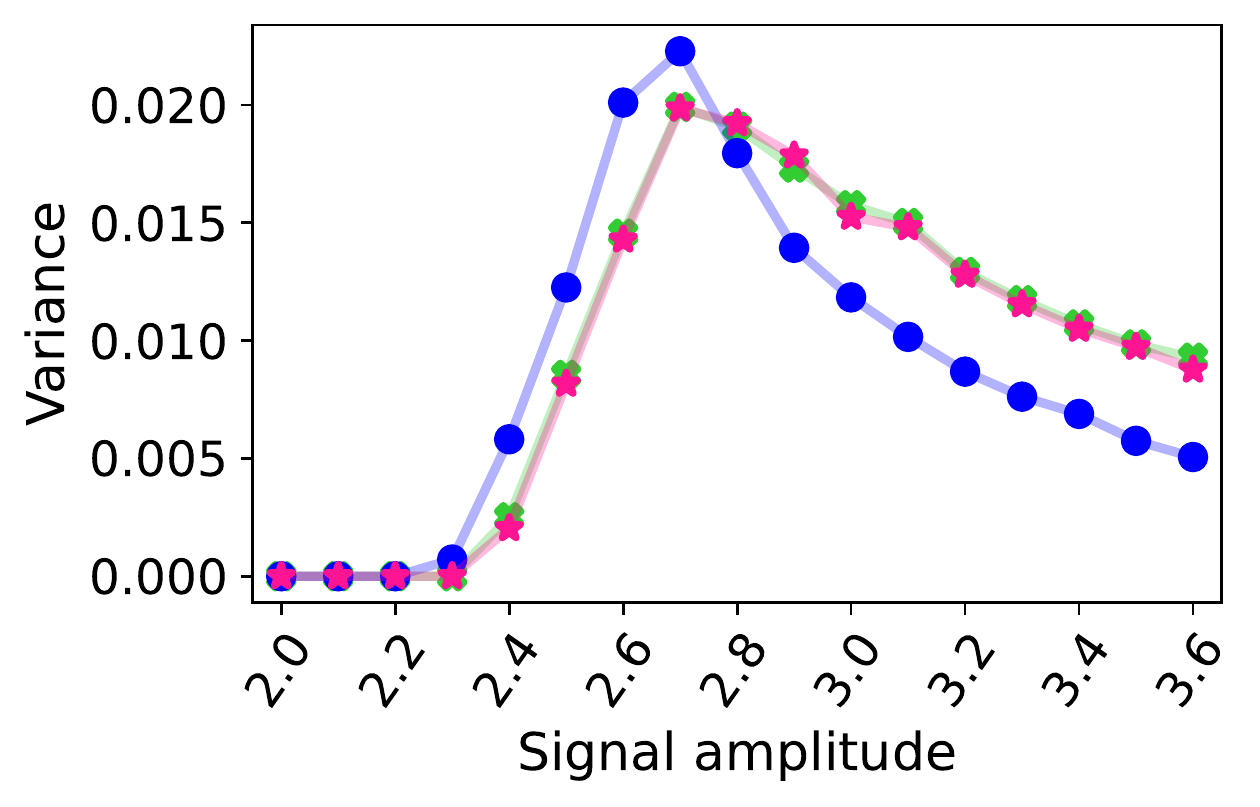}
  \caption{Performance on synthetic data of \texttt{E-AdaDetect} applied with different model weighting schemes, as a function of the signal strength.
The model weights are designed to prioritize the results of randomized analyses based on models that are more effective at separating inliers from outliers.
The t-test approach tends to lead to higher power.
Other details are as in Figure~\ref{fig:data_difficulty}. 
  }
\label{fig:weight}
  \end{figure*}

To further demonstrate the effectiveness of data-driven weighting, we turn to analyze the performance of \texttt{E-AdaDetect} on four real-world outlier detection data sets: {\em musk}, {\em shuttle}, {\em KDDCup99}, and {\em credit card}. We refer to Supplementary Section \ref{app:real-experiments} for more information regarding these data. 
Similar to Figure \ref{fig:weight}, our \texttt{E-AdaDetect} method is applied $K=10$ times to each data set, each time leveraging a different predictive model as follows. Half of the models are random forests implemented with varying max-depth hyper-parameters (10, 12, 20, 30, and 7), while the other half are support vector machines with an RBF kernel with varying width hyper-parameters (0.1, 0.001, 0.5, 0.2, and 0.03). 
This setup is interesting because different models often tend to perform differently in practice, and it is usually unclear a-priori which combination of model and hyper-parameters is optimal for a given data set. 
Figure \ref{fig:weight-real-data} summarizes the results, demonstrating that both data-driven weighting schemes (``t-test'' and ``avg. score'') lead to more numerous discoveries compared to the ``uniform'' weighting baseline, and that the ``t-test'' approach is the most powerful weighting scheme here. These results are in line with the synthetic experiment presented in Figure \ref{fig:weight}. Lastly, the variance metrics reported in Figure \ref{fig:weight-real-data} also suggest that data-driven weighting further enhances the algorithmic stability.

\begin{figure*}[!htb]
  \centering
\includegraphics[width=0.32\textwidth]{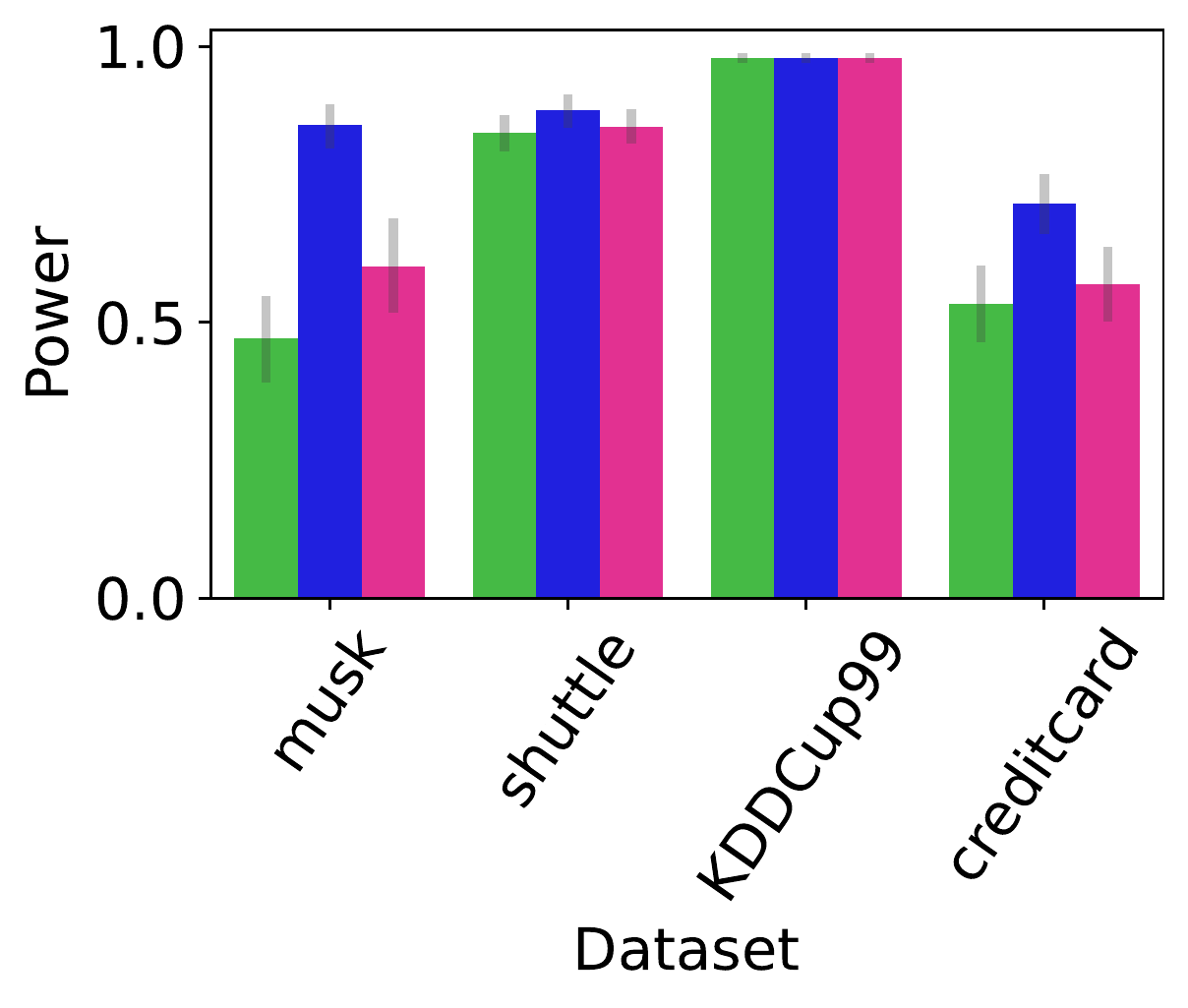}
\includegraphics[width=0.32\textwidth]{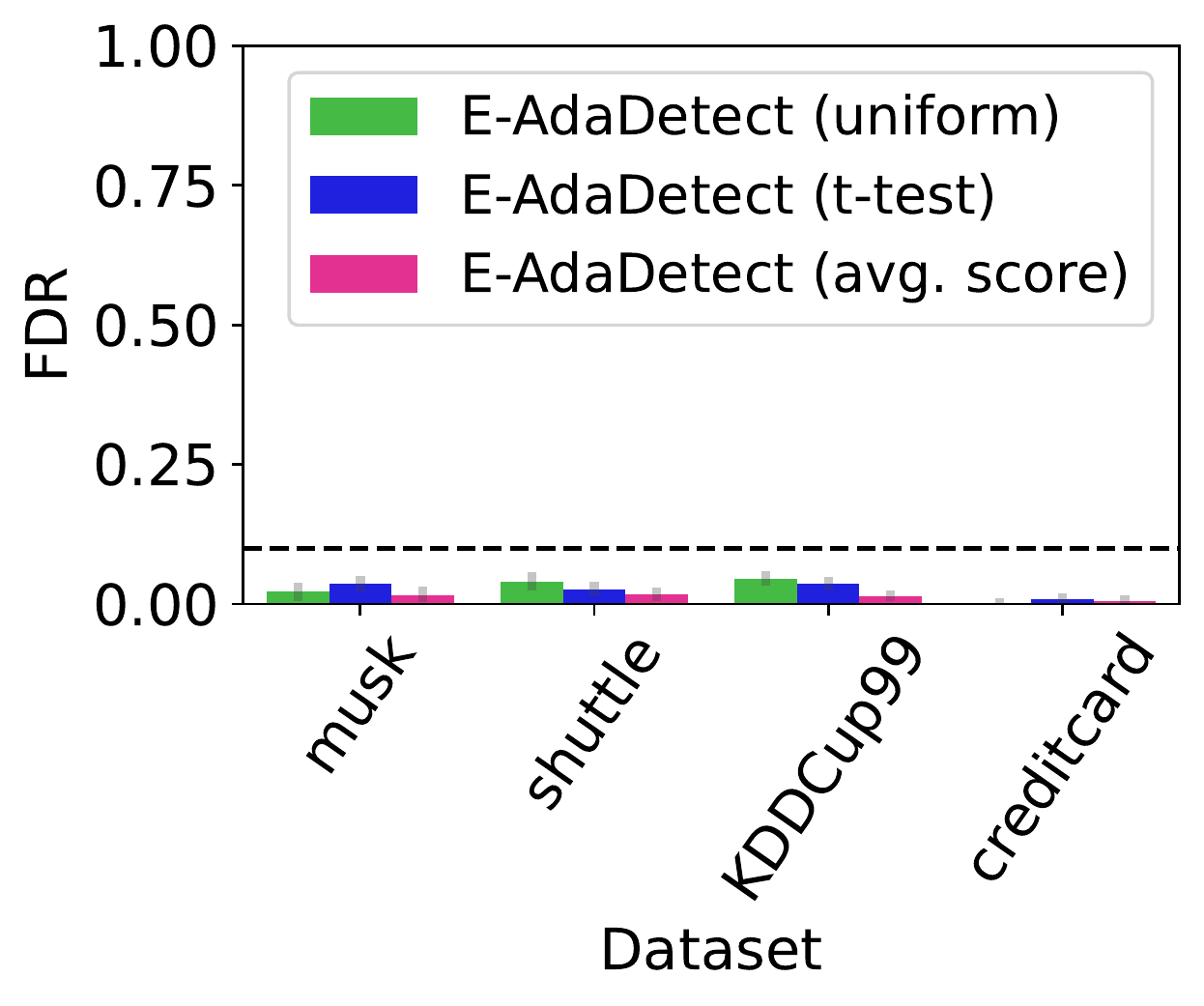}
\includegraphics[width=0.32\textwidth]{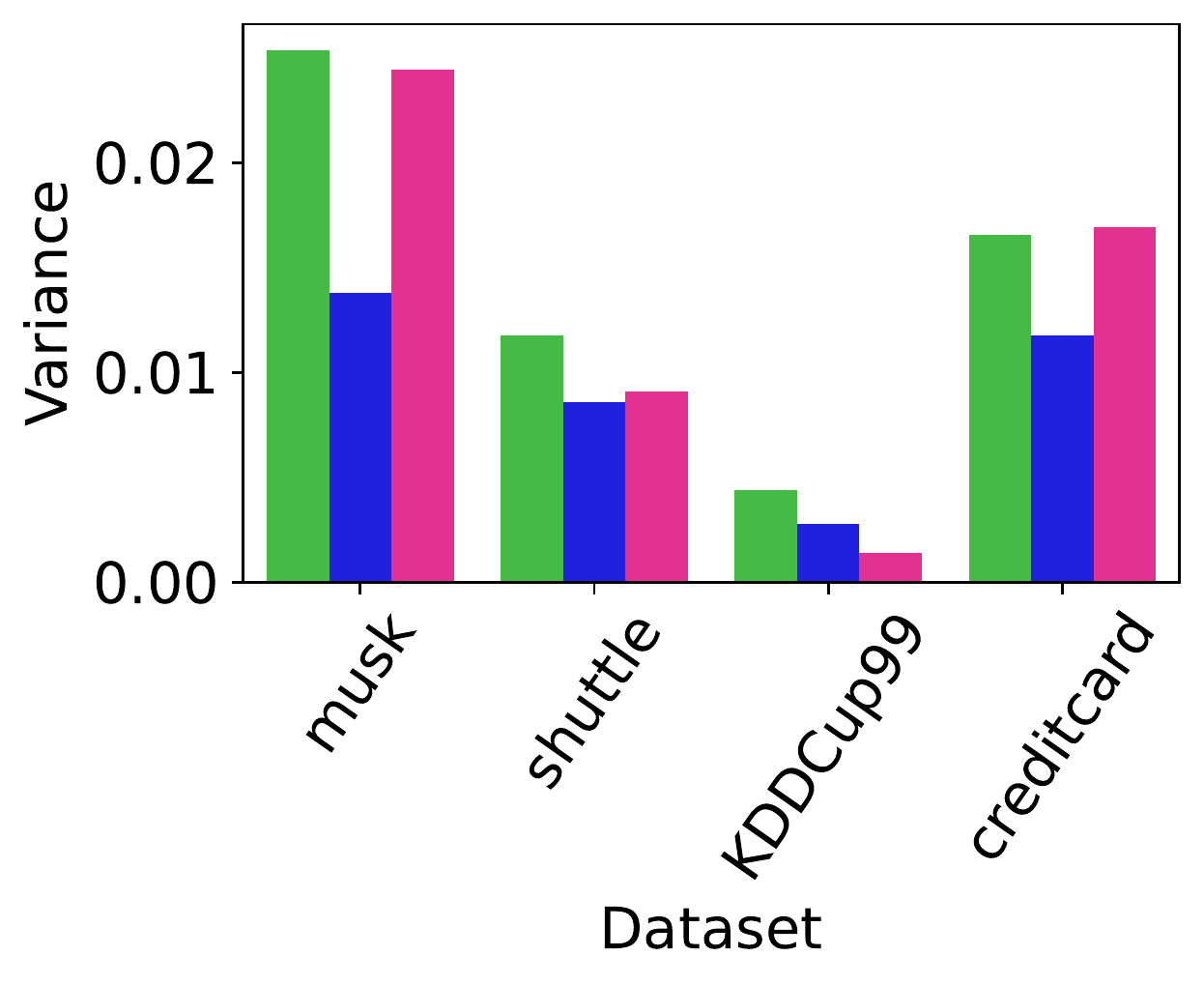}
  \caption{Performance of \texttt{E-AdaDetect}, applied with different model weighting schemes, on four real data sets. All methods utilize a combination of random forest and support vector machine models. The ``t-test'' weighting approach leads to the highest power and lowest algorithmic variability, all while controlling the FDR below the nominal 10\% level.
  }
\label{fig:weight-real-data}
  \end{figure*}

\subsection{Additional results from experiments with synthetic and real data}

Sections~\ref{app:synthetic-experiments}--\ref{app:real-experiments} in the Supplementary Material present the results of several additional experiments.
In particular, Section \ref{app:synthetic-experiments} focuses on experiments with synthetic data.
Section \ref{app:baselines} describes comparisons with alternative derandomization approaches based on different types of p-to-e calibrators \citep{e-value} operating one test point at a time, which turn out to yield lower power compared to our martingale-based method. 
The results also show that our method compares favorably to an alternative derandomization method based on an e-value construction that first appeared in~\citet{ignatiadis2023evalues}, especially if the test set contains numerous outliers or if the nominal FDR level is not too low.
Finally, Section \ref{app:real-experiments} describes additional numerical experiments based on several real data sets also studied in \citet{conformal-p-values} and \citet{ml-fdr}. These results confirm that our martingale-based e-value method can mitigate the algorithmic randomness of standard conformal inferences and \texttt{AdaDetect} while retaining relatively high power.

\section{Discussion}

Our experience suggests that e-values are often less powerful than p-values in measuring the statistical evidence against a {\em single hypothesis}. Yet, e-values can be useful to aggregate multiple dependent tests of the same hypothesis \citep{e-value}---a task that would otherwise require very conservative adjustments within the p-value framework \citep{vovk2022admissible}.
Further, we have shown that e-values lend themselves well to {\em multiple testing} because they allow efficient FDR control under arbitrary dependence \citep{eBH}, and even their relatively weak {\em individual} evidence can accumulate rapidly when a large number of hypotheses is probed.
The opportunity arising from the combination of these two key properties was recently leveraged to derandomize knockoffs \citep{drand-kn}, but until now it had not been fully exploited in the context of conformal inference.

While this paper has focused on derandomizing split-conformal and \texttt{AdaDetect} inferences for novelty detection, the key ideas could be easily extended. For example, one may utilize e-values to derandomize conformal prediction intervals in regression~\citep{lei2014distribution,romano2019conformalized} or prediction sets for classification \citep{lei2013distribution,romano2020classification} while controlling the false coverage rate over a large test set \citep{weinstein2020online}.
A different direction for future research may explore the derandomization of cross-validation+ \citep{barber2019predictive}. 

We conclude by discussing two limitations of this work.
First, the proposed method is more computationally expensive compared to standard conformal inference or \texttt{AdaDetect}, and therefore one may be limited to applying it with relatively small numbers $K$ of analysis repetitions when working with very large data sets.
That being said, it is increasingly recognized that stability is an important goal in data science \citep{murdoch2019definitions}, and thus mitigating algorithmic randomness may often justify the deployment of additional computing resources.
Second, we have shown that our method sometimes leads to a reduction in the average number of findings compared to randomized alternatives, especially in applications where few discoveries are expected. Therefore, in situations with few anticipated discoveries, practitioners considering applying our method should carefully weigh the anticipated gains in stability versus a possible reduction in power.

Mathematical proofs and additional supporting results are in the Supplementary Material. Software implementing the algorithms described in this paper and enabling the reproduction of the associated numerical experiments is available at \href{https://github.com/Meshiba/derandomized-novelty-detection}{https://github.com/Meshiba/derandomized-novelty-detection}.

\clearpage
\section*{Acknowledgements}

Y.~R. and M.~B.~were supported by the Israel Science Foundation (grant No. 729/21). Y.~R.~thanks the Career Advancement Fellowship, Technion, for providing research support. Y.~R.~also thanks Citi Bank for the generous financial support.
M.~S.~was partially supported by NSF grant DMS 2210637 and by an Amazon Research Award.

% \bibliography{bib}
\bibliographystyle{abbrvnat}  % different style?

\clearpage

\begin{center}
{
\hrule height 4pt
\vskip 20pt
\vskip -\parskip%
\LARGE \textbf{Supplementary Material for: ``Derandomized novelty detection with FDR control via conformal e-values''}
\vskip 20pt
\vskip -\parskip%
\hrule height 1pt
\vskip 9pt%
}
\end{center}

\begin{abstract}
This document contains mathematical proof, additional details, comparisons to baseline methods, and other supporting information accompanying the paper ``Derandomized novelty detection with FDR control via conformal e-values''.
\end{abstract}

%%%%%%%%%%%%%%%%%%%%%%%%%%%%%%%%%%%%%%%%%%%%%%%%%%%%%%%%%%%%%%%%%%%%%%%%%%%%%%%
%%%%%%%%%%%%%%%%%%%%%%%%%%%%%%%%%%%%%%%%%%%%%%%%%%%%%%%%%%%%%%%%%%%%%%%%%%%%%%%
% APPENDIX
%%%%%%%%%%%%%%%%%%%%%%%%%%%%%%%%%%%%%%%%%%%%%%%%%%%%%%%%%%%%%%%%%%%%%%%%%%%%%%%
%%%%%%%%%%%%%%%%%%%%%%%%%%%%%%%%%%%%%%%%%%%%%%%%%%%%%%%%%%%%%%%%%%%%%%%%%%%%%%%

\appendix
% Special numbering for appendix (Use "S" instead of "A" in Supplement)
\renewcommand{\thesection}{S\arabic{section}}
\renewcommand{\theequation}{S\arabic{equation}}
\renewcommand{\thetheorem}{S\arabic{theorem}}
\renewcommand{\thecorollary}{S\arabic{corollary}}
\renewcommand{\theproposition}{S\arabic{proposition}}
\renewcommand{\thelemma}{S\arabic{lemma}}
\renewcommand{\thetable}{S\arabic{table}}
\renewcommand{\thefigure}{S\arabic{figure}}
\renewcommand{\thealgorithm}{S\arabic{algorithm}}

\setcounter{section}{0}
\setcounter{figure}{0}
\setcounter{table}{0}
\setcounter{theorem}{0}
\setcounter{algorithm}{0}

\section*{Contents}

The supplementary material is organized as follows:
\begin{itemize}
    \item All algorithmic details are summarized in Section~\ref{app:algorithms}.
    \item Mathematical proofs of theorems presented in the paper can be found in Section~\ref{app:proofs}.
    \item Section~\ref{app:experiments-details} provides details on the training strategy and choice of hyper-parameters for the models utilized in the paper, along with information about the computational resources needed to conduct the experiments.
    \item Additional synthetic experiments involving our derandomization framework in combination with \texttt{AdaDetect} and \texttt{OC-Conformal} are in Section~\ref{app:synthetic-experiments}.
    \item A discussion of alternative approaches for constructing e-values and corresponding comparisons to our martingale-based e-value construction are in Section~\ref{app:baselines}.
    \item Real data experiments using \texttt{AdaDetect} and \texttt{OC-Conformal}, along with their derandomized versions, are in Section~\ref{app:real-experiments}.
\end{itemize}
\clearpage

\section{Algorithmic details} \label{app:algorithms}

\begin{algorithm}[!htb]
\caption{Aggregation of conformal e-values with fixed model weights}
\label{drand-e-value}
\begin{algorithmic}[1]
\STATE \textbf{Input:} {
{inlier data set $\D\equiv \left\{ X_i\right\}_{i=1}^n$};
{test set $\D_{\mathrm{test}}$};
{size of calibration-set $n_{\mathrm{cal}}$};
{number of iterations $K$};
{one-class or binary black-box classification algorithm $\mathcal{A}$};
{normalized model weights $w^{(k)}$, for $k \in [K]$};
{hyper-parameter $\alpha_{\mathrm{bh}} \in (0,1)$};
}
\FOR{$k=1,...,K$}
\STATE Randomly split $\D$ into $\D_{\mathrm{cal}}^{(k)}$ and $\D_{\mathrm{train}}^{(k)}$, with $|\D_{\mathrm{cal}}^{(k)}|=n_{\mathrm{cal}}$
\STATE Train the model: $\mathcal{M}^{(k)}\gets\mathcal{A}(\D_{\mathrm{train}}^{(k)})$ \COMMENT{possibly including additional labeled outlier data if available}
\STATE Compute the calibration scores $S^{(k)}_i=\mathcal{M}^{(k)}(X_i)$, for all $i\in \D_{\mathrm{cal}}^{(k)}$
\STATE Compute the test scores $S^{(k)}_j=\mathcal{M}^{(k)}(X_j)$, for all $j\in \D_{\mathrm{test}}$
\STATE Compute the threshold $\hat{t}^{(k)}$ according to \eqref{eq:threshold} \COMMENT{this depends on the hyper-parameter $\alpha_{\mathrm{bh}}$ }
\STATE Compute the e-values $e^{(k)}_{j}$ for all $j\in \left| \D_{\mathrm{test}}\right|$ according to \eqref{e-value}
\ENDFOR
\STATE Aggregate the e-values $\bar{e}_j = \sum_{k=1}^K w^{(k)} \cdot e^{(k)}_j$
\STATE \textbf{Output:} e-values $\bar{e}_j$ for all $j \in \mathcal{D}_{\mathrm{test}}$ that can be filtered with Algorithm~\ref{alg:ebh} to control the FDR.
\end{algorithmic}
\end{algorithm}

\begin{algorithm}[!htb]
\caption{eBH filter of \citet{eBH}} \label{alg:ebh}
\begin{algorithmic}[1]
\STATE \textbf{Input:} {e-values $\left\{ e_j \right\}_{j=1}^N$ corresponding to $N$ null hypotheses to be tested};
{target FDR level $\alpha \in (0,1)$}
\STATE Compute the order statistics of the e-values: $e_{(1)}\geq \dots \geq e_{(N)}$
\STATE Find the rejection threshold $i_{\max} = \max\{i \in [N]: e_{(i)} \geq N / (\alpha \cdot i) \}$
\STATE Construct the rejection set $\mathcal{R} = \{ j \in [N] :\ e_j \geq e_{(i_{\max})}\}$
\STATE \textbf{Output:} a list of rejected null hypotheses $\mathcal{R} \subseteq [N]$.
\end{algorithmic}
\end{algorithm}

\begin{algorithm}[!htb]
\caption{Aggregation of conformal e-values with data-adaptive model weights}
\label{drand-e-value-adaptive}
\begin{algorithmic}[1]
\STATE \textbf{Input:} {
{inlier data set $\D\equiv \left\{ X_i\right\}_{i=1}^n$};
{test set $\D_{\mathrm{test}}$};
{size of calibration-set $n_{\mathrm{cal}}$};
{number of iterations $K$};
{one-class or binary black-box classification algorithm $\mathcal{A}$};
{a model weighting function $\omega$};
{hyper-parameter $\alpha_{\mathrm{bh}} \in (0,1)$};
%{the false discovery rate $\left\{\alpha_{t,i} \right\}_{i=1}^T, \alpha \in (0,1)$}
}
\FOR{$k=1,...,K$}
\STATE Randomly split $\D$ into $\D_{\mathrm{cal}}^{(k)}$ and $\D_{\mathrm{train}}^{(k)}$, with $|\D_{\mathrm{cal}}^{(k)}|=n_{\mathrm{cal}}$
\STATE Train the model: $\mathcal{M}^{(k)}\gets\mathcal{A}(\D_{\mathrm{train}}^{(k)})$ \COMMENT{possibly including additional labeled outlier data if available}
\STATE Compute the calibration scores $S^{(k)}_i=\mathcal{M}^{(k)}(X_i)$, for all $i\in \D_{\mathrm{cal}}^{(k)}$
\STATE Compute the test scores $S^{(k)}_j=\mathcal{M}^{(k)}(X_j)$, for all $j\in \D_{\mathrm{test}}$
\STATE Compute the weights $\tilde{w}^{(k)} = \omega \left( \{ S_{i}^{(k)} \}_{i \in \D_{\mathrm{test}} \cup \D_{\mathrm{cal}}^{(k)} }  \right)$ \COMMENT{invariant un-normalized model weights}
\STATE Compute the threshold $\hat{t}^{(k)}$ according to \eqref{eq:threshold} \COMMENT{this depends on the hyper-parameter $\alpha_{\mathrm{bh}}$} \STATE Compute the e-values $e^{(k)}_{j}$ for all $j\in \left| \D_{\mathrm{test}}\right|$ according to \eqref{e-value}
\ENDFOR
\FOR{$k=1,...,K$}
\STATE $w^{(k)} = \tilde{w}^{(k)} / \sum_{k'=1}^{K} \tilde{w}^{(k')}$ \COMMENT{normalize the model weights}
\ENDFOR
\STATE Aggregate the e-values $\bar{e}_j = \sum_{k=1}^K w^{(k)} \cdot e^{(k)}_j$
\STATE \textbf{Output:} e-values $\bar{e}_j$ for all $j \in \mathcal{D}_{\mathrm{test}}$ that can be filtered with Algorithm~\ref{alg:ebh} to control the FDR.
\end{algorithmic}
\end{algorithm}

\begin{algorithm}[!htb]
\caption{Aggregation of conformal e-values with data-adaptive model weights and \texttt{AdaDetect} training}
\label{drand-e-value-adaptive-ada}
\begin{algorithmic}[1]
\STATE \textbf{Input:} {
{inlier data set $\D\equiv \left\{ X_i\right\}_{i=1}^n$};
{test set $\D_{\mathrm{test}}$};
{size of calibration-set $n_{\mathrm{cal}}$};
{number of iterations $K$};
{black-box binary classification algorithm $\mathcal{A}$};
{a model weighting function $\omega$};
{hyper-parameter $\alpha_{\mathrm{bh}} \in (0,1)$};
%{the false discovery rate $\left\{\alpha_{t,i} \right\}_{i=1}^T, \alpha \in (0,1)$}
}
\FOR{$k=1,...,K$}
\STATE Randomly split $\D$ into $\D_{\mathrm{cal}}^{(k)}$ and $\D_{\mathrm{train}}^{(k)}$, with $|\D_{\mathrm{cal}}^{(k)}|=n_{\mathrm{cal}}$
\STATE Train the binary classifier, $\mathcal{M}^{(k)}\gets\mathcal{A}(\D_{\mathrm{train}}^{(k)}, \D_{\mathrm{cal}}^{(k)}\cup \D_{\mathrm{test}})$ \COMMENT{treating the data in $\D_{\mathrm{cal}}^{(k)}\cup \D_{\mathrm{test}}$ as outliers}
\STATE Compute the calibration scores $S^{(k)}_i=\mathcal{M}^{(k)}(X_i)$, for all $i\in \D_{\mathrm{cal}}^{(k)}$
\STATE Compute the test scores $S^{(k)}_j=\mathcal{M}^{(k)}(X_j)$, for all $j\in \D_{\mathrm{test}}$
\STATE Compute the weights $\tilde{w}^{(k)} = \omega \left( \{ S_{i}^{(k)} \}_{i \in \D_{\mathrm{test}} \cup \D_{\mathrm{cal}}^{(k)} }  \right)$ \COMMENT{invariant un-normalized model weights}
\STATE Compute the threshold $\hat{t}^{(k)}$ according to \eqref{eq:threshold} \COMMENT{this depends on the hyper-parameter $\alpha_{\mathrm{bh}}$}
\STATE Compute the e-values $e^{(k)}_{j}$ for all $j\in \left| \D_{\mathrm{test}}\right|$ according to \eqref{e-value}
\ENDFOR
\FOR{$k=1,...,K$}
\STATE $w^{(k)} = \tilde{w}^{(k)} / \sum_{k'=1}^{K} \tilde{w}^{(k')}$ \COMMENT{normalize the model weights}
\ENDFOR
\STATE Aggregate the e-values $\bar{e}_j = \sum_{k=1}^K w^{(k)} \cdot e^{(k)}_j$
\STATE \textbf{Output:} e-values $\bar{e}_j$ for all $j \in \mathcal{D}_{\mathrm{test}}$ that can be filtered with Algorithm~\ref{alg:ebh} to control the FDR.
\end{algorithmic}
\end{algorithm}

\begin{algorithm}[!htb]
\caption{Adaptive model weighting via t-tests} \label{alg:t-weights}
\begin{algorithmic}[1]
\STATE \textbf{Input:} {Scores $\{ S_{i} \}_{i=1}^{N} $}; a guess $\gamma$ for the proportion of outliers in the data.
\STATE Compute the order statistics of the scores: $S_{(1)}\leq \dots \leq S_{(N)}$
\STATE Denote $n_2= \lceil \gamma N \rceil$ and $n_1 = N - n_2$
\STATE Divide the sorted scores into two groups with size $n_1$ and $n_2$: $\mathcal{I}_1=\{S_{(i)}\}_{i=1}^{n_1}$, $\mathcal{I}_2=\{S_{(i)}\}_{i=n_1 + 1}^{N}$
\STATE Estimate the means of the two groups: $\mu_1 = (1/n_1) \sum_{i\in \mathcal{I}_1} S_{(i)}$ and $\mu_2 = (1/n_2) \sum_{i\in \mathcal{I}_2} S_{(i)}$
\STATE Estimate the pooled variance: $z = (1/(n_1+n_2-2)) \left[ \sum_{i\in \mathcal{I}_1} (S_{(i)} - \mu_1)^2 + \sum_{i\in \mathcal{I}_2} (S_{(i)} - \mu_2)^2 \right]$
\STATE Compute the t-statistic: $\tilde{v}= (\mu_1 - \mu_2) / \sqrt{z \cdot (1/n_1 + 1/n_2)}$
\STATE \textbf{Output:} model weight $\tilde{w} = |\tilde{v}|$
\end{algorithmic}
\end{algorithm}

\begin{algorithm}[!htb]
\caption{Adaptive model weighting via trimmed mean} \label{alg:avg-weights}
\begin{algorithmic}[1]
\STATE \textbf{Input:} {Scores $\{ S_{i} \}_{i=1}^{N} $}; a guess $\gamma$ for the proportion of outliers in the data.
\STATE Compute the order statistics of the scores: $S_{(1)}\leq \dots \leq S_{(N)}$
\STATE Denote $\tilde{n}= N - \lceil \gamma N \rceil$
% \STATE Define the trimmed group: $\mathcal{I} = \left\{ S_{(i)} \right\}_{i=1}^{\tilde{n}}$
\STATE Compute the mean of the trimmed group: $\hat{\mu} = \sum_{i=1}^{\tilde{n}} S_{(i)}$
\STATE \textbf{Output:} model weight $\tilde{w} = e^{-\hat{\mu}}$
\end{algorithmic}
\end{algorithm}

\clearpage

\section{Mathematical proofs} \label{app:proofs}

\begin{proof}[Proof of Theorem~\ref{thm:e-fdr}]
This result is implied by Theorem~\ref{thm:e-fdr-adaptive}, to whose proof we refer.
\end{proof}

\begin{proof}[Proof of Theorem~\ref{thm:e-fdr-adaptive}]

The proof follows a martingale argument similar to that of \citet{fairness}.
For each fixed $k$, define the following two quantities as functions of $t \in \mathbb{R}$:
\begin{align}
    V_{\mathrm{test}}^{(k)}(t) = \sum_{j \in \D_{\mathrm{test}}^{\mathrm{null}}} \mathbb{I} \left\{ S_j^{(k)}\geq t \right\},
\end{align}
and
\begin{align}
    V_{\mathrm{cal}}^{(k)}(t) = \sum_{i \in \D_{\mathrm{cal}}^{(k)}} \mathbb{I} \left\{ S_i^{(k)}\geq t \right\}.
\end{align}
For each $k \in [K]$, define also the unordered set of conformity scores for non-null test points as:
\begin{align*}
  \tilde{\mathcal{D}}^{(k)}_{\mathrm{test-nn}} = \{ \hat{S}_{i}^{(k)} \}_{i \in \D_{\mathrm{test}} \setminus \D_{\mathrm{test}}^{\mathrm{null}}},
\end{align*}
and the unordered set of conformity scores for all calibration and test points as:
\begin{align*}
\tilde{\mathcal{D}}^{(k)}_{\mathrm{cal-test}} = \{ \hat{S}_{i}^{(k)} \}_{i \in \D_{\mathrm{test}} \cup \D_{\mathrm{cal}}^{(k)} }.
\end{align*}

With this premise, we can write:
\begin{align*}
  \sum_{j\in \Hnull} \E \left[ \bar{e}_j\right]
  & =  \sum_{j\in \Hnull}  \E \left[ \sum_{k=1}^{K} w^{(k)} e_j^{(k)} \right] \\
  & = \sum_{k=1}^{K} \sum_{j\in \Hnull} \E \left[ w^{(k)} e_j^{(k)} \right]  \\
  & = \sum_{k=1}^{K} \sum_{j\in \Hnull} \E \left[ w^{(k)} \left( 1 +  n_{\mathrm{cal}} \right) \frac{ \mathbb{I} \left\{ S_j^{(k)}\geq \hat{t}^{(k)}\right\}}{1 + V_{\mathrm{cal}}^{(k)}(\hat{t}^{(k)})} \right] \\
  & = \sum_{k=1}^{K} \E \left[ w^{(k)} \left( 1 +  n_{\mathrm{cal}} \right) \mathbb{E}\left[ \frac{ V_{\mathrm{test}}^{(k)}(\hat{t}^{(k)}) }{1 + V_{\mathrm{cal}}^{(k)}(\hat{t}^{(k)})} \mid \tilde{\mathcal{D}}^{(k)}_{\mathrm{cal-test}} , \tilde{\mathcal{D}}^{(k)}_{\mathrm{test-nn}} \right] \right] \\
  & = \sum_{k=1}^{K} \E \left[ w^{(k)} \cdot n_{\mathrm{test}}^{\mathrm{null}}  \right] \\
  & = n_{\mathrm{test}}^{\mathrm{null}} \leq n_{\mathrm{test}}.
\end{align*}
Above, the third-to-last equality follows from the assumption that $w^{(k)}$ is a deterministic function of $\tilde{\mathcal{D}}^{(k)}_{\mathrm{cal-test}}$, and the second-to-last equality follows from the fact that $M^{(k)}(t)$, defined as
\begin{align*}
    M^{(k)}(t) = \frac{V_{\mathrm{test}}^{(k)}(t)}{1 + V_{\mathrm{cal}}^{(k)}(t)},
\end{align*}
is a martingale conditional on $\tilde{\mathcal{D}}^{(k)}_{\mathrm{cal-test}}$ and $\tilde{\mathcal{D}}^{(k)}_{\mathrm{test-nn}}$, and therefore it is possible to show that
\begin{align*}
  \mathbb{E}\left[ M^{(k)}(\hat{t}^{(k)}) \mid \tilde{\mathcal{D}}^{(k)}_{\mathrm{cal-test}}, \tilde{\mathcal{D}}^{(k)}_{\mathrm{test-nn}}\right]
    = \frac{n_{\mathrm{test}}^{\mathrm{null}}}{1+n_{\mathrm{cal}}}
\end{align*}
by applying the optional stopping theorem.
This last statement is proved below, following the same strategy as in \citet{fairness}.

For each $l \in \{1,\ldots, n_{\mathrm{test}}^{\mathrm{null}} + n_{\mathrm{cal}}\}$, define $t_l$ as the unique discrete threshold belonging to $\tilde{\mathcal{D}}^{(k)}_{\mathrm{cal-test}}$ at which exactly $l$ inliers have scores exceeding $t_l$, across all calibration and null test points; i.e.,
\begin{align*}
  t_l^{(k)} = \inf \left\{ t \in \tilde{\mathcal{D}}^{(k)}_{\mathrm{cal-test}} : V_{\mathrm{test}}^{(k)}(t) + V_{\mathrm{cal}}^{(k)}(t) = l \right\}.
\end{align*} 
Note that this is always well-defined as long as there are no ties between scores (which can always be achieved by adding a negligible noise).

By convention, define also $t_0^{(k)} = \infty$. Consider then a discrete-time filtration indexed by $l$:
\begin{align*}
  \mathcal{F}_{l}^{(k)} = \left\{ \sigma\left( V_{\mathrm{test}}^{(k)}(t_{l'}^{(k)}), V_{\mathrm{cal}}^{(k)}(t_{l'}^{(k)}) \right) \right\}_{l \leq l' \leq n_{\mathrm{test}}^{\mathrm{null}} + n_{\mathrm{cal}}}.
\end{align*}
Note that $\mathcal{F}_{l}^{(k)}$ is a backward-running filtration because $\mathcal{F}_{l_2}^{(k)} \subset \mathcal{F}_{l_1}^{(k)}$ for any $l_1 < l_2$.

It now remains to be proved that $M^{(k)}(t)$ is a martingale.
Since we assumed that there are no ties between scores, we get that for every two consecutive thresholds $t_{l}^{(k)}$ and $t_{l-1}^{(k)}$ the following holds by definition:
\begin{align*}
    V_{\mathrm{test}}^{(k)}(t_{l-1}^{(k)}) + V_{\mathrm{cal}}^{(k)}(t_{l-1}^{(k)}) = V_{\mathrm{test}}^{(k)}(t_{l}^{(k)}) + V_{\mathrm{cal}}^{(k)}(t_{l}^{(k)}) - 1.
\end{align*}
The discrepancy between these two thresholds corresponds to a singular score whose value is larger than $t_{l}^{(k)}$ but smaller than $t_{l-1}^{(k)} > t_{l}^{(k)}$. 
This score can either correspond to a calibration or null test point. Therefore, we should consider the following two mutually exclusive events:
\begin{align*}
    E_1 = \left\{V_{\mathrm{cal}}^{(k)}(t_{l-1}^{(k)}) = V_{\mathrm{cal}}^{(k)}(t_{l}^{(k)}) \right\} \cap \left\{ V_{\mathrm{test}}^{(k)}(t_{l-1}^{(k)}) = V_{\mathrm{test}}^{(k)}(t_{l}^{(k)}) - 1\right\}, \\
    E_2 = \left\{V_{\mathrm{cal}}^{(k)}(t_{l-1}^{(k)}) = V_{\mathrm{cal}}^{(k)}(t_{l}^{(k)}) - 1\right\} \cap \left\{ V_{\mathrm{test}}^{(k)}(t_{l-1}^{(k)}) = V_{\mathrm{test}}^{(k)}(t_{l}^{(k)})\right\}.
\end{align*}
By Assumption~\ref{exchangeable},
\begin{align*}
    \p \left( E_1 \mid \mathcal{F}_{l}^{(k)} , \tilde{\mathcal{D}}^{(k)}_{\mathrm{cal-test}}, \tilde{\mathcal{D}}^{(k)}_{\mathrm{test-nn}} \right) = \frac{V_{\mathrm{test}}^{(k)}(t_{l}^{(k)})}{V_{\mathrm{test}}^{(k)}(t_{l}^{(k)}) + V_{\mathrm{cal}}^{(k)}(t_{l}^{(k)})}, \\
    \p \left( E_2 \mid \mathcal{F}_{l}^{(k)} , \tilde{\mathcal{D}}^{(k)}_{\mathrm{cal-test}}, \tilde{\mathcal{D}}^{(k)}_{\mathrm{test-nn}} \right) = \frac{V_{\mathrm{cal}}^{(k)}(t_{l}^{(k)})}{V_{\mathrm{test}}^{(k)}(t_{l}^{(k)}) + V_{\mathrm{cal}}^{(k)}(t_{l}^{(k)})}.
\end{align*}

Then, for any $l \in \{1,\ldots, n_{\mathrm{test}}^{\mathrm{null}} + n_{\mathrm{cal}}\}$, it follows from the law of total probability that
\begin{align*}
  \mathbb{E} & \left[ M^{(k)}(t_{l-1}^{(k)}) \mid \mathcal{F}_{l}^{(k)} , \tilde{\mathcal{D}}^{(k)}_{\mathrm{cal-test}}, \tilde{\mathcal{D}}^{(k)}_{\mathrm{test-nn}} \right] \\
  & \qquad = \frac{V_{\mathrm{test}}^{(k)}(t_{l}^{(k)})-1}{1 + V_{\mathrm{cal}}^{(k)}(t_{l}^{(k)})} \cdot \frac{V_{\mathrm{test}}^{(k)}(t_{l}^{(k)})}{V_{\mathrm{cal}}^{(k)}(t_{l}^{(k)}) + V_{\mathrm{test}}^{(k)}(t_{l}^{(k)})} + \frac{V_{\mathrm{test}}^{(k)}(t_{l}^{(k)})}{V_{\mathrm{cal}}^{(k)}(t_{l}^{(k)})} \cdot \frac{V_{\mathrm{cal}}^{(k)}(t_{l}^{(k)})}{V_{\mathrm{cal}}^{(k)}(t_{l}^{(k)}) + V_{\mathrm{test}}^{(k)}(t_{l}^{(k)})}\\
  & \qquad = \frac{V_{\mathrm{test}}^{(k)}(t_{l}^{(k)})-1}{1 + V_{\mathrm{cal}}^{(k)}(t_{l}^{(k)})} \cdot \frac{V_{\mathrm{test}}^{(k)}(t_{l}^{(k)})}{V_{\mathrm{cal}}^{(k)}(t_{l}^{(k)}) + V_{\mathrm{test}}^{(k)}(t_{l}^{(k)})} + \frac{V_{\mathrm{test}}^{(k)}(t_{l}^{(k)})}{V_{\mathrm{cal}}^{(k)}(t_{l}^{(k)}) + V_{\mathrm{test}}^{(k)}(t_{l}^{(k)})}\\
  & \qquad = \frac{V_{\mathrm{test}}^{(k)}(t_{l}^{(k)}) \cdot [V_{\mathrm{test}}^{(k)}(t_{l}^{(k)})-1] + V_{\mathrm{test}}^{(k)}(t_{l}^{(k)}) \cdot [1 + V_{\mathrm{cal}}^{(k)}(t_{l}^{(k)})]}{[1 + V_{\mathrm{cal}}^{(k)}(t_{l}^{(k)})] \cdot [V_{\mathrm{cal}}^{(k)}(t_{l}^{(k)}) + V_{\mathrm{test}}^{(k)}(t_{l}^{(k)})]}\\
  & \qquad = \frac{V_{\mathrm{test}}^{(k)}(t_{l}^{(k)})}{1 + V_{\mathrm{cal}}^{(k)}(t_{l}^{(k)})} \\
  & \qquad = M^{(k)}(t_{l}^{(k)}).
\end{align*}

By the optional stopping theorem, this implies that
\begin{align*}
  \mathbb{E}\left[ M^{(k)}(\hat{t}^{(k)}) \mid \tilde{\mathcal{D}}^{(k)}_{\mathrm{cal-test}}, \tilde{\mathcal{D}}^{(k)}_{\mathrm{test-nn}} \right]
  & = \mathbb{E}\left[ M^{(k)}(t_{n_{\mathrm{test}}^{\mathrm{null}}+n_{\mathrm{cal}}}) \mid \tilde{\mathcal{D}}^{(k)}_{\mathrm{cal-test}}, \tilde{\mathcal{D}}^{(k)}_{\mathrm{test-nn}} \right]
    = \frac{n_{\mathrm{test}}^{\mathrm{null}}}{1+n_{\mathrm{cal}}},
\end{align*}
and this completes the proof.
\end{proof}

\clearpage

\section{Implementation details for one-class and binary classifiers}
\label{app:experiments-details}

We have applied the following models, from the \texttt{scikit-learn} \citep{sklearn_api} Python library, to compute the scores.

\begin{itemize}
    \item Synthetic experiments:
\begin{itemize}
    \item Binary classifer: logistic regression with \texttt{scikit-learn} default ridge regularization parameter.
    \item  One class classifier: one-class kernel SVM with RBF kernel with \texttt{scikit-learn} default kernel width parameter.
\end{itemize}
\item Real-data experiments:
\begin{itemize}
    \item Binary classifier: random forest with 100 estimators with a maximum depth of 10. All other hyper-parameters are set to \texttt{scikit-learn} default values.
    \item  One class classifier: isolation forest with 100 estimators, each is fitted to a random subset of $\min(256, n\_\text{samples})$ training samples. All other hyper-parameters are set to \texttt{scikit-learn} default values.
\end{itemize}
\end{itemize}

Unless specified otherwise, our derandomization method is implemented by setting $\alpha_{\mathrm{bh}} = 0.1\cdot 
\alpha$, where $\alpha$ is the target FDR level. All the experiments were conducted on our local CPU cluster.

\section{Additional synthetic experiments} \label{app:synthetic-experiments}

\subsection{Derandomized AdaDetect} \label{app:synthetic-experiments-adadetect}

Section~\ref{sec:synthetic} of the main manuscript studies the performance of the proposed method focusing on the algorithmic variability for different analyses of the same synthetic data set. Here, we conduct the following additional experiments.
\begin{itemize}
    \item Figure~\ref{fig-app:synth-classic-FDR-Power} confirms the reproducibility of the results presented in Figure~\ref{fig:data_difficulty} of the main manuscript, showing that the FDR is controlled over 100 independent realizations of the data. These results are investigated as a function of the signal strength.
    \item Figure~\ref{fig:iterations} of the main manuscript studies the effect of the number $K$ of derandomized analyses for strong signal amplitude. We conduct a similar study in Figure~\ref{fig-app:iterations-low} but for a lower power regime. When the number of iterations is relatively high, our method achieves much smaller algorithmic variably as measured by the selection variance, although at the cost of reduced power. This drop in power can be explained by noting that, in the low power regime, the base outlier detection methods tend to make inconsistent selections across multiple analyses of the same data, which are likely to be filtered out by the derandomization procedure. 

    \item  Figure~\ref{fig-app:synth-classic-FDR-Power-iterations} confirms the reproducibility of the results observed in Figure~\ref{fig:iterations} of the main manuscript and Figure~\ref{fig-app:iterations-low} of the Supplementary Material, by evaluating the average FDR and power over 100 independent realizations of the data. These results are investigated as a function of the number of analysis repetitions $K$.
    
\end{itemize}

\begin{figure}[!htb]
  \centering
\includegraphics[width=0.45\textwidth]{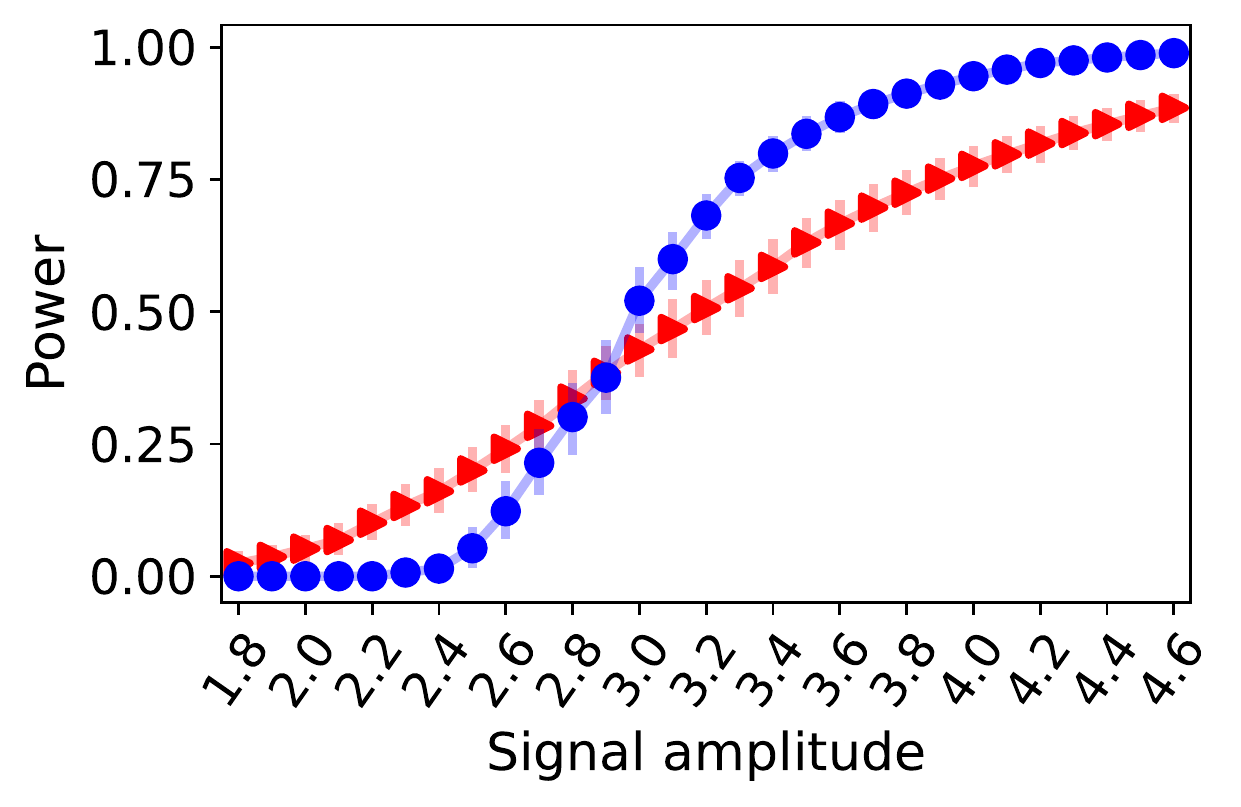}
\includegraphics[width=0.45\textwidth]{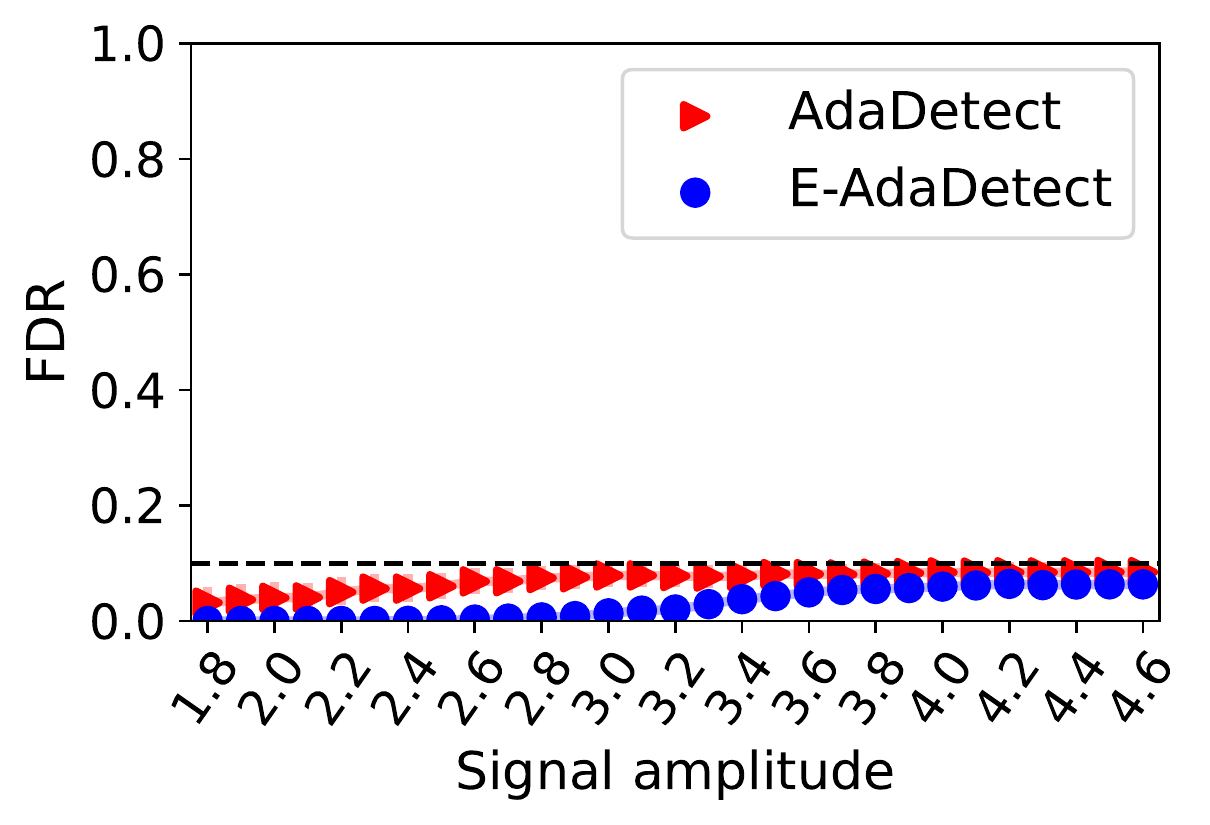}
  \caption{
Performance on synthetic data of the proposed derandomized outlier detection method, \texttt{E-AdaDetect}, applied with $K=10$ analysis repetitions. The results are compared to the performance of the randomized benchmark, \texttt{AdaDetect}, as a function of the signal strength, averaging over 100 independent realizations of the data. 
Left: average proportion of true outliers that are discovered (higher is better). Right: average proportion of false discoveries (lower is better).
Other results are as in Figure~\ref{fig:data_difficulty}.
}
\label{fig-app:synth-classic-FDR-Power}
  \end{figure}

\begin{figure*}[!htb]
  \centering
\includegraphics[width=0.32\textwidth]{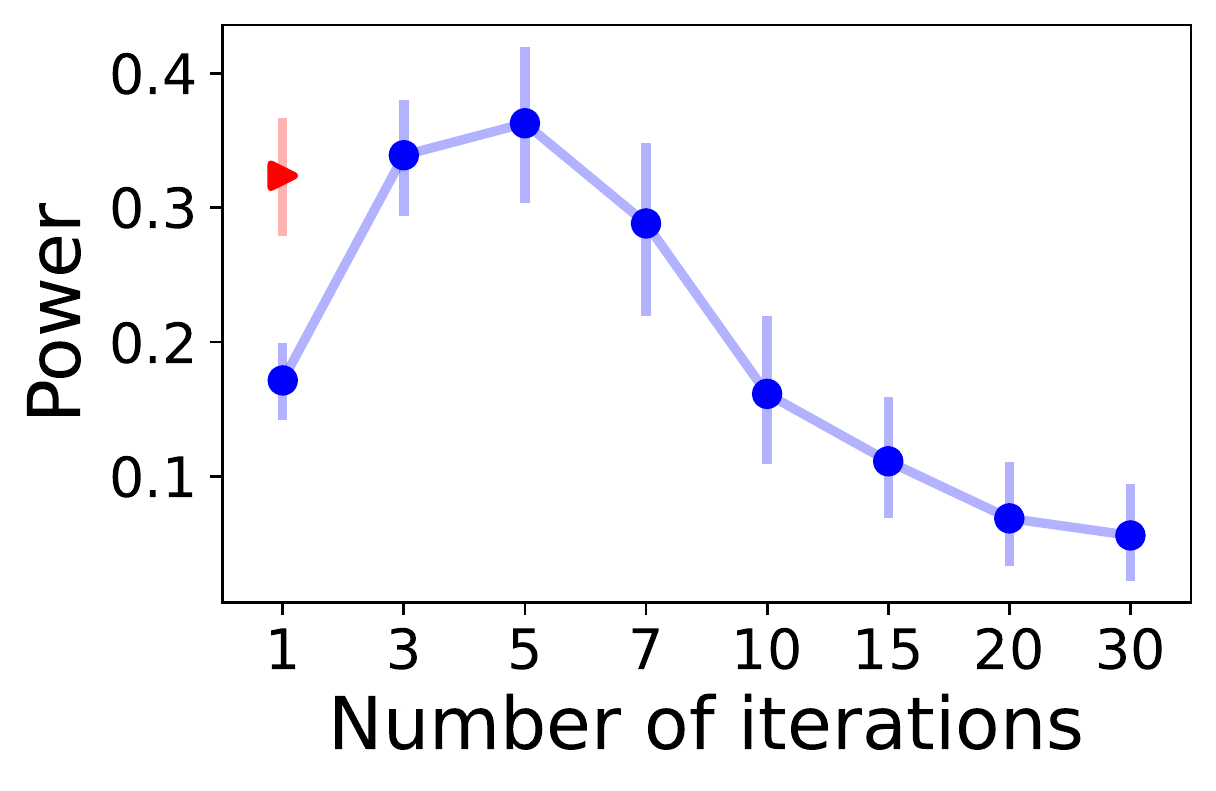}
\includegraphics[width=0.32\textwidth]{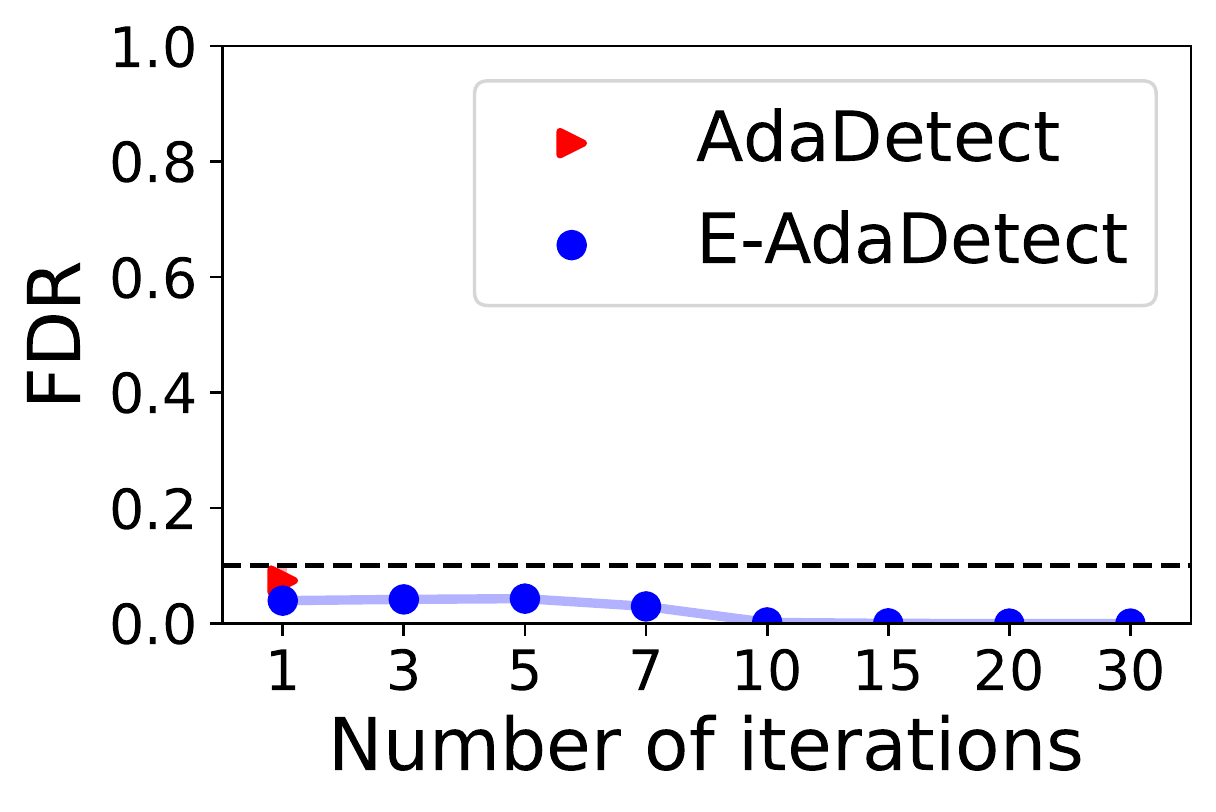} 
\includegraphics[width=0.32\textwidth]{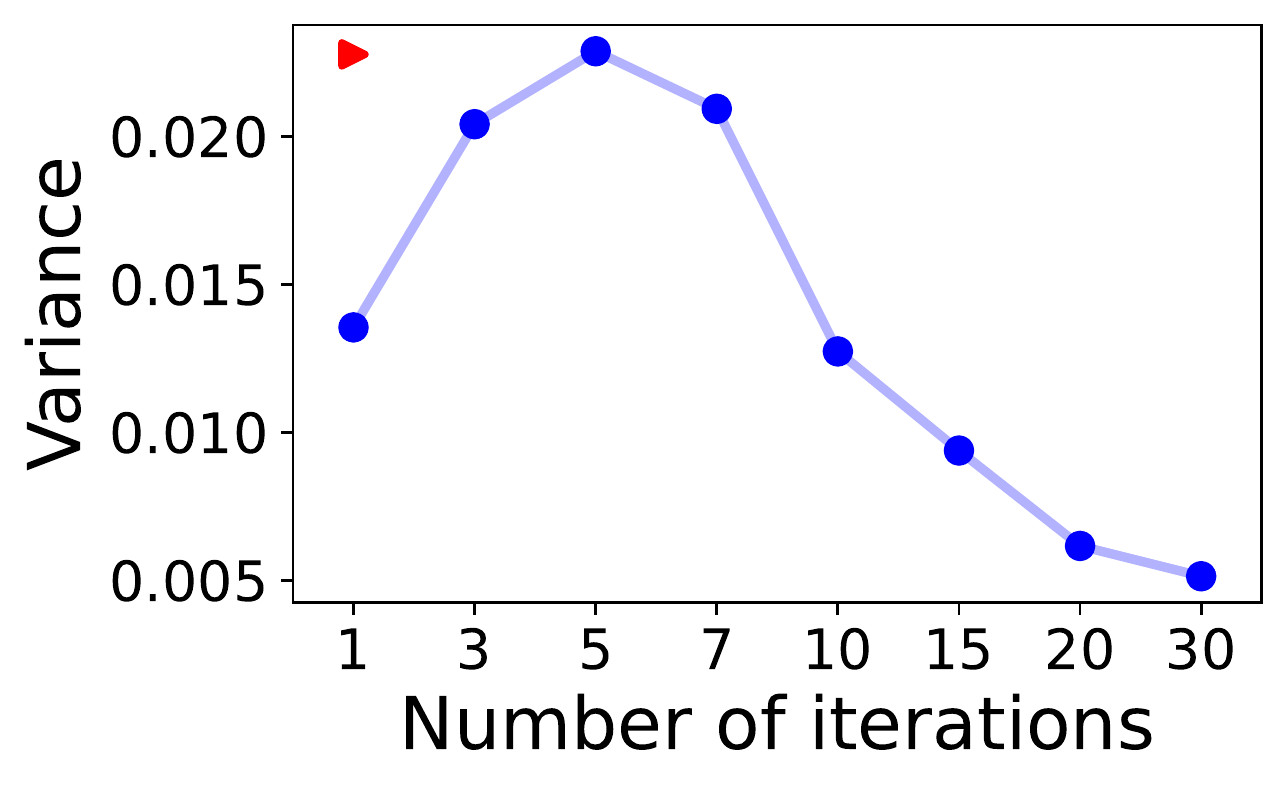}
  \caption{Performance on synthetic data of \texttt{E-AdaDetect}, as a function of the number $K$ of analysis repetitions, compared to \texttt{AdaDetect}. Note that the latter can only be applied with a single data split (or iteration).
Low power regime with signal amplitude $2.8$.
Other details are as in Figure~\ref{fig:iterations}.
}
\label{fig-app:iterations-low}
  \end{figure*}

\begin{figure}[!htb]
  \centering
\includegraphics[width=0.45\textwidth]{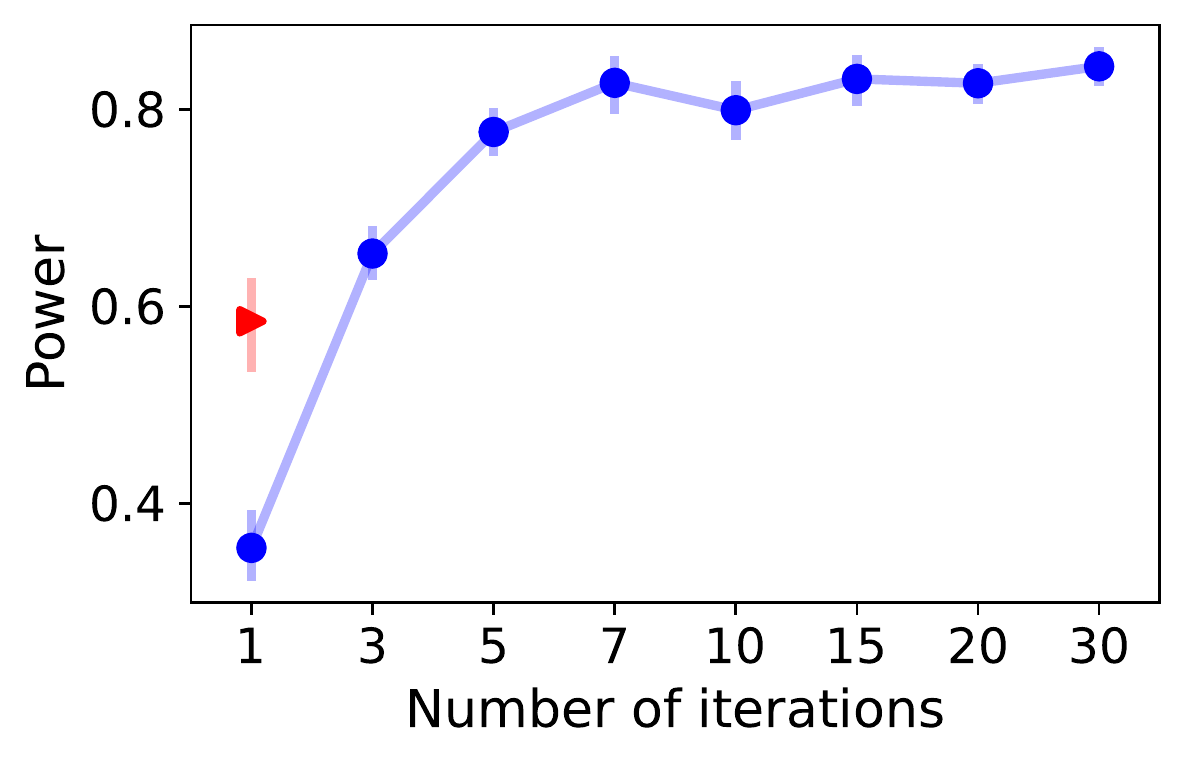}
\includegraphics[width=0.45\textwidth]{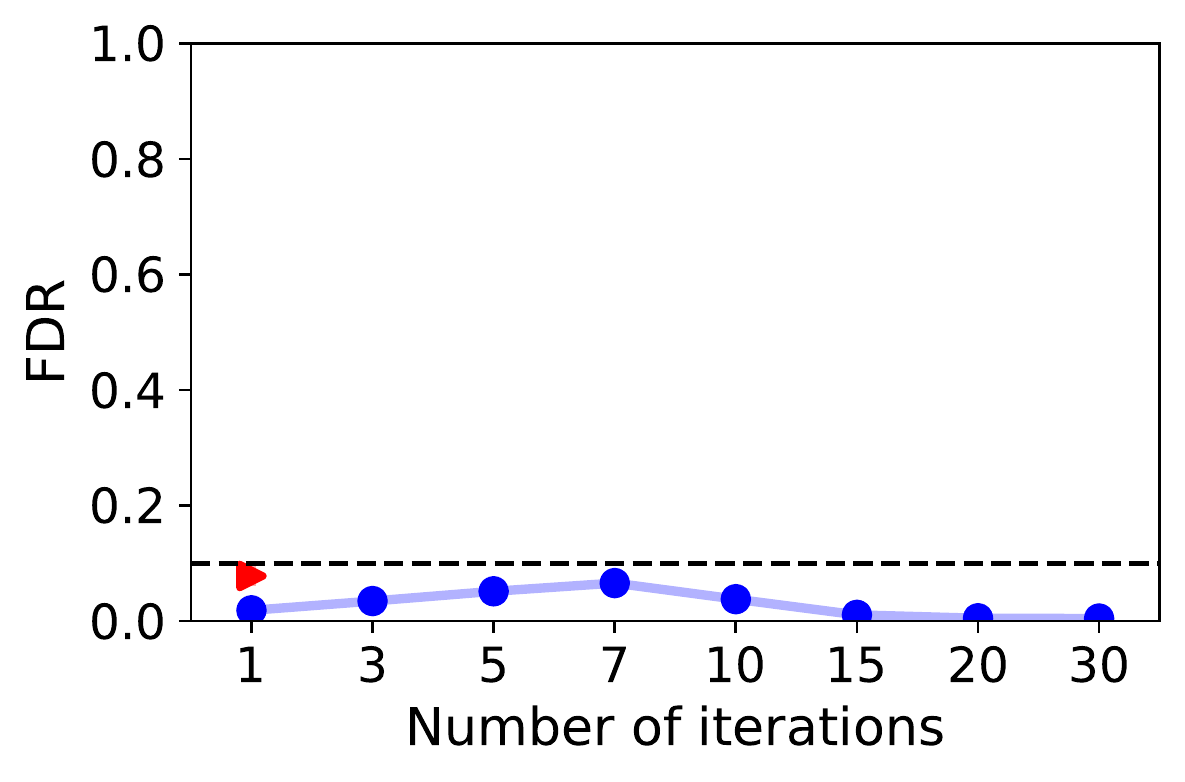}
\\
\includegraphics[width=0.45\textwidth]{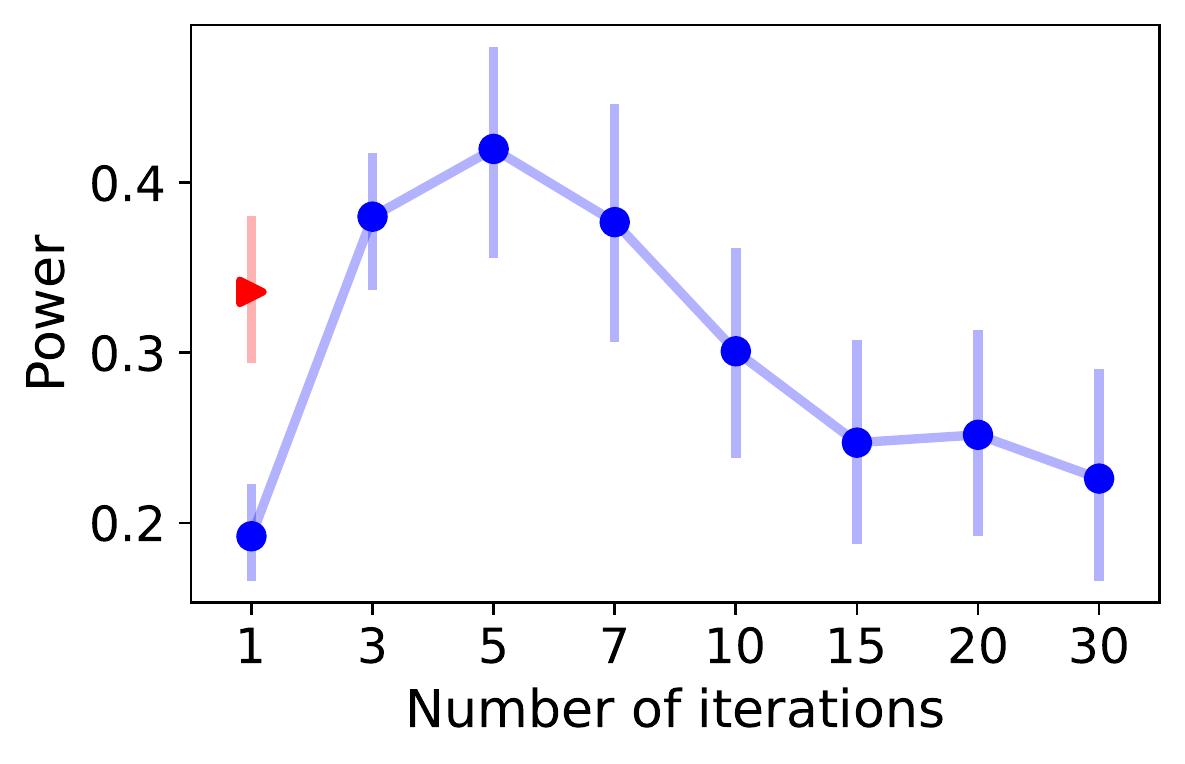}
\includegraphics[width=0.45\textwidth]{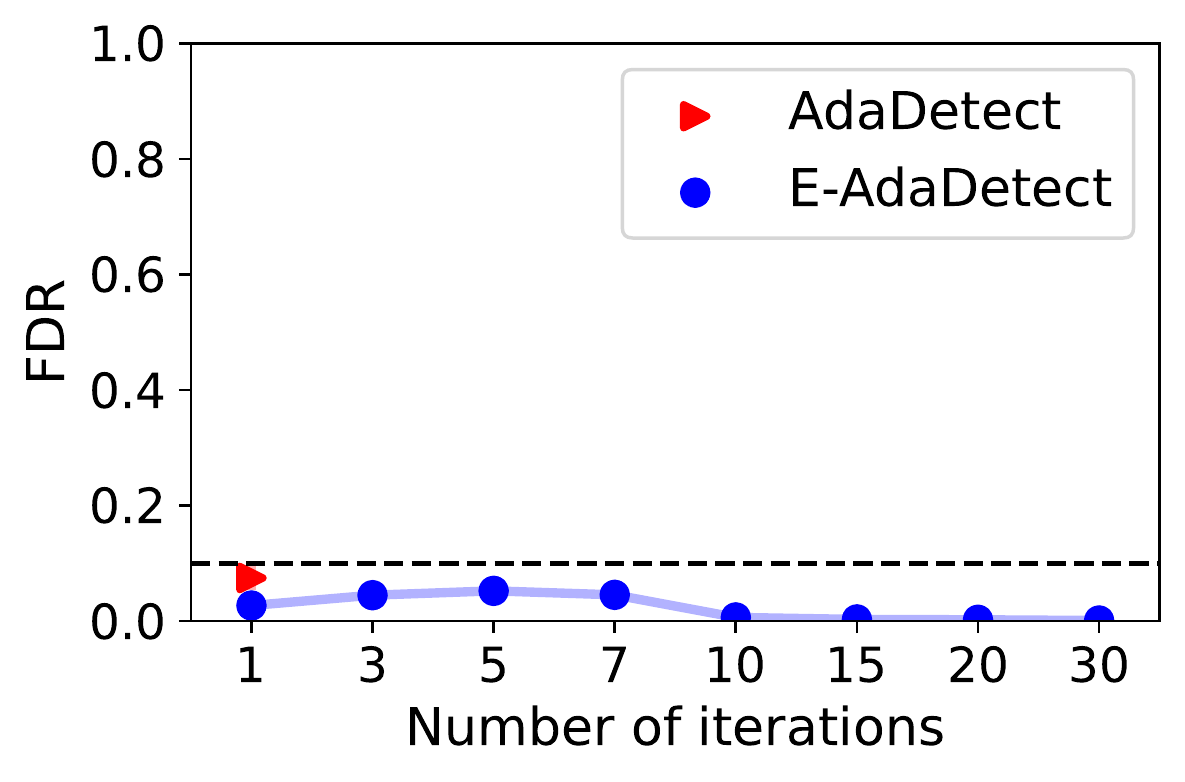}
  \caption{Performance on synthetic data of \texttt{E-AdaDetect}, as a function of the number $K$ of analysis repetitions, compared to that of its randomized benchmark \texttt{AdaDetect}.
The results are averaged over 100 independent realizations of the data. Top: high-power regime with signal amplitude $3.4$. Bottom: low-power regime with signal amplitude $2.8$.
Other results are as in Figure~\ref{fig:iterations}.
}
\label{fig-app:synth-classic-FDR-Power-iterations}
  \end{figure}

\subsection{Derandomized One-Class Conformal} \label{app:synthetic-experiments-OC}

We now turn to explore the effect of our derandomization method on the performance of \texttt{OC-Conformal}. Specifically, Figure~\ref{fig-app:data_difficulty_ocsvm} compares the power, false discovery proportion, and variance of \texttt{OC-Conformal} to those of \texttt{E-OC-Conformal} as a function of the signal amplitude, on one realization of the synthetic data. 
The results show that the false discovery proportion is controlled for both methods, but this error metric is lower for \texttt{E-OC-Conformal}. The variance of \texttt{E-OC-Conformal} is also reduced, but often at the cost of reduced power. Figure~\ref{fig-app:synth-classic-FDR-Power-data_difficulty_ocsvm} reports the corresponding average FDR and power across 100 realizations, confirming the validity of our method.

\begin{figure}[!htb]
  \centering
  \includegraphics[width=0.32\textwidth]{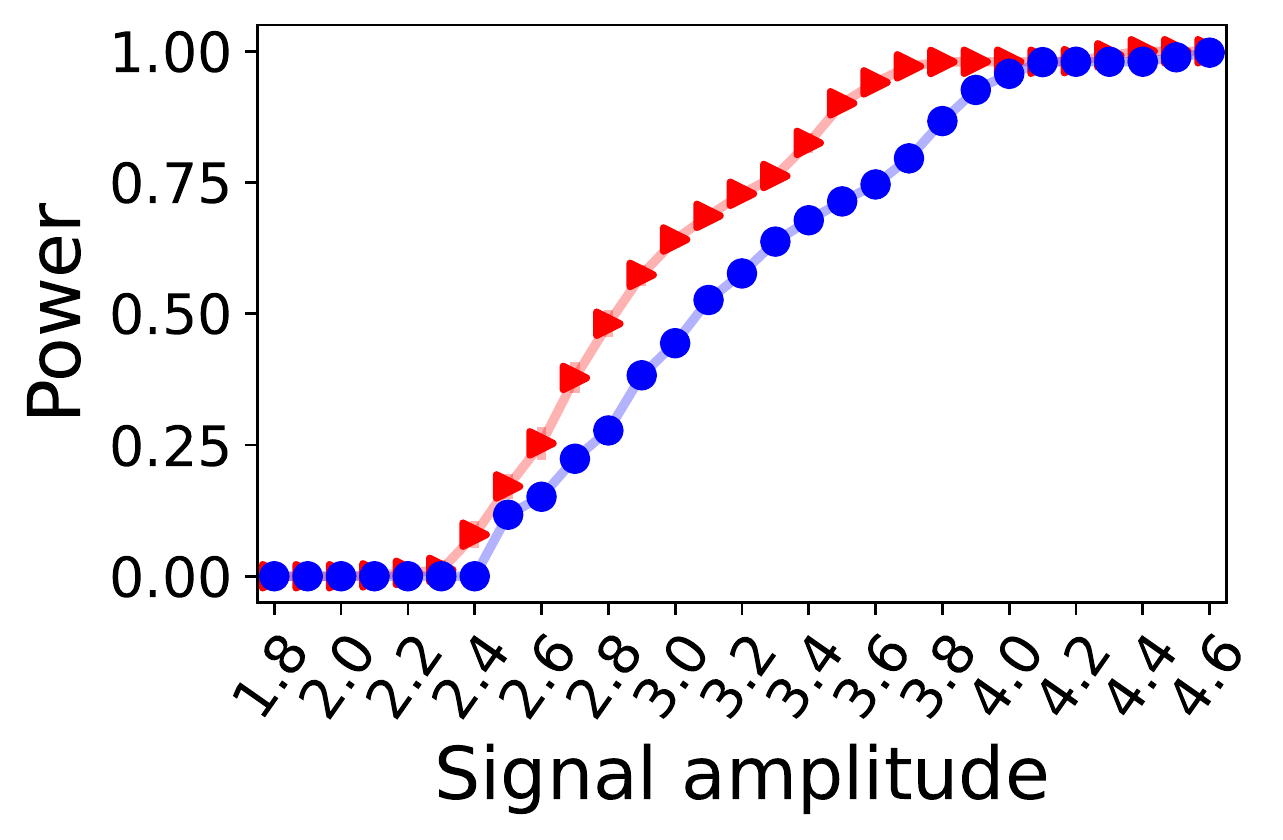}
  \includegraphics[width=0.32\textwidth]{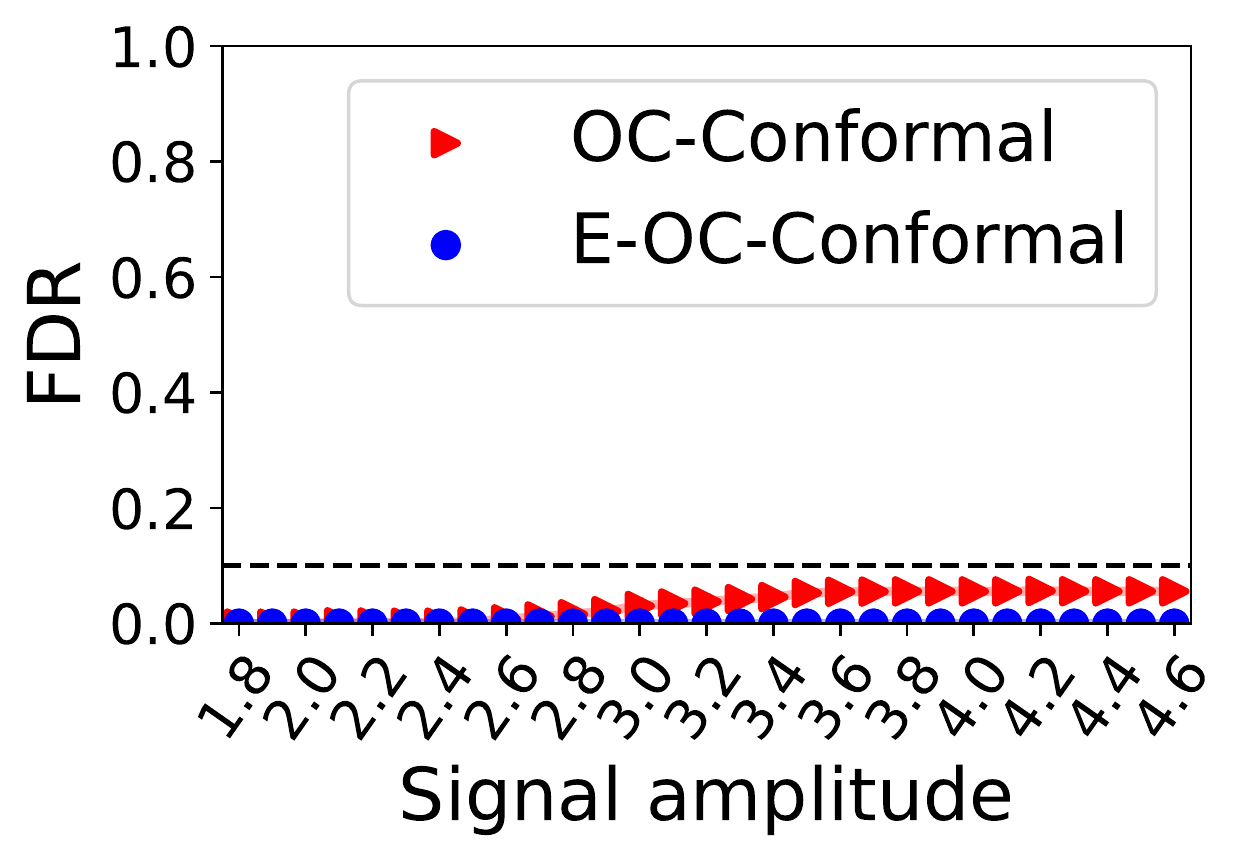}
  \includegraphics[width=0.32\textwidth]{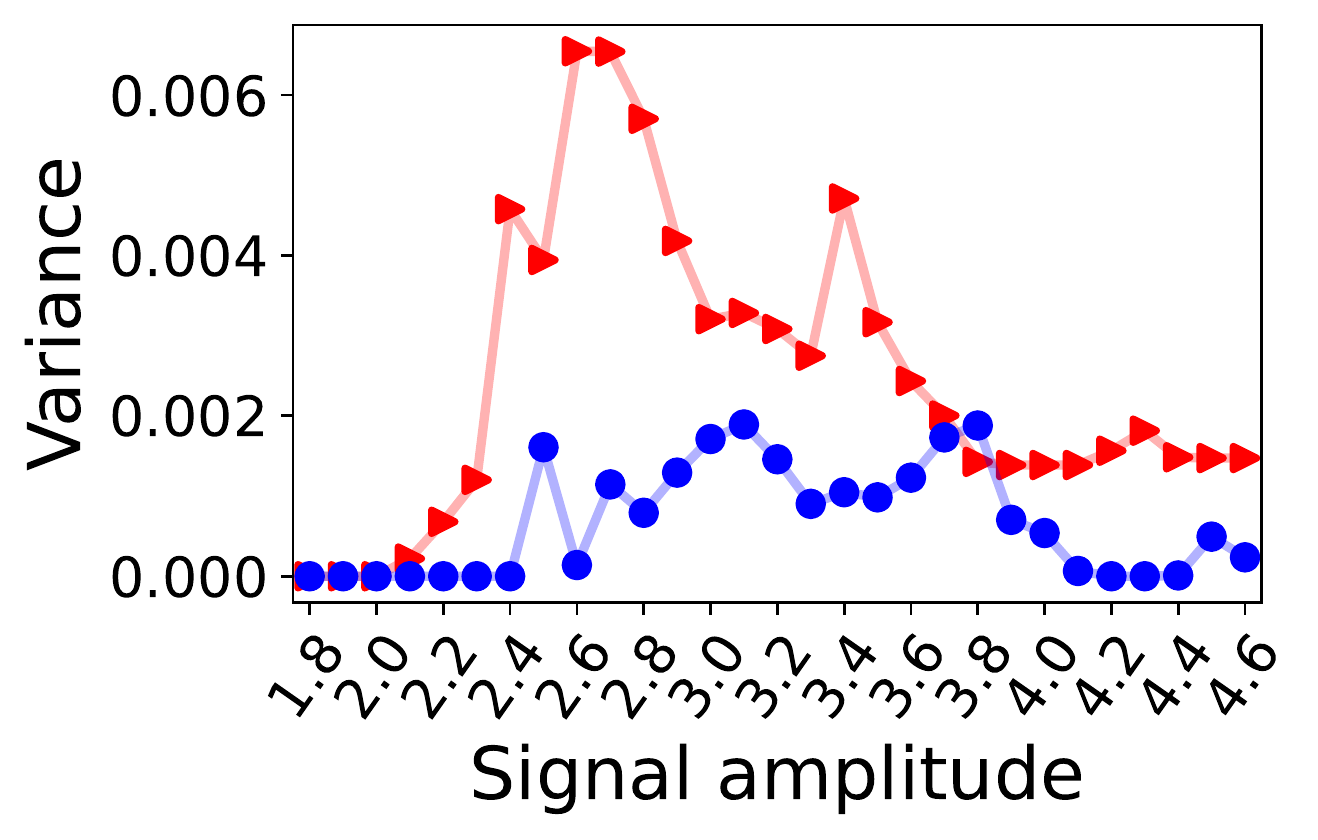}
  \caption{Performance on synthetic data of the proposed derandomized outlier detection method, \texttt{E-OC-Conformal}, applied with $K=10$ analysis repetitions. The results are compared to the performance of the randomized benchmark, \texttt{OC-Conformal}, as a function of the signal strength. Both methods leverage a one-class support vector classifier. Left: average proportion of true outliers that are discovered (higher is better). Center: average proportion of false discoveries (lower is better). Right: variability of the findings (lower is better).
Other details are as in Figure~\ref{fig:data_difficulty}. Note that these results correspond to 100 repeated experiments based on a single realization of the labeled and test data, hence why the results appear a little noisy.
}
\label{fig-app:data_difficulty_ocsvm}
  \end{figure}

\begin{figure}[!htb]
  \centering
  \includegraphics[width=0.45\textwidth]{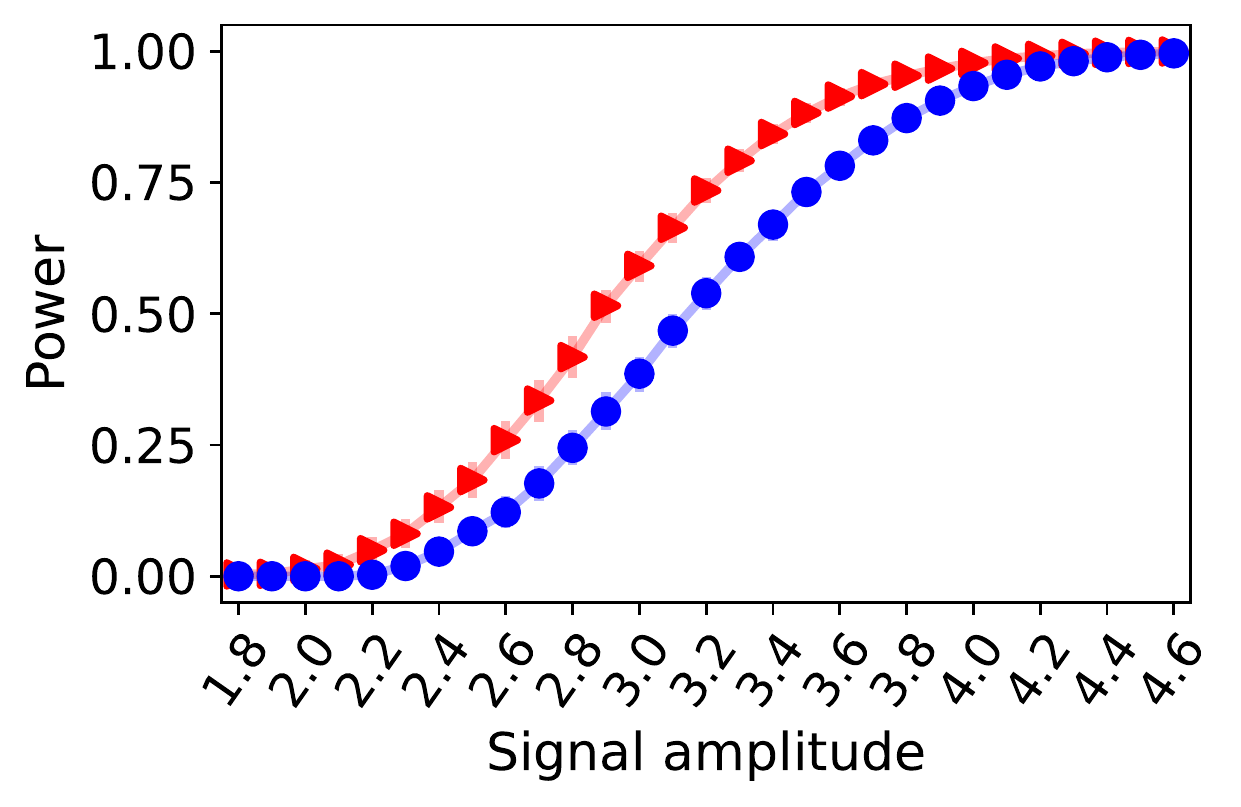}
  \includegraphics[width=0.45\textwidth]{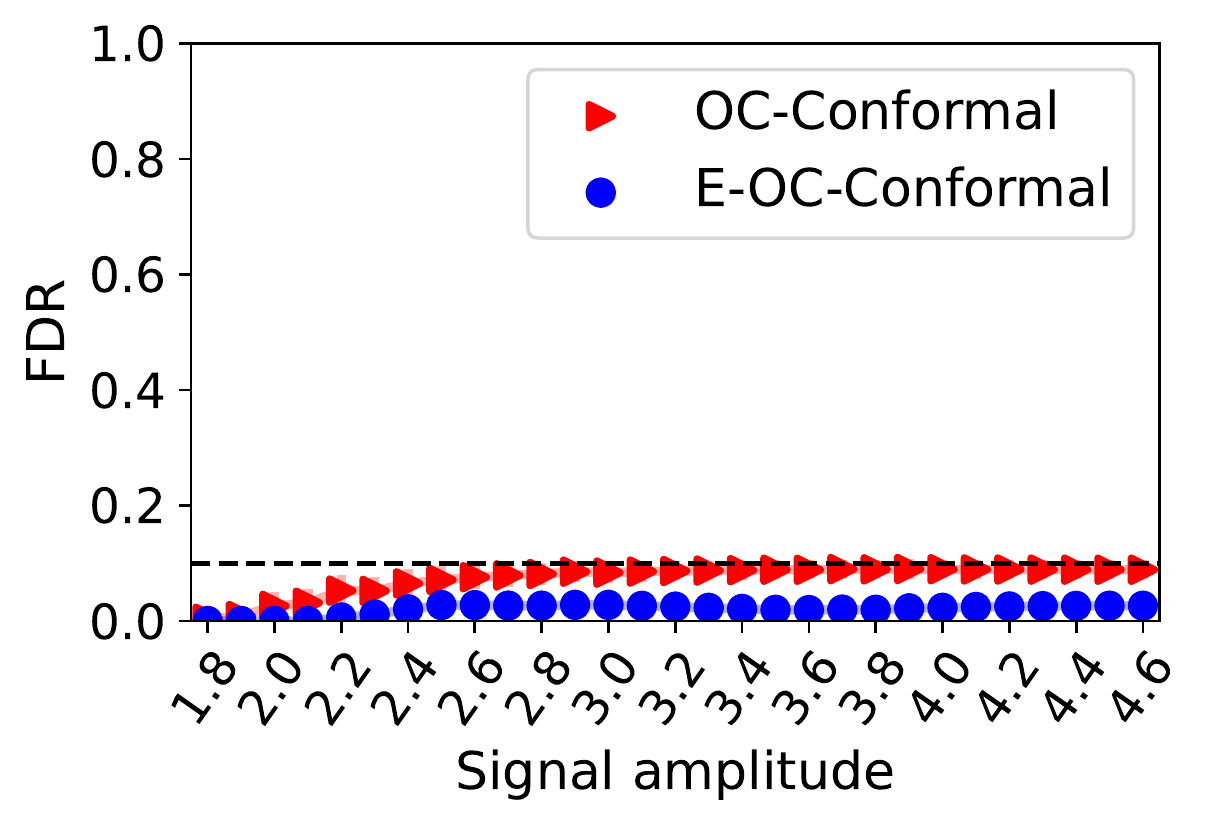}
  \caption{Performance on synthetic data of the proposed derandomized outlier detection method, \texttt{E-OC-Conformal}, applied with $K=10$ analysis repetitions. The results are compared to the performance of the randomized benchmark, \texttt{OC-Conformal}, as a function of the signal strength. Both methods leverage a one-class support vector classifier. The results are averaged over 100 independent realizations of the data. Left: average proportion of true outliers that are discovered (higher is better). Right: average proportion of false discoveries (lower is better).
Other details are as in Figure~\ref{fig:data_difficulty}.
}
\label{fig-app:synth-classic-FDR-Power-data_difficulty_ocsvm}
  \end{figure}

\FloatBarrier

\subsection{Hyper-parameter tuning} \label{app:alpha_bh}

In all of the experiments presented in the main manuscript, we set $\alpha_{\mathrm{bh}}$ to be $\alpha / 10$. We found this to be a reasonable choice in general, although it may not always be optimal.
In this section, we explore the effect of $\alpha_{\mathrm{bh}}$ across four scenarios: low and high power regimes, as well as small and large proportions of outliers in the test set. Following Figure~\ref{app-fig:alpha_bh-OC}, we can see that the choice $\alpha_{\mathrm{bh}}=\alpha / 10$ is not always ideal for the \texttt{E-OC-Conformal} algorithm. Here, a larger value of $\alpha_{\mathrm{bh}}=\alpha / 2$ seems to be a better choice when the proportion of outliers is large, as it results in more test outlier samples with non-zero e-values. By contrast, a smaller value of $\alpha_{\mathrm{bh}}=\alpha / 10$ is suitable when the proportion of outliers is small, as it leads to larger e-values for the outliers.

We now repeat the same experiment but with \texttt{E-AdaDetect}, summarizing the results in Figure~\ref{app-fig:alpha_bh-AdaDetect}. In contrast with \texttt{E-OC-Conformal}, a fixed $\alpha / 10$ is an appropriate choice for $\alpha_{\mathrm{bh}}$ when applying our method with \texttt{AdaDetect} for all the scenarios we studied, indicating that this method is more robust to the choice of $\alpha_{\mathrm{bh}}$. In the remaining supplementary experiments, we will utilize a fixed $\alpha_{\mathrm{bh}}=\alpha / 2$ for \texttt{E-OC-Conformal} and a fixed $\alpha_{\mathrm{bh}}=\alpha / 10$ for \texttt{E-AdaDetect}.

\begin{figure*}[!htb]
  \centering
  \begin{subfigure}[b]{0.3\textwidth}
    \includegraphics[width=\textwidth]{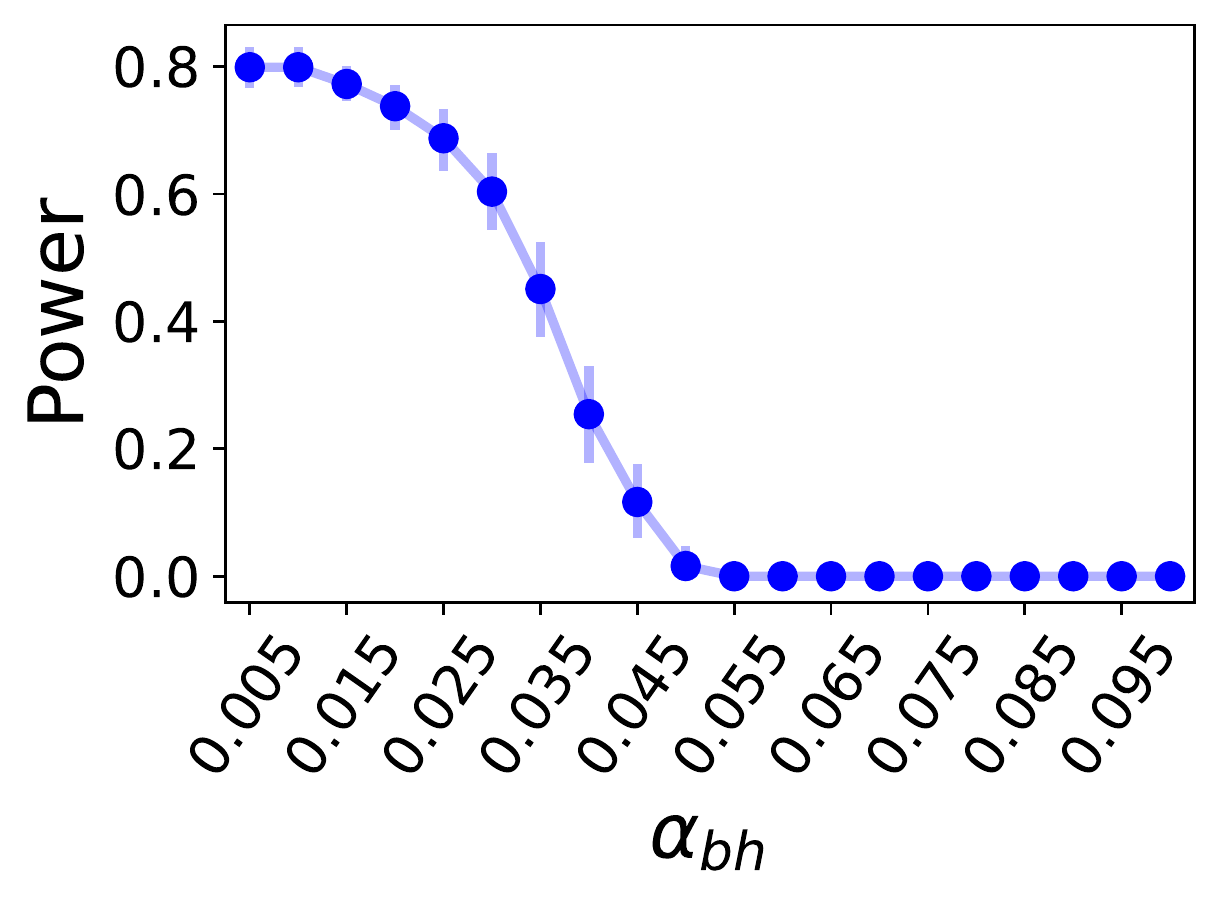}
    \caption{high power, $10\%$ outliers}
  \end{subfigure}
  \begin{subfigure}[b]{0.3\textwidth}
    \includegraphics[width=\textwidth]{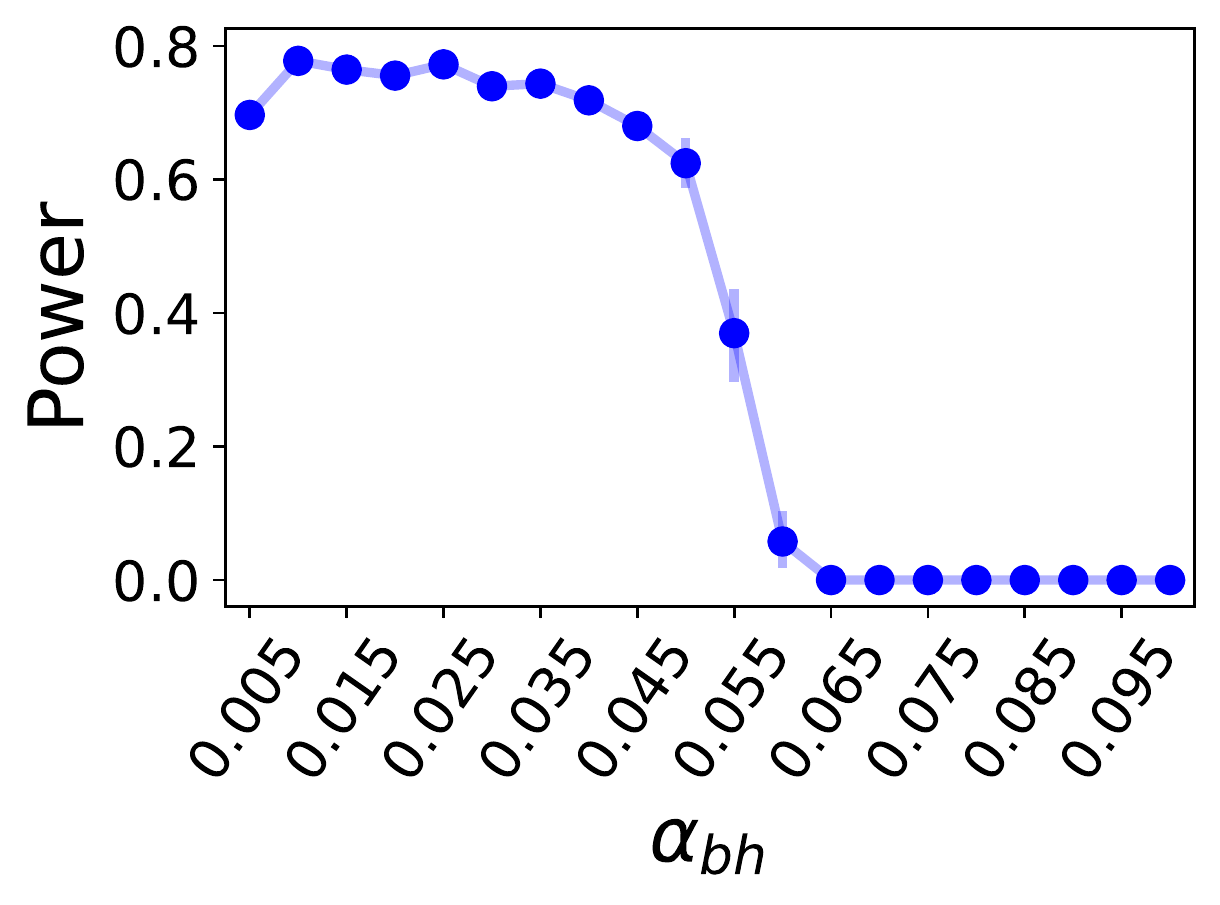}
    \caption{high power, $50\%$ outliers}
  \end{subfigure}
    \\
  \begin{subfigure}[b]{0.3\textwidth}
    \includegraphics[width=\textwidth]{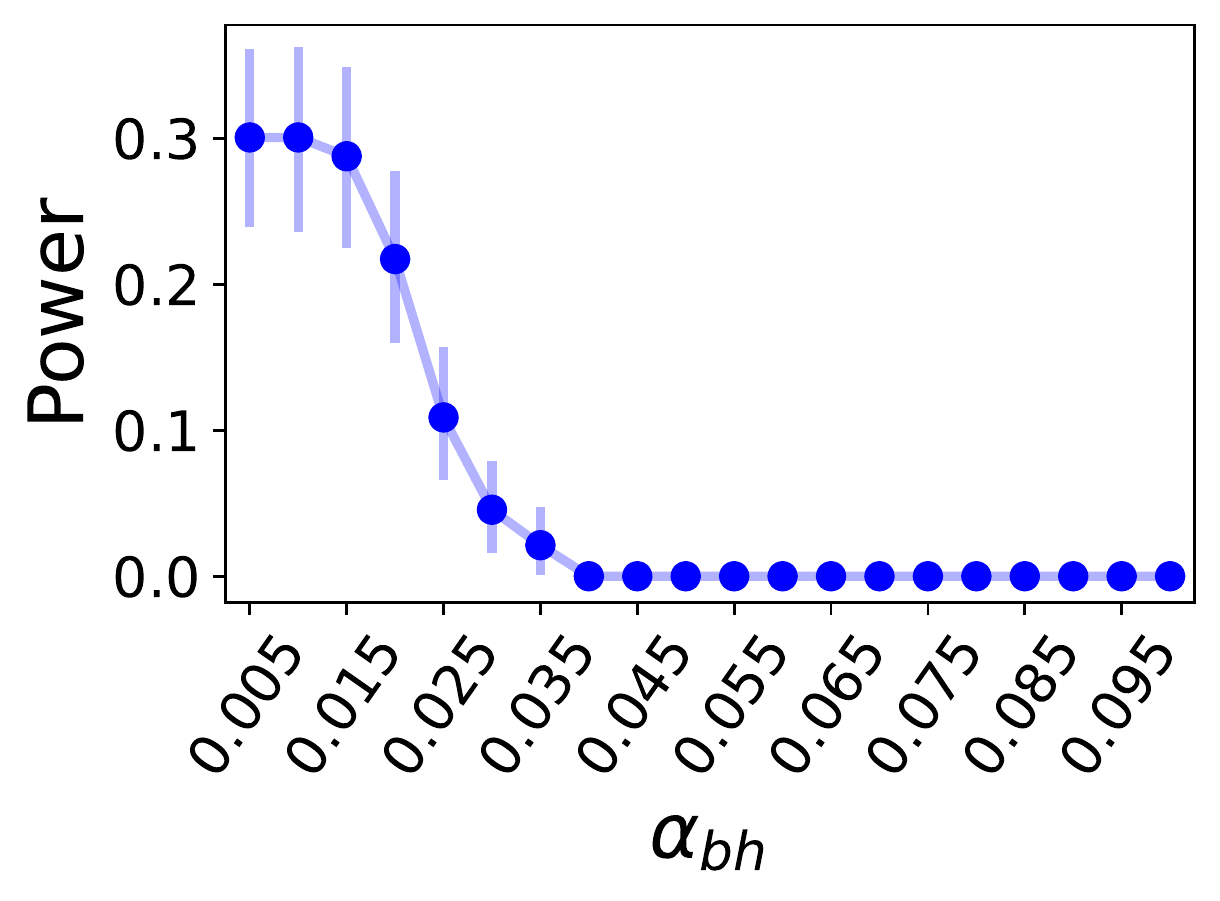}
    \caption{low power, $10\%$ outliers}
  \end{subfigure}
  \begin{subfigure}[b]{0.3\textwidth}
    \includegraphics[width=\textwidth]{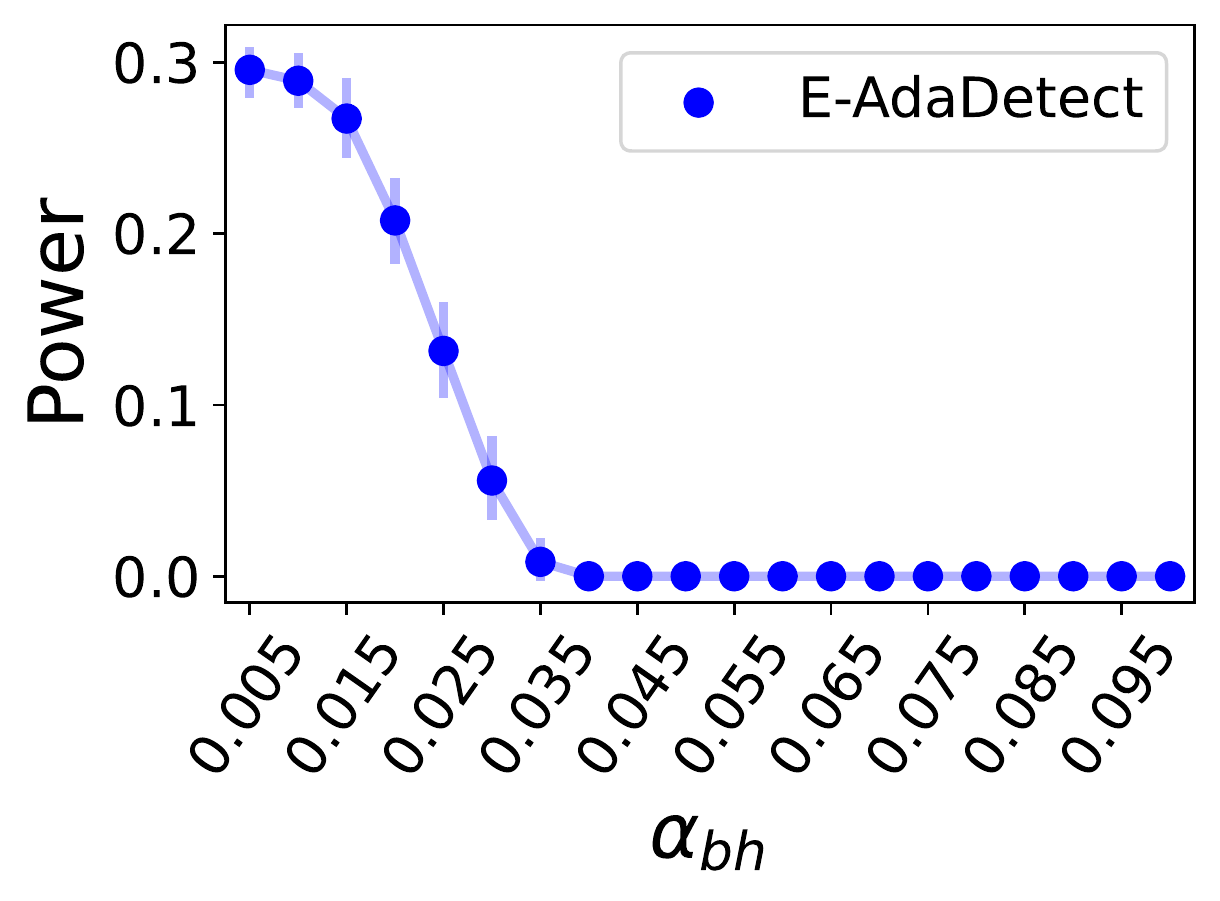}
    \caption{low power, $50\%$ outliers}
  \end{subfigure}
  \caption{
  Performance on synthetic data of the proposed derandomized outlier detection method, \texttt{E-AdaDetect}, applied with $K=10$ as a function of $\alpha_{\mathrm{bh}}$. The results are averaged over 100 independent realizations of the data.
  Top: high-power regime with signal amplitude $3.4$ for $10\%$ outliers and $1.6$ for $50\%$ outliers. Bottom: low-power regime with signal amplitude $2.8$ for $10\%$ outliers and $1.1$ for $50\%$ outliers. Left: $10\%$ outliers in the test-set. Right: $50\%$ outliers in the test-set. Other details are as in Figure~\ref{fig:data_difficulty}.}
\label{app-fig:alpha_bh-AdaDetect}
  \end{figure*}

  \begin{figure*}[!htb]
  \centering
  \begin{subfigure}[b]{0.3\textwidth}
    \includegraphics[width=\textwidth]{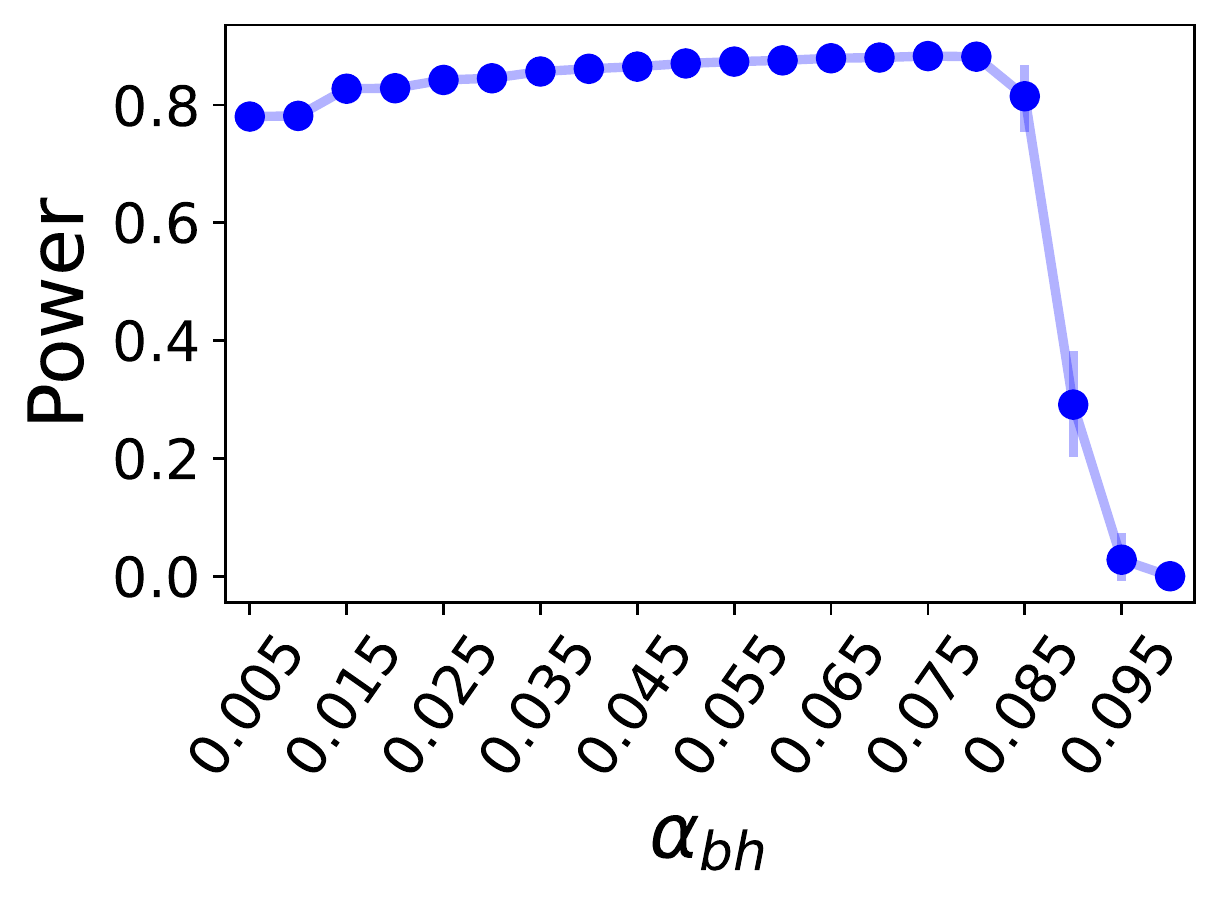}
    \caption{high power, $10\%$ outliers}
  \end{subfigure}
  \begin{subfigure}[b]{0.3\textwidth}
    \includegraphics[width=\textwidth]{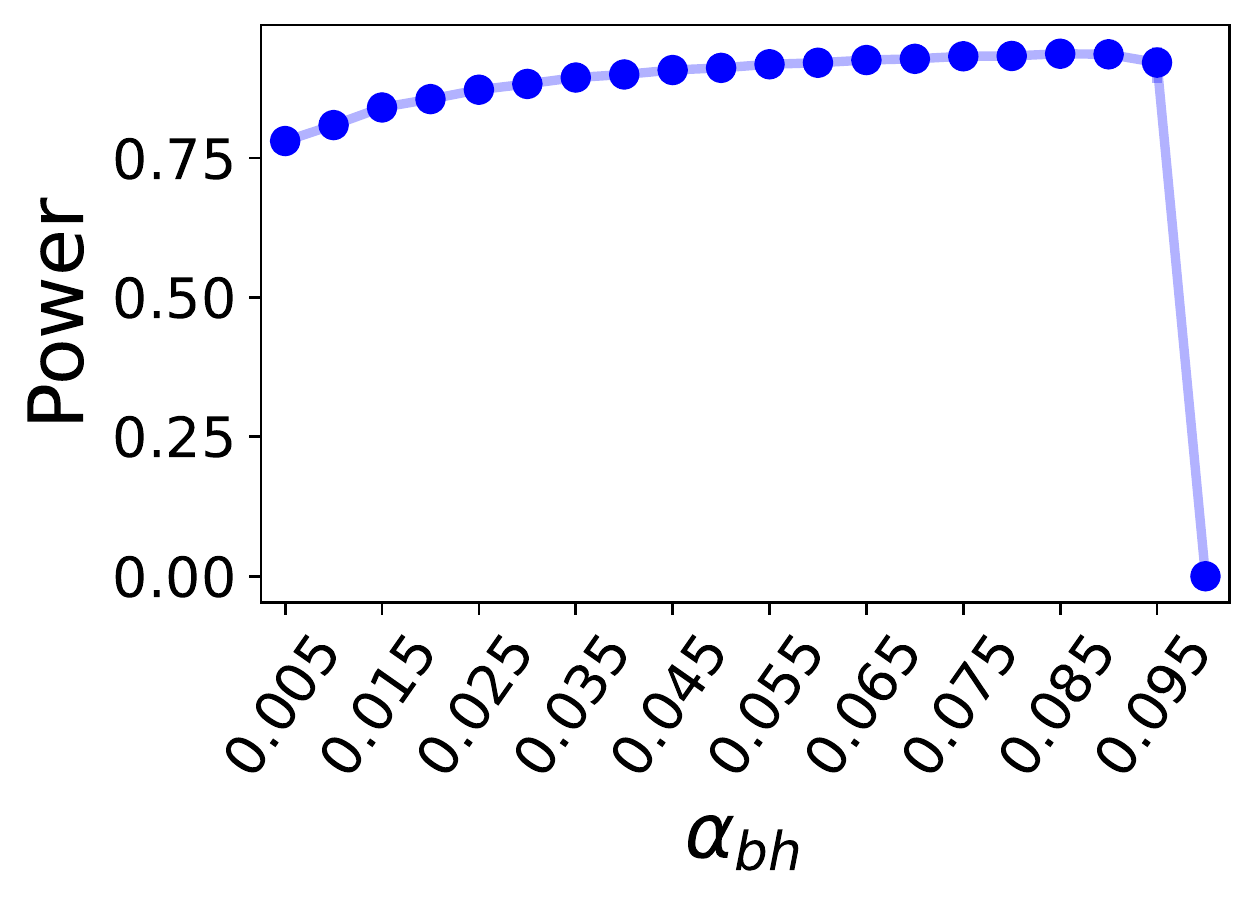}
    \caption{high power, $50\%$ outliers}
  \end{subfigure}
    \\
  \begin{subfigure}[b]{0.3\textwidth}
    \includegraphics[width=\textwidth]{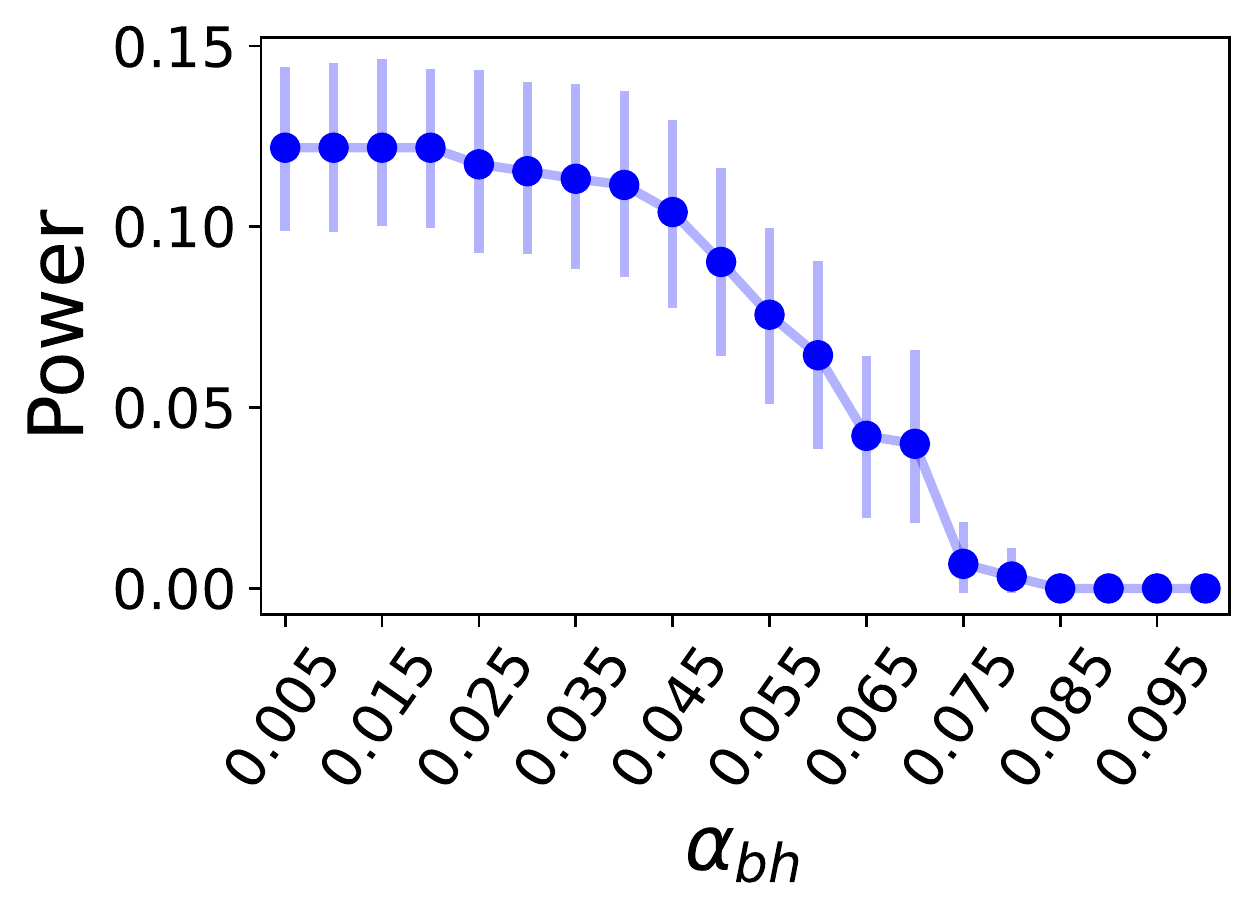}
    \caption{low power, $10\%$ outliers}
  \end{subfigure}
  \begin{subfigure}[b]{0.3\textwidth}
    \includegraphics[width=\textwidth]{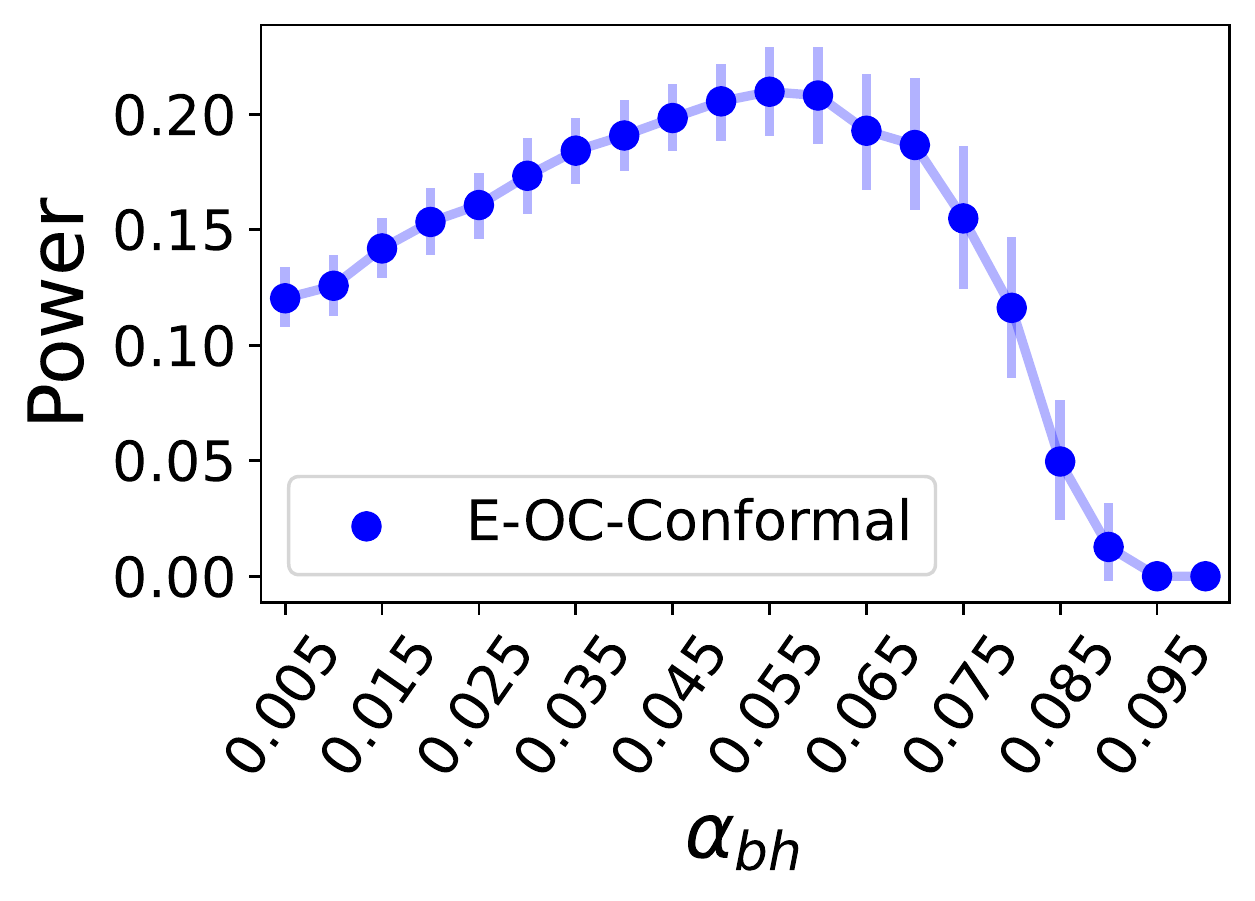}
    \caption{low power, $50\%$ outliers}
  \end{subfigure}
  \caption{Performance on synthetic data of the proposed derandomized outlier detection method, \texttt{E-OC-Confromal}, applied with $K=10$ as a function of $\alpha_{\mathrm{bh}}$. The method leverages a one-class support vector classifier. The results are averaged over 100 independent realizations of the data.
  Top: high-power regime with signal amplitude $3.6$ for $10\%$ outliers and $3.4$ for $50\%$ outliers. Bottom: low-power regime 
  with signal amplitude $2.6$ for $10\%$ outliers and $2.3$ for $50\%$ outliers. Left: $10\%$ outliers in the test-set. Right: $50\%$ outliers in the test-set. Other details are as in Figure~\ref{fig:data_difficulty}.}
\label{app-fig:alpha_bh-OC}
  \end{figure*}

  \FloatBarrier
  \clearpage

\section{Comparisons to alternative e-values constructions}
\label{app:baselines}

In this section, we discuss alternative methods for constructing conformal e-values to derandomize split conformal inferences, comparing their performance to that of our proposed martingale-based approach.
These alternative e-value constructions were proposed in prior works, as detailed below, but they had not been previously utilized for the purpose of de-randomizing conformal inferences.

\subsection{Review of {p-to-e calibrators}}
\label{app:p-to-e-calibrators}
One strategy for generating e-values is to utilize {\em p-to-e calibrators}, whose goal is to transform valid p-values into valid e-values \citep[Section~2]{e-value}. This strategy can be specifically applied to conformal p-values $\hat{u}$. Various types of calibrators are available, including Shafer's calibrator:
\begin{equation}
\label{eq:def-shafer-calibrator}
    S(\hat{u}):= \frac{1}{\sqrt{\hat{u}}} - 1,
\end{equation}
as well as the following family of calibrators,
\begin{equation}
\label{eq:def-class-calibrator}
    F(\hat{u})= \epsilon \cdot \hat{u}^{\epsilon - 1},
\end{equation}
where $\epsilon \in (0,1)$ is a hyper-parameter. To bypass the problem of choosing the hyper-parameter $\epsilon$ in \eqref{eq:def-class-calibrator}, one can use an over-optimistic estimate of the maximum of \eqref{eq:def-class-calibrator}, known as the VS calibrator \citep{e-value}: 
\begin{equation}
\label{eq:vs}
\text{VS}(\hat{u}):={\sup}_{\epsilon} \ \epsilon \hat{u}^{\epsilon - 1} = \left\{
  \begin{array}{@{}ll@{}}
    -e^{-1} / (\hat{u}\text{ ln }\hat{u}), & \text{if}\ \hat{u}\leq e^{-1} \\
    1, & \text{otherwise} 
  \end{array}\right..
\end{equation}
The VS calibrator does not produce a valid e-value, but it can still serve as an informative baseline because it approximates the most powerful possible calibrator within the family of functions~\eqref{eq:def-class-calibrator} \citep{e-value}.
An alternative way to eliminate the influence of $\epsilon$ is to integrate over it, leading to the following calibrator:
\begin{equation}
\label{eq:calibrator-integral}
F(\hat{u}):= \int_0^{1} \epsilon \hat{u}^{\epsilon - 1} d\epsilon = \frac{1-\hat{u}+\hat{u}\ln{\hat{u}}}{\hat{u}(-\ln{\hat{u}})^2}.
\end{equation}

Armed with a p-to-e calibrator, one can then derandomize split conformal inferences by proceeding similarly to the main manuscript; this approach is summarized for completeness in Algorithm~\ref{drand-p-to-e-adaptive}.

\begin{algorithm}[!htb]
\caption{Aggregation of conformal e-values computed by a p-to-e calibrator with data-adaptive model weights}
\label{drand-p-to-e-adaptive}
\begin{algorithmic}[1]
\STATE \textbf{Input:} {
{inlier data set $\D\equiv \left\{ X_i\right\}_{i=1}^n$};
{test set $\D_{\mathrm{test}}$};
{size of calibration-set $n_{\mathrm{cal}}$};
{number of iterations $K$};
{p-to-e calibrator function $F$};
{one-class or binary black-box classification algorithm $\mathcal{A}$};
{a model weighting function $\omega$};
}
\FOR{$k=1,...,K$}
\STATE Randomly split $\D$ into $\D_{\mathrm{cal}}^{(k)}$ and $\D_{\mathrm{train}}^{(k)}$, with $|\D_{\mathrm{cal}}^{(k)}|=n_{\mathrm{cal}}$
\STATE Train the model: $\mathcal{M}^{(k)}\gets\mathcal{A}(\D_{\mathrm{train}}^{(k)})$ \COMMENT{possibly including additional labeled outlier data if available}
\STATE Compute the calibration scores $S^{(k)}_i=\mathcal{M}^{(k)}(X_i)$, for all $i\in \D_{\mathrm{cal}}^{(k)}$
\STATE Compute the test scores $S^{(k)}_j=\mathcal{M}^{(k)}(X_j)$, for all $j\in \D_{\mathrm{test}}$
\STATE Compute the weights $\tilde{w}^{(k)} = \omega \left( \{ S_{i}^{(k)} \}_{i \in \D_{\mathrm{test}} \cup \D_{\mathrm{cal}}^{(k)} }  \right)$ \COMMENT{invariant un-normalized model weights}
\STATE Compute the p-values for all $j\in \left| \D_{\mathrm{test}} \right|$ : $\hat{u}_j^{(k)} = ( 1 + \sum_{i \in \D_{\mathrm{cal}}} \mathbb{I}\{ S_j^{(k)} \leq S_i^{(k)} \})  / (1+n_{\mathrm{cal}})$ \STATE Compute the e-values $e^{(k)}_{j}$ for all $j\in \left| \D_{\mathrm{test}}\right|$ using the p-to-e calibrator: $e^{(k)}_{j}=F\left(\hat{u}_j^{(k)}\right)$
\ENDFOR
\FOR{$k=1,...,K$}
\STATE $w^{(k)} = \tilde{w}^{(k)} / \sum_{k'=1}^{K} \tilde{w}^{(k')}$ \COMMENT{normalize the model weights}
\ENDFOR
\STATE Aggregate the e-values $\bar{e}_j = \sum_{k=1}^K w^{(k)} \cdot e^{(k)}_j$
\STATE \textbf{Output:} e-values $\bar{e}_j$ for all $j \in \mathcal{D}_{\mathrm{test}}$ that can be filtered with Algorithm~\ref{alg:ebh} to control the FDR.
\end{algorithmic}
\end{algorithm}

\FloatBarrier

\subsubsection{Comparing p-to-e to the martingale-based approach} \
\begin{figure}[!htb]
    \centering
    \includegraphics[width=0.45\textwidth]{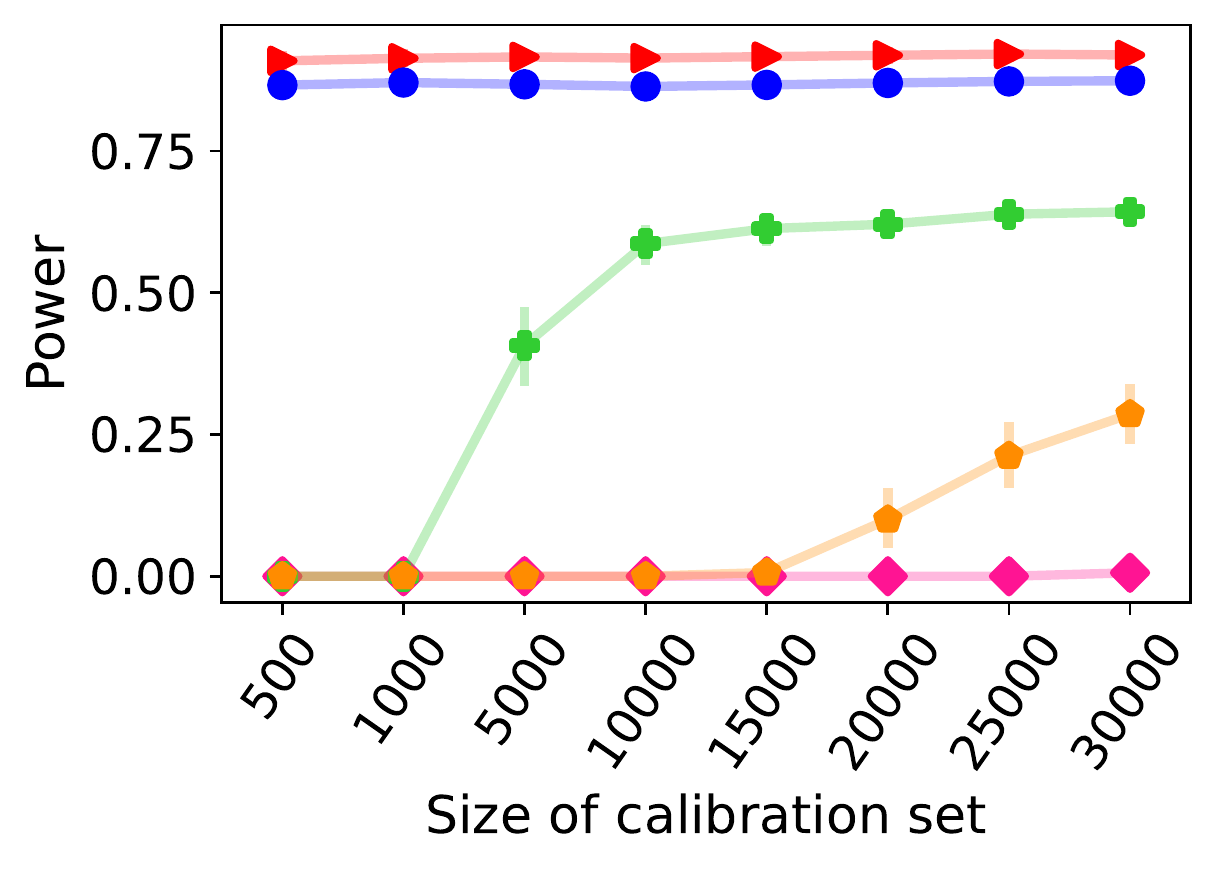}
    \includegraphics[width=0.45\textwidth]{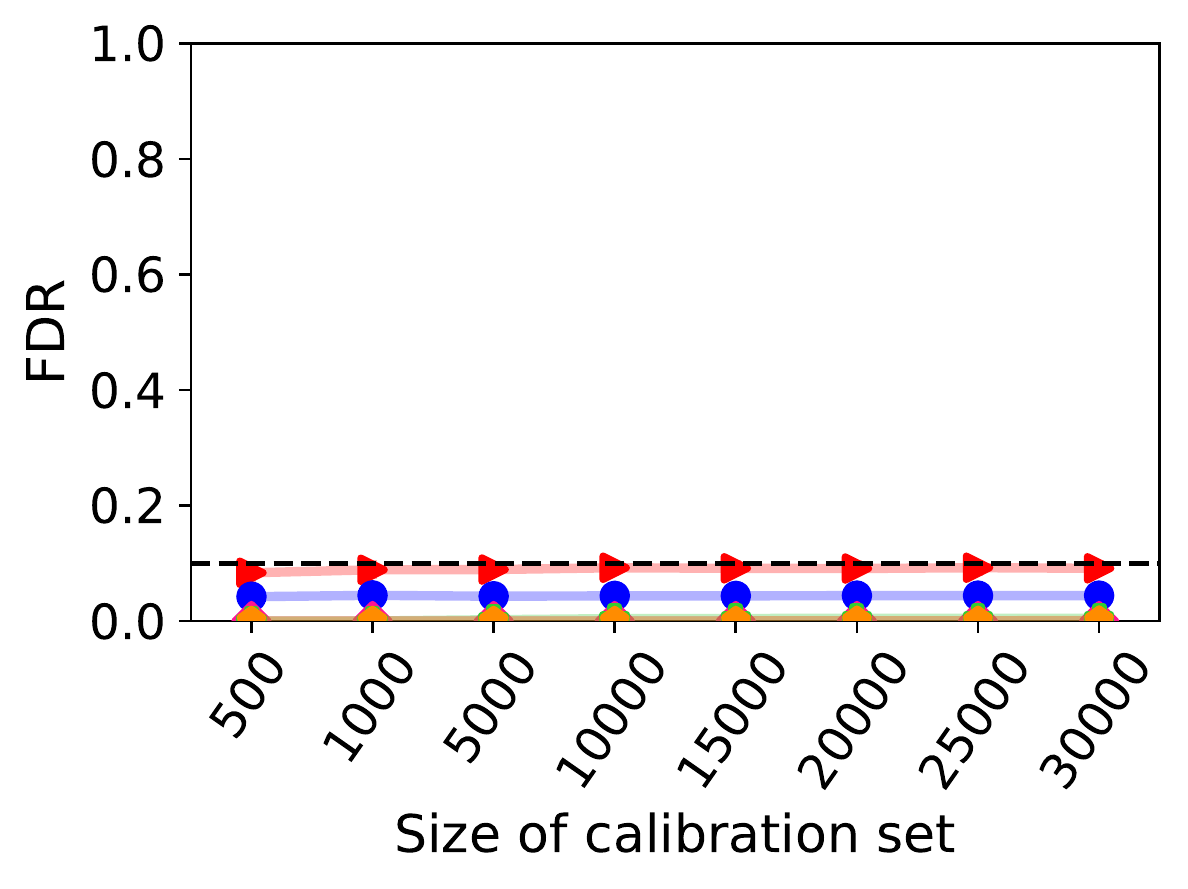}
    \\
    \includegraphics[width=0.45\textwidth]{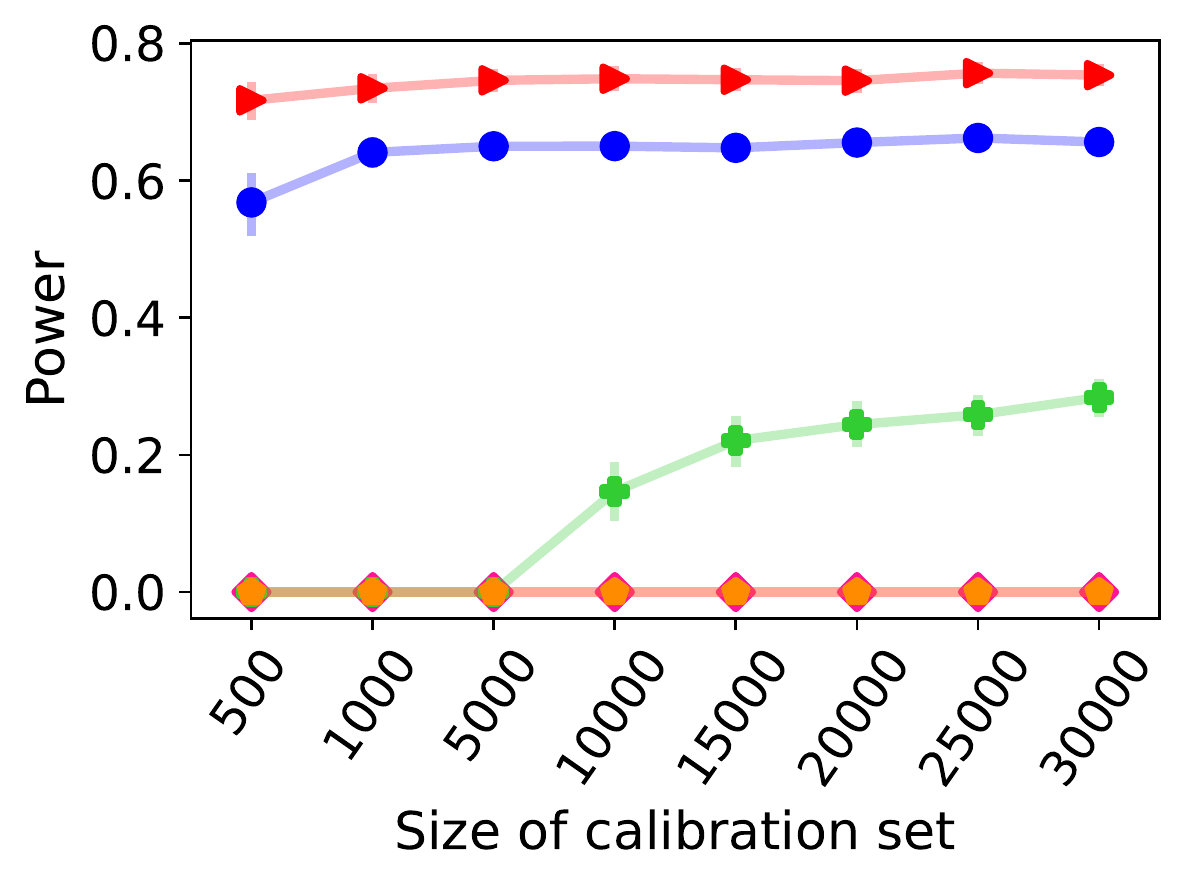}
    \includegraphics[width=0.45\textwidth]{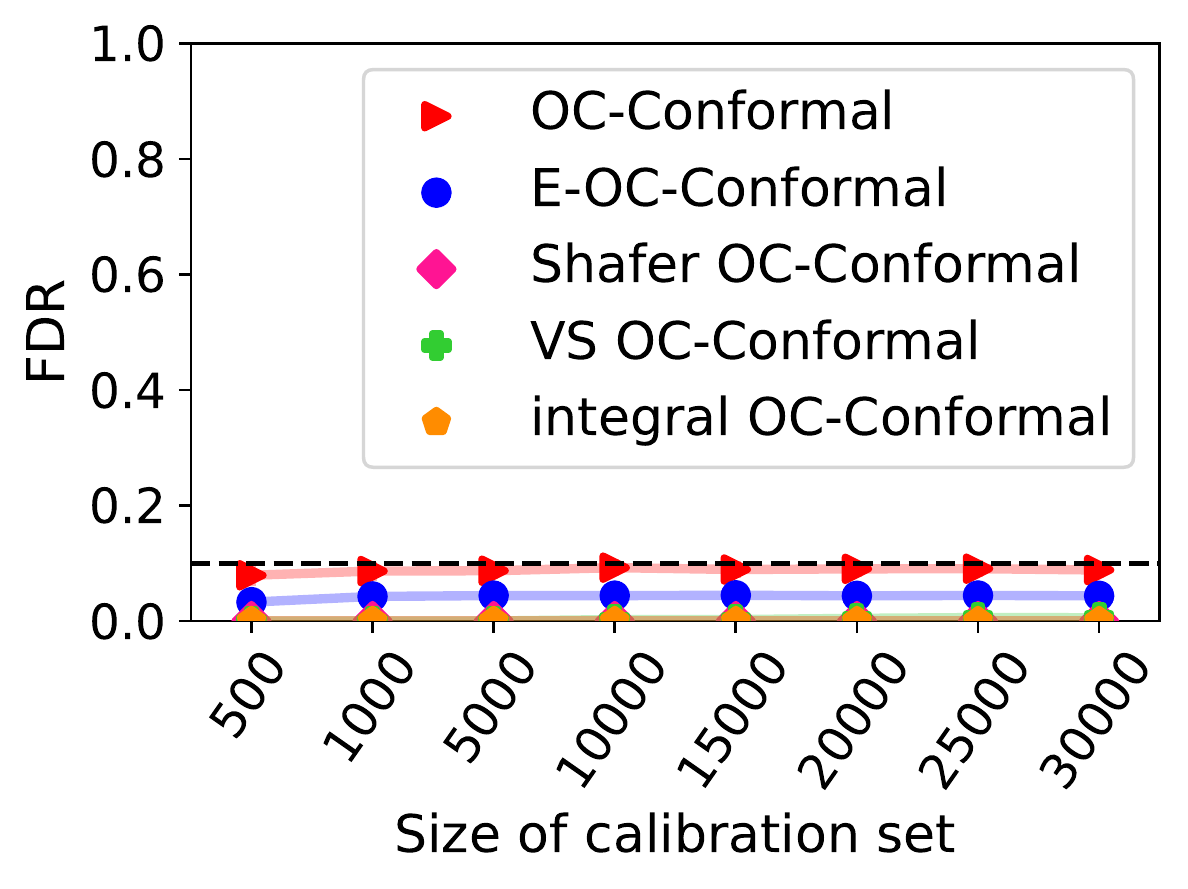}
    \caption{Performance on synthetic data of the proposed derandomized outlier detection method, \texttt{E-OC-Conformal}, applied with $K=10$, compared to that of its randomized benchmark, \texttt{OC-Conformal}. We also compare the performance of these methods to p-to-e calibrators, applied with $K=10$, as a function of the number of inlier calibration points. The number of training inliers is fixed and equals $1000$. All methods leverage a one-class support vector classifier.
  Top: high-power regime with signal amplitude $3.6$. Bottom: low-power regime with signal amplitude $3.2$. The dashed horizontal line indicates the nominal false discovery rate level $\alpha=0.1$. The results are averaged over 100 independent realizations of the data.}
    \label{app-fig:calibrators-calib_size-OCSVM-test_p_0.1-df}
\end{figure} % calibrators-calib_size-OCSVM-test_p_0.1-df

When implementing the p-to-e derandomization approach in the synthetic experiments described in Section~\ref{sec:synthetic}, we observed that the power was nearly zero. Increasing the size of the calibration set can be beneficial to improve the power of this approach. This is because the size of the calibration set determines the minimum attainable conformal p-value, given by ${1}/({n_{\text{cal}} + 1})$. Consequently, the size of this set influences the maximum achievable e-value through p-to-e calibrators: smaller input p-values result in larger outputs from the calibrator functions \eqref{eq:def-shafer-calibrator}, \eqref{eq:vs}, and \eqref{eq:calibrator-integral}.

Following the above discussion, we compare the performance of the p-to-e approach to our martingale-based method as a function of the size of the calibration set. According to Figure~\ref{app-fig:calibrators-calib_size-OCSVM-test_p_0.1-df}, we can see that the p-to-e approach has lower power than our method, where both derandomization methods are combined with \texttt{OC-Conformal}.  Among the studied p-to-e calibrators, the VS calibrator demonstrates relatively higher power. However, it should be noted that the VS calibrator generates invalid e-values, as this calibrator outputs an overly optimistic estimate of the maximum value in \eqref{eq:def-class-calibrator}. Nevertheless, even the VS calibrator is considerably less powerful than our proposed method.

Since there is a significant performance gap between our martingale-based e-value construction and the p-to-e approaches, we do not provide any further comparisons between these methods. We also do not repeat this experiment with 
\texttt{AdaDetect}, since the p-to-e approach requires a large calibration set to yield meaningful power, which is far from an ideal setup for \texttt{AdaDetect}. The latter approach suggests fitting a binary classifier on the observed data, treating both the calibration and test points as outliers.
A large amount of inlier calibration points can lower \texttt{AdaDetect}'s power since the wrong labeling of the calibration points as outliers is likely to reduce the classifier’s ability to provide large scores for test outliers.
 
 \FloatBarrier

\subsection{Review of soft-rank permutation e-test}
\label{app:soft-rank}
The soft-rank e-values introduced by \citet{ignatiadis2023evalues} is another approach to construct e-values for permutation tests, including split conformal. This method constructs an e-value for each test point by comparing its relative rank to the calibration samples, employing a similar methodology to the construction of conformal p-values described in \citet{conformal-p-values}. Here, we present a slightly modified version that begins with normalizing the conformity score for each test point as well as the calibration scores. 
Consider a single hypothesis corresponding to a single test point. Let $S_0$ be the corresponding test conformity score, and $S_1,\dots,S_{n_{\mathrm{cal}}}$ be the $n_{\mathrm{cal}}$ conformity scores correspond to the calibration set. 
With this in place, denote by $S_{\mathrm{max}}$ and $S_{\mathrm{min}}$ the maximum and minimum scores among $S_0,\dots, S_{n_{\mathrm{cal}}}$. Then, for each $b\in [0,n_{\mathrm{cal}}]$ we define the normalized score as
\begin{equation}
\label{eq:soft-rank-normalize}
L_b = \frac{S_b - S_{\mathrm{min}}}{S_{\mathrm{max}} - S_{\mathrm{min}}}.
\end{equation}
Having defined the normalized score, we construct an e-value for each test point by following the set of steps described in \citet{ignatiadis2023evalues}. Define $L_*=\min_{b=0,\dots,n_{\mathrm{cal}}}L_b$. For $b=0,\dots,n_{\mathrm{cal}}$, compute the transformed statistic as
\begin{equation}
\label{eq:soft-rank-transform}
    R_b= \frac{e^{rL_b} - e^{rL_*}}{r},
\end{equation}
where $r>0$ is a hyper-parameter.
In the case where $r=0$, the transformed statistic simplifies to $R_b = L_b - L_*$. Overall, the soft-ranking transformation presented above preserves the ordering of the test statistics while ensuring that the random variable $R_b$ is non-negative.
Leveraging the transformed non-negative variables, we can construct a valid e-value for the test point  by computing \citep{ignatiadis2023evalues}
\begin{equation}
\label{eq:soft-rank}
    e_0:= (n_{\mathrm{cal}}+1) \frac{R_0}{\sum_{i=0}^{n_{\mathrm{cal}}} R_i}.
\end{equation}
Similarly to p-to-e calibrators, we can combine the soft-rank e-value approach with our novel derandomization procedure, as outlined in Algorithm~\ref{drand-soft-rank-adaptive}, and further compare the performance of this method to our martingale-based e-values.

Before doing so, we pause to discuss the choice of the hyper-parameter $r$. Since there is no simple rule on how to set this parameter, we repeat the same analysis from Section~\ref{app:alpha_bh} and study the effect of $r$ across four scenarios: low/high power regimes and small/large proportions of test outliers. The results, presented in Figures~\ref{app-fig:soft_rank-r-AdaDetect} and \ref{app-fig:soft_rank-r-OC}, suggest that a suitable choice for $r$ could be $500$ for \texttt{AdaDetect} and $r=75$ for \texttt{OC-Conformal}. The latter choice takes into account the trade-off in power across Figures~\ref{app-fig:soft_rank-r-OC-high-0.5} and \ref{app-fig:soft_rank-r-OC-low-0.1}. We use these choices for all the soft-rank experiments provided in this Supplementary Material.

\begin{algorithm}[!htb]
\caption{Aggregation of soft-rank e-values with data-adaptive model weights}
\label{drand-soft-rank-adaptive}
\begin{algorithmic}[1]
\STATE \textbf{Input:} {
{inlier data set $\D\equiv \left\{ X_i\right\}_{i=1}^n$};
{test set $\D_{\mathrm{test}}$};
{size of calibration-set $n_{\mathrm{cal}}$};
{number of iterations $K$};
{one-class or binary black-box classification algorithm $\mathcal{A}$};
{a model weighting function $\omega$};
{hyper-parameter $r \in [0,\infty) $};
}
\FOR{$k=1,...,K$}
\STATE Randomly split $\D$ into $\D_{\mathrm{cal}}^{(k)}$ and $\D_{\mathrm{train}}^{(k)}$, with $|\D_{\mathrm{cal}}^{(k)}|=n_{\mathrm{cal}}$
\STATE Train the model: $\mathcal{M}^{(k)}\gets\mathcal{A}(\D_{\mathrm{train}}^{(k)})$ \COMMENT{possibly including additional labeled outlier data if available}
\STATE Compute the calibration scores $S^{(k)}_i=\mathcal{M}^{(k)}(X_i)$, for all $i\in \D_{\mathrm{cal}}^{(k)}$
\STATE Compute the test scores $S^{(k)}_j=\mathcal{M}^{(k)}(X_j)$, for all $j\in \D_{\mathrm{test}}$
\STATE Compute the weights $\tilde{w}^{(k)} = \omega \left( \{ S_{i}^{(k)} \}_{i \in \D_{\mathrm{test}} \cup \D_{\mathrm{cal}}^{(k)} }  \right)$ \COMMENT{invariant un-normalized model weights}
\STATE Normalize the scores according to \eqref{eq:soft-rank-normalize}
\STATE Compute the transformed score for all $j\in \left| \D_{\mathrm{test}} \right|$ according to \eqref{eq:soft-rank-transform} \COMMENT{this depends on the hyper-parameter $r$} \STATE Compute the e-values $e^{(k)}_{j}$ for all $j\in \left| \D_{\mathrm{test}}\right|$ according to \eqref{eq:soft-rank} 
\ENDFOR
\FOR{$k=1,...,K$}
\STATE $w^{(k)} = \tilde{w}^{(k)} / \sum_{k'=1}^{K} \tilde{w}^{(k')}$ \COMMENT{normalize the model weights}
\ENDFOR
\STATE Aggregate the e-values $\bar{e}_j = \sum_{k=1}^K w^{(k)} \cdot e^{(k)}_j$
\STATE \textbf{Output:} e-values $\bar{e}_j$ for all $j \in \mathcal{D}_{\mathrm{test}}$ that can be filtered with Algorithm~\ref{alg:ebh} to control the FDR.
\end{algorithmic}
\end{algorithm}

\begin{figure*}[!htb]
  \centering
  \begin{subfigure}[b]{0.3\textwidth}
    \includegraphics[width=\textwidth]{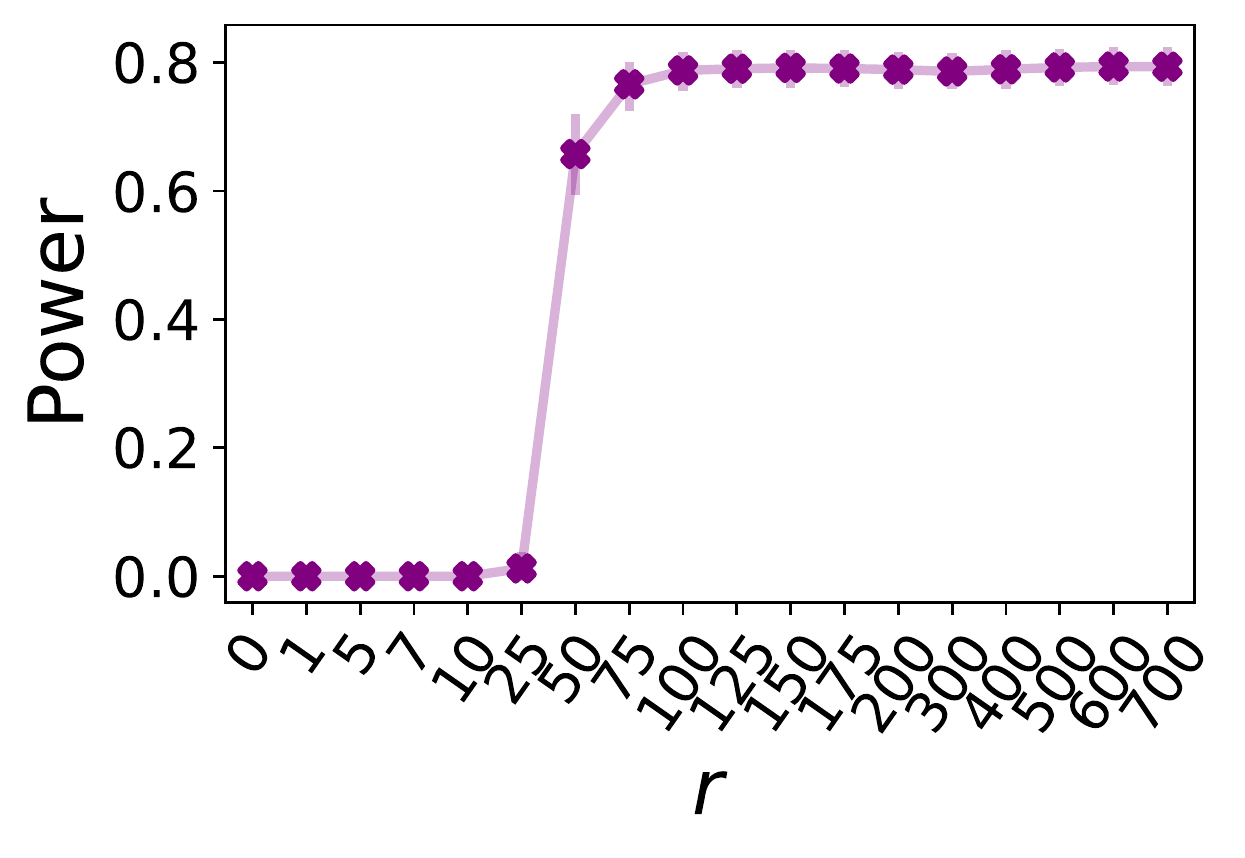}
    \caption{high power, $10\%$ outliers}
  \end{subfigure}
  \begin{subfigure}[b]{0.3\textwidth}
    \includegraphics[width=\textwidth]{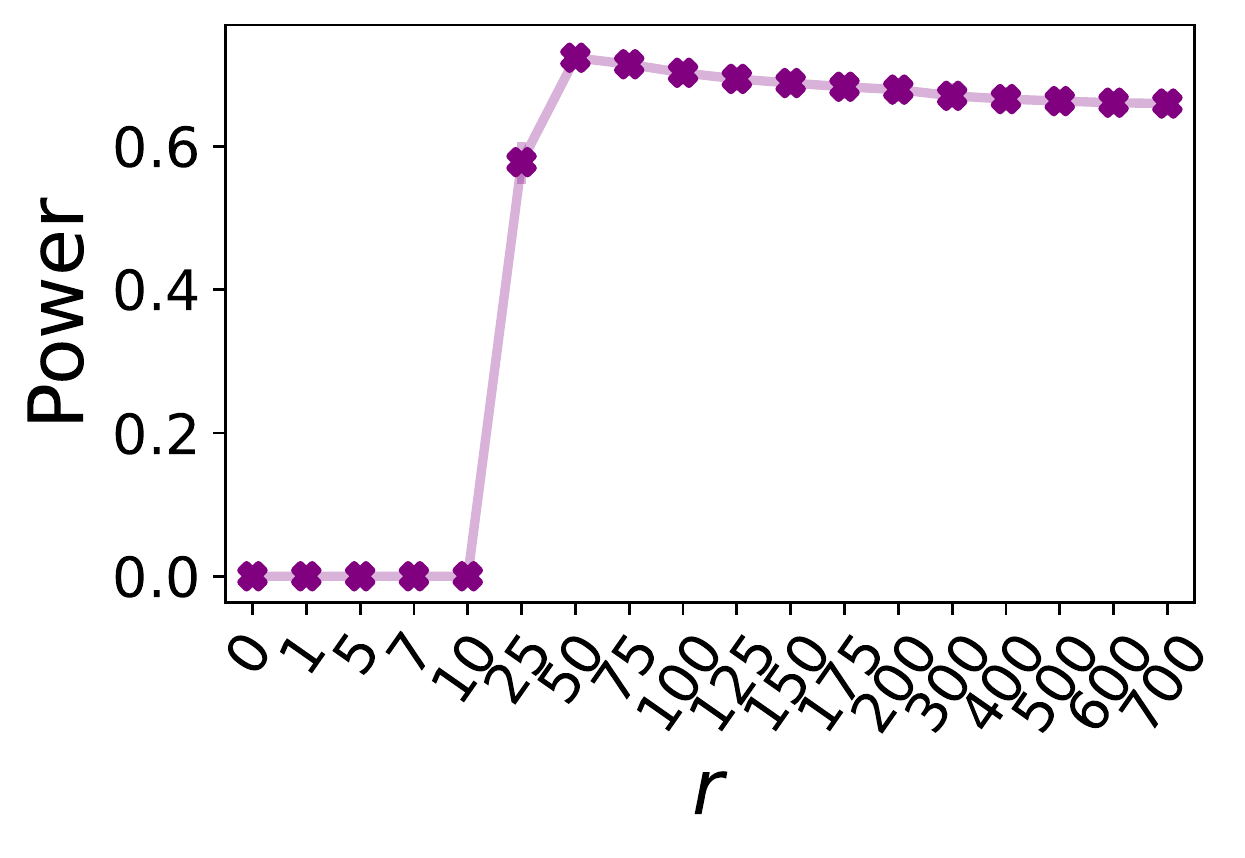}
    \caption{high power, $50\%$ outliers}
  \end{subfigure}
    \\
  \begin{subfigure}[b]{0.3\textwidth}
    \includegraphics[width=\textwidth]{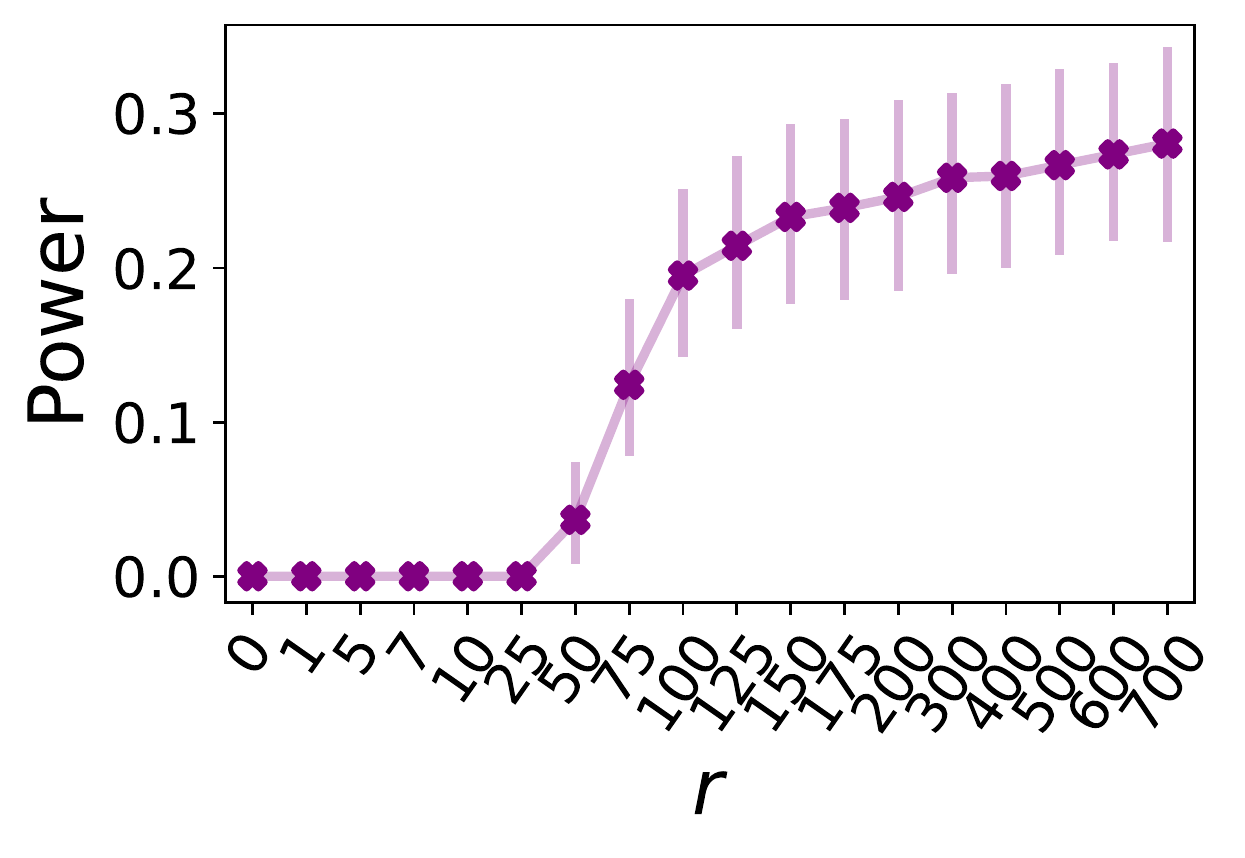}
    \caption{low power, $10\%$ outliers}
  \end{subfigure}
  \begin{subfigure}[b]{0.3\textwidth}
    \includegraphics[width=\textwidth]{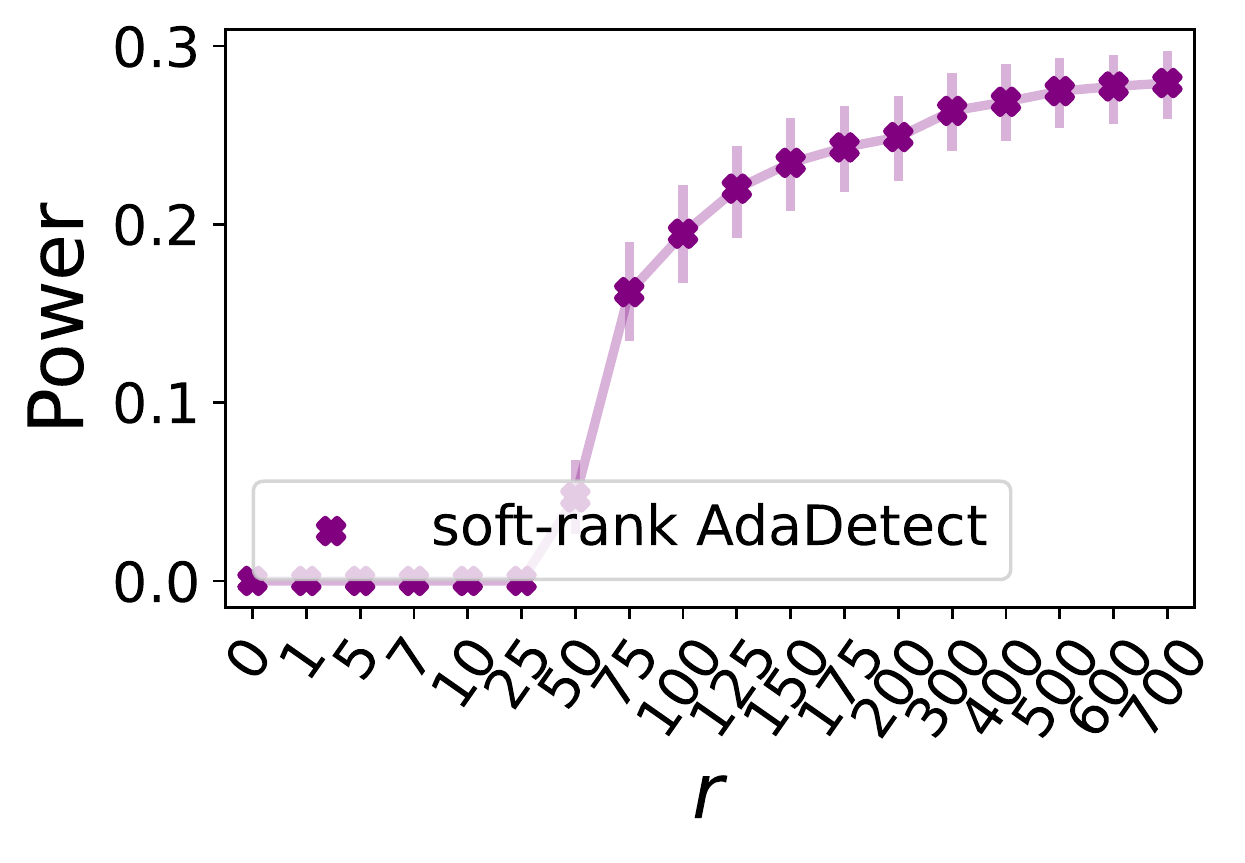}
    \caption{low power, $50\%$ outliers}
  \end{subfigure}
  \caption{Performance on synthetic data of the proposed derandomized outlier detection method applied with soft-rank e-values, \texttt{soft-rank OC-Conformal}, applied with $K=10$ as a function of $r$ hyper-parameter. The results are averaged over 100 independent realizations of the data.
  Top: high-power regime with signal amplitude $3.4$ for $10\%$ outliers and $1.6$ for $50\%$ outliers. Bottom: low-power regime with signal amplitude $2.8$ for $10\%$ outliers and $1.1$ for $50\%$ outliers. Left: $10\%$ outliers in the test-set. Right: $50\%$ outliers in the test-set. Other details are as in Figure~\ref{fig:data_difficulty}.}
\label{app-fig:soft_rank-r-AdaDetect}
  \end{figure*}

\begin{figure*}[!htb]
  \centering
  \begin{subfigure}[b]{0.3\textwidth}
    \includegraphics[width=\textwidth]{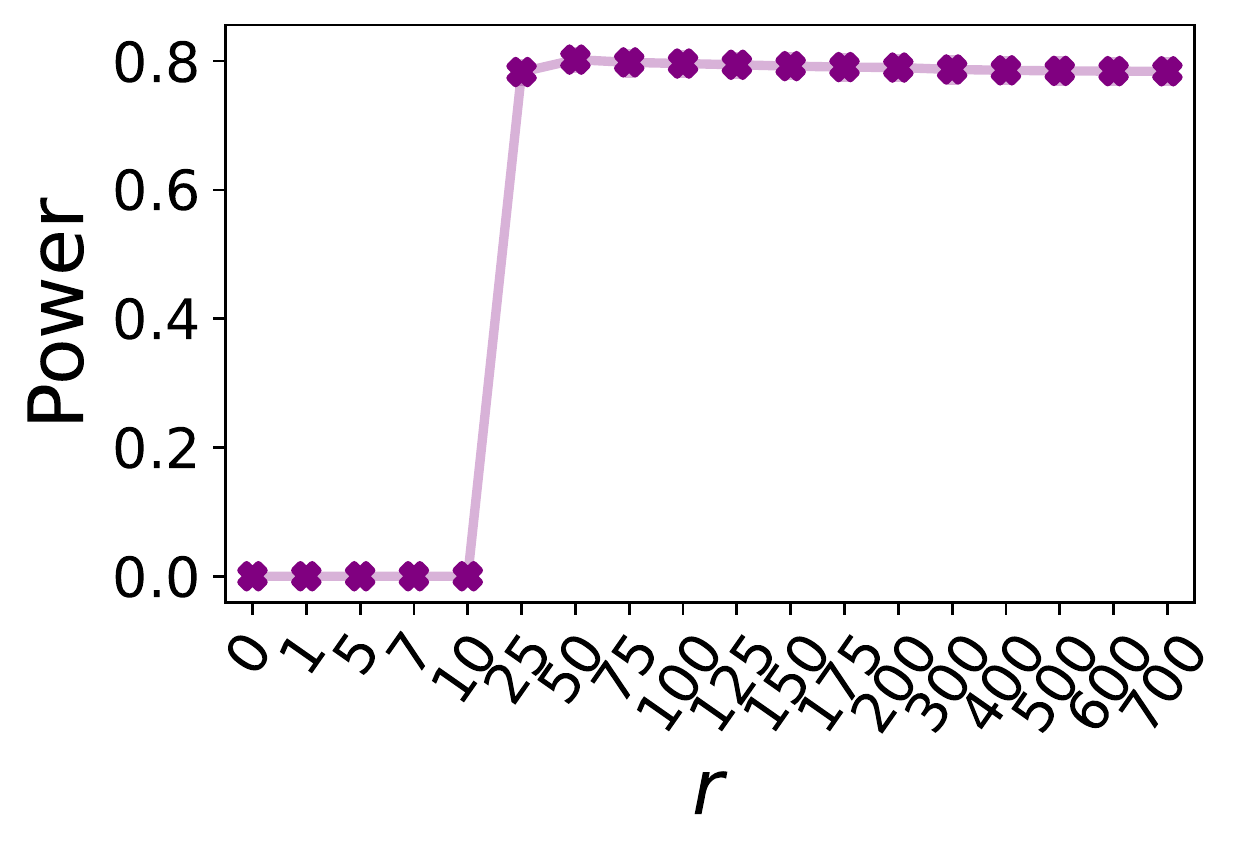}
    \caption{high power, $10\%$ outliers}
    \label{app-fig:soft_rank-r-OC-high-0.1}
  \end{subfigure}
  \begin{subfigure}[b]{0.3\textwidth}
    \includegraphics[width=\textwidth]{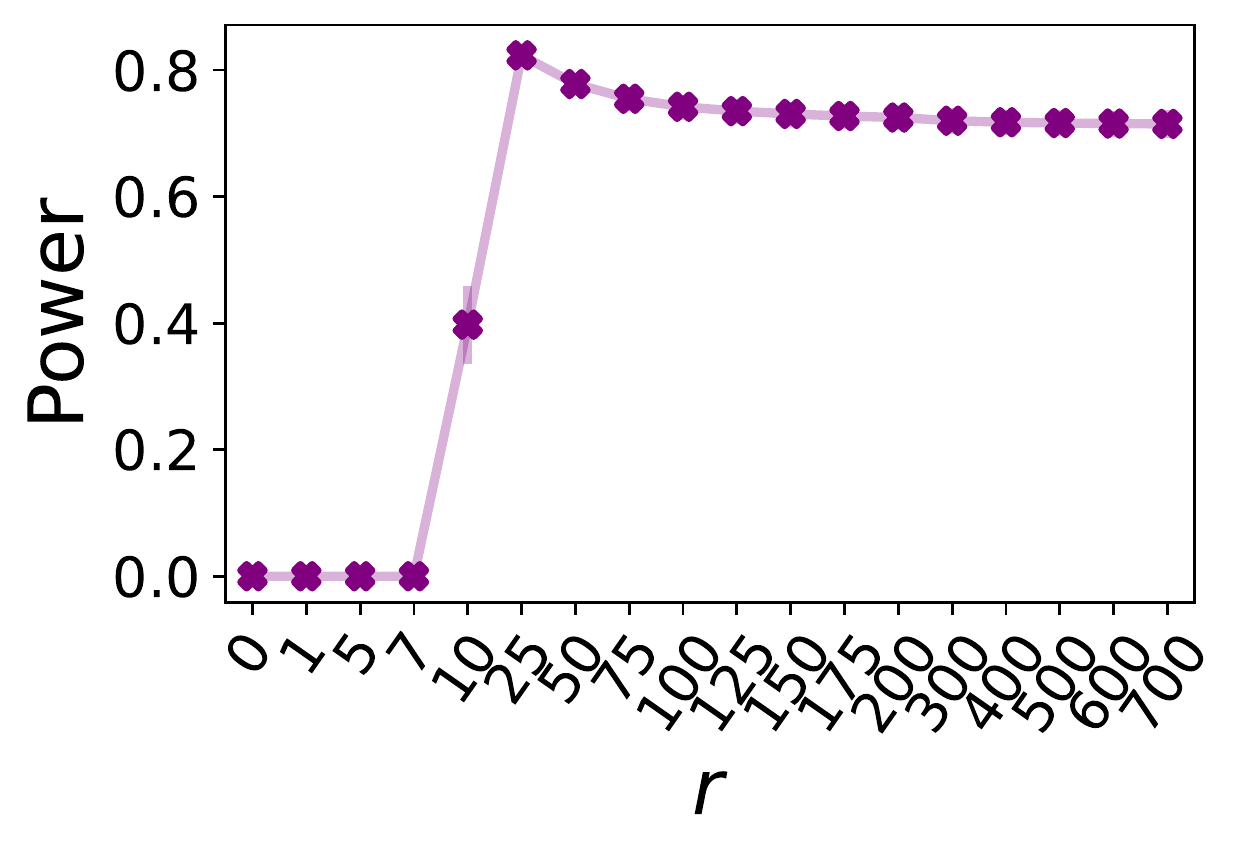}
    \caption{high power, $50\%$ outliers}
    \label{app-fig:soft_rank-r-OC-high-0.5}
  \end{subfigure}
    \\
  \begin{subfigure}[b]{0.3\textwidth}
    \includegraphics[width=\textwidth]{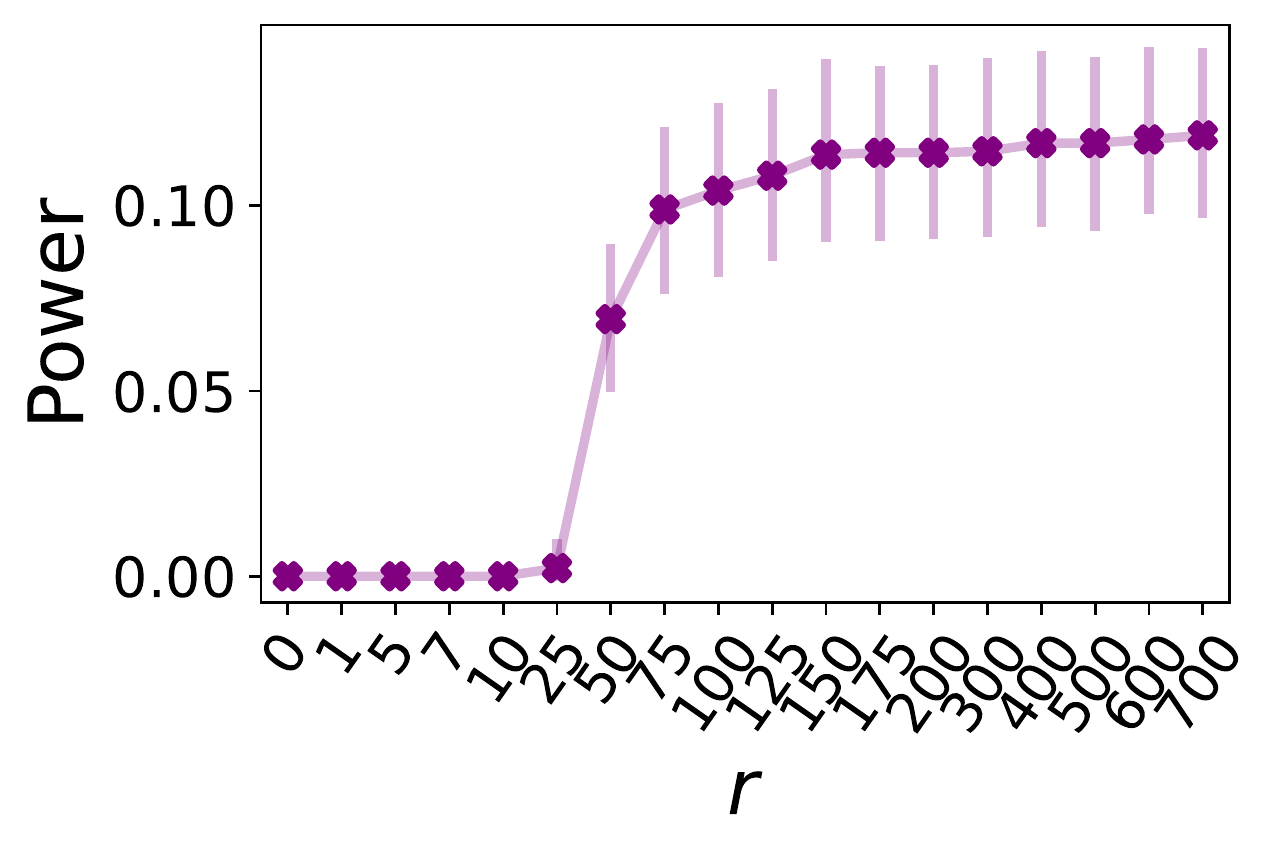}
    \caption{low power, $10\%$ outliers}
    \label{app-fig:soft_rank-r-OC-low-0.1}
  \end{subfigure}
  \begin{subfigure}[b]{0.3\textwidth}
    \includegraphics[width=\textwidth]{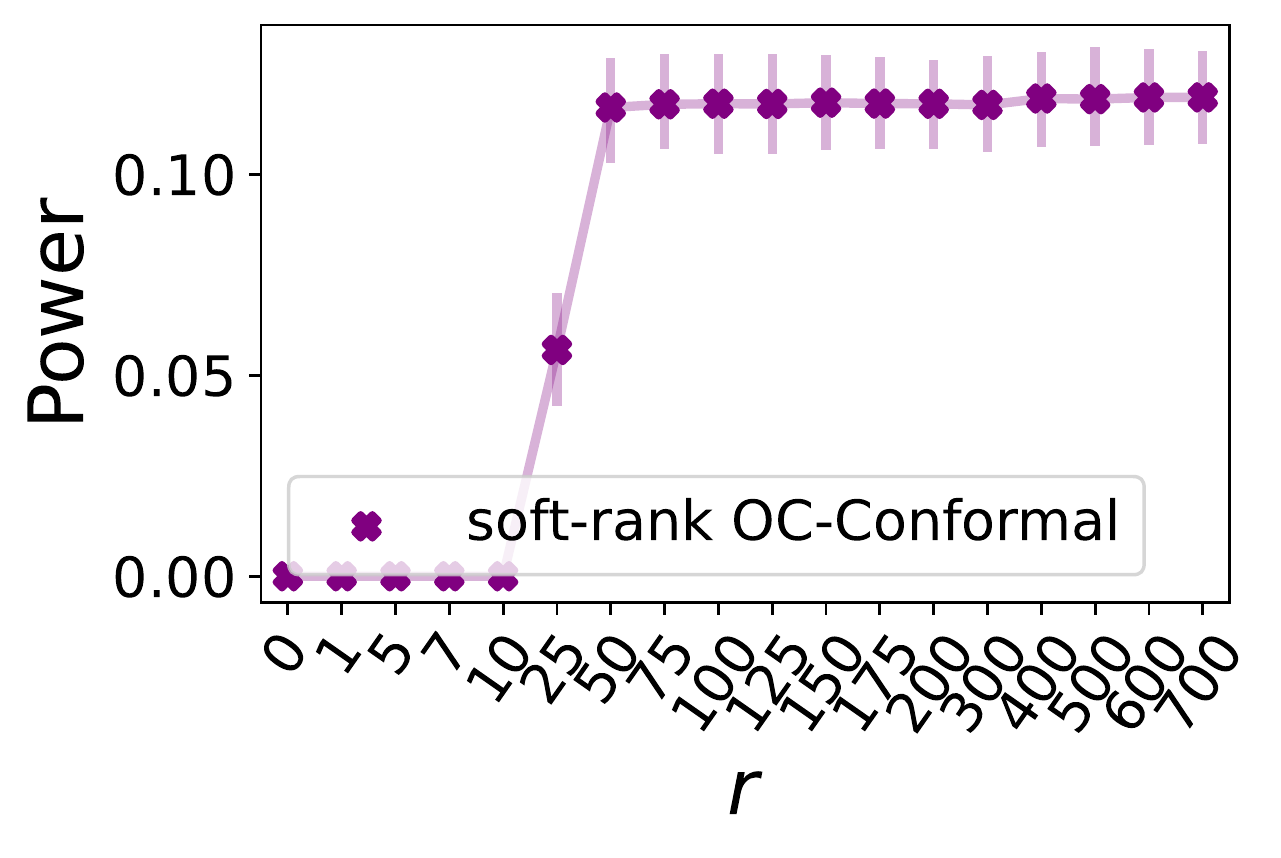}
    \caption{low power, $50\%$ outliers}
    \label{app-fig:soft_rank-r-OC-low-0.5}
  \end{subfigure}
  \caption{Performance on synthetic data of the proposed derandomized outlier detection method applied with soft-rank e-values, \texttt{soft-rank OC-Conformal}, applied with $K=10$ as a function of $r$ hyper-parameter. The method leverages a one-class support vector classifier. The results are averaged over 100 independent realizations of the data.
  Top: high-power regime with signal amplitude $3.6$ for $10\%$ outliers and $3.4$ for $50\%$ outliers. Bottom: low-power regime with signal amplitude $2.6$ for $10\%$ outliers and $2.3$ for $50\%$ outliers. Left: $10\%$ outliers in the test set. Right: $50\%$ outliers in the test set. Other details are as in Figure~\ref{fig:data_difficulty}.}
\label{app-fig:soft_rank-r-OC}
  \end{figure*}

 \FloatBarrier

\subsubsection{Comparing soft-rank e-values to the martingale-based e-values}

In striking contrast with the soft-rank e-values that are constructed separately for each test point, our martingale-based e-values are constructed jointly by looking at all test scores. Intuitively, by leveraging the additional information present in the test set, the martingale-based e-values may achieve higher power. Mathematically, recall that the soft-rank e-value is valid by construction, implying that $\E[e] \leq 1$ under the null hypothesis. By contrast, our martingale-based e-values satisfy a more relaxed property for which $\sum_{j\in \Hnull} \E \left[ {e}_j\right] \leq n_{\mathrm{test}}$. Consequently, in settings where the proportion of test outliers is large, each of the inlier e-values can exceed the value 1 as long as their sum is bounded by $n_{\mathrm{test}}$, in expectation. This can be attractive since we anticipate the non-null e-values that correspond to outlier points to have larger values than the null ones. Indeed, the following experiments indicate that the martingale-based approach tends to be more powerful than the soft-rank e-values when the proportion of outliers in the test set is relatively large. 

In more detail, we compare in Figure~\ref{app-fig:outliers-proportion} the soft-rank approach with our martingale-based method by varying the proportion of outliers present in the test set. That figure is obtained by adjusting the signal amplitude level such that the power of the randomized method (\texttt{AdaDetect}/\texttt{OC-Conformal}) is fixed at around 80\% for all the range of outlier proportions we studied. It is evident from that figure that the gap between the soft-rank e-value and our martingale-based e-value increases as the proportion of outliers increases, and that our proposal is more powerful than the soft-rank approach. For completeness, a comprehensive comparison considering various proportions of outliers as a function of the signal amplitude can be found in Figure~\ref{app-fig:signal-amplitude_LR} and Figure~\ref{app-fig:signal-amplitude_OC}. 

We also investigate the impact of the target FDR level on the performance of the soft-rank and our martingale-based methods. The performance metrics, shown in Figures~\ref{app-fig:target-FDR} and \ref{app-fig:target-FDR-OC}, reveal that our method is more powerful, displaying greater adaptability to the FDR level. This aligns with our expectations, as our hyper-parameter $\alpha_{\mathrm{bh}}$ is set proportionally to the target FDR level $\alpha$. By contrast, the soft-rank hyper-parameter $r$ remains fixed across different target FDR levels; it is unclear how to refine the choice of this parameter in this setting, or even to conclude whether the best choice of $r$ is affected by the target FDR level.

\begin{figure}[!htb]
    \centering
    \begin{subfigure}[b]{0.48\textwidth}
    \includegraphics[width=0.47\textwidth]{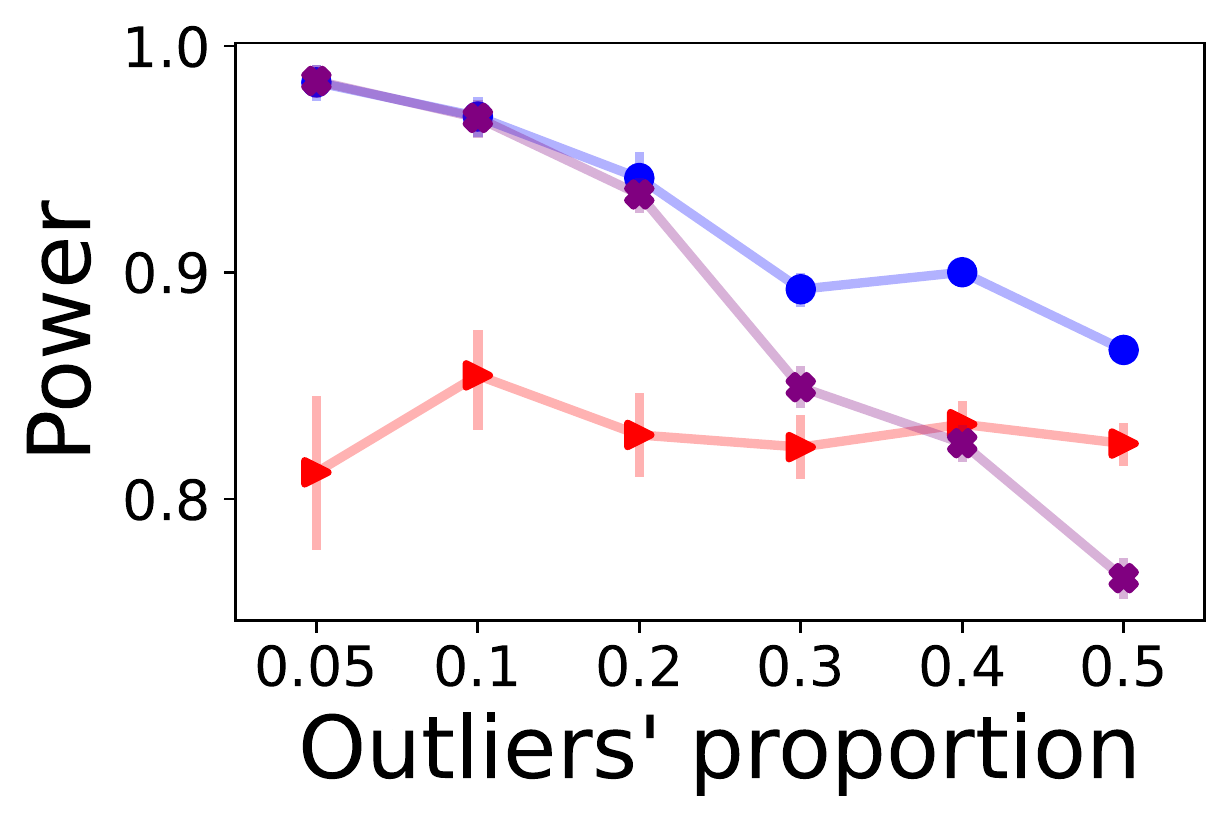}
    \includegraphics[width=0.47\textwidth]{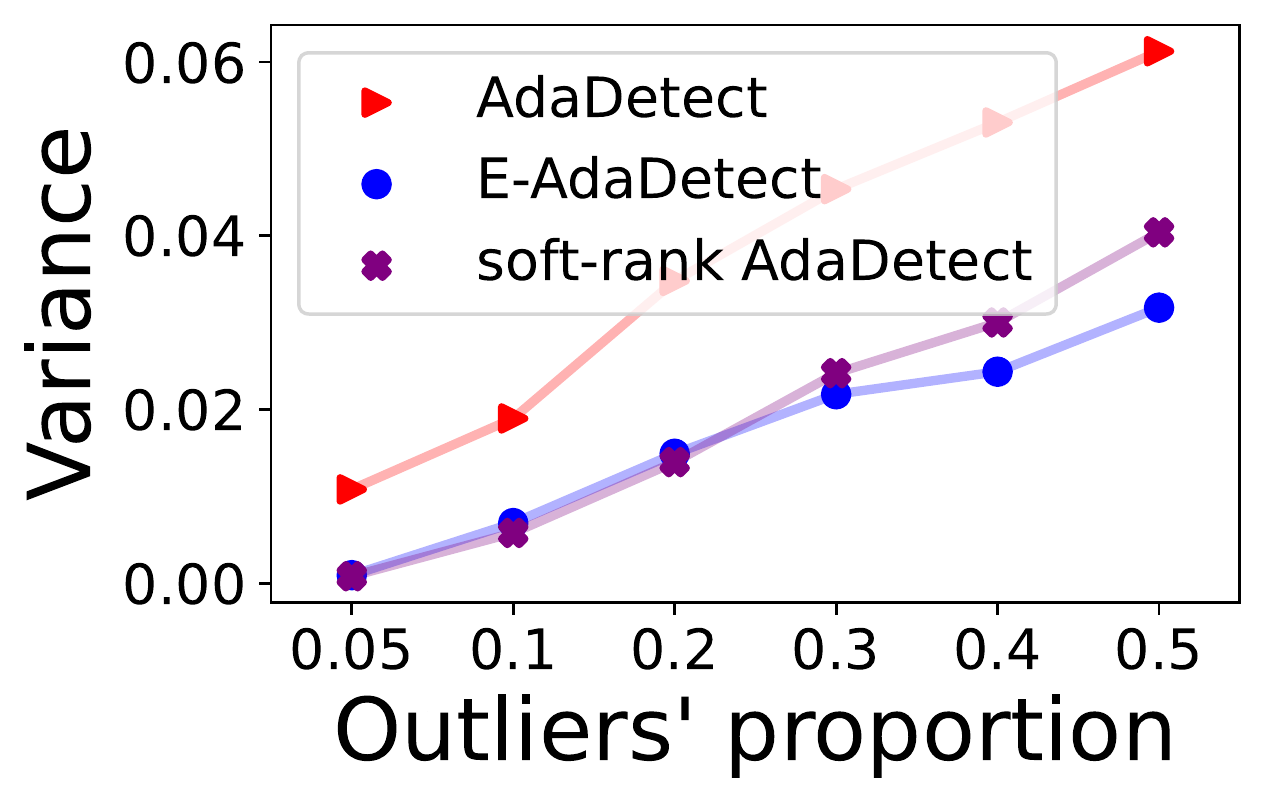}
    \caption{\texttt{AdaDetect}}
    \label{fig:outliers-proportion-adadetect}
    \end{subfigure}
    \begin{subfigure}[b]{0.48\textwidth}
    \includegraphics[width=0.47\textwidth]{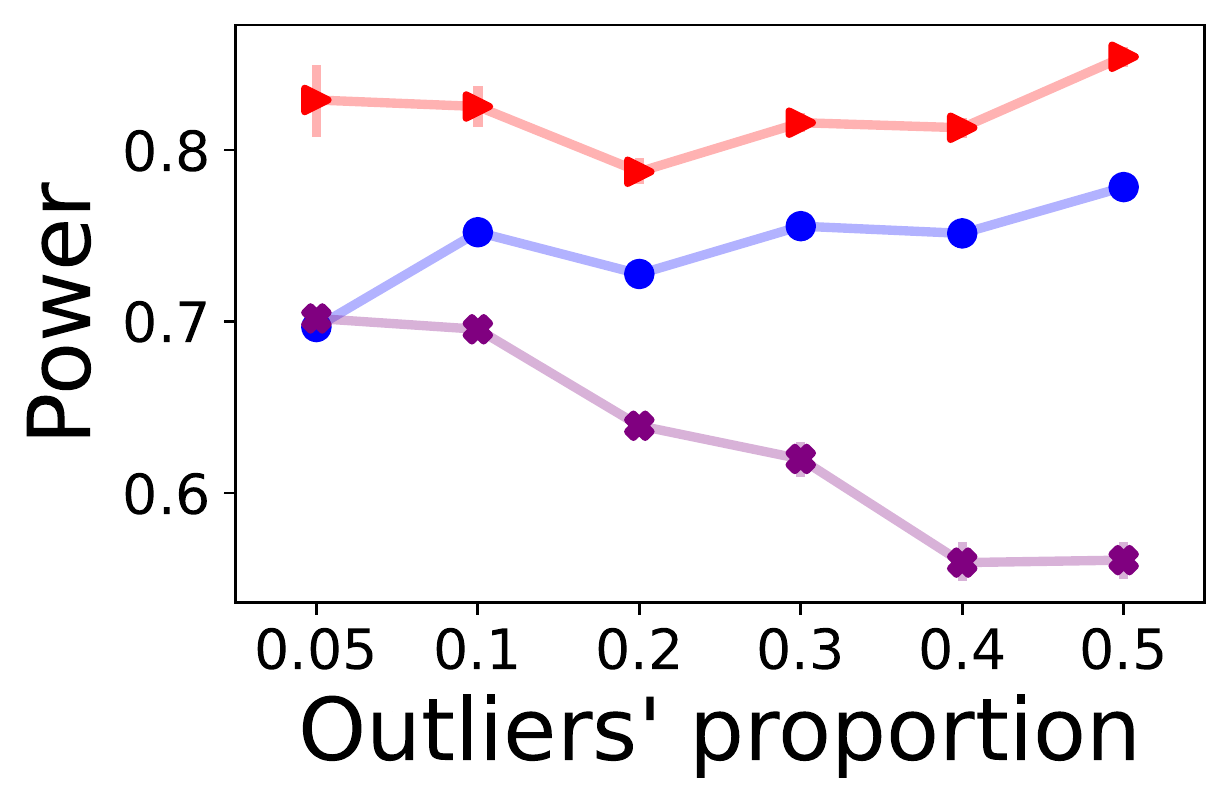}
    \includegraphics[width=0.5\textwidth]{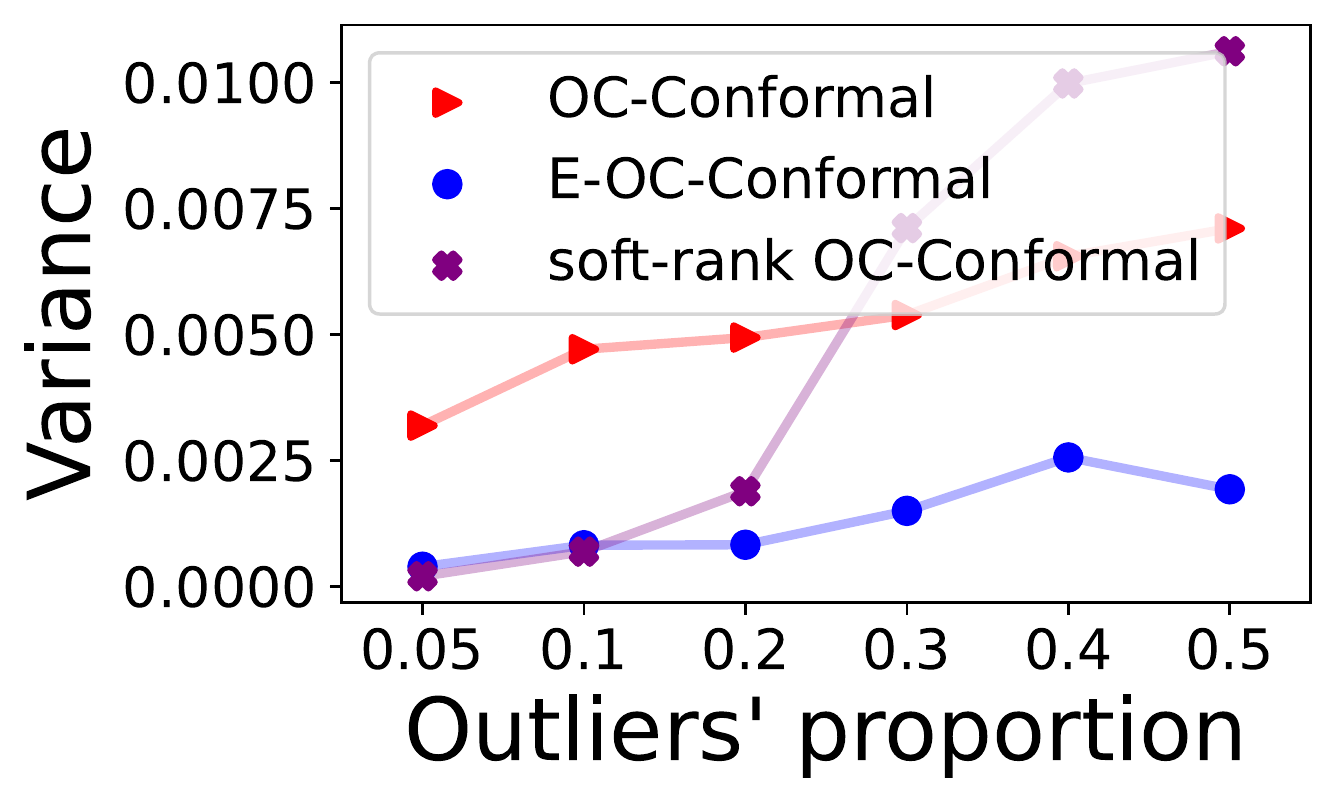}
    \caption{\texttt{OC-Conformal}}
    \label{fig:outliers-proportion-oc-conformal}
    \end{subfigure}
    \caption{Performance on synthetic data of the proposed derandomized outlier detection method, \texttt{E-AdaDetect} (\texttt{E-OC-Conformal}), applied with $K=10$, compared to that of its randomized benchmark, \texttt{AdaDetect} (\texttt{OC-Conformal}). We also compare these methods to the soft-rank method, applied with $K=10$, \texttt{soft-rank AdaDetect} (\texttt{soft-rank OC-Conformal}), as a function of the proportion of outliers in the test set with the corresponding signal strength that results in a stable strength of the randomized benchmarks.}
    \label{app-fig:outliers-proportion}
\end{figure}

\begin{figure}[!htb]
    \centering
    \begin{subfigure}[b]{0.32\textwidth}
\includegraphics[width=\textwidth]{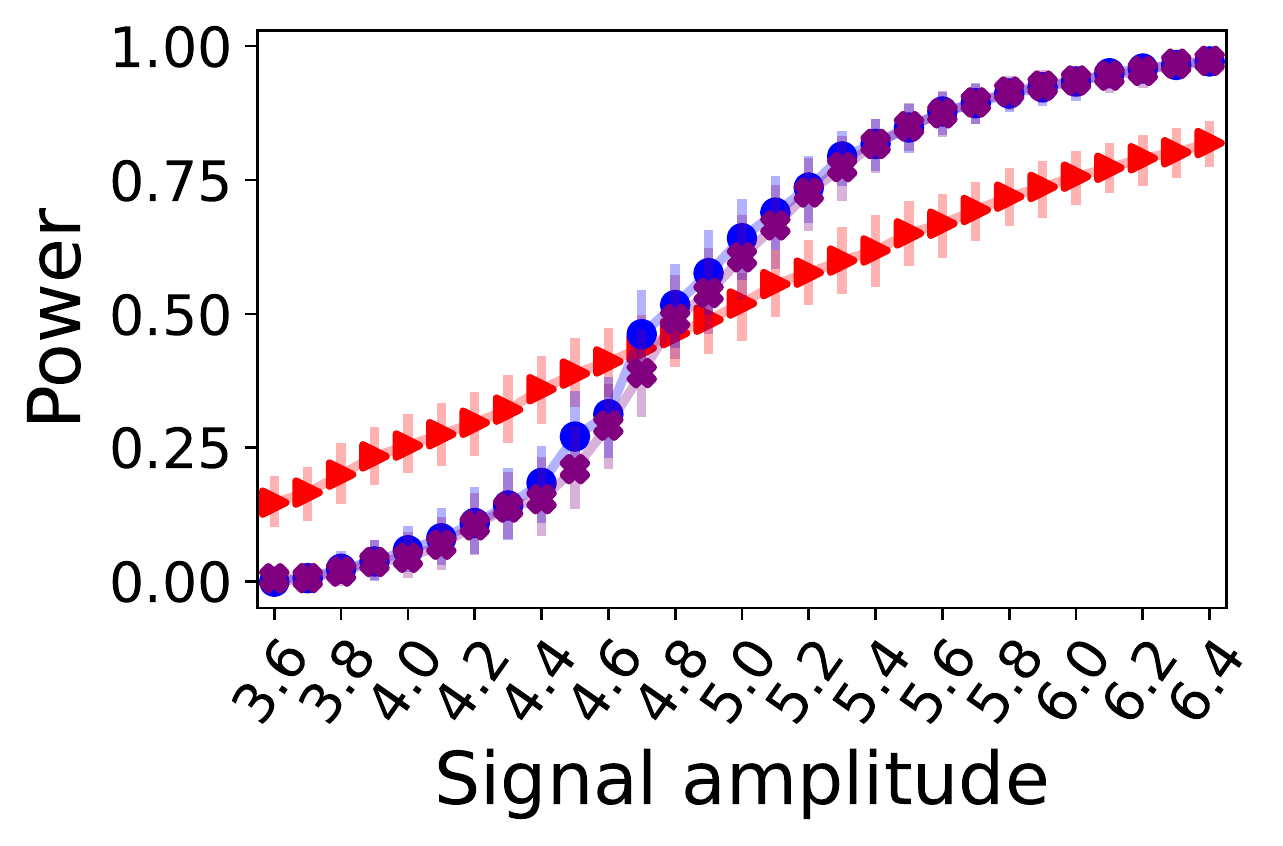}
    \caption{$5\%$ outliers}
\end{subfigure}
    \begin{subfigure}[b]{0.32\textwidth}
\includegraphics[width=\textwidth]{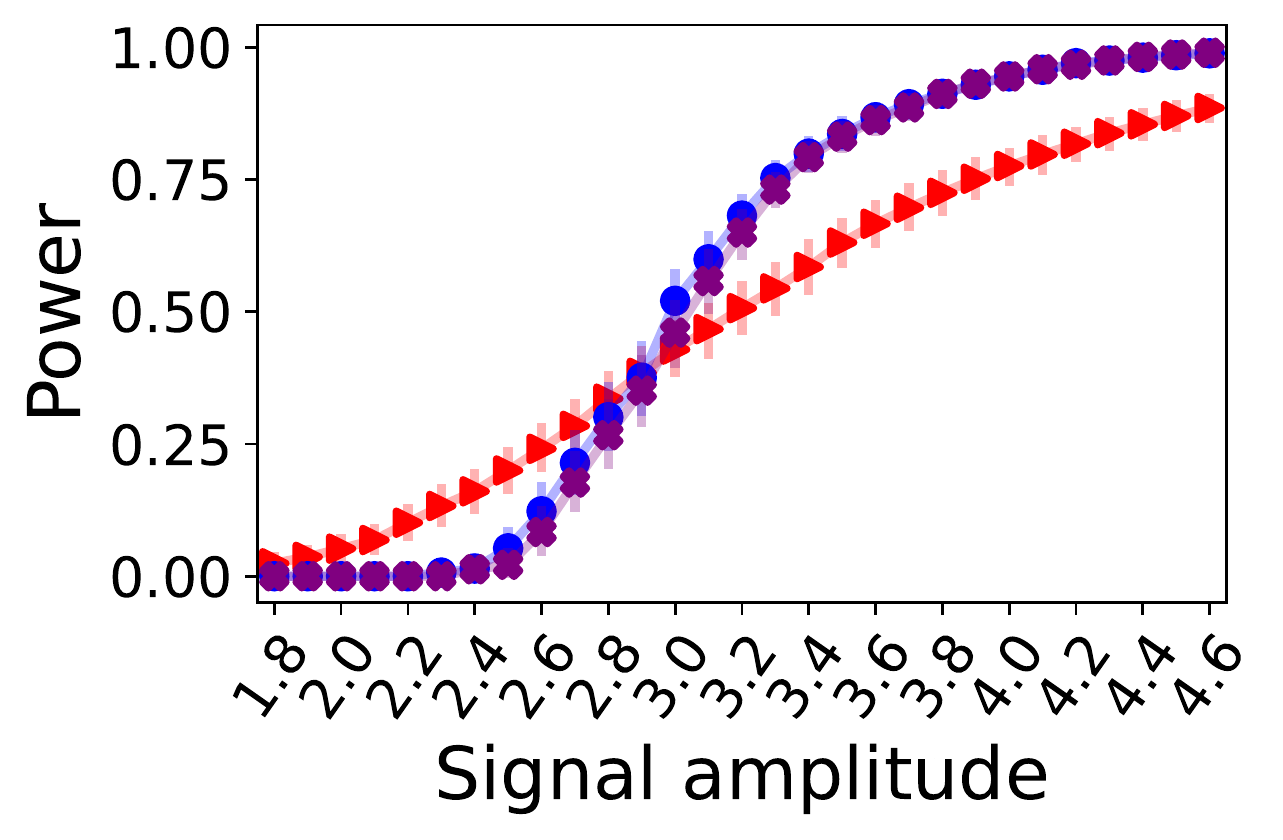}
    \caption{$10\%$ outliers}
\end{subfigure}
    \begin{subfigure}[b]{0.32\textwidth}
\includegraphics[width=\textwidth]{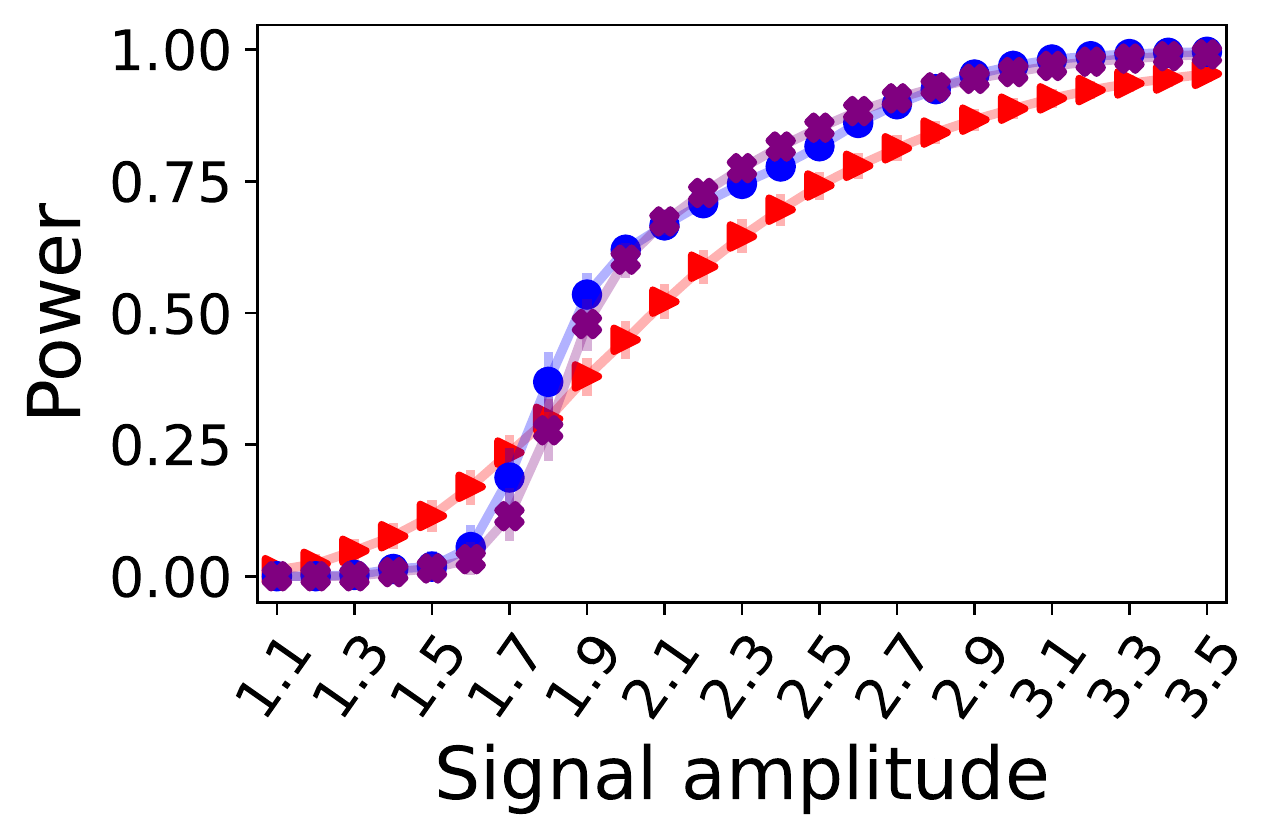}
    \caption{$20\%$ outliers}
\end{subfigure}
    \begin{subfigure}[b]{0.32\textwidth}
\includegraphics[width=\textwidth]{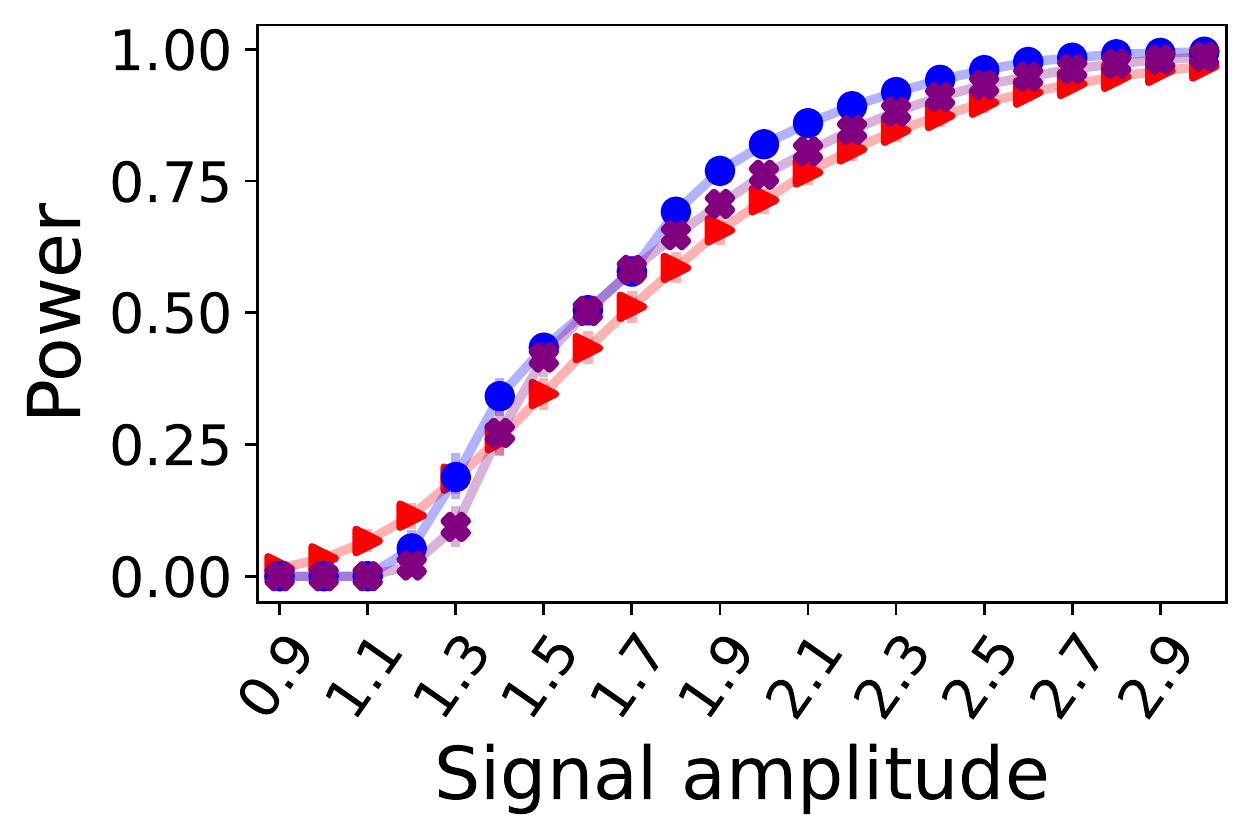}
    \caption{$30\%$ outliers}
\end{subfigure}
    \begin{subfigure}[b]{0.32\textwidth}
\includegraphics[width=\textwidth]{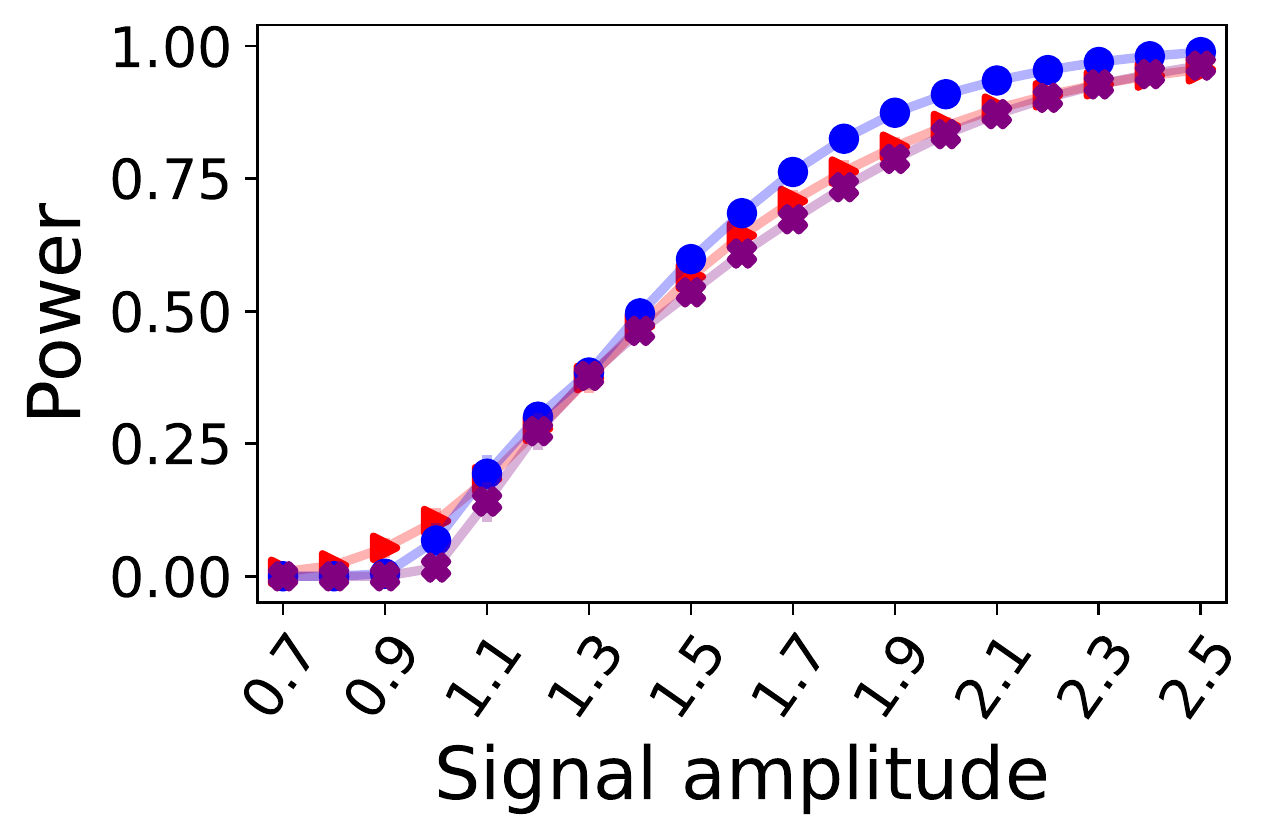}
    \caption{$40\%$ outliers}
\end{subfigure}
    \begin{subfigure}[b]{0.32\textwidth}
\includegraphics[width=\textwidth]{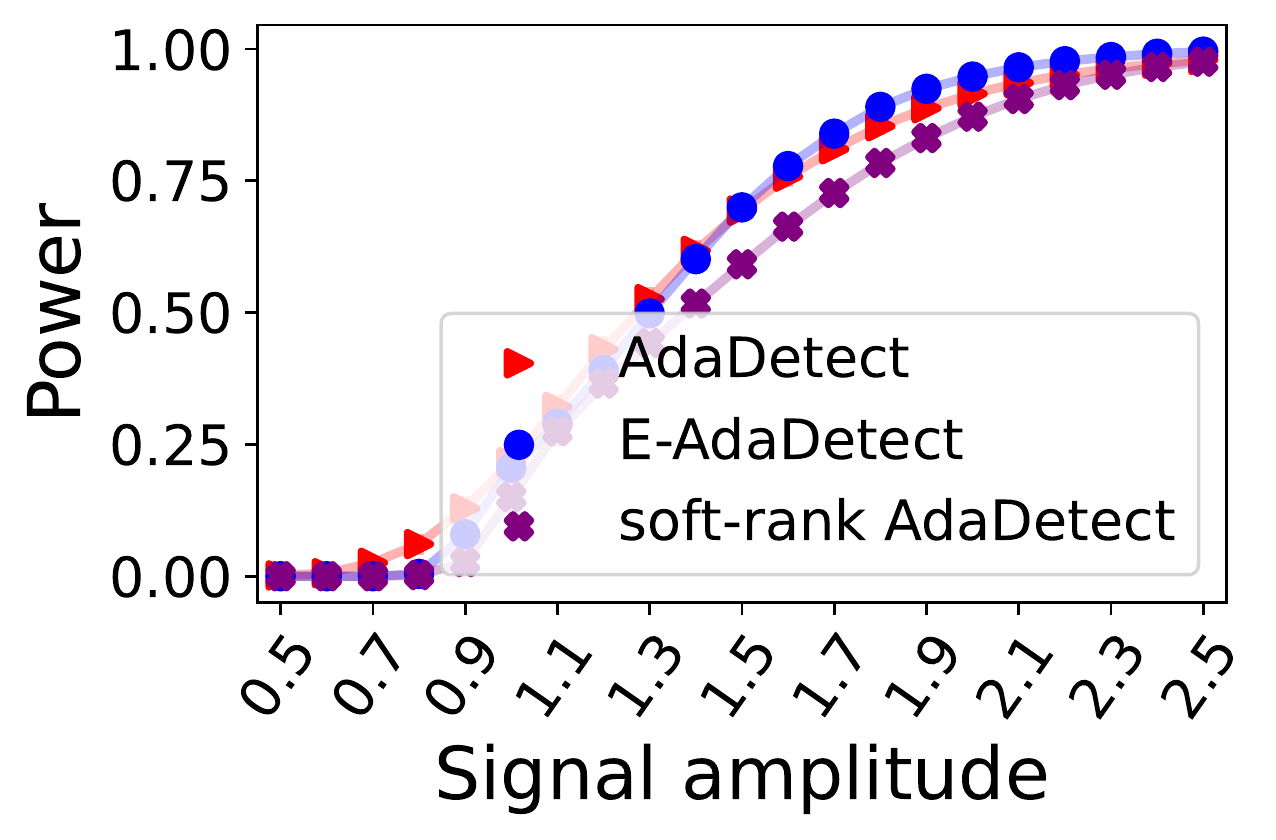}
    \caption{$50\%$ outliers}
\end{subfigure}
\caption{Performance on synthetic data of the proposed derandomized outlier detection method, \texttt{E-AdaDetect}, applied with $K=10$, compared to that of its randomized benchmark, \texttt{AdaDetect}. We compare these methods to the soft-rank method, applied with $K=10$, \texttt{soft-rank AdaDetect}, as a function of the signal strength for varying proportions of outliers in the test set. The results are averaged over 100 independent realizations of the data. Other results are as in Figure~\ref{fig:data_difficulty}.}
    \label{app-fig:signal-amplitude_LR}
\end{figure}

\begin{figure}[!htb]
    \centering
    \begin{subfigure}[b]{0.32\textwidth}
\includegraphics[width=\textwidth]{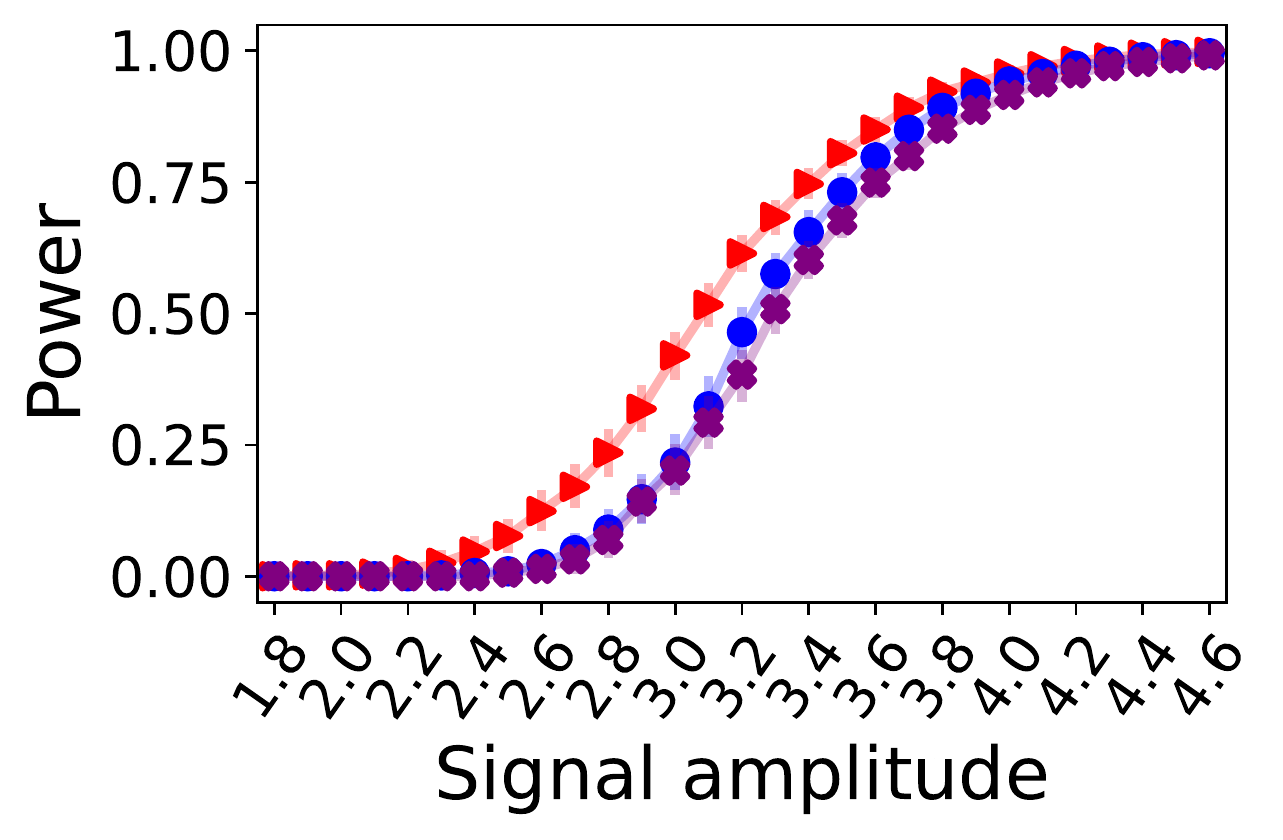}
    \caption{$5\%$ outliers}
\end{subfigure}
    \begin{subfigure}[b]{0.32\textwidth}
\includegraphics[width=\textwidth]{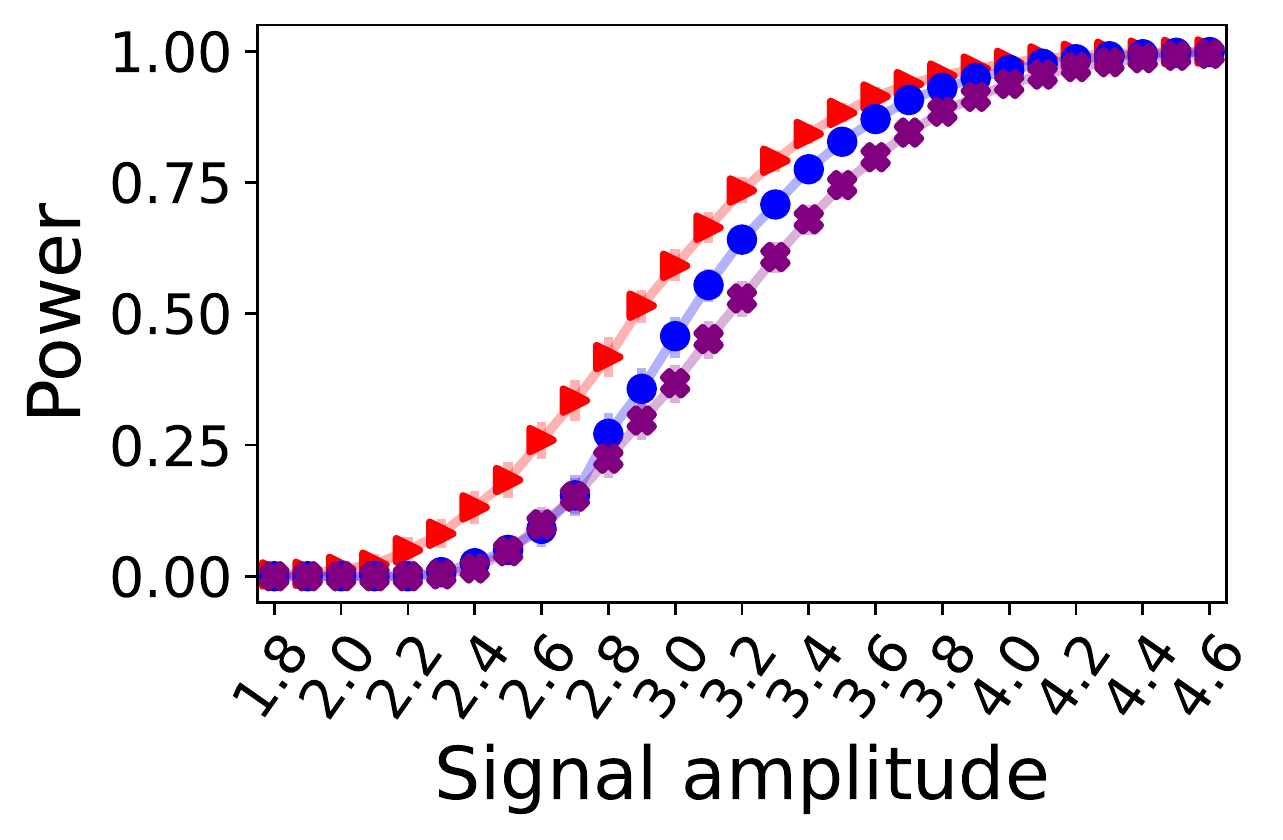}
    \caption{$10\%$ outliers}
\end{subfigure}
    \begin{subfigure}[b]{0.32\textwidth}
\includegraphics[width=\textwidth]{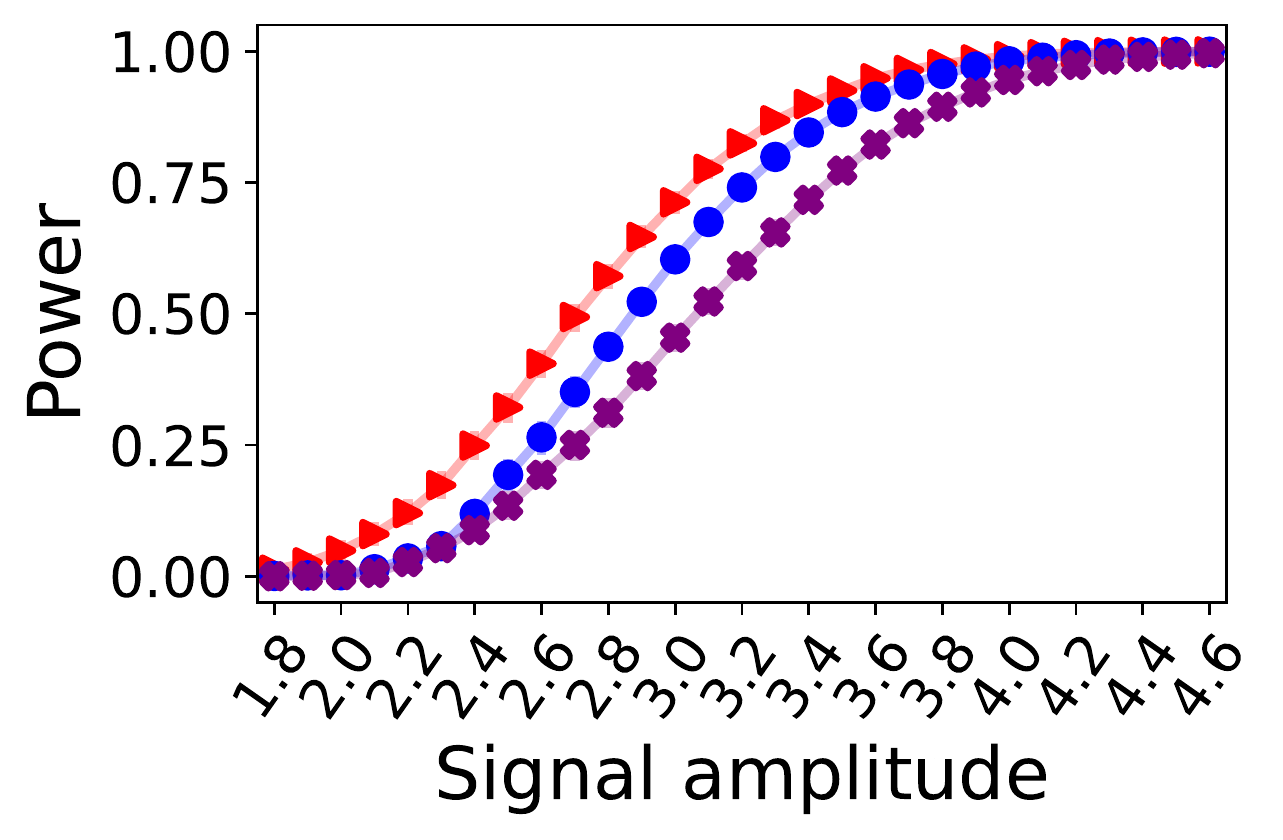}
    \caption{$20\%$ outliers}
\end{subfigure}
    \begin{subfigure}[b]{0.32\textwidth}
\includegraphics[width=\textwidth]{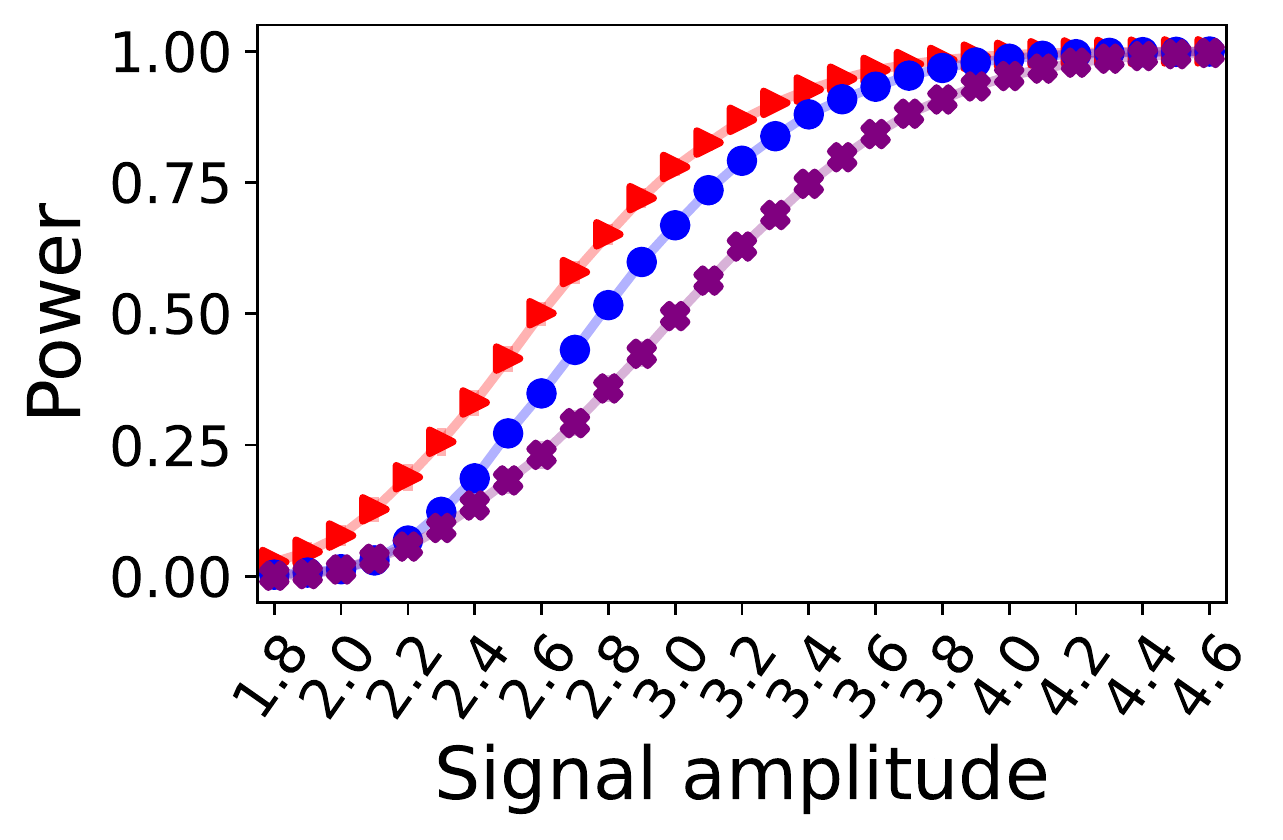}
    \caption{$30\%$ outliers}
\end{subfigure}
    \begin{subfigure}[b]{0.32\textwidth}
\includegraphics[width=\textwidth]{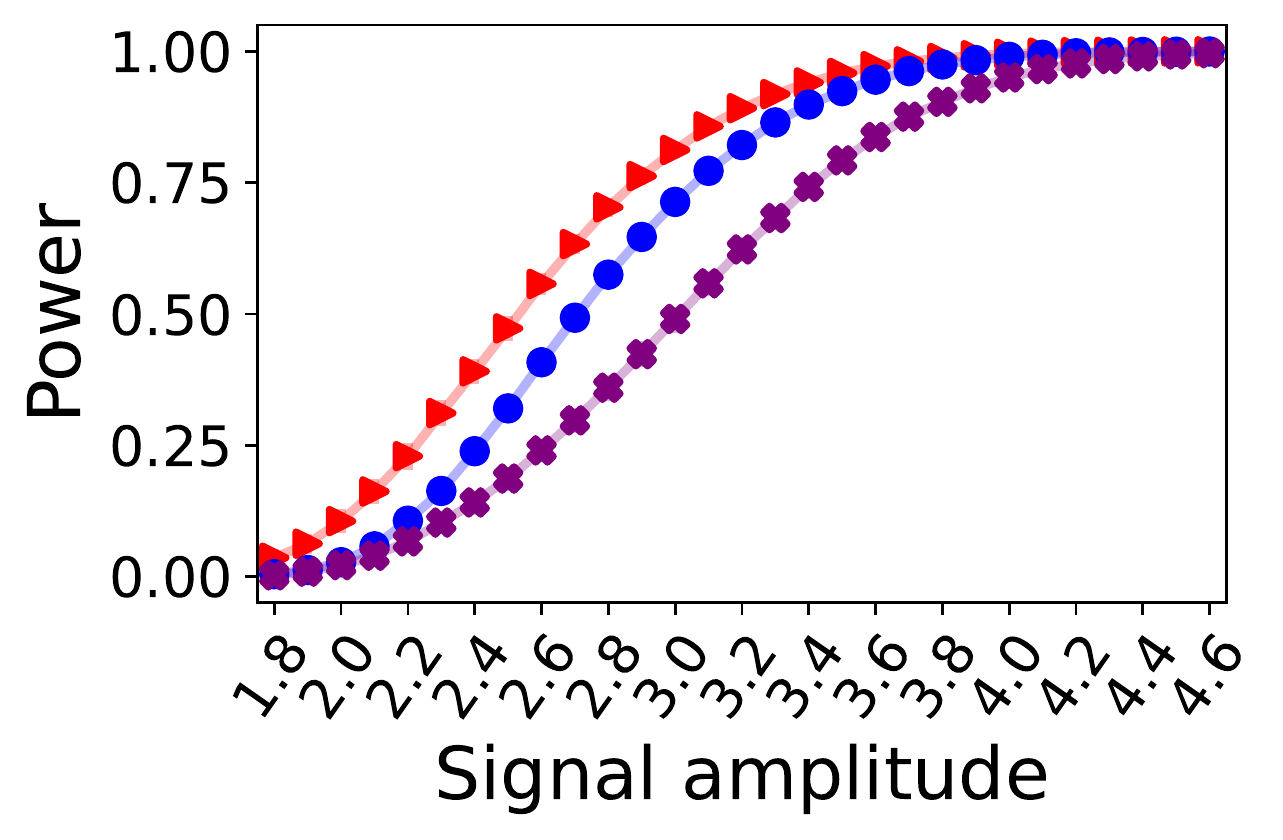}
    \caption{$40\%$ outliers}
\end{subfigure}
    \begin{subfigure}[b]{0.32\textwidth}
\includegraphics[width=\textwidth]{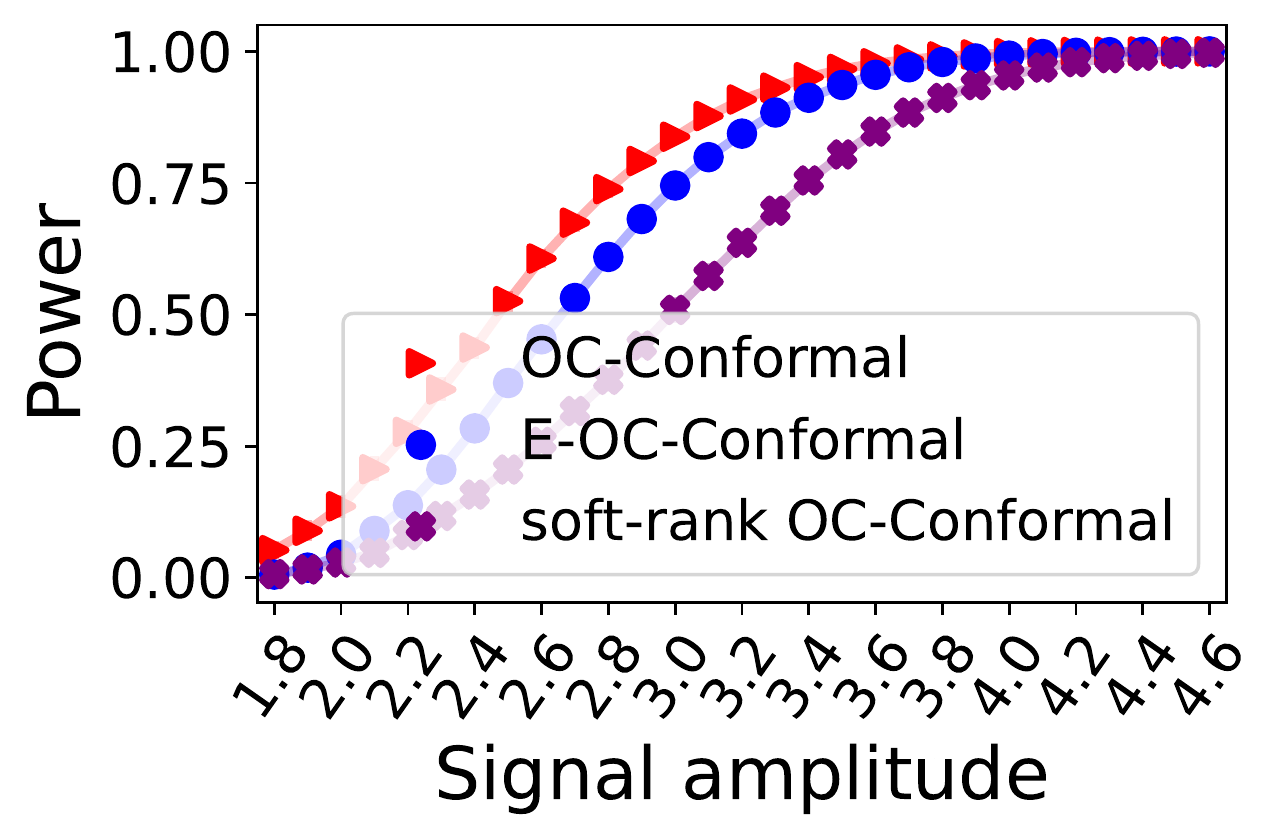}
    \caption{$50\%$ outliers}
\end{subfigure}
\caption{Performance on synthetic data of the proposed derandomized outlier detection method, \texttt{E-OC-Conformal}, applied with $K=10$, compared to that of its randomized benchmark, \texttt{OC-Conformal}. These methods are also compared to the soft-rank method, applied with $K=10$, \texttt{soft-rank OC-Conformal}, as a function of the signal strength for varying proportion of outliers in the test set. The results are averaged over 100 independent realizations of the data. Other results are as in Figure~\ref{fig-app:synth-classic-FDR-Power-data_difficulty_ocsvm}.}
    \label{app-fig:signal-amplitude_OC}
\end{figure}

\begin{figure}[!htb]
    \centering
    \begin{subfigure}[b]{\textwidth}
    \includegraphics[width=0.32\textwidth]{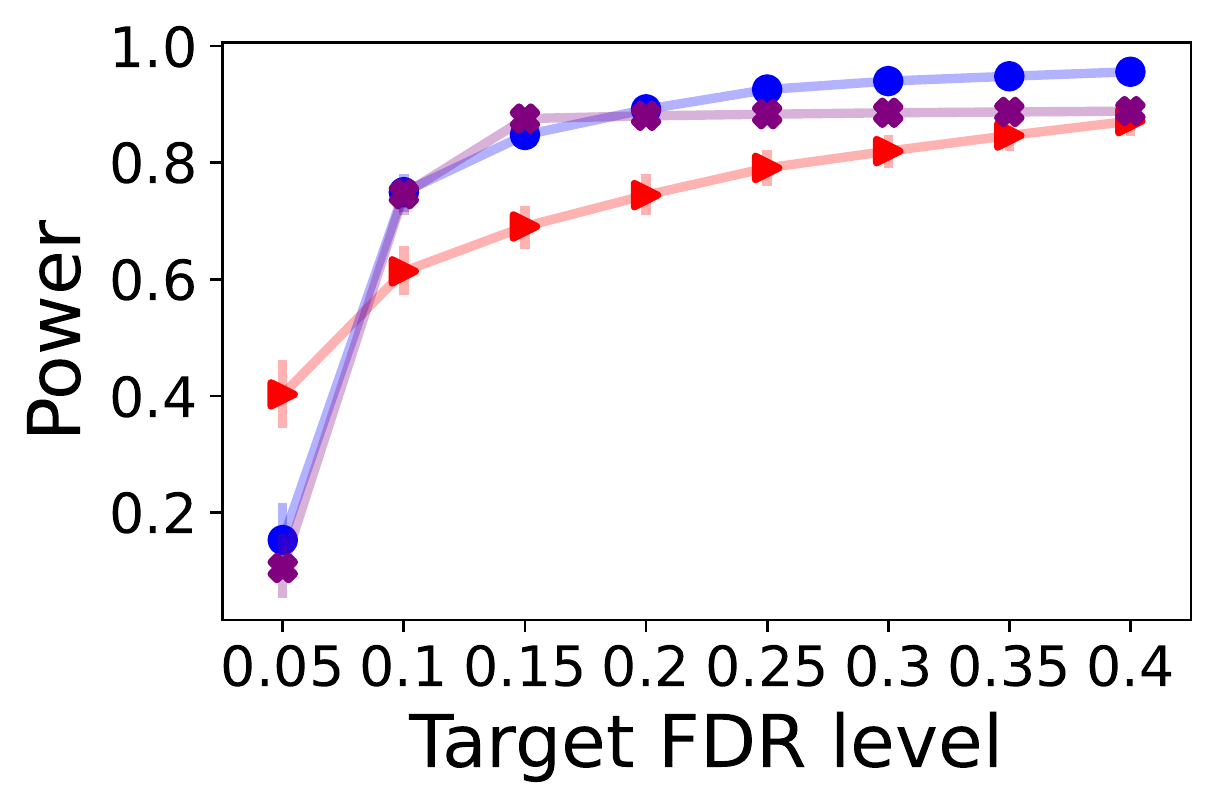}
    \includegraphics[width=0.32\textwidth]{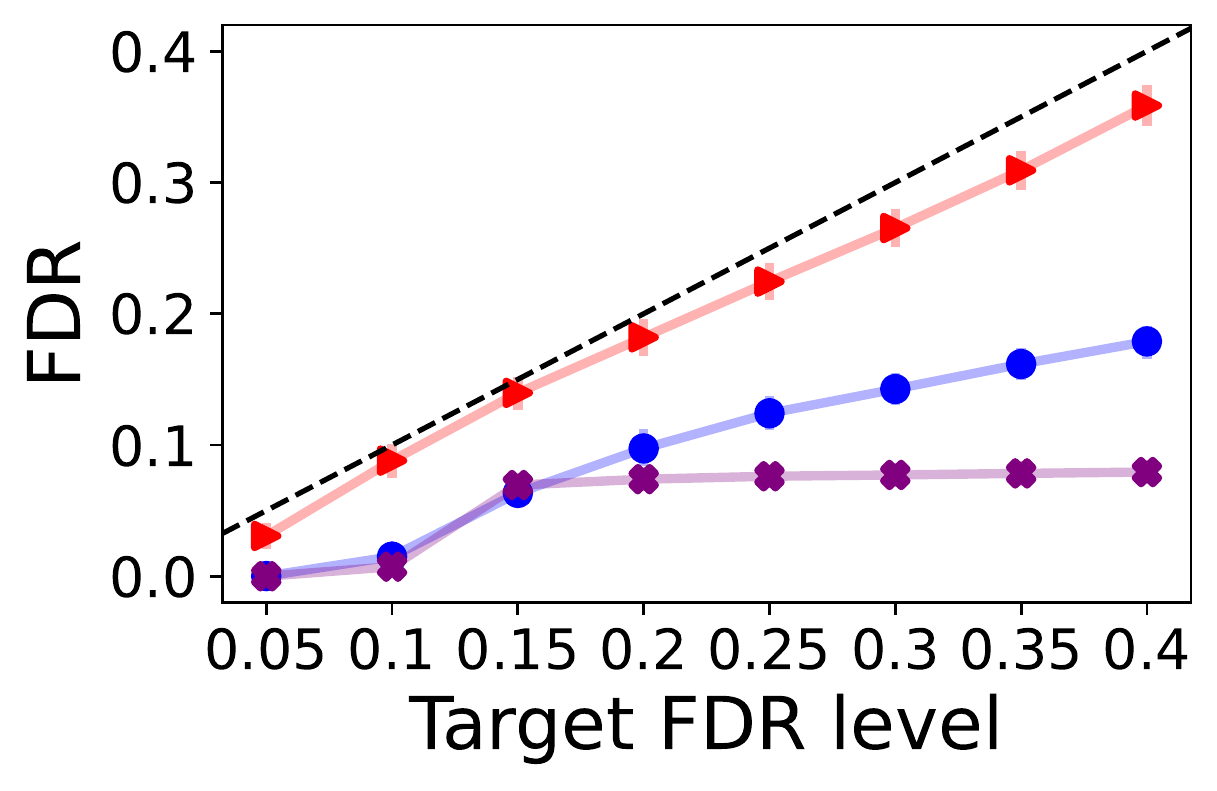}
    \includegraphics[width=0.32\textwidth]{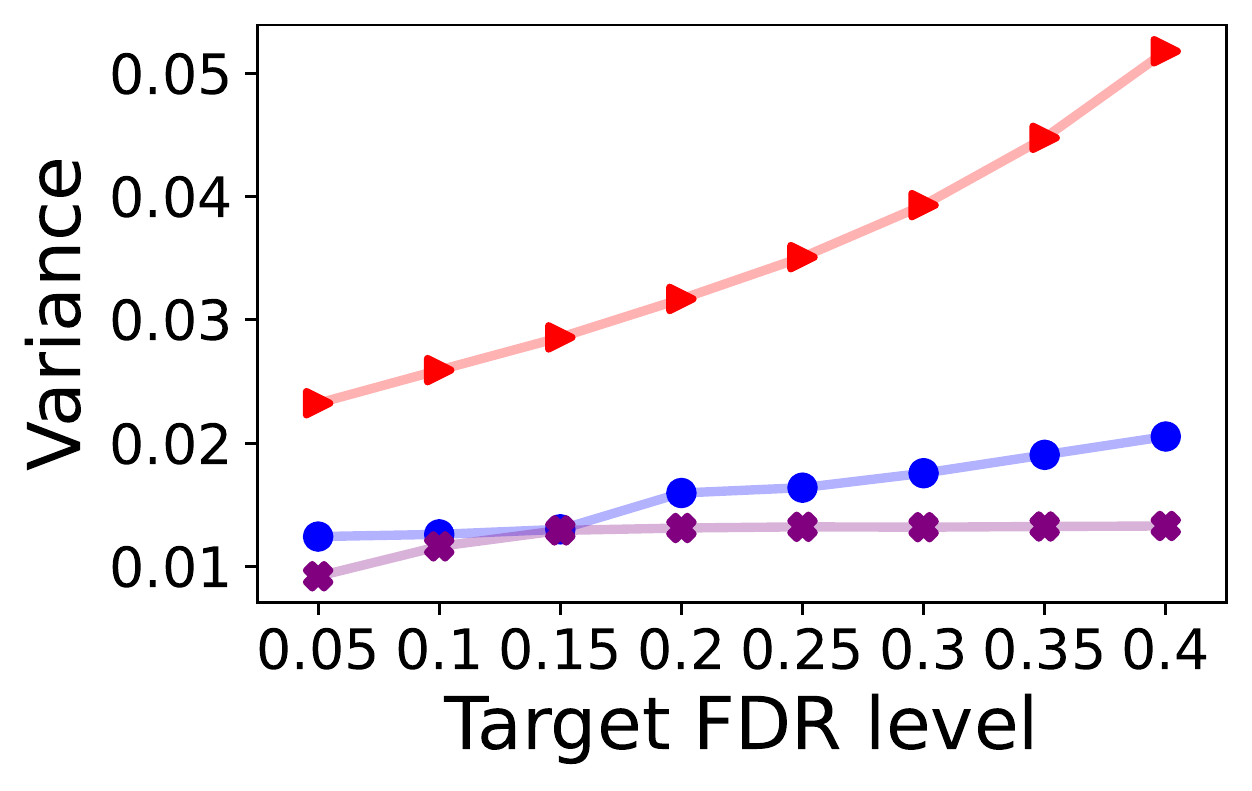}
    \caption{$10\%$ outliers}
    \label{app-fig:target-FDR-outliers-0.1}
    \end{subfigure}
    \\
    \begin{subfigure}[b]{\textwidth}
    \includegraphics[width=0.32\textwidth]{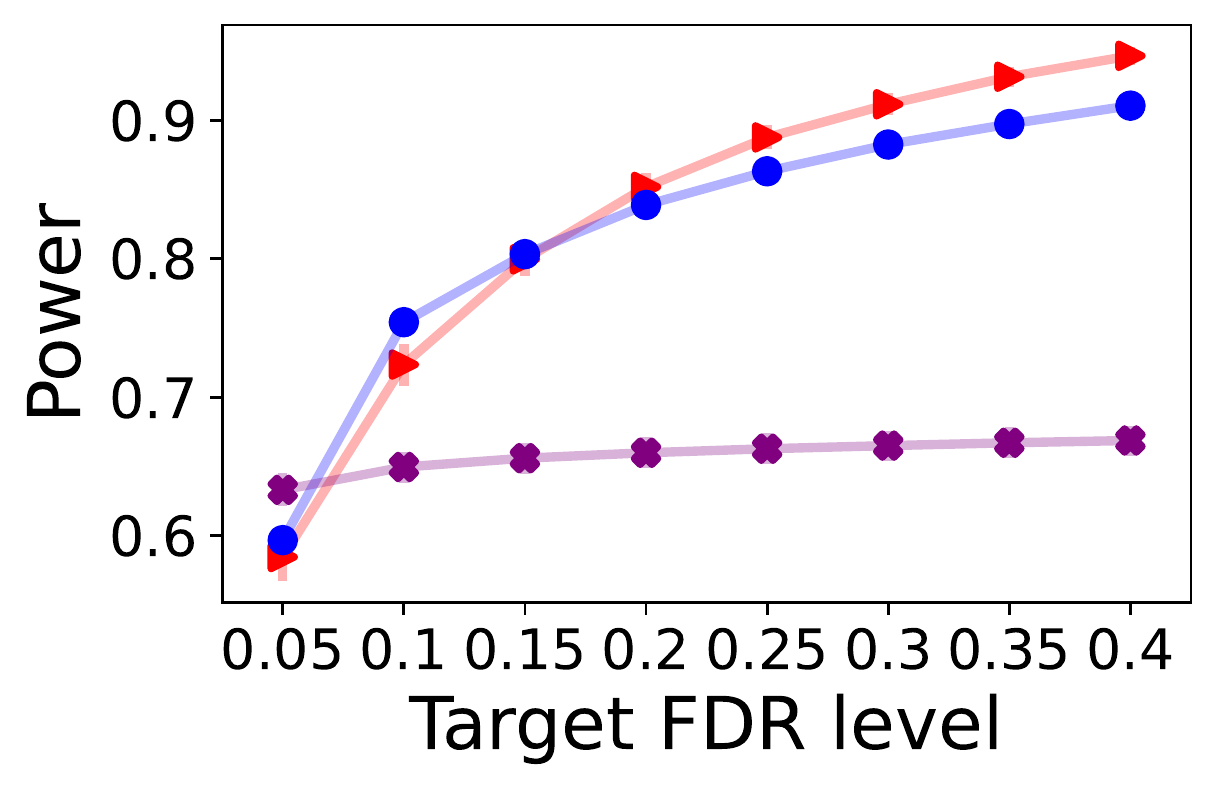}
    \includegraphics[width=0.32\textwidth]{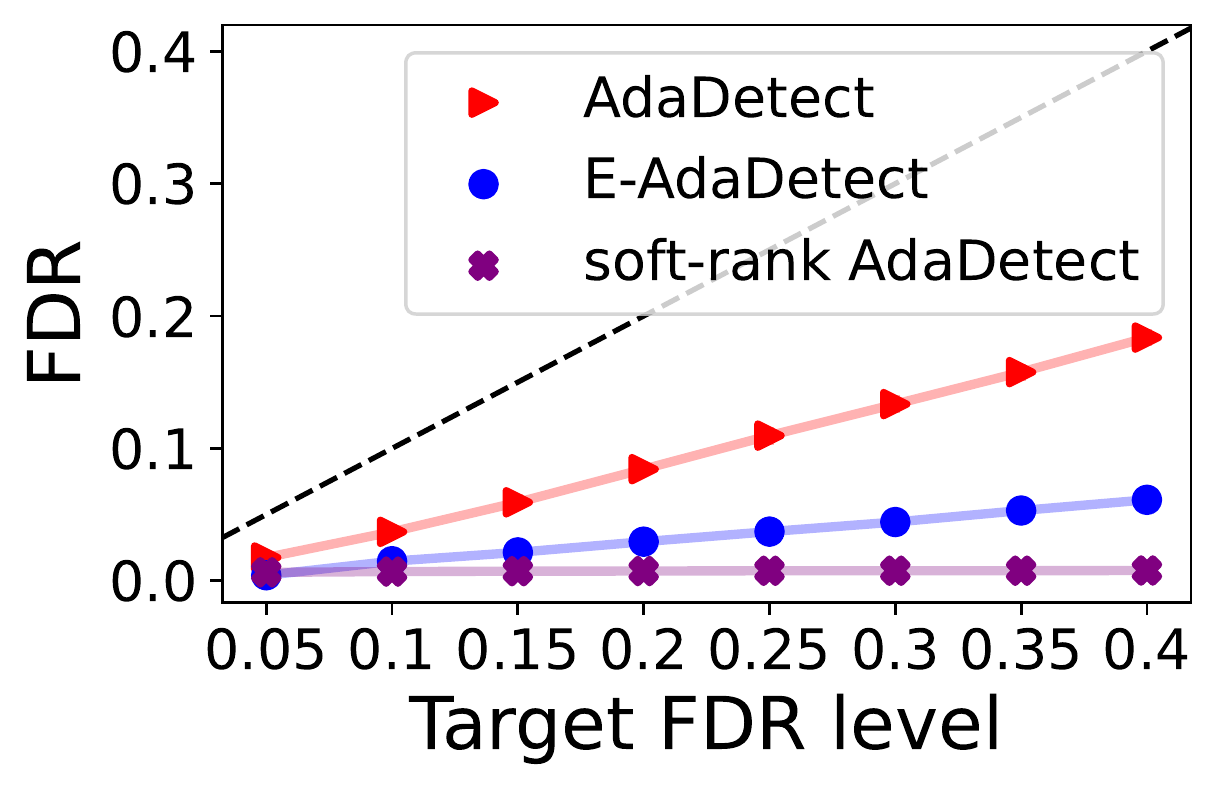}
    \includegraphics[width=0.32\textwidth]{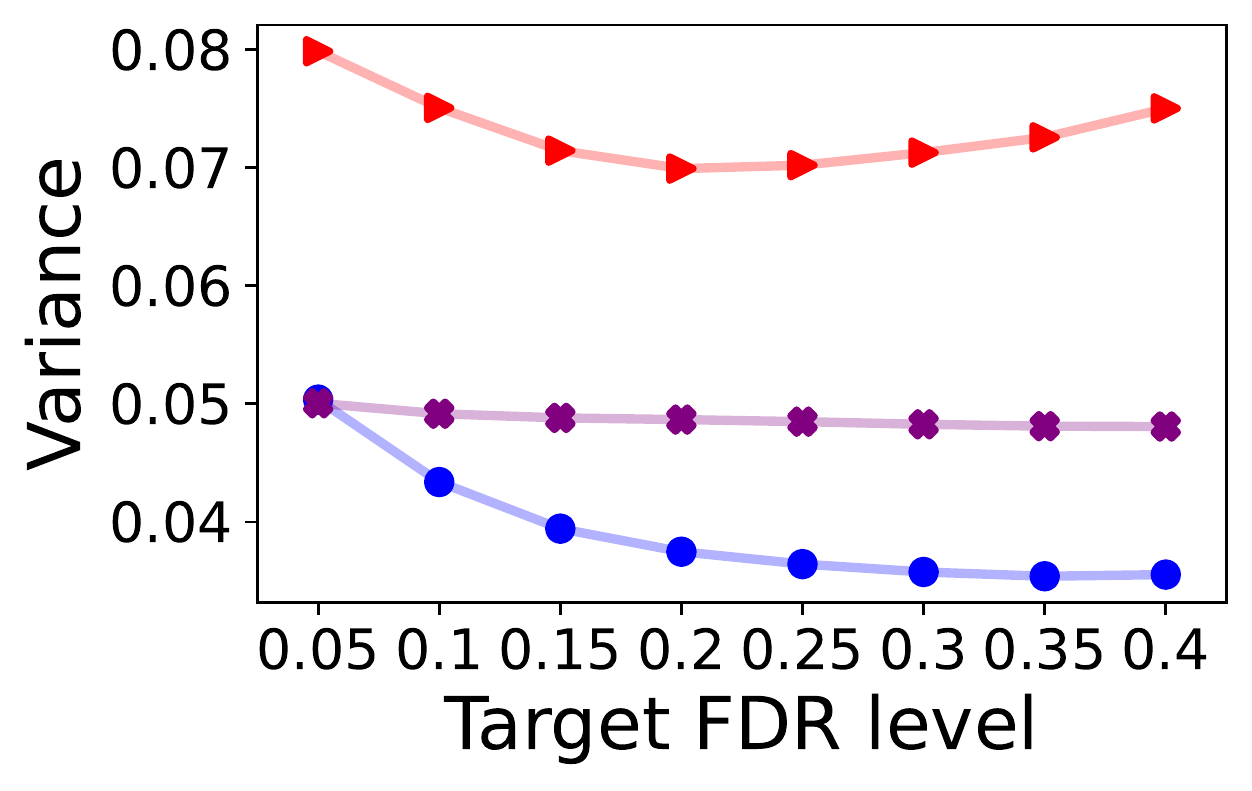}
    \caption{$50\%$ outliers}
    \label{app-fig:target-FDR-outliers-0.5}
    \end{subfigure}
    \caption{Performance on synthetic data of the proposed derandomized outlier detection method, \texttt{E-AdaDetect}, applied with $K=10$, compared to that of its randomized benchmark, \texttt{AdaDetect}. These methods are also compared to the soft-rank method, applied with $K=10$, \texttt{soft-rank AdaDetect}, as a function of the target FDR level. All methods leverage a logistic regression binary classifier.
  \ref{app-fig:target-FDR-outliers-0.1} presents the performance in high-power regime with signal amplitude $3.4$ when there are $10\%$ outliers. \ref{app-fig:target-FDR-outliers-0.5} presents the performance in high-power regime with signal amplitude $1.6$ when there are $50\%$ outliers. The dashed line indicates the nominal false discovery rate level. Note that these results correspond to 100 repeated experiments based on a single realization of the labeled and test data, hence why the results appear a little noisy.}
    \label{app-fig:target-FDR}
\end{figure}

\begin{figure}[!htb]
    \centering
    \begin{subfigure}[b]{\textwidth}
    \includegraphics[width=0.32\textwidth]{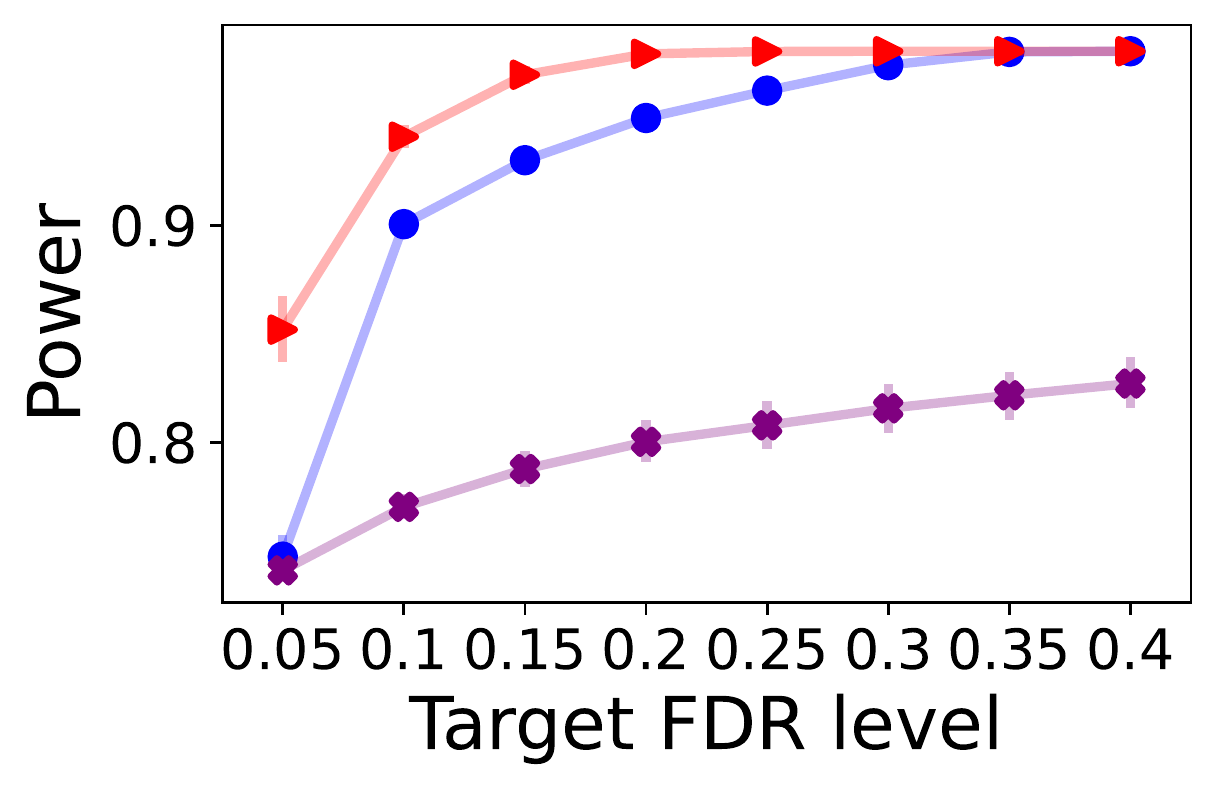}
    \includegraphics[width=0.32\textwidth]{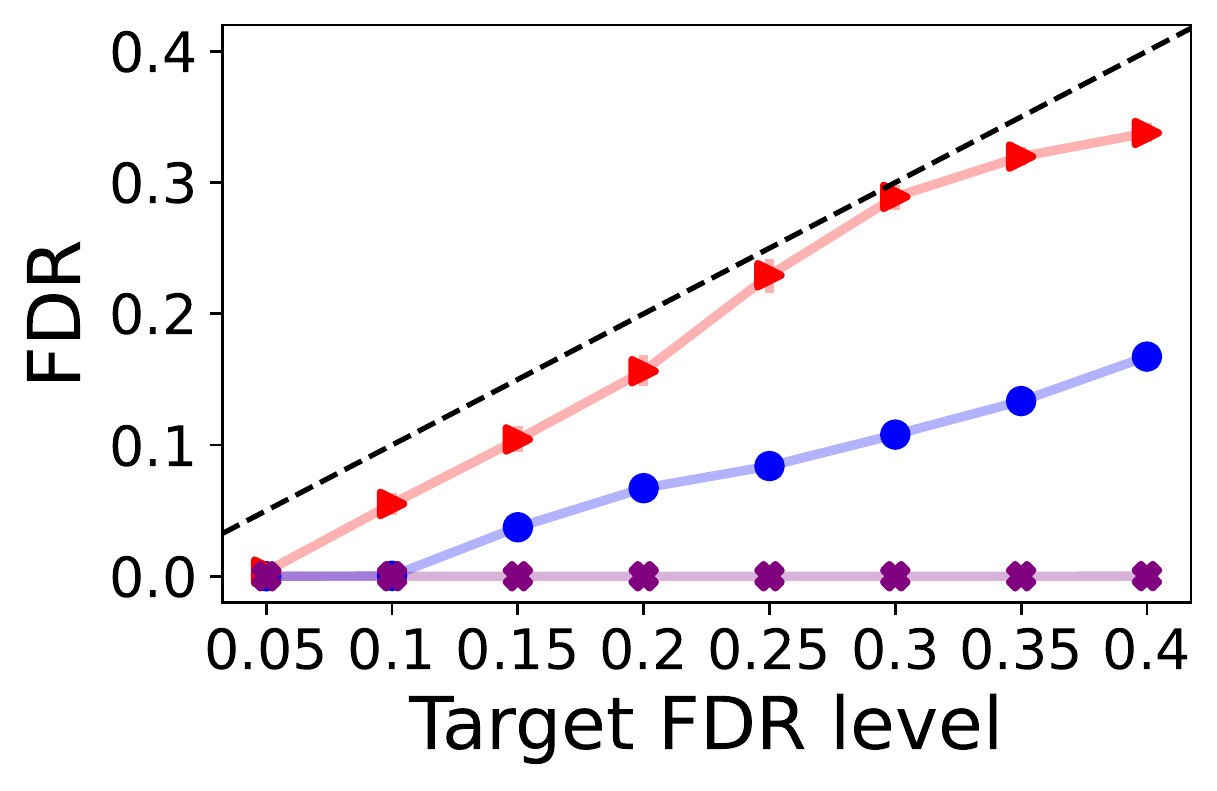}
    \includegraphics[width=0.32\textwidth]{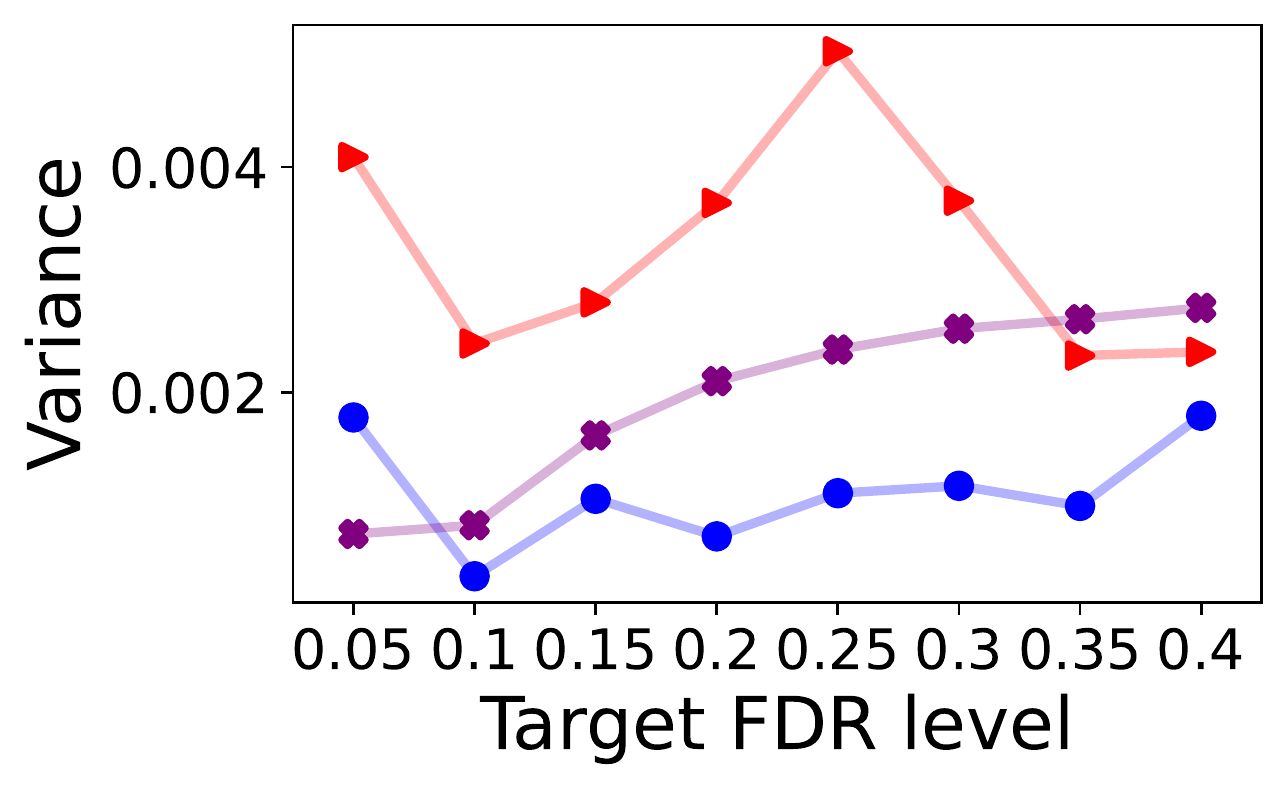}
    \caption{$10\%$ outliers}
    \label{app-fig:target-FDR-outliers-0.1-OC}
    \end{subfigure}
    \\
    \begin{subfigure}[b]{\textwidth}
    \includegraphics[width=0.32\textwidth]{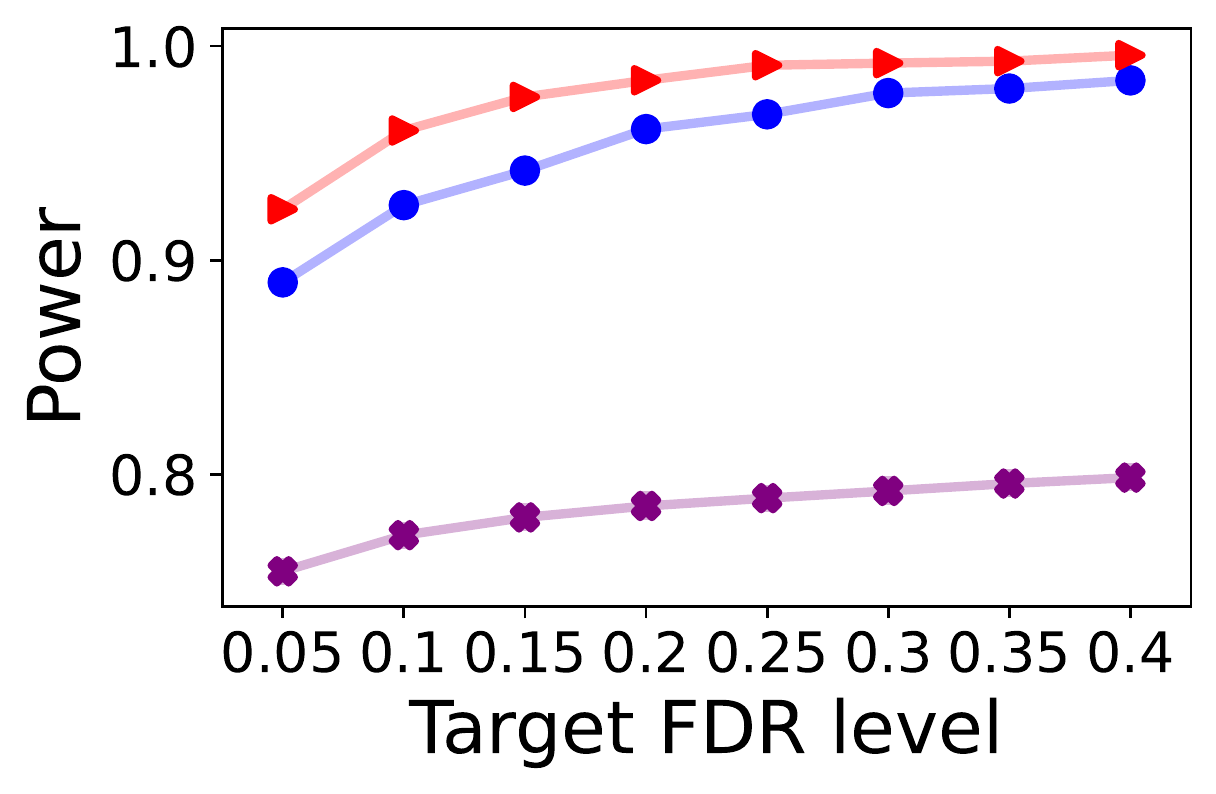}
    \includegraphics[width=0.32\textwidth]{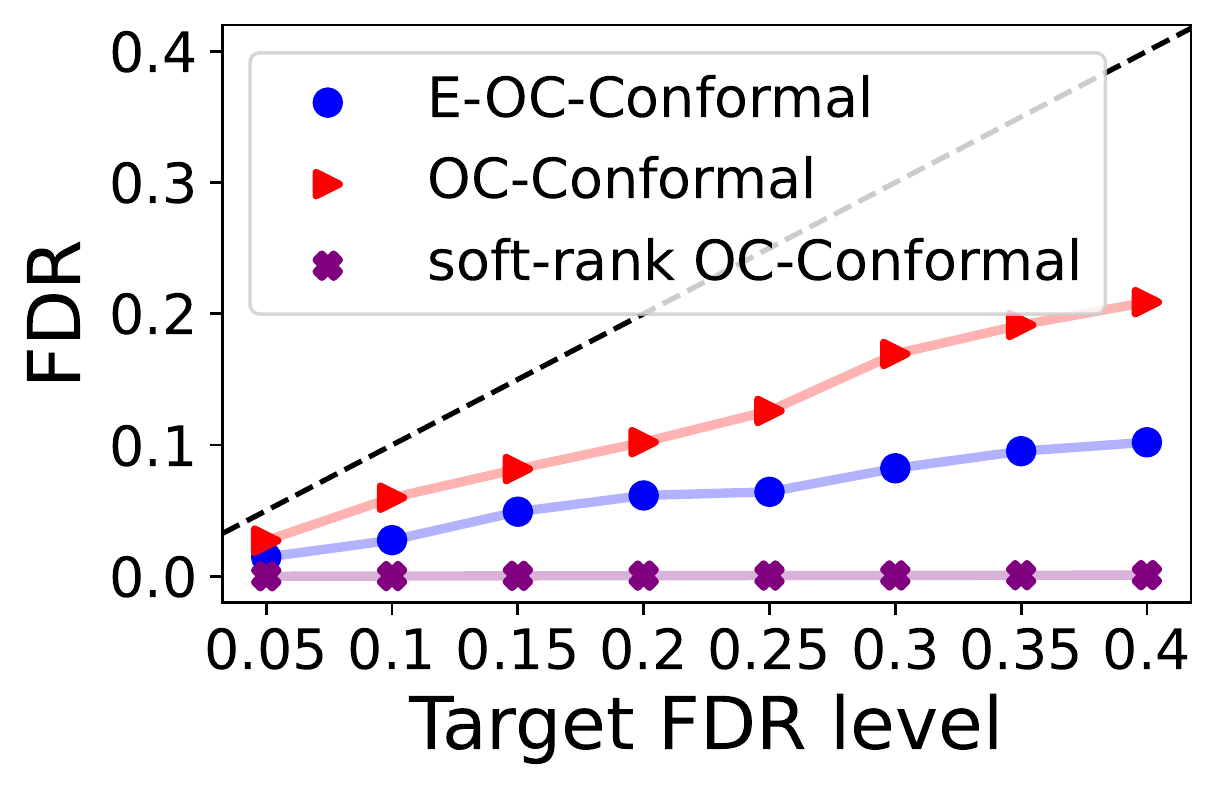}
    \includegraphics[width=0.32\textwidth]{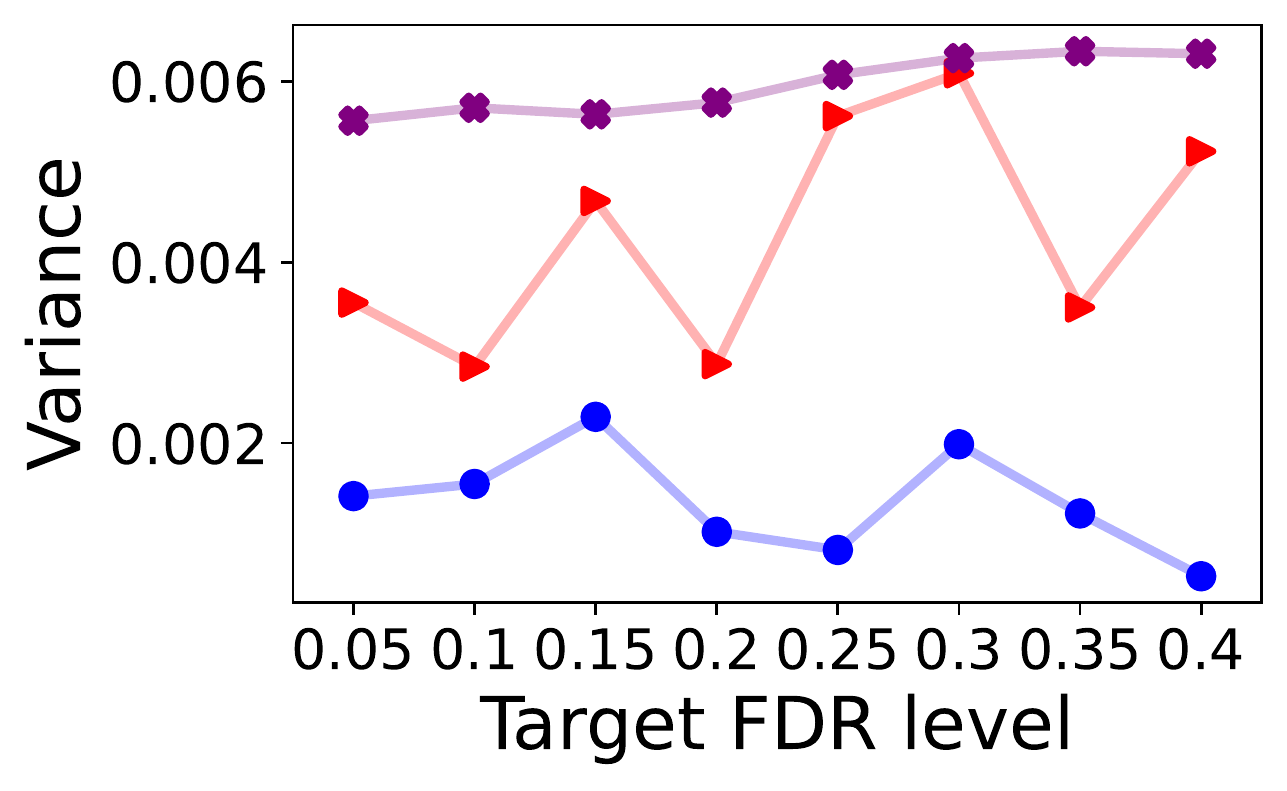}
    \caption{$50\%$ outliers}
    \label{app-fig:target-FDR-outliers-0.5-OC}
    \end{subfigure}
    \caption{Performance on synthetic data of the proposed derandomized outlier detection method, \texttt{E-OC-Conformal}, applied with $K=10$, compared to that of its randomized benchmark, \texttt{OC-Conformal}. These methods are also compared to the soft-rank method, applied with $K=10$, \texttt{soft-rank OC-Conformal}, as a function of the target FDR level. All methods leverage a one-class support vector classifier.
  \ref{app-fig:target-FDR-outliers-0.1-OC} presents the performance in high-power regime with signal amplitude $3.6$ when there are $10\%$ outliers. \ref{app-fig:target-FDR-outliers-0.5-OC} presents the performance in high-power regime with signal amplitude $3.4$ when there are $50\%$ outliers. The dashed line indicates the nominal false discovery rate level. Note that these results correspond to 100 repeated experiments based on a single realization of the labeled and test data, hence why the results appear a little noisy.}
    \label{app-fig:target-FDR-OC}
\end{figure}

\section{Experiments with real data} \label{app:real-experiments}

\subsection{Derandomized AdaDetect}
\label{app:real-experiments-adadetect}

In this section, we evaluate the performance of our method on several benchmark data sets for outlier detection, also studied in \citet{conformal-p-values} and \citet{ml-fdr}: 
% {\em mammography} \cite{mammography}, 
{\em musk} \citep{musk}, {\em shuttle} \citep{shuttle}, {\em KDDCup99} \citep{KDDCup99}, and {\em credit card} \citep{creditcard}.
We refer to \citet{conformal-p-values} and \citet{ml-fdr} for more details about these data sets. 
Similarly to Section~\ref{sec:synthetic}, we construct a reference set and a test set through random sub-sampling. 
The reference set contains 3000 inliers, and the test set contains 1000 samples, of which we control the proportion of outliers. 
We apply our derandomization procedure using the proposed martingale-based e-values and soft-rank e-values implemented in combination with \texttt{AdaDetect}, 
using $K=10$ independent splits of the reference set into training and calibration subsets of size 2000 and 1000, respectively.
All the methods are repeatedly applied to carry out 100 independent analyses of the same data.
In the regime that the test set contains 10\% outliers, we can see from Figure~\ref{app-fig:real-data-outliers-0.1-RF} that all methods control the average proportion of false discoveries below $\alpha=0.1$ and achieve similar power, but the findings obtained with the derandomized methods are far more stable. By contrast, when increasing the proportion of outliers to $40\%$ (Figure~\ref{app-fig:real-data-outliers-0.4-RF}) we can see that the martingale-based approach tends to be more powerful than the soft-rank method. 
Finally, Figure~\ref{app-fig:real-data-outliers-proportion-AdaDetect-RF} confirms the reproducibility of these results by reporting the average FDR and power over 100 independent realizations of the sub-sampled data considered in Figure \ref{app-fig:real-data-soft_rank-RF}; 
these performance metrics are presented as a function of the outlier proportion.

\subsection{Derandomized One-Class Conformal} \label{app:real-experiments-OC}
We turn to study the effect of our approach on \texttt{OC-Conformal} on the same real data sets, by following the experimental protocol from Section~\ref{app:real-experiments-adadetect}.
In general, we observe that \texttt{OC-Conformal} is less powerful and less stable than \texttt{AdaDetect} on the studied data sets, and therefore we increase the number of analyses of our randomization procedure to $K=70$.

Figure~\ref{app-fig:real-data-soft_rank-IF} indicates our martingale-based derandomization procedure indeed reduces the algorithmic variability while controlling the false discovery proportion. These results also demonstrate the trade-off between stability and power: the selections of \texttt{E-OC-Conformal} are more stable at the cost of having lower power compared to the base \texttt{OC-Conformal}. Focusing on the derandomization methods, when the test set contains 10\% outliers
(Figure~\ref{app-fig:real-data-outliers-0.1-IF}), the soft-rank method has a comparable and possibly slightly higher power compared to our martingale-based method. On the other hand, our method exhibits greater power for a larger proportion of outliers (Figure~\ref{app-fig:real-data-outliers-0.4-IF}). One explanation for this behavior is in our choice of  $\alpha_{\mathrm{bh}} = 0.05$, which may not be optimal for situations with low proportions of outliers when using the one-class conformal algorithm. Observe also that for the KDDCup99 data set, when the test set contains 40\% outliers, the selection variance of the soft-rank approach is the highest among the methods we study. One way to reduce this variance is to further increase the number of analyses $K$. We conclude this experiment with Figure~\ref{app-fig:real-data-outliers-proportion-OC-Conformal}, which provides a comprehensive comparison of the FDR and power for varying proportions of outliers in the test set. This figure confirms the reproducibility of our method---observe how the FDR, evaluated over 100 random sub-samples of the data, is controlled.

\begin{figure*}[!htb]
  \centering
  \begin{subfigure}[b]{\textwidth}
\includegraphics[width=0.32\textwidth]{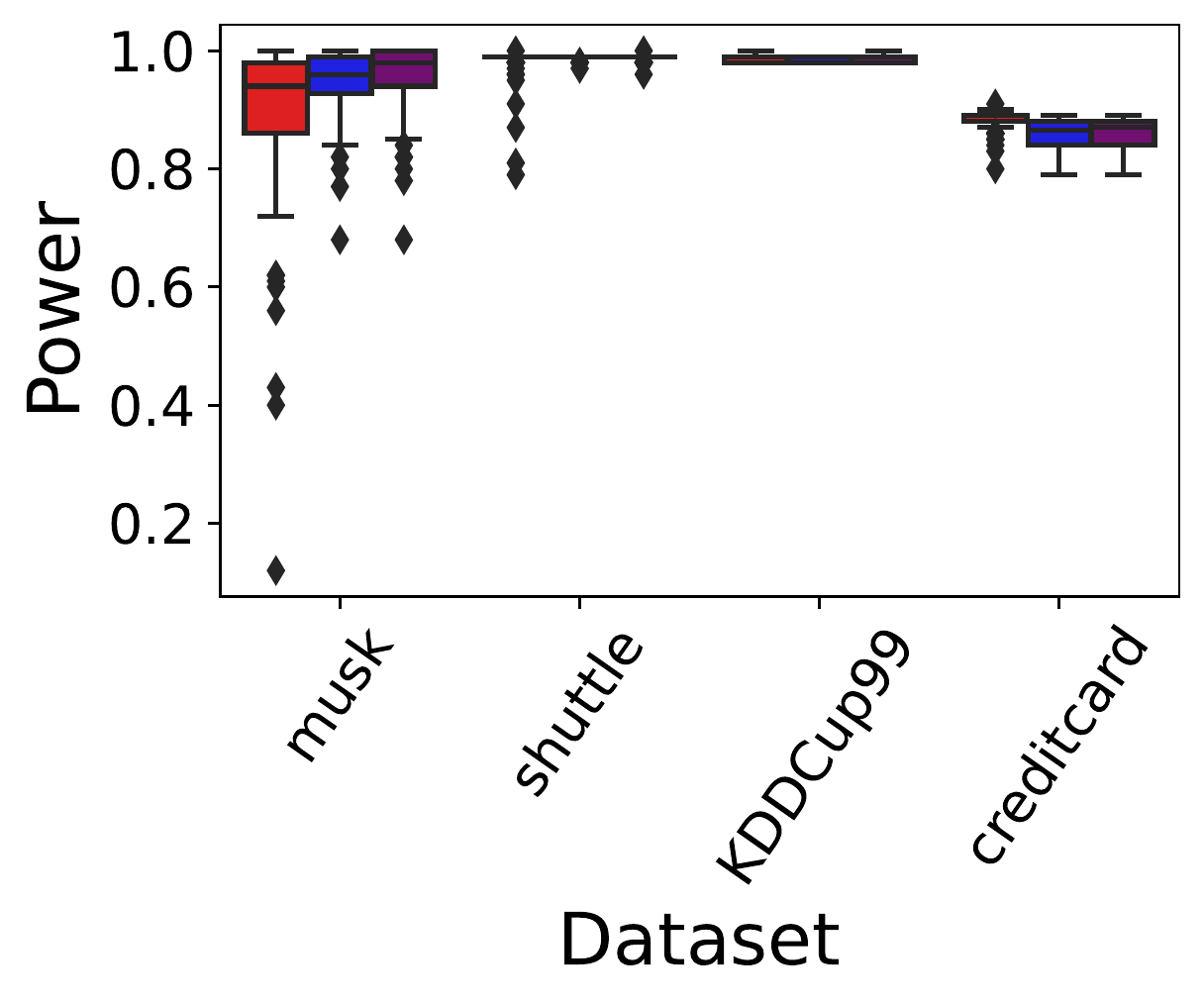}
\includegraphics[width=0.32\textwidth]{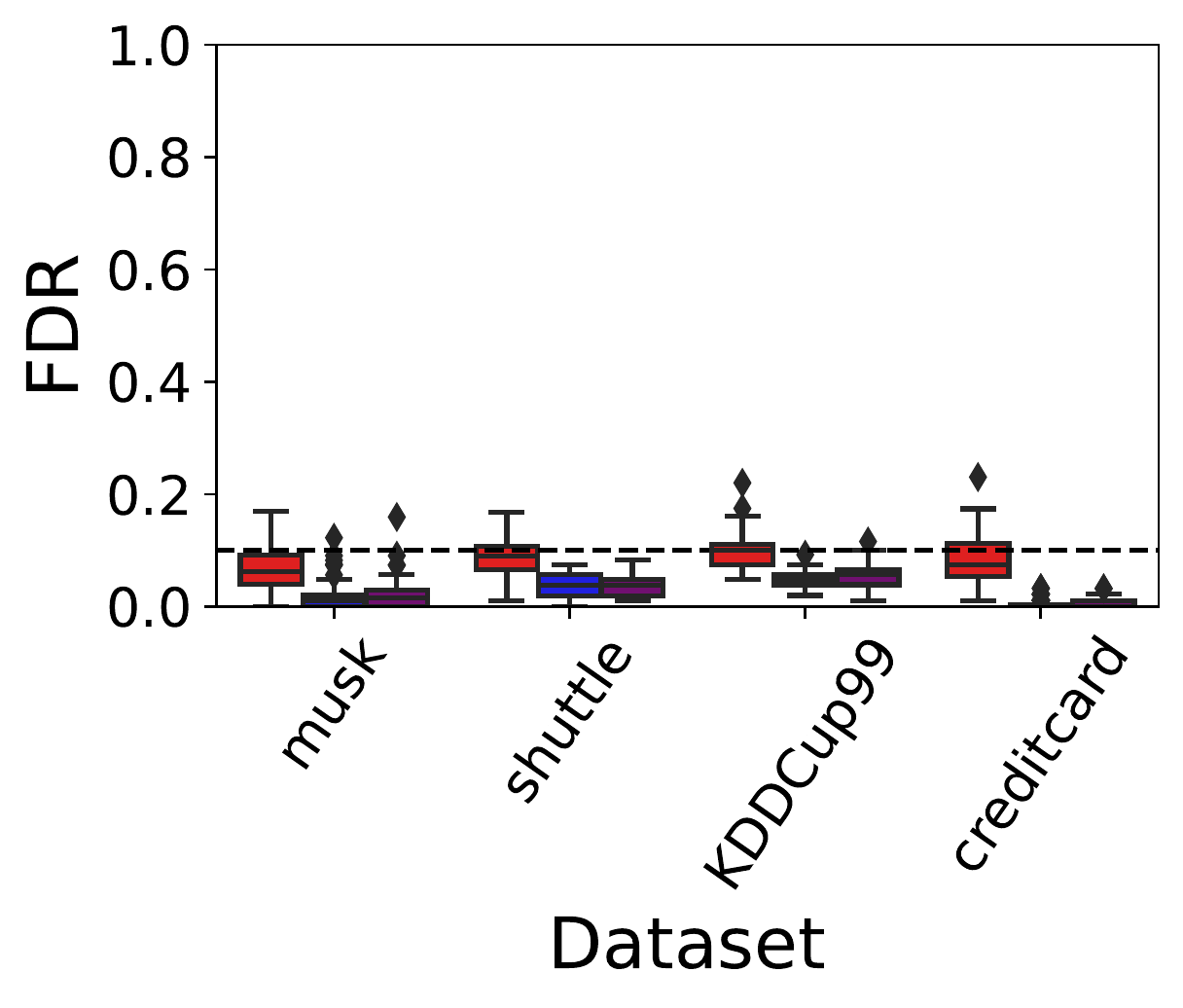}
\includegraphics[width=0.32\textwidth]{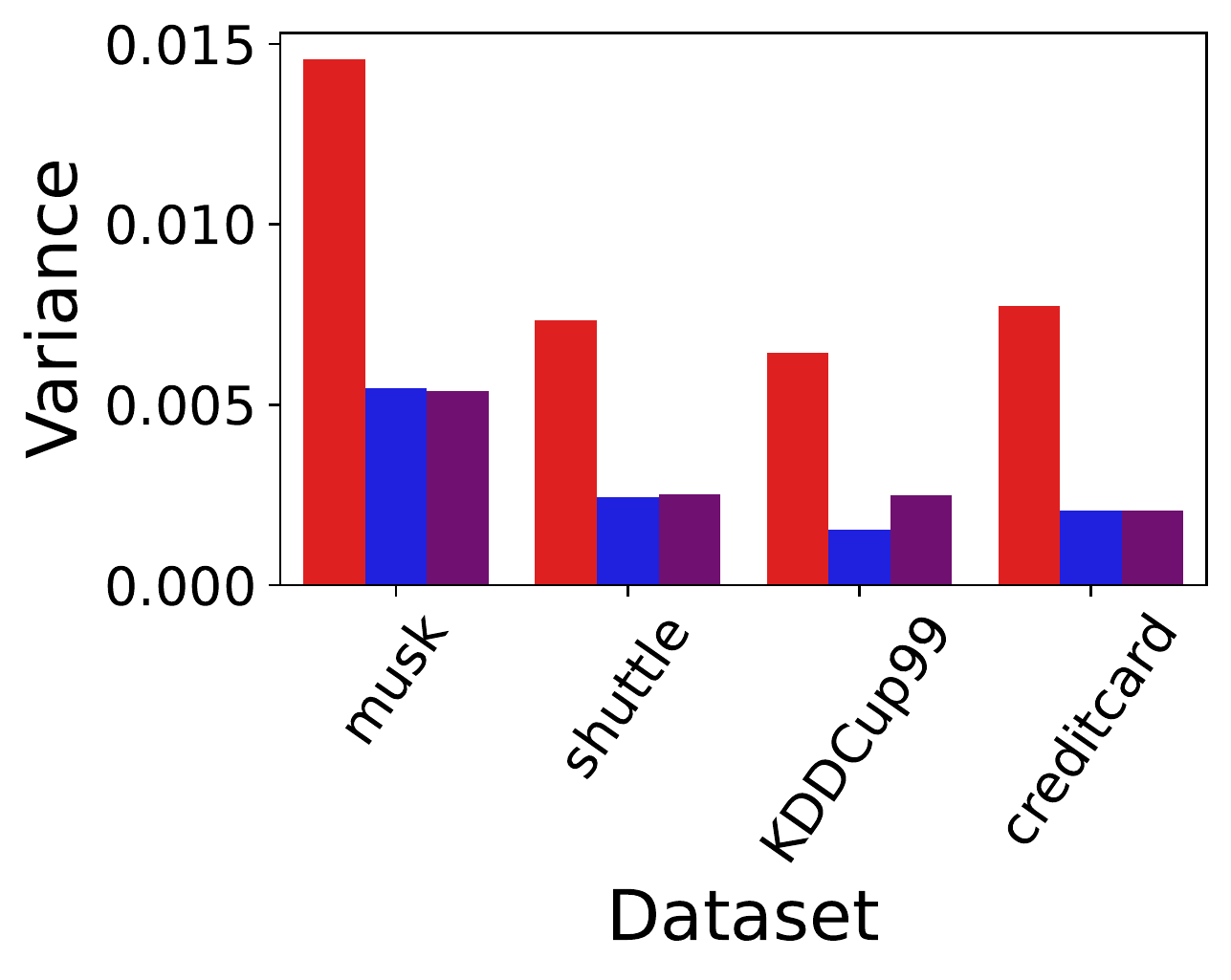}
\caption{$10\%$ outliers}
\label{app-fig:real-data-outliers-0.1-RF}
  \end{subfigure}
  \begin{subfigure}[b]{\textwidth}
\includegraphics[width=0.32\textwidth]{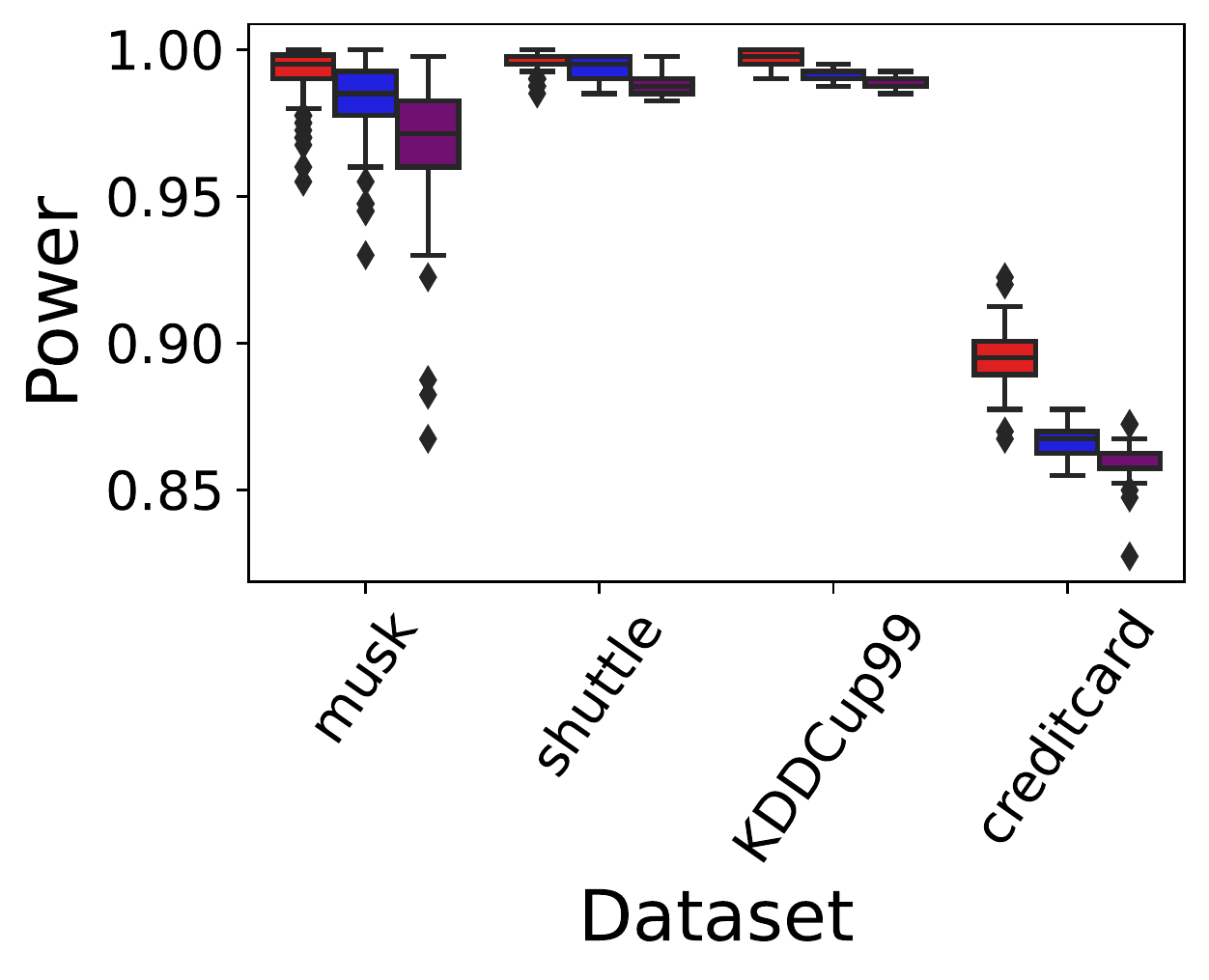}
\includegraphics[width=0.32\textwidth]{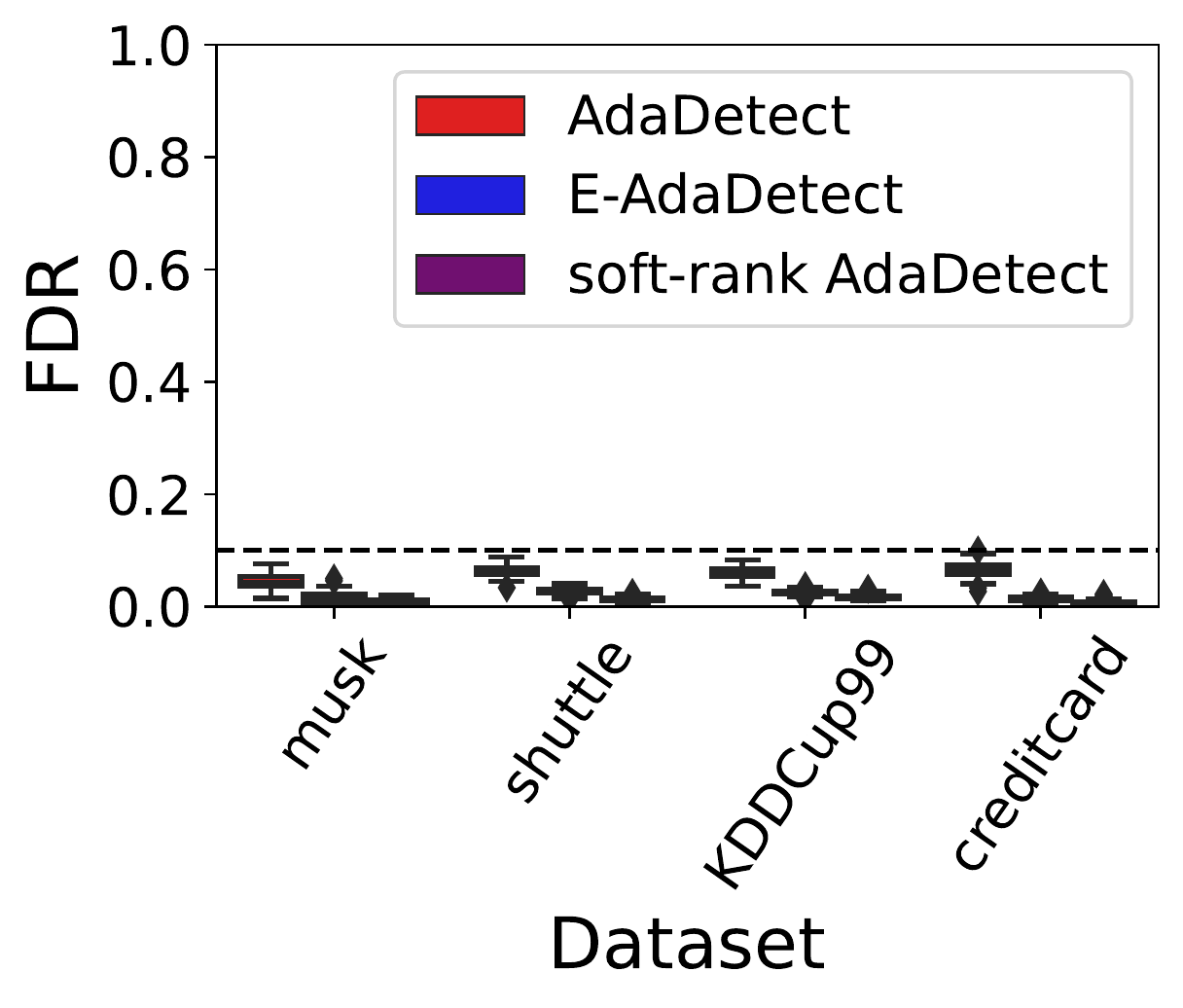}
\includegraphics[width=0.32\textwidth]{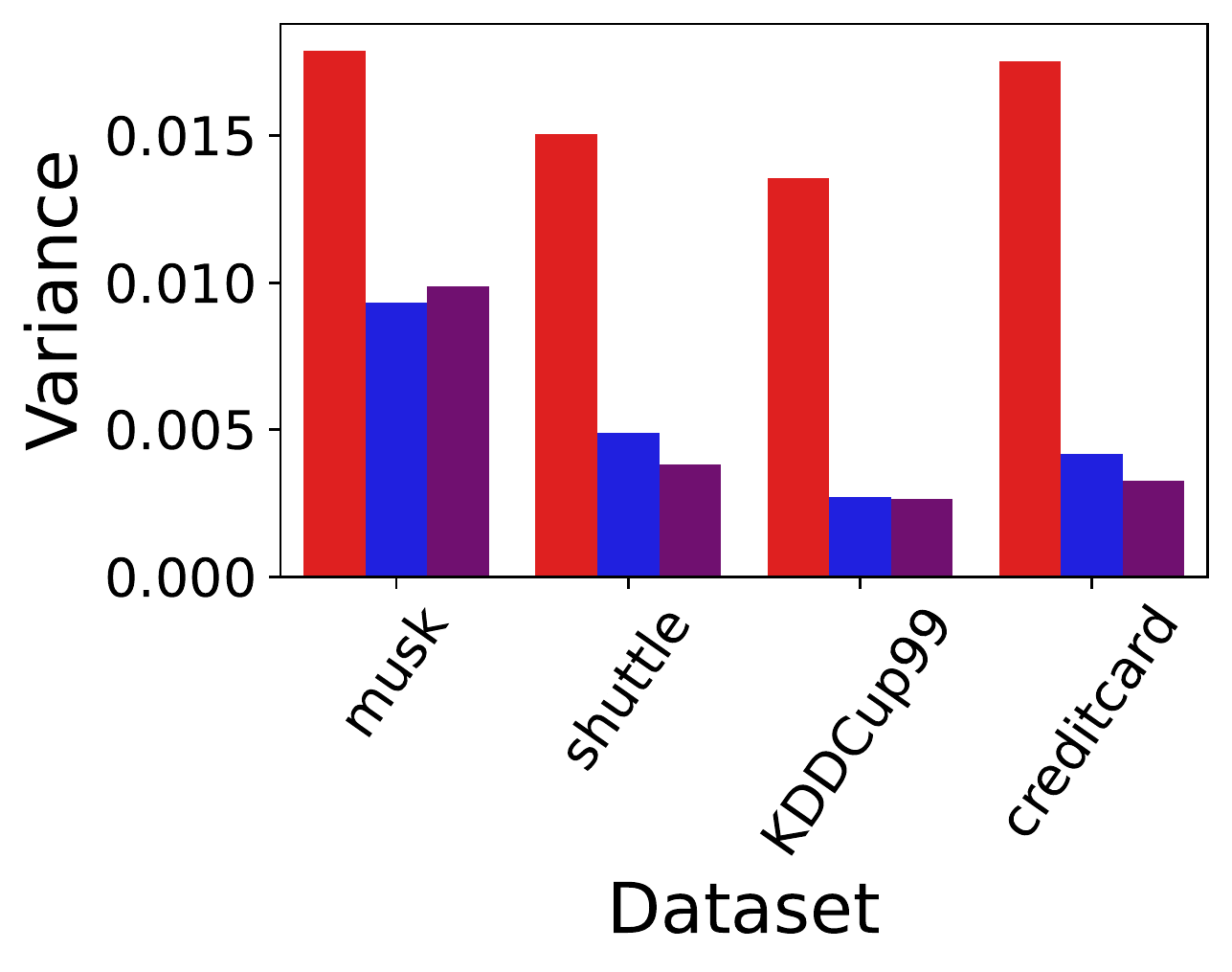}
\caption{$40\%$ outliers}
\label{app-fig:real-data-outliers-0.4-RF}
\end{subfigure}
  \caption{
    Performance on real data of \texttt{E-AdaDetect}, its randomized version, \texttt{AdaDetect}, and \texttt{soft-rank AdaDetect}. \ref{app-fig:real-data-outliers-0.1-RF} and \ref{app-fig:real-data-outliers-0.4-RF} present the performance of all methods for $10\%$ and $40\%$ outliers in the test-set, respectively.
All methods leverage a random forest binary classifier. Left: average proportion of true outliers that are discovered (higher is better). Right: variability of the findings (lower is better).
}
  \label{app-fig:real-data-soft_rank-RF}
  \end{figure*}

\begin{figure*}[!htb]
  \centering
  \begin{subfigure}[b]{\textwidth}
\includegraphics[width=0.32\textwidth]{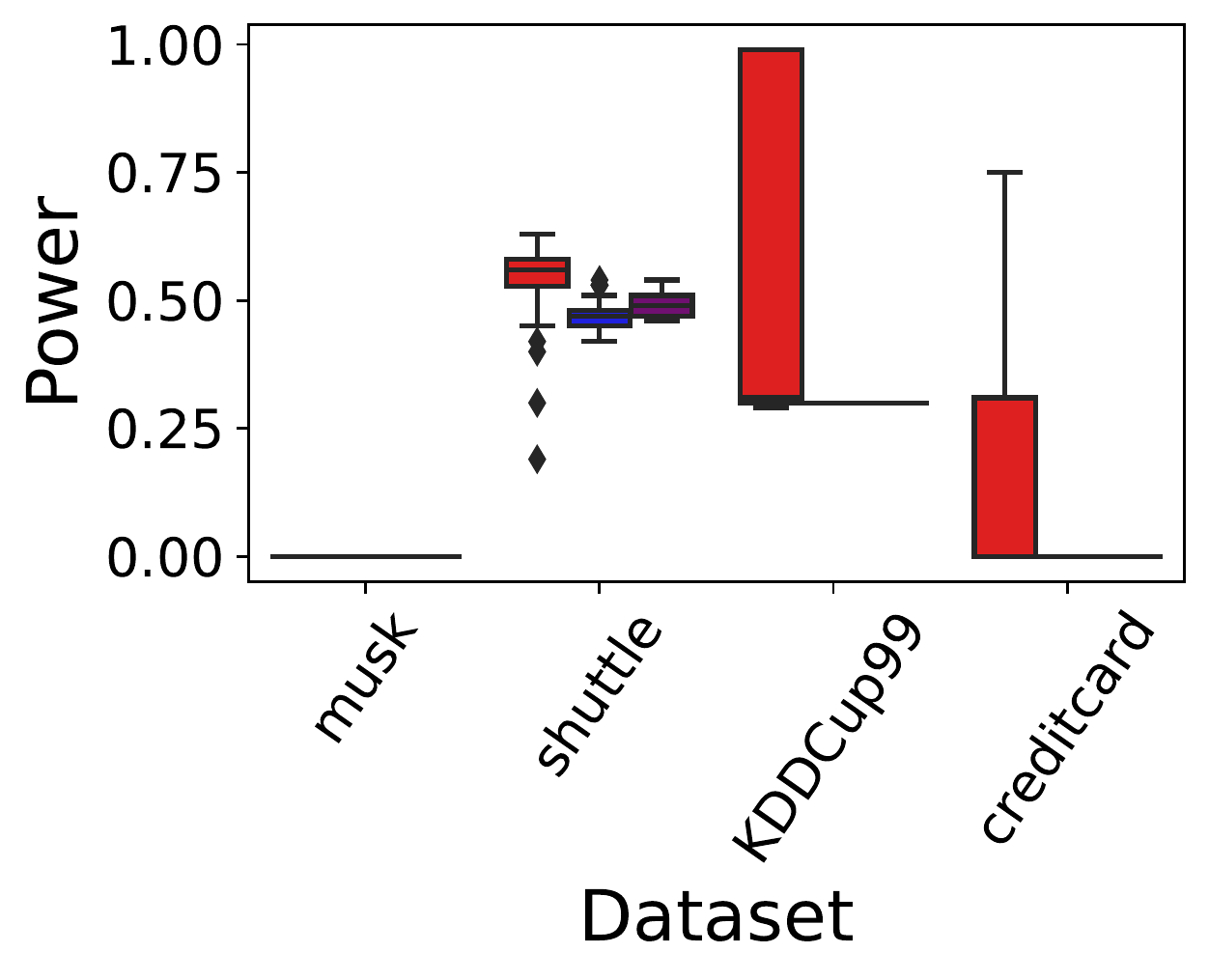}
\includegraphics[width=0.32\textwidth]{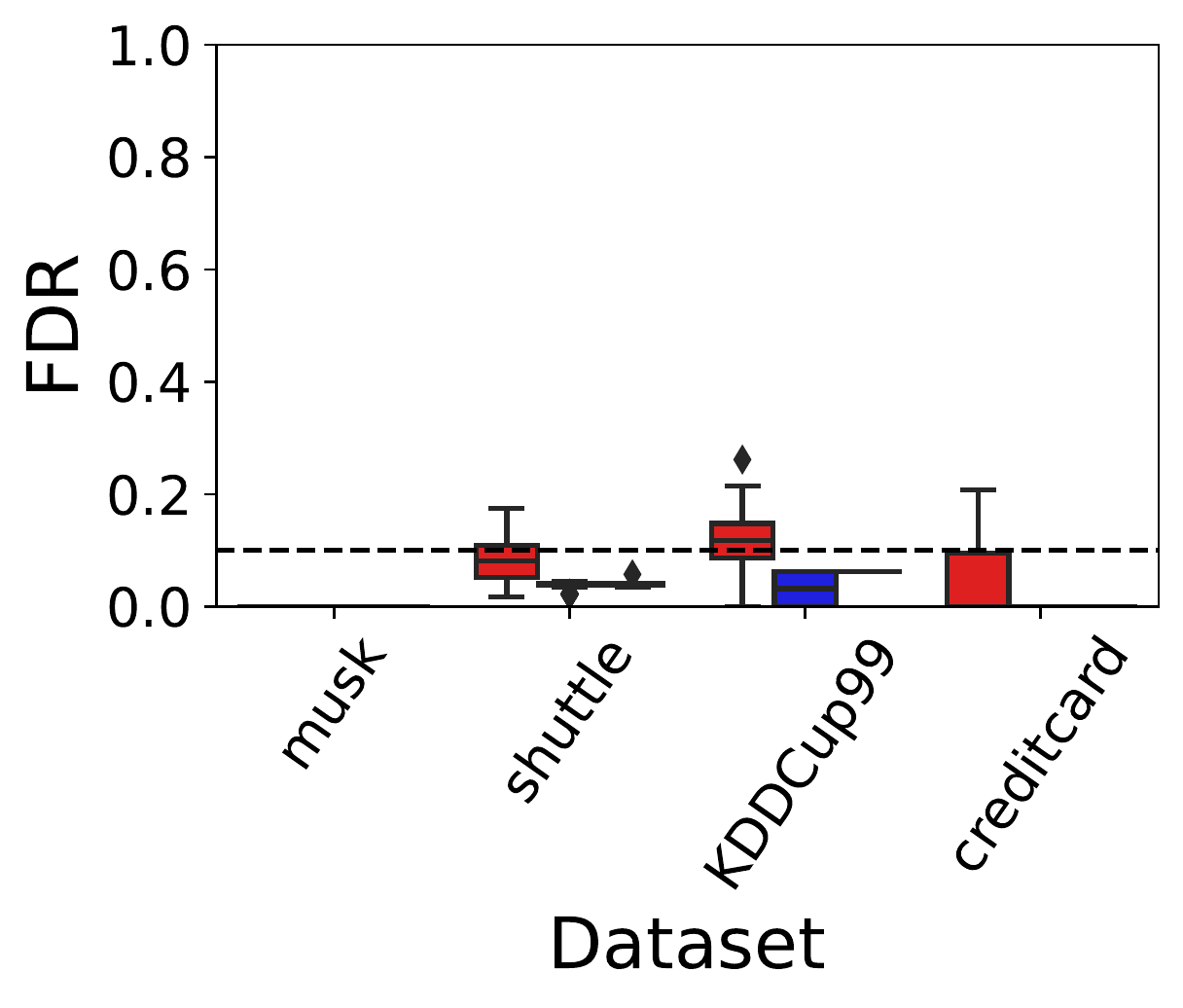}
\includegraphics[width=0.32\textwidth]{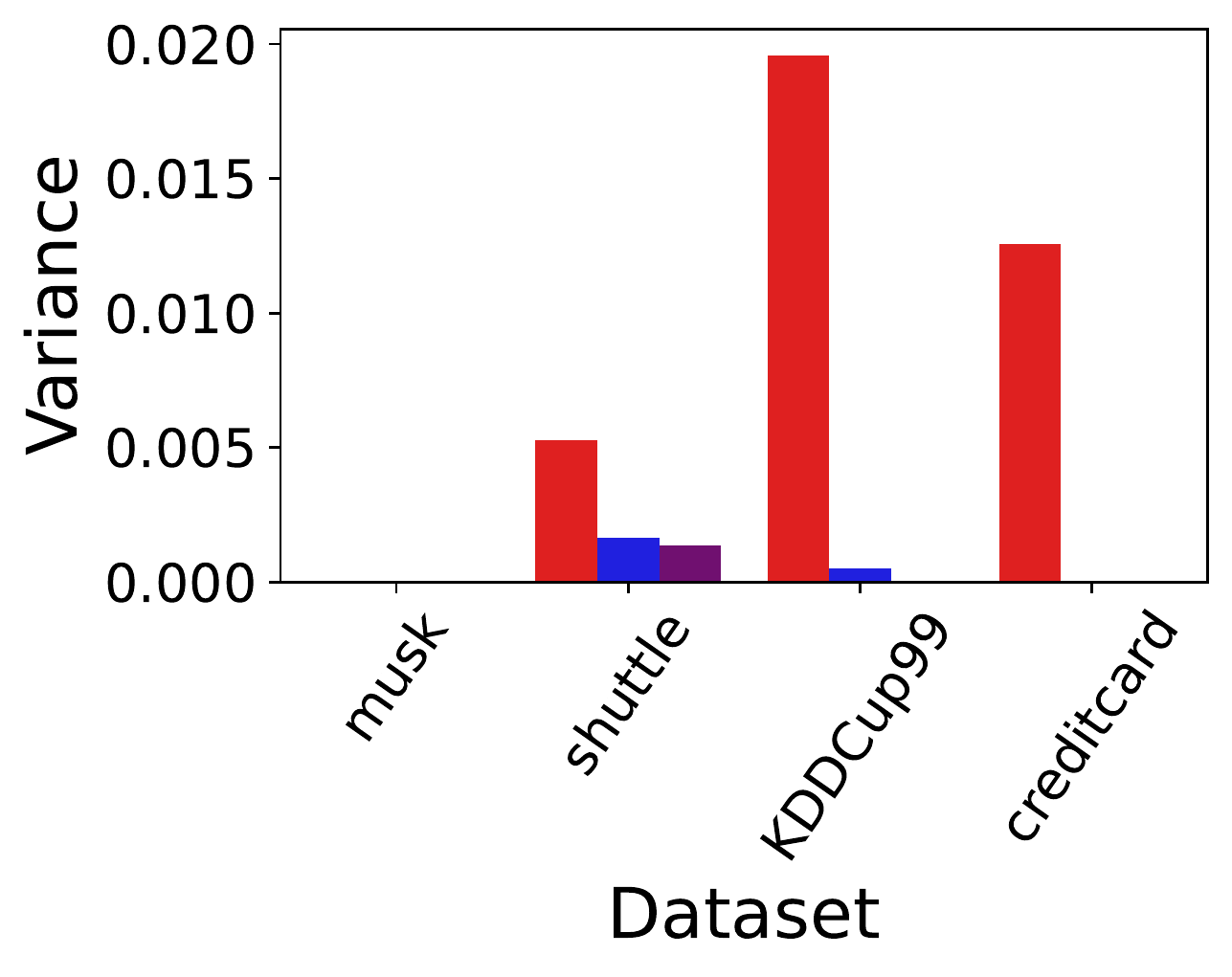}
\caption{$10\%$ outliers}
\label{app-fig:real-data-outliers-0.1-IF}
  \end{subfigure}
  \begin{subfigure}[b]{\textwidth}
\includegraphics[width=0.32\textwidth]{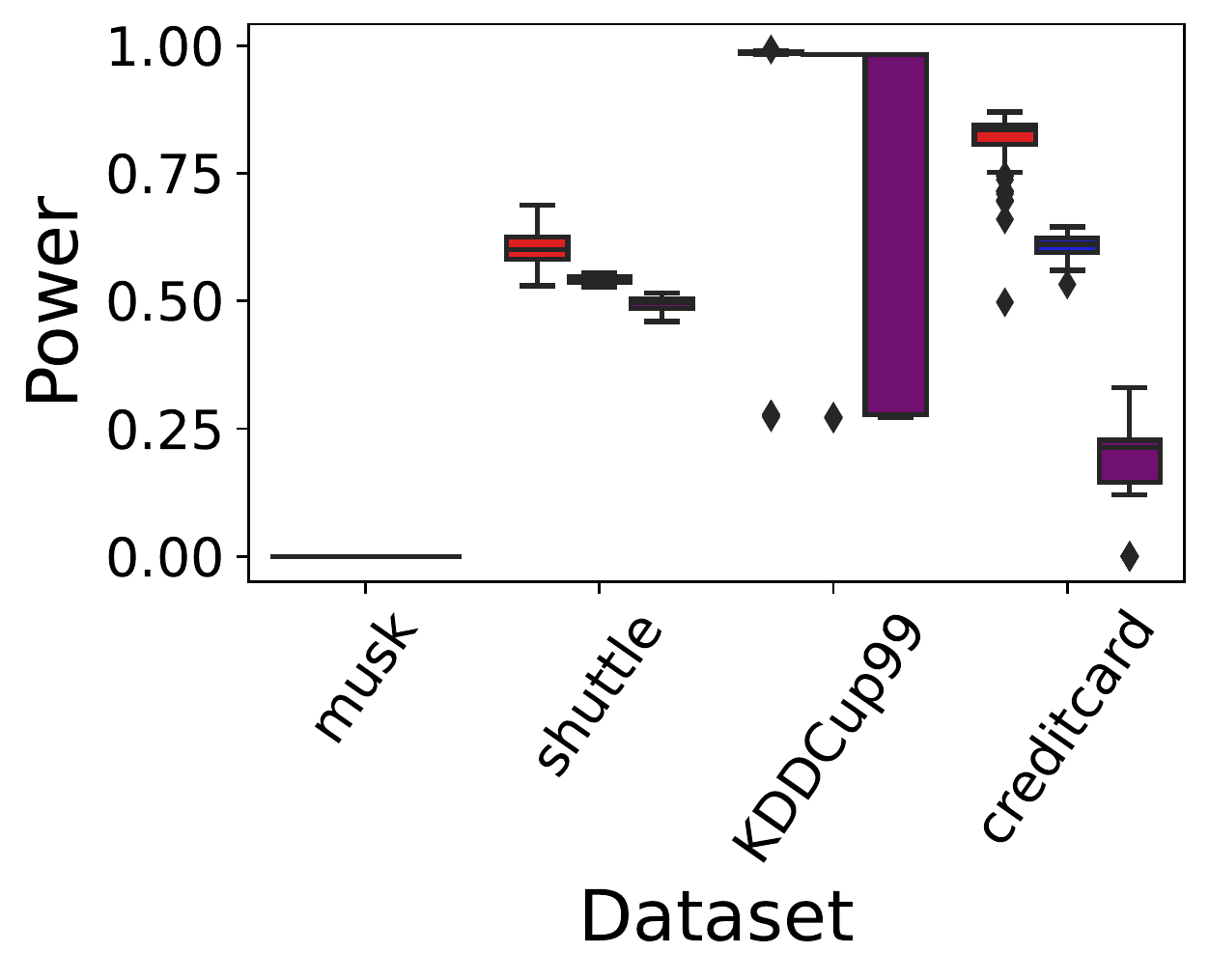}
\includegraphics[width=0.32\textwidth]{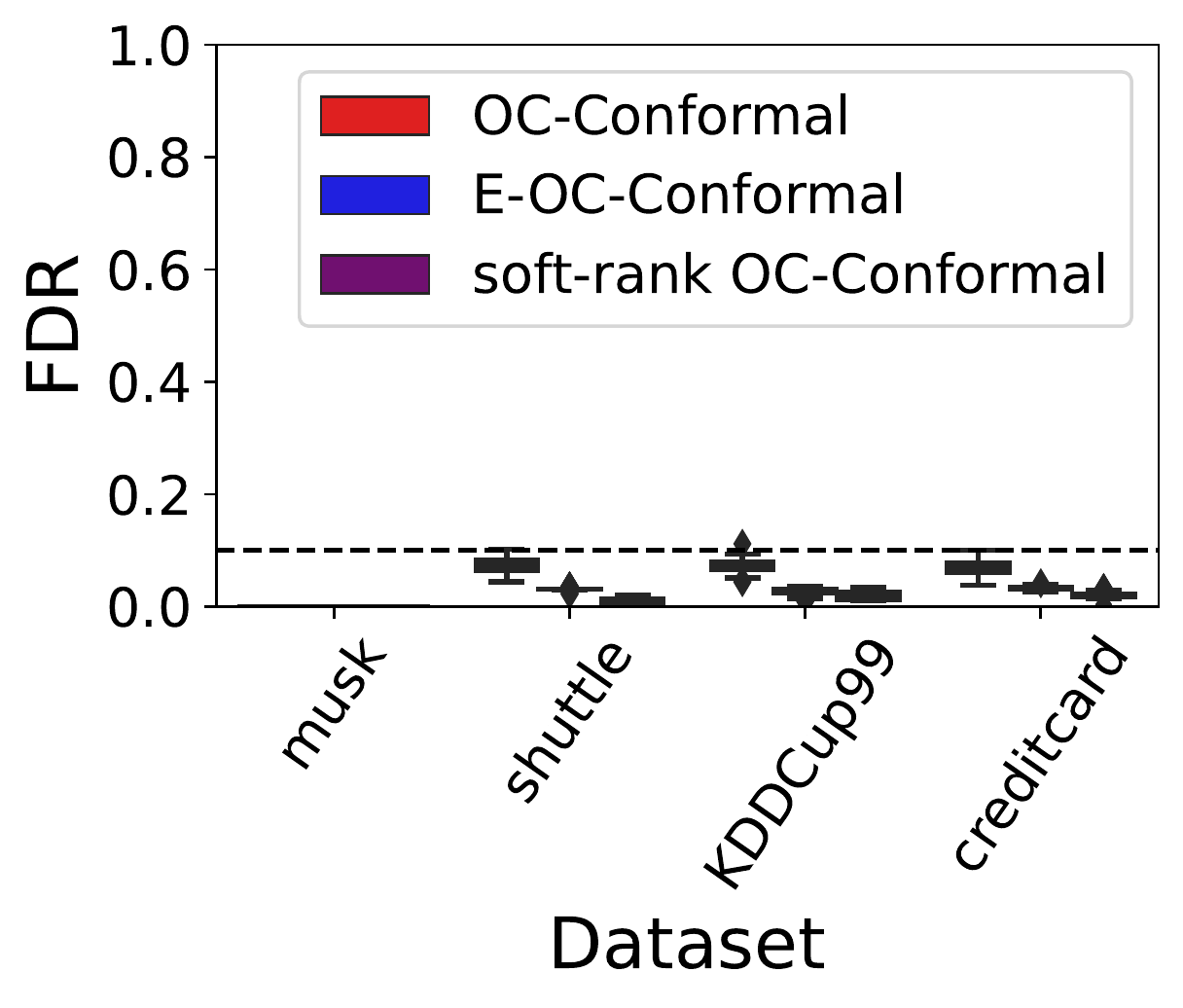}
\includegraphics[width=0.32\textwidth]{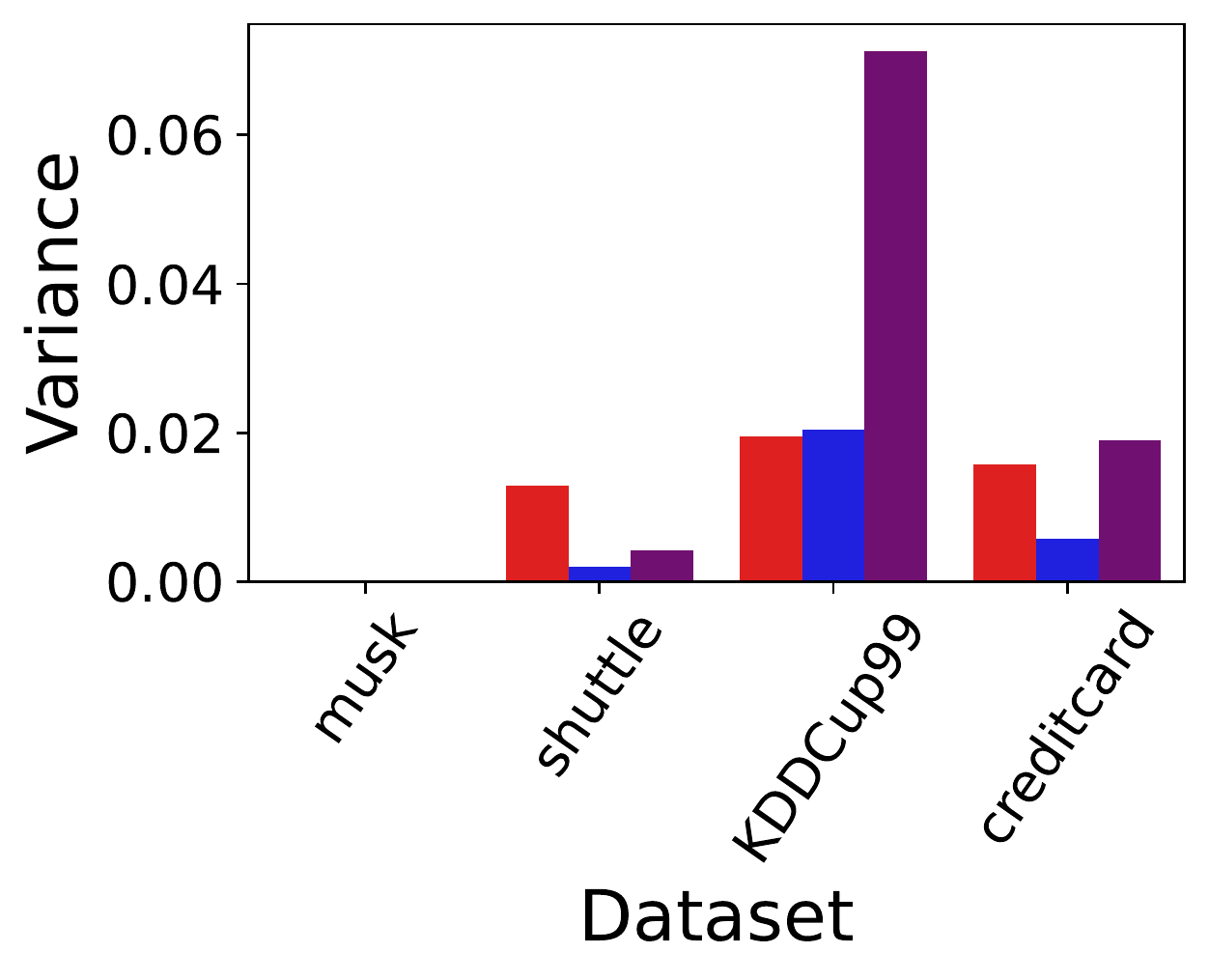}
\caption{$40\%$ outliers}
\label{app-fig:real-data-outliers-0.4-IF}
\end{subfigure}
  \caption{
    Performance on real data of \texttt{E-OC-Conformal}, its randomized version, \texttt{OC-Conformal}, and \texttt{soft-rank OC-Conformal}. \ref{app-fig:real-data-outliers-0.1-IF} and \ref{app-fig:real-data-outliers-0.4-IF} present the performance of all methods for $10\%$ and $40\%$ outliers in the test-set, respectively.
All methods leverage an isolation forest classifier. Left: average proportion of true outliers that are discovered (higher is better). Right: variability of the findings (lower is better).
}
  \label{app-fig:real-data-soft_rank-IF}
  \end{figure*}

\begin{figure}[!htb]
  \centering
  \begin{subfigure}[b]{\textwidth}
  \centering
\includegraphics[width=0.45\textwidth]{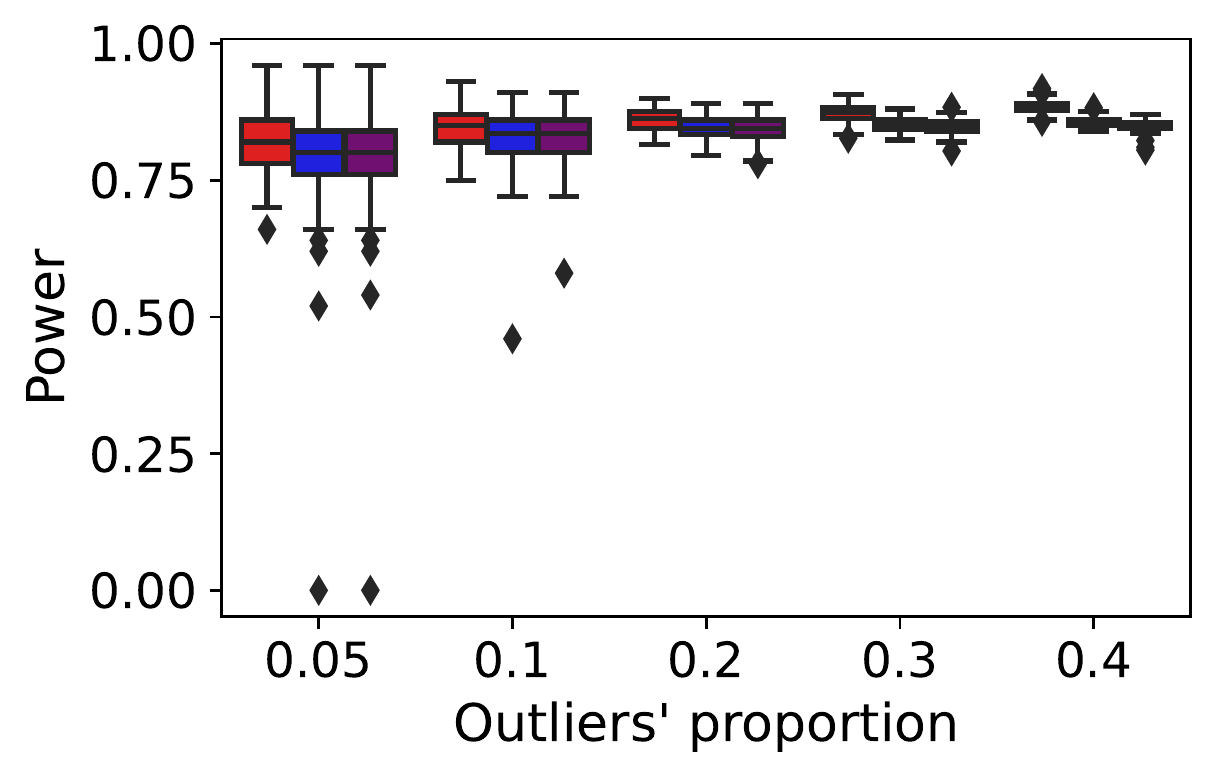}
\includegraphics[width=0.45\textwidth]{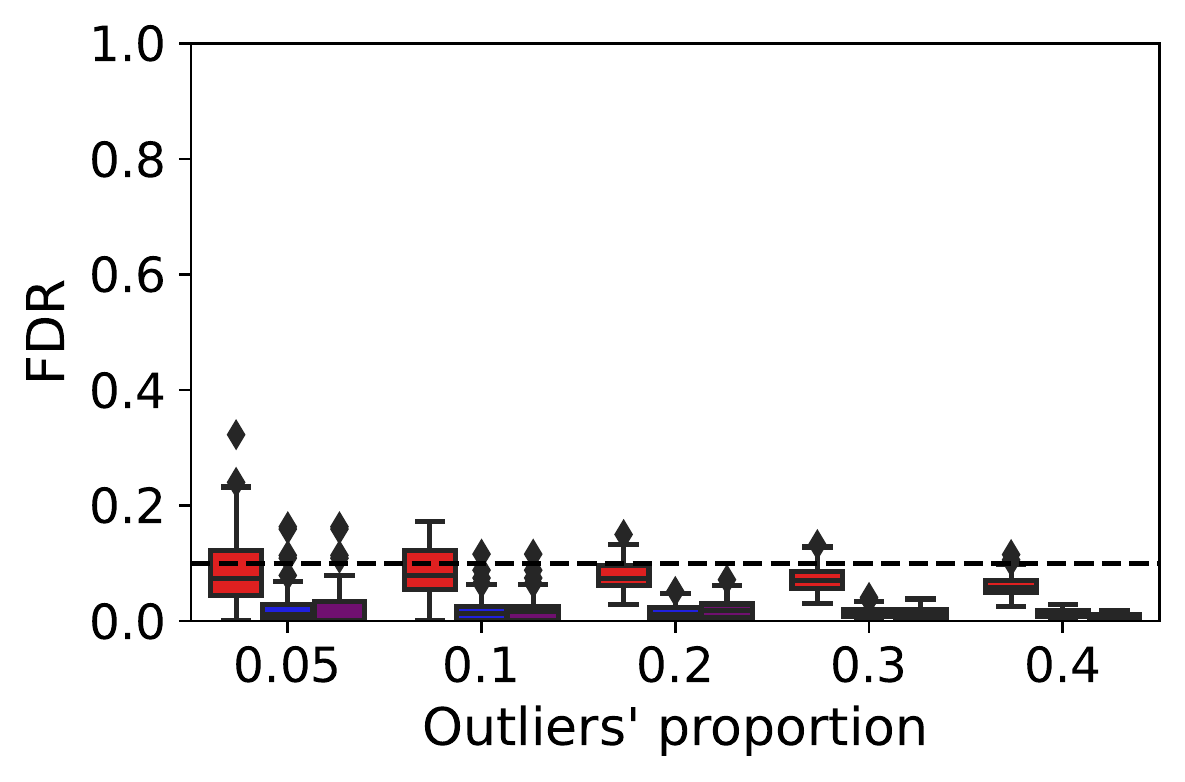}
\caption{creditcard}
\end{subfigure}
  \begin{subfigure}[b]{\textwidth}
  \centering
\includegraphics[width=0.45\textwidth]{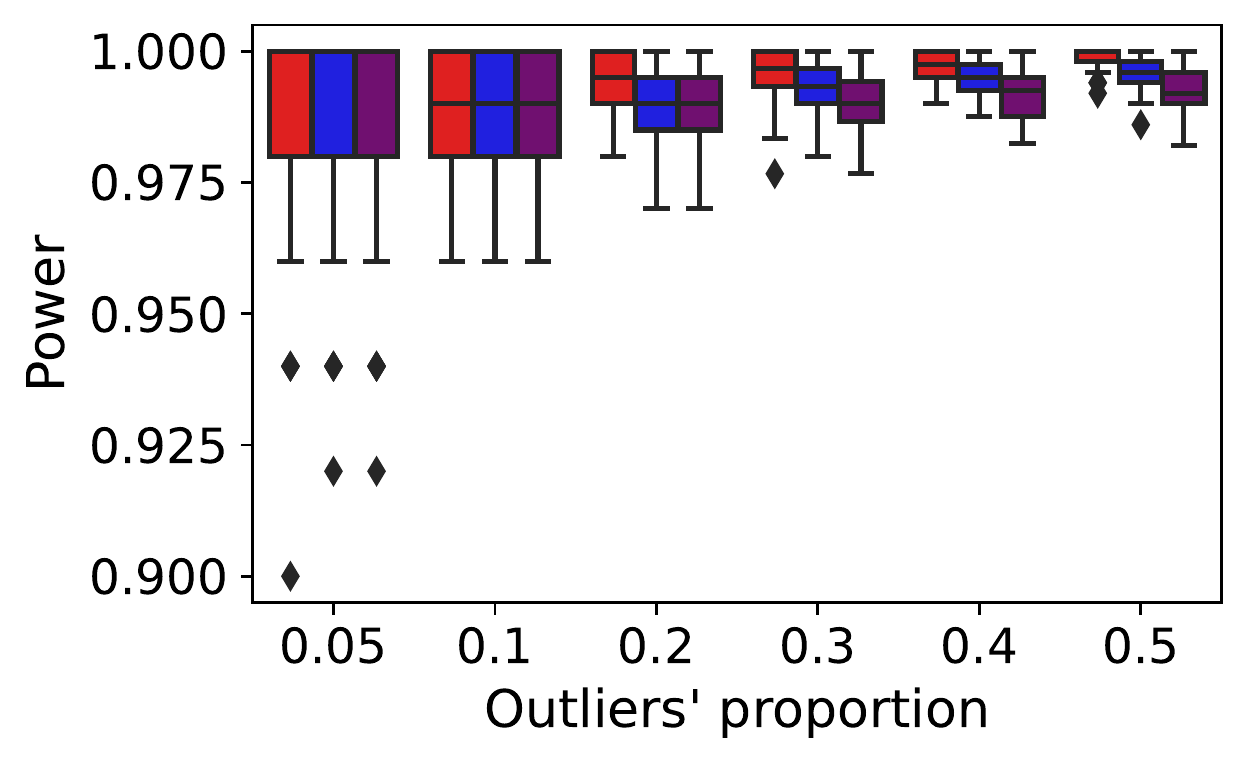}
\includegraphics[width=0.45\textwidth]{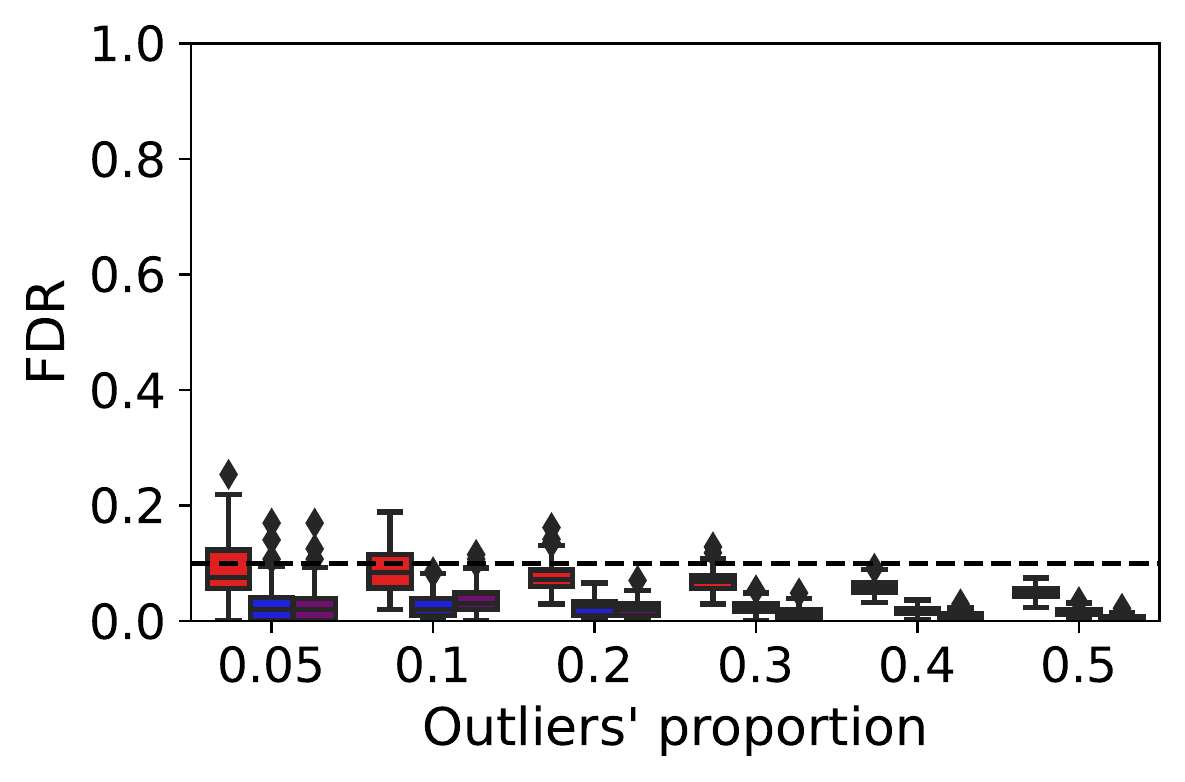}
\caption{KDDCup99}
\end{subfigure}
\\
  \begin{subfigure}[b]{\textwidth}
  \centering
\includegraphics[width=0.45\textwidth]{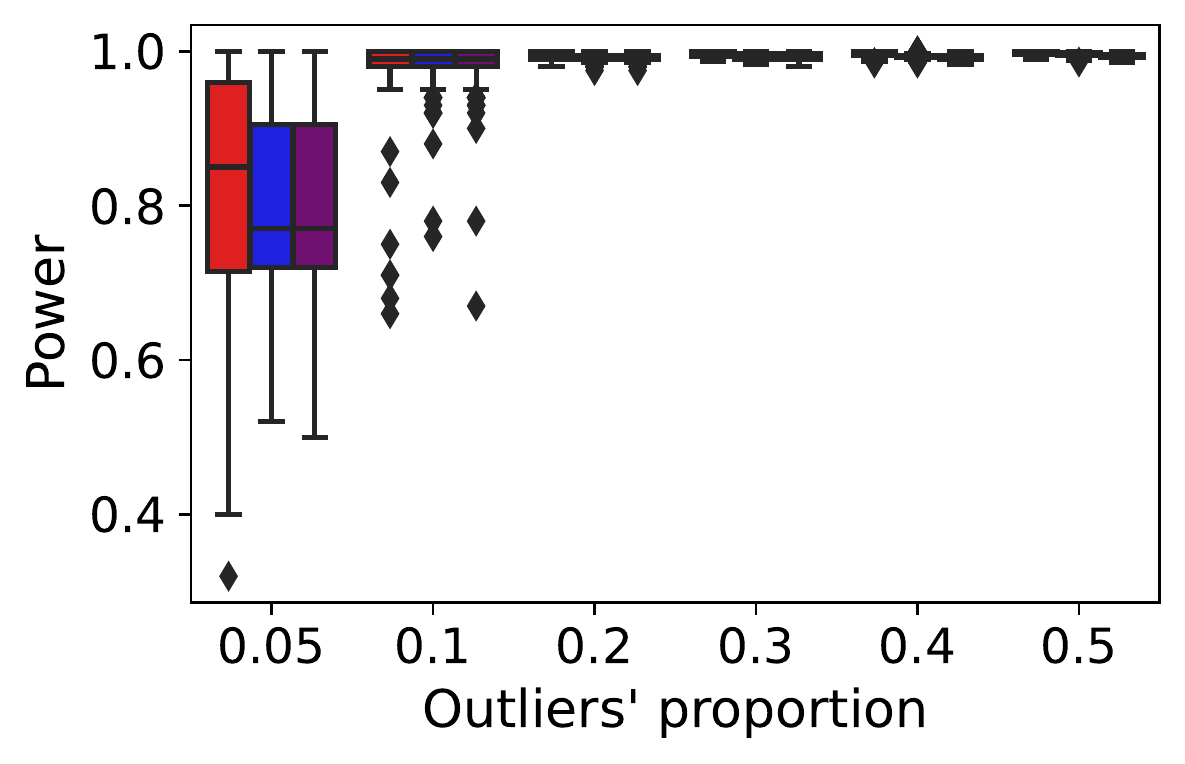}
\includegraphics[width=0.45\textwidth]{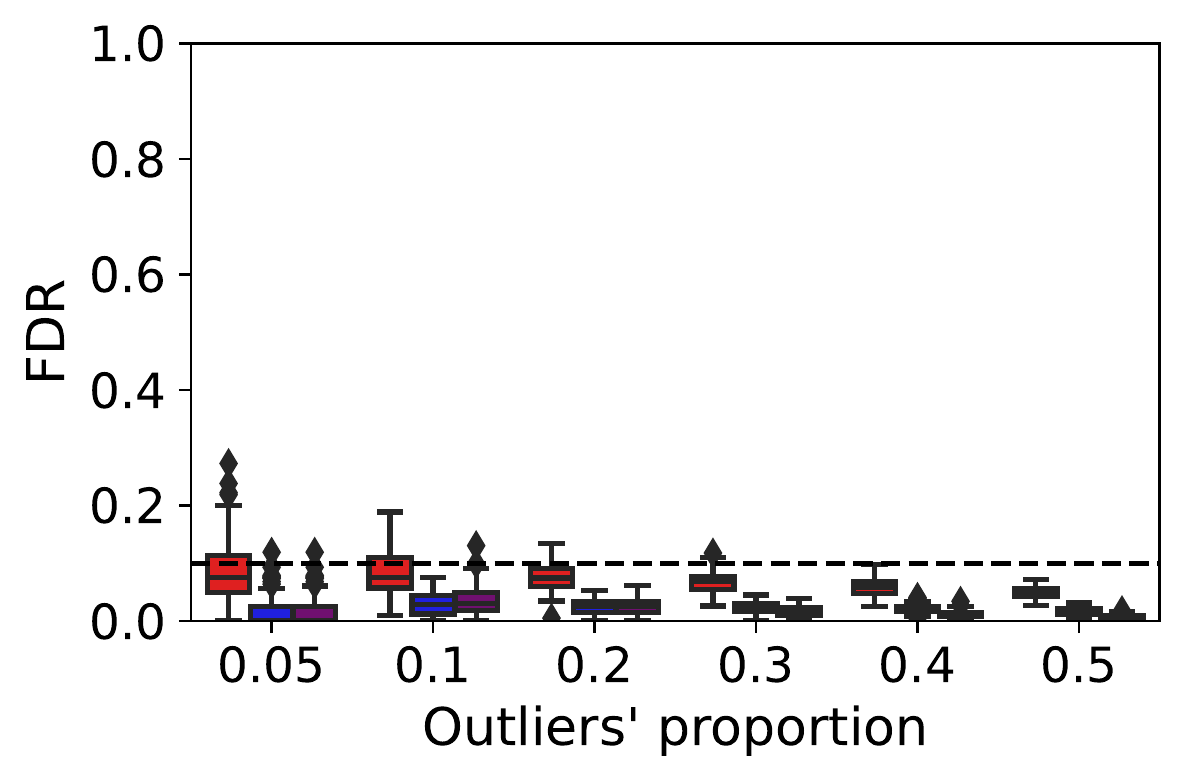}
\caption{shuttle}
\end{subfigure}
  \begin{subfigure}[b]{\textwidth}
  \centering
\includegraphics[width=0.45\textwidth]{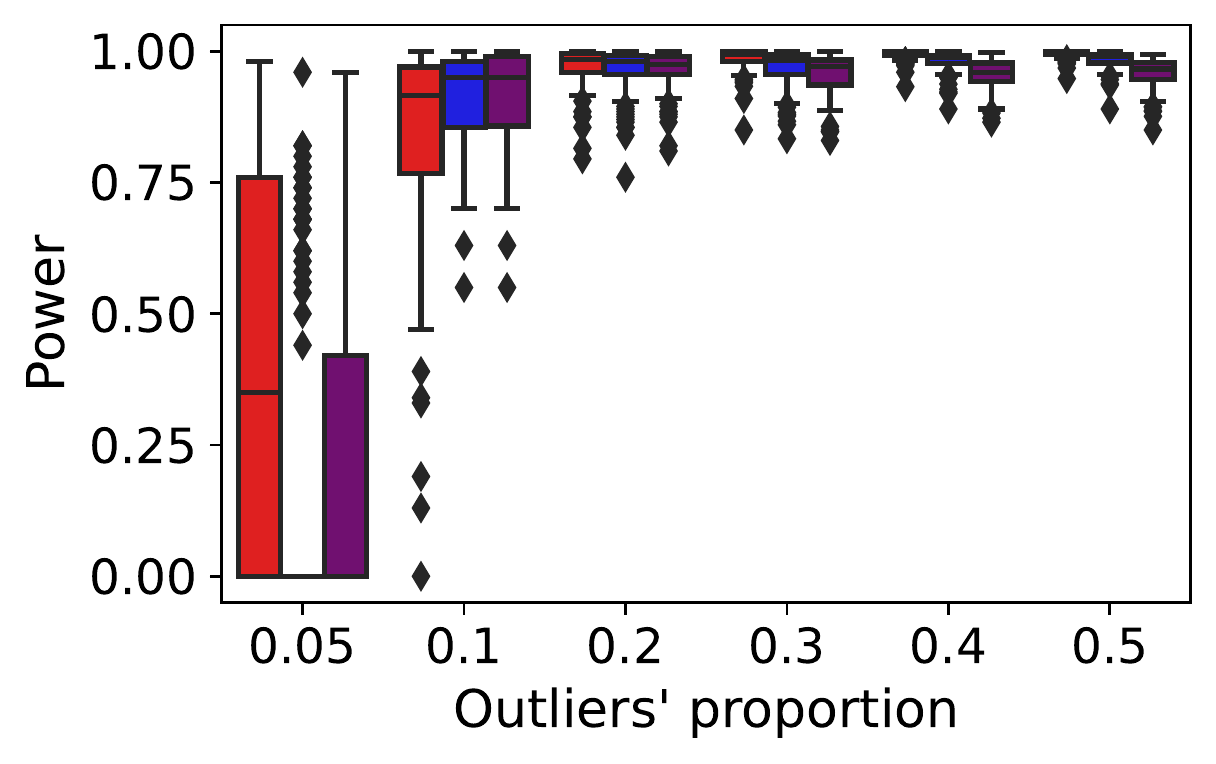}
\includegraphics[width=0.45\textwidth]{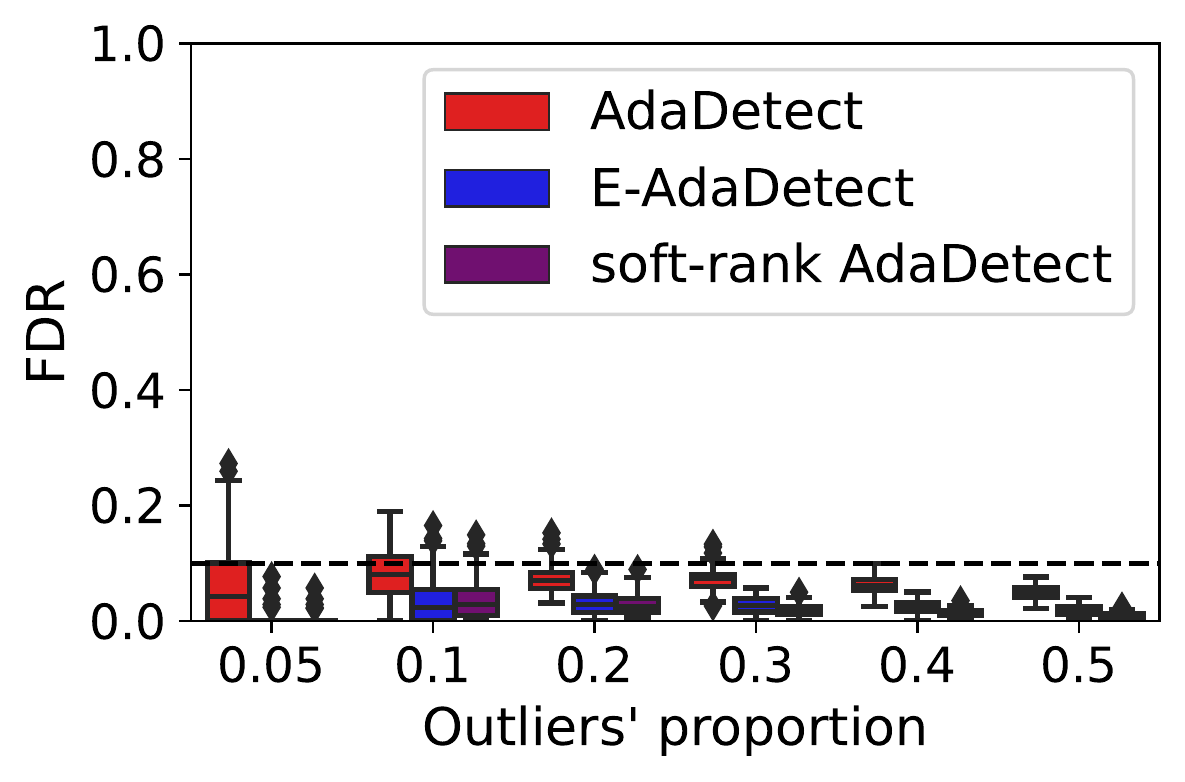}
\caption{musk}
\end{subfigure}
  \caption{Performance on real data of \texttt{E-AdaDetect}, its randomized version, \texttt{AdaDetect}, and \texttt{soft-rank AdaDetect} as a function of the outliers proportion in the test-set. Each sub-figure corresponds to a different dataset.
All methods leverage a random forest binary classifier. The results are averaged over 100 independent realizations of the data, which are randomly subsampled from the raw data sources. Other details are as in Figure~\ref{app-fig:real-data-soft_rank-RF}.}
  \label{app-fig:real-data-outliers-proportion-AdaDetect-RF}
\end{figure}

\begin{figure}[!htb]
  \centering
  \begin{subfigure}[b]{\textwidth}
  \centering
\includegraphics[width=0.45\textwidth]{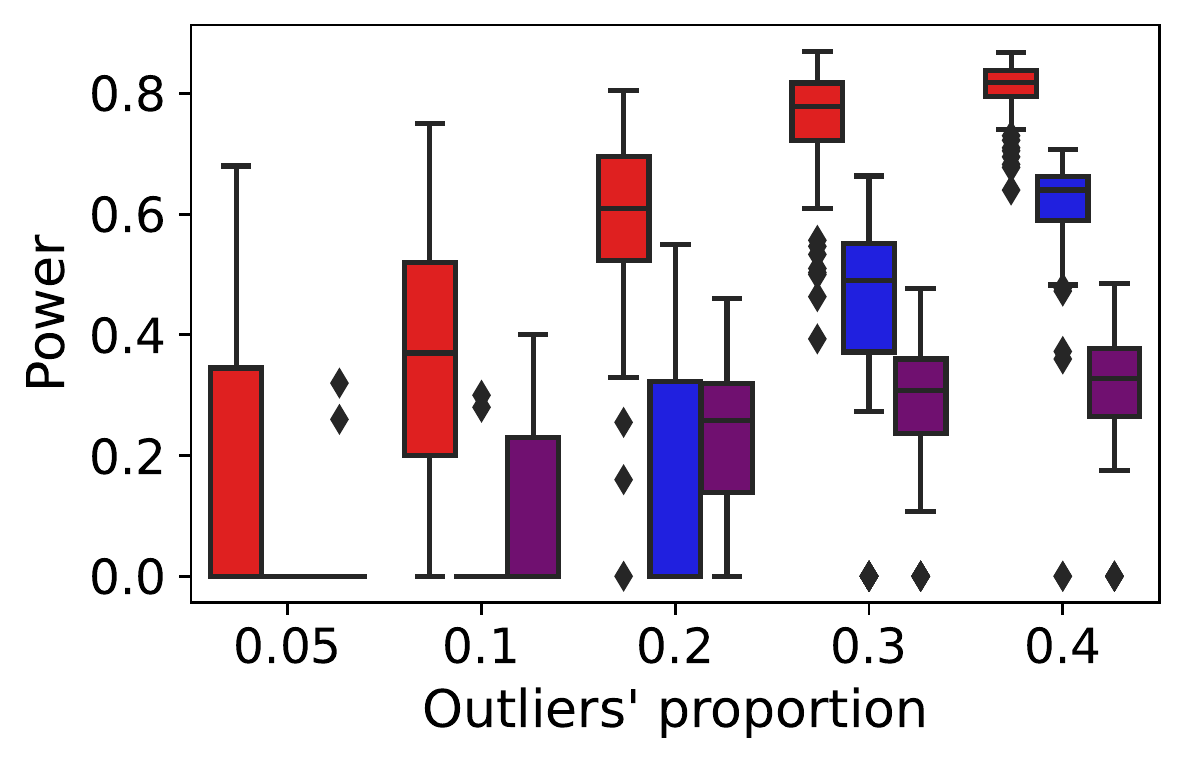}
\includegraphics[width=0.45\textwidth]{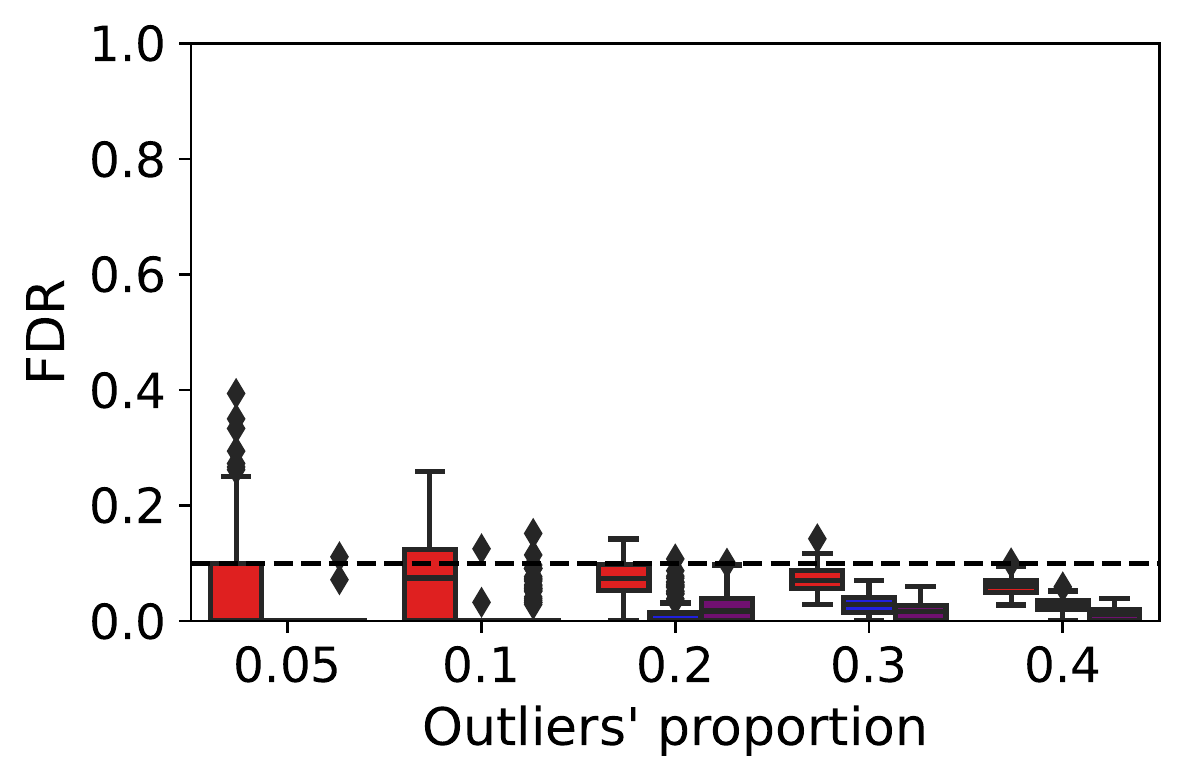}
\caption{creditcard}
\end{subfigure}
  \begin{subfigure}[b]{\textwidth}
  \centering
\includegraphics[width=0.45\textwidth]{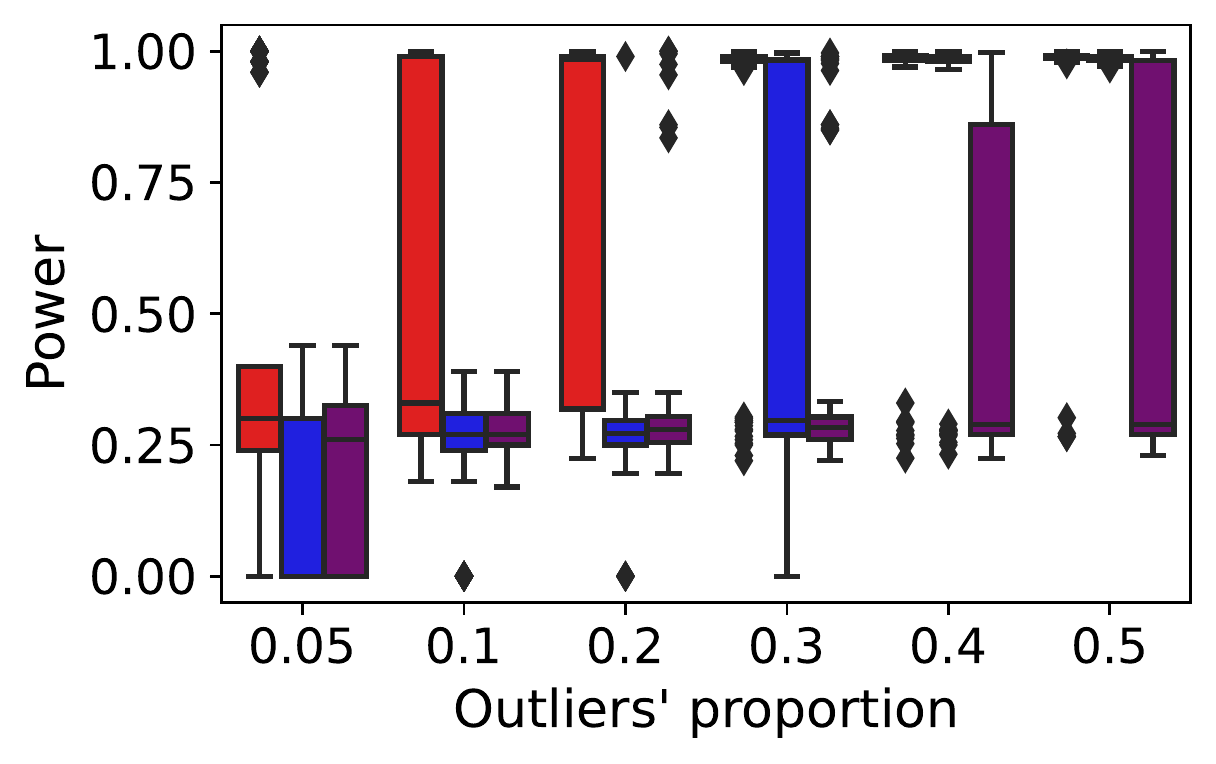}
\includegraphics[width=0.45\textwidth]{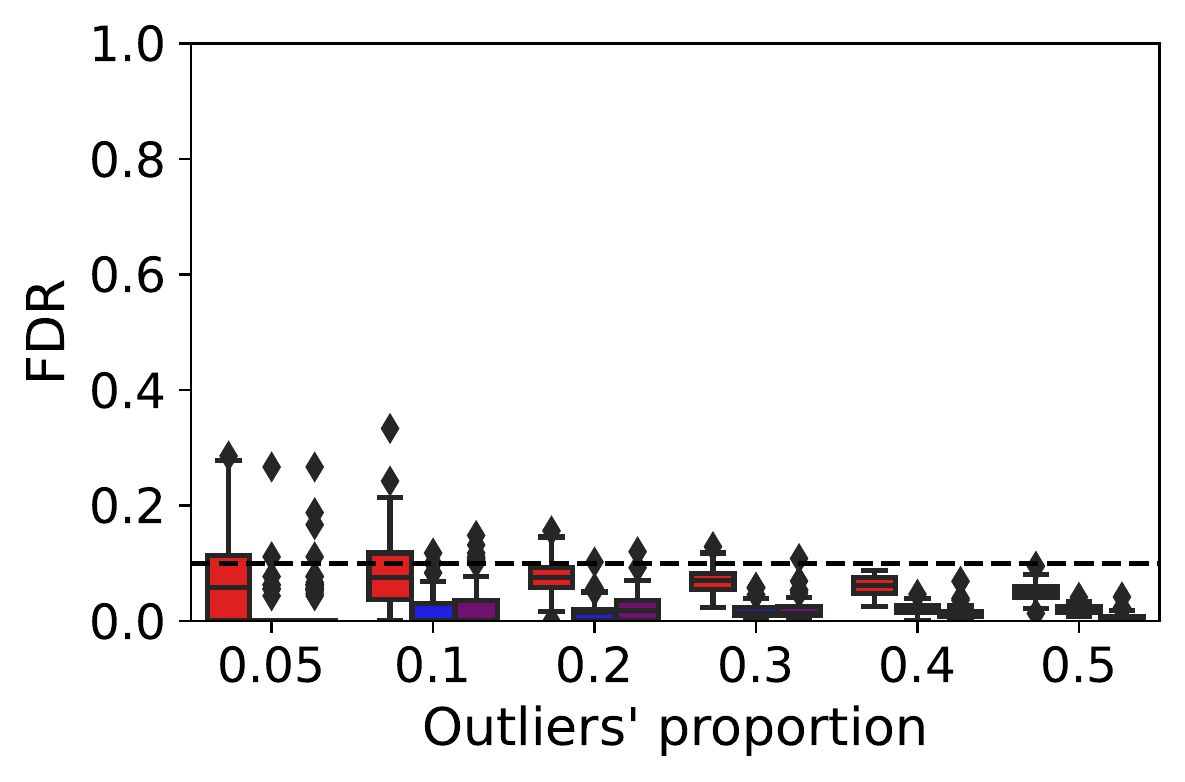}
\caption{KDDCup99}
\end{subfigure}
  \begin{subfigure}[b]{\textwidth}
  \centering
\includegraphics[width=0.45\textwidth]{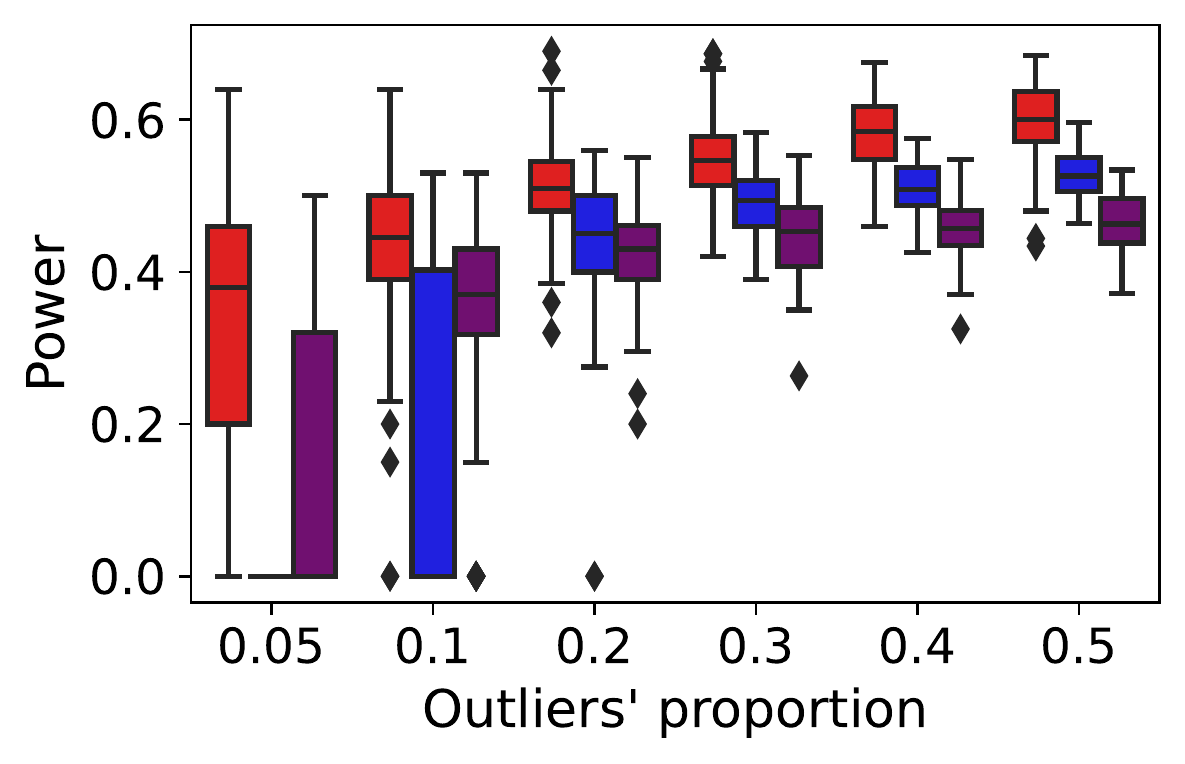}
\includegraphics[width=0.45\textwidth]{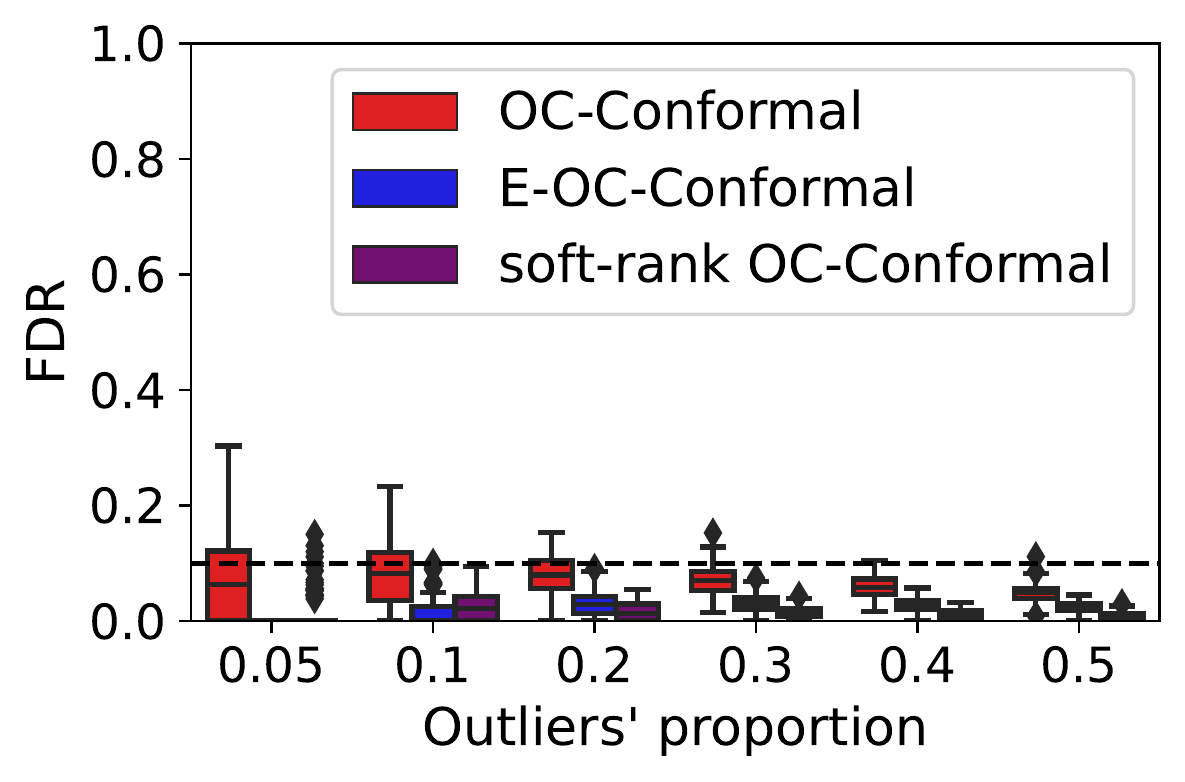}
\caption{shuttle}
\end{subfigure}
  \caption{Performance on real data of \texttt{E-OC-Conformal}, its randomized version, \texttt{OC-Conformal}, and \texttt{soft-rank OC-Conformal} as a function of the outliers proportion in the test-set. Each sub-figure corresponds to a different dataset. The power obtained for musk dataset is zero for all methods and thus is not shown.
All methods leverage an isolation forest classifier. The results are averaged over 100 independent realizations of the data, which are randomly subsampled from the raw data sources. Other details are as in Figure~\ref{app-fig:real-data-soft_rank-RF}.
  }
  \label{app-fig:real-data-outliers-proportion-OC-Conformal}
  \end{figure}

%%%%%%%%%%%%%%%%%%%%%%%%%%%%%%%%%%%%%%%%%%%%%%%%%%%%%%%%%%%%

% \clearpage
% \bibliography{bib}
% \bibliographystyle{abbrvnat}  % different style?

\end{document}